\pdfoutput=1
\documentclass[twoside,11pt]{article}

\usepackage[preprint]{jmlr2e}
\usepackage[utf8]{inputenc}
\usepackage{amsmath,amssymb,mathtools}

\usepackage{algorithm}
\usepackage{algpseudocode}
\usepackage{algorithmicx}
\algrenewcommand\algorithmicrequire{\textbf{Input:}}
\algrenewcommand\algorithmicensure{\textbf{Output:}}

\usepackage[math]{cellspace}
\usepackage{enumitem}
\usepackage{array,multirow,diagbox}
\usepackage{siunitx}
\usepackage{xcolor,colortbl}
\usepackage[framemethod=TikZ]{mdframed} 
\usepackage{graphicx}
\graphicspath{{./fig/main/}{./fig/SI/}}
\usepackage{bm}
\usepackage{bbm}
\usepackage{framed}
\usepackage{float}
\usepackage{wrapfig}
\usepackage{tikz-cd}
\usepackage{moreverb}
\usepackage{epstopdf}
\usepackage{multicol}
\usepackage[nonumberlist]{glossaries}
\usepackage[normalem]{ulem}

\usepackage{xspace}

\usepackage{caption}
\usepackage{subcaption}
\usepackage{capt-of} 
\DeclareCaptionLabelFormat{adja-page}{\hrulefill\\#1 #2 \emph{(previous page)}}

\makeatletter
\@ifundefined{proposition}{\newtheorem{proposition}[theorem]{Proposition}}{}
\@ifundefined{corollary}{\newtheorem{corollary}[theorem]{Corollary}}{}
\@ifundefined{conjecture}{}{}
\@ifundefined{myassum}{\newtheorem{myassum}{Assumption}[section]}{}
\makeatother

\allowdisplaybreaks[4]

\numberwithin{equation}{section}

\newenvironment{mat}{\left[\begin{array}{ccccccccccccccc}}{\end{array}\right]}
\newcommand\bcm{\begin{mat}}
\newcommand\ecm{\end{mat}}

\newenvironment{rmat}{\left[\begin{array}{rrrrrrrrrrrrr}}{\end{array}\right]}
\newcommand\brm{\begin{rmat}}
\newcommand\erm{\end{rmat}}
\newcommand\beq{\begin{equation}}
\newcommand\eeq{\end{equation}}

\def\eps{{\epsilon}}

\def\cD{{\mathcal{D}}}
\def\cE{{\mathcal{E}}}

\def\cG{{\mathcal{G}}}

\def\Skill{\ensuremath{\mathit{Skills}}}

\def\MDP{{\mathtt{MDP}}}

\def\MMDP{{\mathtt{MMDP}}}

\def\ends{\mathtt{end}}

\def\emptyvalue{{\mathtt{null}}}
\def\endaction{{a^{\underline{\mathtt{end}}}}}
\def\endactionset{{\cA^{\mathtt{end}}}}
\def\endactionfactor{{{a^{\ends}}}}

\def\diag{{\mathrm{diag}}}
\def\NA{{\mathrm{NA}}}
\def\hint{{\mathrm{hint}}}
\def\bound{{\mathrm{bounds}}}
\def\thresh{{\mathrm{thresh}}}
\def\err{{\mathrm{err}}}
\def\Err{{\mathrm{Err}}}

\def\prev{{\mathtt{prev}}}

\def\cur{s_{\mathtt{cur}}}
\def\curs{\mathtt{cur}}
\def\curnext{s'_{\mathtt{cur}}}
\def\dest{{\mathtt{dest}}}
\def\dests{s_{\dest}}
\def\destnext{s'_{\dest}}
\def\goals{{\mathtt{goal}}}
\def\goal{s_{\goals}}

\def\dir{{\mathtt{dir}}}
\def\means{{\mathtt{means}}}

\def\motor{{\mathtt{mc}}}
\def\car{{\mathtt{car}}}
\def\loc{{\omega}}
\def\Block{{\gridworld_{\mathtt{blocks}}}}
\def\noBlock{{\gridworld^c_{\mathtt{blocks}}}}

\def\noblock{{\noBlock}}

\def\room{{\mathtt{room}}}
\def\agent{{\cur}}
\def\agentnext{{\curnext}}

\def\obstacles{{\mathtt{obstacles}}}

\def\pick{{\mathtt{pick}}}
\def\key{{\mathtt{key}}}
\def\keyvec{{s_{\underline{\mathtt{key}}}}}

\def\keypickvec{{s_{\underline{\mathtt{pick}}}}}
\def\keypickvecnext{{s'_{\underline{\mathtt{pick}}}}}

\def\door{{\mathtt{door}}}
\def\sopen{1}
\def\sclosed{0}
\def\neardoor{\mathcal{N}}
\def\open{{\mathtt{open}}}

\def\doorvec{{s_{\underline{\mathtt{door}}}}}

\def\dooropenvec{{s_{\underline{\mathtt{open}}}}}
\def\dooropenvecnext{{s'_{\underline{\mathtt{open}}}}}

\def\goalpick{{s_{\mathtt{done}}}}
\def\goalpicknext{{s'_{\mathtt{done}}}}

\def\indic{{\mathbbm{1}}}

\def\cA{{\mathcal{A}}}

\def\cS{{\mathcal{S}}}
\def\cSA{{\cS\cA}}
\def\cSAS{{\cS\cA\cS}}

\def\gridworld{{\Omega}}
\def\jam{{\mathtt{jams}}}
\def\jams{{\gridworld_{\mathtt{jams}}}}
\def\djams{{\cD_{\mathtt{jams}}}}
\def\changemeans{{\cD_{\mathtt{pause}}}}
\def\nochangemeans{{\cD^c_{\mathtt{pause}}}}
\def\nojams{{\gridworld^c_{\mathtt{jams}}}}
\def\dnojams{{\cD^c_{\mathtt{jams}}}}

\def\piex{\ensuremath{\overline{\pi}}\xspace}
\def\Sinit{\ensuremath{\cS^{\text{init}}}\xspace}

\def\Sterm{\ensuremath{\cS^{\text{end}}}\xspace}
\def\policysequence{sequence of generator sets}

\def\chng{{\mathtt{pause}}}

\def\r{{\text{r}}}

\def\argmin{{\text{argmin}}}

\def\comp{{\text{comp}}}
\def\decomp{{\text{decomp}}}
\DeclareMathOperator*{\argmax}{arg\,max}

\newcommand{\mult}{\cdot}

\definecolor{Gray}{gray}{0.8}
\definecolor{LightCyan}{rgb}{0.88,1,1}
\definecolor{LightRed}{rgb}{1,0.8,0.8}
\newcolumntype{a}{>{\columncolor{Gray}}c}

\newcommand{\levidx}[1]{^{{#1}}}

\usepackage{lastpage}
\ShortHeadings{Multi-level meta-reinforcement learning}{Yang and Maggioni}
\firstpageno{1}

\usepackage{hyperref}
\usepackage{cleveref}

\begin{document}

\title{Multi-level meta-reinforcement learning with skill-based curriculum}

\author{
\name Sichen Yang\thanks{Corresponding author.} \email syang114@jhu.edu \\
\addr Department of Applied Mathematics \& Statistics \\ Johns Hopkins University, 3400 N. Charles Street, Baltimore, MD 21218, USA
\AND
\name Mauro Maggioni\email mauromaggionijhu@icloud.com \\
\addr Department of Mathematics, Department of Applied Mathematics \& Statistics \\ Johns Hopkins University, 3400 N. Charles Street, Baltimore, MD 21218, USA
}

\maketitle
\begin{abstract}
We consider problems in sequential decision making with natural multi-level structure, where sub-tasks are assembled together to accomplish complex goals. Systematically inferring and leveraging hierarchical structure has remained a longstanding challenge; we describe an efficient multi-level procedure for repeatedly compressing Markov decision processes (MDPs), wherein a parametric family of policies at one level is treated as single actions in the compressed MDPs at higher levels, while preserving the semantic meanings and structure of the original MDP, and mimicking the natural logic to address a complex MDP. Higher-level MDPs are themselves independent MDPs with less stochasticity, and may be solved using existing algorithms. As a byproduct, spatial or temporal scales may be coarsened at higher levels, making it more efficient to find long-term optimal policies. 
The multi-level representation delivered by this procedure decouples sub-tasks from each other and usually greatly reduces unnecessary stochasticity and the policy search space, leading to fewer iterations and computations when solving the MDPs. 
A second fundamental aspect of this work is that these multi-level decompositions plus the factorization of policies into embeddings (problem-specific) and skills (including higher-order functions) yield new transfer opportunities of skills across different problems and different levels.
This whole process is framed within curriculum learning, wherein a teacher organizes the student agent's learning process in a way that gradually increases the difficulty of tasks and ensures the abundance of transfer opportunities across different MDPs and different levels within and across curricula. 
The consistency of this framework and its benefits can be guaranteed under mild assumptions. 
We demonstrate abstraction, transferability, and curriculum learning in some illustrative domains, including MazeBase+, a more complex variant of the MazeBase example. 
\end{abstract}


\tableofcontents

\section{Introduction}

\label{s:intro}

Discovering and exploiting \emph{multi-level structure} is a longstanding challenge in sequential decision making. Classical hierarchical reinforcement learning (HRL) decomposes tasks into reusable sub-policies—e.g., options and semi-Markov abstractions~\citep{SuttonOptions,DietterichHRL,ParrRussell1997,DayanHinton1993}. 
These constructions are often restricted to one/two levels, or rely on hand-specified subgoals, which can hinder principled planning and transfer at scale \citep{BartoHRLReview}. 
Recent developments in deep HRL automates parts of this pipeline (option-critic, FeUdal, HIRO, HAC, \citep{Bacon2017OptionCritic, Vezhnevets2017FeUdal, Nachum2018HIRO, Levy2019HAC, Dwiel2019GoalSpace}). While powerful, these methods can entangle sub-tasks and propagate unnecessary stochasticity across levels, complicating long-horizon planning.

\noindent\textbf{Our contribution.} We introduce a teacher–student–assistant meta-reinforcement learning (RL) framework that (i) repeatedly compresses MDPs so that parametric families of policies at one level become single, abstract actions, yet having semantic meaning, at the next level, yielding  higher-level MDPs, which are easier to solve  and that can capture the high-level logic needed in complex problems;
(ii) factors policies into embeddings and skills to enable transfer across levels and across MDPs, accelerating solutions and enabling the creation of a dictionary of highly reusable policies, at different levels of abstraction, to be applied as new problems are encountered; and (iii) organizes learning as a skill-based curriculum aligned with these abstractions. 
This unified view yields fewer iterations at lower per-iteration cost when used in conjunction with existing optimization algorithms such as value iteration, scaling in sparse-reward domains where many prior approaches struggle. 

\noindent{\textbf{Context and relationships with related work.}}
Beside the connections briefly mentioned above, a separate thread in HRL learns skills \emph{without} external rewards to build general-purpose primitives \citep{Eysenbach2019DIAYN, Gregor2016VIC}. While broadly useful, na\"ive long-horizon composition can reintroduce stochasticity and tangled credit assignment; our multi-level compression \emph{decouples sub-tasks and reduces stochasticity} at higher levels.
{\em{Abstraction}} further motivates compression, with each compression step \emph{preserves the semantics} of the original MDP while shrinking variance and branching, so solving a long-horizon task reduces to a stack of cleaner subproblems solvable with standard methods \citep{Puterman}. We also exploit natural tensorization of action factors and function composition to induce transferable, high-level behaviors. This is conceptually very different from spectral and spatial techniques~\citep{MahadevanMaggioni:JMLR:07, Machado2017Laplacian, DeanGivan1997, LiWalshLittman2006, RavindranBarto2003}.

\emph{Curriculum learning}~\citep{JMLR:v21:20-212,Bengio2009Curriculum,Matiisen2017TSCL, Florensa2017RCG, Sukhbaatar2017ASP, Sukhbaatar2018HSP} supplies the second pillar. \emph{Many curricula operationalize difficulty by time to solve rather than time to learn, and often restrict the final task to a concatenation of subtasks.} Our \emph{compression-aligned} curriculum instead defines difficulty via natural human logic, mirroring how humans tackle complex tasks and improve optimization and transfer in sparse-reward regimes, with higher levels coarsening space/time while broadening scope as a natural byproduct. 

\emph{Transfer} is the third pillar, providing rapid learning and policy improvement across related tasks~\citep{Barreto2017SF, Barreto2018DeepSF, Andreas2017PolicySketches, Rusu2016ProgressiveNets}. We factor policies into \emph{embeddings} (problem-specific perception/featurization) and \emph{skills} (including reusable higher-order functions) and then compress families of such policies into single abstract actions. This supports \emph{transfer across levels and across MDPs, both within and across curricula}—even when state spaces differ—without relying on rote replay, targeting semantic reuse of skills rather than state memorization, without the need of storing and revisiting states~\citep{Ecoffet2021GoExplore}.

Recent evidence from meta-learning for compositionality shows that optimizing a standard neural network for compositional skills yields human-like systematic generalization~\cite{LakeBaroni2023Nature}. Because multi-level structure is a prior over \emph{families} of problems, our work connects to meta-RL~\citep{Duan2016RL2, Finn2017MAML, Rakelly2019PEARL, Frans2018MLSH}. Our compression and policy factorization provides constructive meta-generalization: higher levels expose slower, more stable dynamics, while lower levels encapsulate fast feedback, enabling cross-task reuse.

Factored Markov decision processes (FMDPs) are \emph{flat} MDPs whose dynamics and rewards admit a compact structured representation: the state is a tuple of variables, transitions are encoded by a two-slice dynamic Bayesian network (DBN), and rewards often decompose additively over small scopes \citep{BoutilierDeardenGoldszmidt1995,BoutilierDeardenGoldszmidt2000SDPFactored,Degris2013FMDPChapter}. 
Importantly for positioning, \emph{``factored'' is not ``hierarchical''}: FMDPs exploit \emph{structural} independence within a single-timescale model, whereas HRL and our multi-level framework primarily exploit \emph{temporal} abstraction. 
Our use of factorization is also different: 
rather than assuming a pre-specified factored (two-slice DBN) model of the environment as in FMDPs, we impose structure on the action/policy side: we decompose \emph{policies} into composable partial policies together with skill/embedding generators, enabling reusable skill extraction and transfer across tasks and abstraction levels.
Moreover, for generalization we allow higher-level action sets to be constructed using \emph{subcomponents} of action factors, and in particular to form a \emph{proper subset} of the full Cartesian product; this is crucial in our multi-level setting because higher-level action sets produced by compression are typically not full products.

Finally, our framework is compatible with inverse reinforcement learning (IRL) and imitation \cite{NgRussell2000IRL,AbbeelNg2004Apprenticeship,Ziebart2008MaxEntIRL,Finn2016GCL,HoErmon2016GAIL,Fu2018AIRL}, and  Hierarchical IRL \cite{Krishnan2016HIRL}. Because each compression step yields an independent, semantically preserved MDP, we can invert the process: estimate rewards or subgoal structures at appropriate levels, then learn skills and curricula consistent with demonstrations, improving sample-efficiency and interpretability relative to flat IRL \cite{Adams2022IRLSurvey}.

{An extended literature review, further connections and comparisons are in Sec.\ref{s:related-work}.}

In a follow-up paper we show how Q-learning can be extended to the multi-level MDPs we introduce here, allowing for efficient multi-level exploration and learning at different levels of abstractions; we will also introduce virtual policies, to represent potential skills yet-to-be-learned, that are learned on-demand, only if useful to accomplish a task at hand, increasing the scalability and transferability of our approach. We will also introduce virtual policies to model potential solutions of yet-to-be-solved sub-problems that may participate in the solution of larger problems, and the possibility of learning multi-level recursive tasks, such as sorting, within an extension of the framework proposed here.

\section{Multi-level Markov decision processes}
\label{s:MMDP}

\subsection{Basic definitions}
\label{s:MMDP-definitions}

A {\bf Markov decision process (MDP)} is a sequential decision making process where the agent interacts with the environment and learns how to maximize the cumulative rewards he could obtain from such interactions. An MDP is modeled as a seven-tuple $\MDP:=(\cS,\Sinit,\Sterm,\cA,P,R,\Gamma)$ of objects in the following list:
\begin{enumerate}[leftmargin=0.5cm,itemsep=0pt]
\item {\bf State space} $\cS$ consisting of all the possible states of the agent. Distinguished subsets include the set of {\bf initial states} $\Sinit\subseteq\cS$  and the set of {\bf terminal states} $\Sterm\subseteq\cS$, at which an episode starts and ends, respectively. We call $\MDP$ episodic if $\Sterm\neq \varnothing$. Without losing generality, we set $\Sinit\subseteq\cS\setminus\Sterm$ to be the set of states starting from which the agent can reach some state in $\Sterm$ if $\Sterm\neq \varnothing$, and set $\Sinit = \cS$ otherwise.

\item {\bf Action set} $\cA:=\cup_{s\in\cS}\cA(s)$ and, for every $s\in\cS$, the set of actions $\cA(s)$ available in $s$. 
We let $\cSA :=\{(s,a):s\in \cS,a\in \cA(s)\}$.  
We assume that $\cA\subseteq\cA_1\times \cdots \times \cA_K$, where $\cA_k$ is called the $k$-th action factor for $k\in [K]:=\{1,\dots,K\}$. 
Without losing generality, we assume each action factor $\cA_k$ contains a special element $\endactionfactor$ that is \textit{universal} across all different MDPs and used to end a policy. 
We let $\endactionset$ 
be the set of actions in which at least one factor is equal to $\endactionfactor$, and assume $\endactionset\subseteq \cA(s)$ for any $s\in\cS$. 
We let $\endaction\in\cA$ be the action whose coordinates are $\endactionfactor$ in all action factors.\footnote{$\endactionfactor$ is underlined to indicate it is a tuple.}
For any $\cA'\subseteq \cA$, we denote $\overline{\cA'} :=\cA'\cup \endactionset$.
We can restrict actions to a subset of factors: for $I\subseteq[K]$, called an active action factor set, we map $a\in\cA$ to $a_I$ by setting to the special value $0$, an ``action'' that cannot be taken, the coordinates of $a$ corresponding to the action factors $\cA_k$ with $k\notin I$, and let $\cA_I$ be the image of $\cA$ under such map.

\item {\bf Transition probabilities} $P:\cSA\times\cS \rightarrow [0,1]$, where $P(s,a,s')$ is the probability of reaching state $s'$ for an agent in state $s$ selecting action $a$. We assume $P(s,a,s') = \indic_{\{s\}}(s')$ for any $a\in \endactionset$, and denote $\cSAS :=\{(s,a,s')\in \cSA\times \cS: P(s,a,s')>0\}$.

\item {\bf Rewards} $R:\cSAS \rightarrow \mathbb{R}$, with $R(s,a,s')$ the reward for an agent in state $s$, selecting action $a$, and reaching state $s'$. 
We will always define $R(s,a,s') = 0$ for any $s\in\Sterm$, $R(s,a,s') = r$ for any $s\notin\Sterm, a\in \endactionset$ and for some $r<0$. 

\item {\bf Discount factors} $\Gamma:\cSAS \rightarrow [0,1]$, with $\Gamma(s,a,s')$ the discount applied to reward $R(s,a,s')$. We always define $\Gamma(s,a,s') = 1$ for any $a\in \endactionset$.
\end{enumerate}

We assume throughout that $\cS,\cA$ are finite sets, albeit the extension to continuous spaces seems rather straightforward (e.g., by using representations on basis functions), and that the reward function $R$ is bounded.
A {\bf policy} $\pi:\cSA\rightarrow[0,1]$ is, for each $(s,a)\in\cSA$, the probability of selecting $a$ when the agent is in state $s$, and satisfies $\sum_{a\in \cA(s)}\pi(s,a)=1$ for any $s\in \cS$. 
Without losing generality, we assume that $\sum_{a\in \endactionset}\pi(s,a)=1$ for any $s\in\Sterm$, since each episode ends when the agent reaches $\Sterm$. 
Most often, and for all of our cases, we will let $\pi(s,a)=\indic_{\{\endaction\}}(a)$ for any $s\in\Sterm$.
This will also be useful later in the construction of multi-level MDPs (MMDPs) and in policy transfer.
The goal of solving $\MDP$ is to learn an optimal policy $\pi_{\ast}$, that maximizes the long-term cumulative rewards, summarized by the {\bf{value function}} $V_\pi(s):=\mathbb{E}_{\tau,S_{1:\tau},A_{0:\tau-1}}[R_{0,\tau}  |S_0=s,A_{0:\tau-1}\sim \pi]$, for any $s\in\Sinit$.
Here, for any two (random) times $T,T'$, $0\le T<T'<\infty (a.s.)$, the random variable $R_{T,T'}$ is the discounted reward accumulated over the interval $[T,T']$: 
\begin{align} \label{e:cumulative-rewards}
R_{T,T'}:=R(S_T,A_{T+1},S_{T+1})+\sum_{t=T+1}^{T'-1}[\Pi_{t'=T}^{t-1}\Gamma(S_{t'},A_{t'+1},S_{t'+1})]\mult R(S_{t},A_{t+1},S_{t+1})\,,
\end{align} 
and $\tau$ being the first time $t$ such that $S_t \in\Sterm$, and $\tau = +\infty$ if such time does not exist and for non-episodic MDPs.
Of course $\pi_\ast$ is independent of the initial state $S_0$ by Markovianity.

We now exploit the tensor product structure underlying the action space to introduce partial policies, that use only a few action factors. They are potentially easier to learn and more transferable, and they can be combined (using the notion of outer products) to obtain general policies. 
We will then introduce partial policy generators to allow the teacher to provide hints to the student on how to potentially restrict the search of an optimal policy to interesting subsets of combinations of partial policies.

For each $I\subset[K]$, a {\bf partial policy} $\pi_I:\cSA_I\rightarrow[0,1]$ is, for each $(s,a_I)\in\cSA_I$, the probability of selecting $a_I\in\cA(s)$ when the agent is in state $s$, and satisfies $\sum_{a\in \cA_I(s)}\pi_I(s,a)=1$ for any $s\in \cS$. 
When $I=[K]$, $\pi_I$ is a policy in the usual sense; otherwise, a partial policy does not prescribe valid/useful actions for the agent, as it prescribes actions with some factors equaling to $0$. 
A {\bf partial policy generator} $g_I$ is a function from a parameter set $\Theta$ to the set of all partial policies $\{\pi_I\}$, yielding a parametric family of partial policies $\{g_I(\theta):\theta\in \Theta\}$. 
Given a finite set $\cG = \{g_{I_1}, g_{I_2}, \dots, g_{I_M}\}$  of partial policy generators (not necessarily sharing the same active action factor set), we define the {\bf set of partial policies generated from $\cG$} as $\widetilde{\Pi}_{\cG}:=\cup_{m=1}^M\{g_{I_m}(\theta):\theta\in \Theta_{m}\}$.
\footnote{Note the following abuse of notation: $I$ in $g_I$ is part of the name of the function $g_I$, but it also indicates that the active action factor set of the function $g_I$ is $I$.
In particular $g_{I_1}$ and $g_{I_2}$ are in general completely unrelated functions (there is no function $g$).}

In order to construct valid policies from partial policies, we appropriately ``combine'' partial policies acting along mutually disjoint action factors that overall ``cover'' all possible action factors.
First of all, given two partial policies $\pi_I, \pi'_{I'}$ with $I\bigcap I'=\varnothing$, we define their {\bf outer-product} as the improper partial policy $\pi_I\otimes \pi'_{I'}:\{(s,(a_I,a_{I'})): s\in\cS, a_I\in\cA_I(s), a_{I'}\in\cA_{I'}(s)\}\rightarrow [0,1]$ satisfying
\beq \label{e:def-otimes}
(\pi_I\otimes \pi'_{I'})(s,(a_I,a_{I'})): = \pi_I(s,a_I)\mult\pi'_{I'}(s,a_{I'})\,.
\eeq 
We call this improper because a (proper) partial policy is defined on $\cSA_{I\cup I'}$, but $(s,(a_I,a_{I'}))$ may not belong to $\cSA_{I\cup I'}$, since $\cA_{I\cup I'}(s)$ is in general not contained in $\cA_I(s)\times \cA_{I'}(s)$. 
However, we can map every improper partial policy to a proper one by restricting it to $\cSA_{I\cup I'}$ and renormalizing it.
This outer product operator, with or without the renormalization, is commutative and associative, so from \eqref{e:def-otimes} it is straightforward to define a unique outer product of multiple partial policies with pairwise disjoint active action factor sets. Hereafter, we may write $(s,a_I,a_{I'})$ instead of $(s,(a_I,a_{I'}))$ for simplicity when there is no ambiguity.

Now, given the set $\widetilde{\Pi}_{\cG}$ of partial policies generated by the finite {\bf{partial policy generator set}} $\cG$ as above, we can define the {\bf{set of polices $\Pi_{\cG}$ generated by $\cG$}} as the union of all policies $\pi$ that are the outer product of finitely many partial policies in $\widetilde{\Pi}_{\cG}$ restricted to $\cA(s)$ and then normalized. 
This means that a policy $\pi$ in $\cG$, before restriction and normalization, can be written as $g_{I_{m_1}}(\theta_{m_1})\otimes \cdots \otimes g_{I_{m_J}}(\theta_{m_J})$ for some $J\in [M], \{m_j:j\in [J]\}\subseteq [M]$, $\theta_{m_j}\in \Theta_{m_j}$ for any $j\in [J]$, and $I_{m_{j_1}}\cap I_{m_{j_2}}=\emptyset$ for $j_1\neq j_2$; it therefore may be represented by a vector with entries $\theta_{m_j}$ for $j\in [J]$ and $\emptyvalue$ for all the other entries, with indices $[M]\setminus\{m_j:j\in [J]\}$. 
Note that $\emptyvalue$ is a special value \textit{universal} across all different MDPs, meaning that the corresponding partial policy generator in $\cG$ is not selected. 
So, all the policies in $\Pi_{\cG}$ could be represented by an element in the product set $(\Theta_1\cup\{\emptyvalue\})\times (\Theta_2\cup\{\emptyvalue\})\times\cdots\times(\Theta_M\cup\{\emptyvalue\})$.

\subsection{Multi-level Markov decision processes (MMDPs)}
\label{s:MMDP-def}

We construct MMDPs starting from a single $\MDP:=(\cS$, $\Sinit$, $\Sterm$, $\cA$, $P$, $R$, $\Gamma)$: it consists of a family of MDPs on the same state space, but with multiple levels that model increasing degrees of ``abstraction'', and enable the agent to enact more and more powerful actions as the level increases.
For instance, actions at level two are constructed as follows. Let $\cG^1$ be a partial policy generator set for the original $\MDP$ (at level one); define the set $\cA^{2}(s):= \overline{\Pi_{\cG^1}}$, and taking action $a\in\cA^2(s)$ at state $s$ will mean running the policy $a\in\Pi_{\cG^1}$ started at $s$ for a certain amount of time as explained below -- with higher levels usually allowing for longer running times. A single action $a\in\cA^2$ corresponds to a whole sequence of actions at level one, determined by one of the policies generated by $\cG^1$. This construction is iterated to construct higher levels.
We show how MMDPs can enable faster solution of an MDP, and transfer across multiple MDPs.

\begin{definition} A {\bf \policysequence} is a sequence of finite partial policy generator sets $\{\cG\levidx{l}\}_{l=1}^\infty$ on an $\MDP$ that satisfies the following: $\cG\levidx{l}$ is defined on $\cSA\levidx{l}$ for any $l\in \mathbb{Z}^+$, where $\cA\levidx{1}:=\cA$ and $\cA\levidx{l+1}(s):=\overline{\Pi_{\cG\levidx{l}}}$ for $l\geq 1$ and all $s\in\cS$. We let $\Pi\levidx{l}:=\Pi_{\cG\levidx{l}}$.
\end{definition}
The inputs for the construction of an MMDP from a given $\MDP$ are the following:
\begin{enumerate}[leftmargin=0.5cm,itemsep=0pt]
\item A {\bf{\policysequence}} $\{\cG\levidx{l}\}$ that yields a sequence of state space-action set pairs $\{\cSA\levidx{l}\}_{l=1}^\infty$. We also assume that the teacher provides another {\bf{sequence of finite partial policy generator sets}} $\{\cG\levidx{l}_{\text{test}}\}_{l=1}^\infty$ with $\cG\levidx{l}_{\text{test}}$ defined on $\cSA\levidx{l}$ for any $l\in \mathbb{Z}^+$ that will be used later to help the student assess the difficulty of $\MDP$ being solved.
\item A {\bf{timescale}} $1\leq t_{\pi}\leq +\infty$, for each policy $\pi\in (\cup_{l=1}^\infty \Pi\levidx{l})\cup(\cup_{l=1}^\infty \Pi_{\cG\levidx{l}_{\text{test}}})$, giving probability $1/t_{\pi}$ of terminating at each time step (if $t_{\pi}= +\infty$, then $1/t_{\pi}:=0$). Optionally, the teacher also provides $t_{\min}$ and $t_{\max}$, a lower bound and an upper bound on the timescale of any policy in a curriculum respectively, which are attributed to $t_{\bound}$ as in Alg.~\ref{algorithm:MMDP-solver}.
\item A {\bf{sequence of negative rewards}} $\{r\levidx{l}\}_{l=1}^\infty$, to penalize choosing an action containing $\endactionfactor$ at a non-terminal state. In particular, $R(s,a,s)=r\levidx{1}$ for $a\in (\cA^1)^{\mathtt{end}}$ and $s\notin \Sterm$. 
\end{enumerate}

We are ready to construct an $\MMDP$ in an inductive manner with the inputs provided above. For the initial step at level one, we let $\MDP\levidx{1}:=(\cS,\Sinit,\Sterm,\cA\levidx{1},P\levidx{1},R\levidx{1},\Gamma\levidx{1})$ to be the given  $\MDP=(\cS,\Sinit,\Sterm,\cA,P,R,\Gamma)$. If the optimal policy $\pi\levidx{1}_{\ast}$ of $\MDP\levidx{1}$ learned by the student is in $\Pi_{\cG\levidx{1}_{\text{test}}}$, then $\MDP\levidx{1}$ is defined to be of difficulty $1$, and we are done. Otherwise, we move to higher levels. 
Inductively, given $\MDP\levidx{l}:=(\cS,\Sinit,\Sterm,\cA\levidx{l},P\levidx{l},R\levidx{l},\Gamma\levidx{l})$, $l\in \mathbb{Z}^+$, we define the $(l+1)$-st level $\MDP\levidx{l+1}$ with:
\begin{enumerate}[leftmargin=0.5cm,itemsep=0pt]
\item {\bf{State space}} $\cS$, set of initial states $\Sinit\subseteq\cS$, and set of terminal states $\Sterm\subseteq\cS$: the same as those of the original $\MDP\levidx{1}$.

\item {\bf{Action set}} $\cA^{l+1}(s):=\overline{\Pi\levidx{l}}$, represented as a subset of $\cA^{l+1}_1\times \cA^{l+1}_2\times\cdots \times \cA^{l+1}_{|\cG\levidx{l}|}$, with $\cA\levidx{l+1}_k := \Theta\levidx{l}_k\cup\{\emptyvalue,\endactionfactor\}$ for $k\in[|\cG\levidx{l}|]$, and $\Theta\levidx{l}_k$ being the domain of the $k$-th generator in $\cG\levidx{l}$. 
Note that $\cA^{l+1}(s)$ does not depend on $s$; this is not as restrictive as it may appear, as a policy in $\overline{\Pi\levidx{l}}$ may choose the action $\endaction$ with probability $1$ at certain states. 

\item {\bf{Transition probabilities}} $P^{l+1}(s,a\levidx{l+1},s')$: when $a^{l+1}\in\Pi\levidx{l}$ is chosen in state $s$, $a^{l+1}$ will be run starting from $s$ till time $\tau$ which is the minimum of the first time an action in $\endactionset$ is chosen and a random stopping time with geometric distribution with parameter $1/{t_{a^{l+1}}}$. \footnote{$t_a=\infty$ is allowed and means that with probability $1$ such time is $\infty$.} Then $P^{l+1}:\cSA\levidx{l+1}\times\cS\rightarrow [0,1]$, for such $a\levidx{l+1}\in \Pi\levidx{l}$, is defined to be the total probability of terminating at the state $s'$:
\begin{align*} 
P^{l+1}(s,a\levidx{l+1},s'):= \mathbb{P}^l_{\tau,S_{1:\tau},A_{0:\tau-1}}[&S_\tau=s'|S_0=s,A_{0:\tau-1}\sim a\levidx{l+1},\\ &\tau = \min\{\tau_0, \min\{\tau':A_{\tau'} \in \endactionset\}\}, \tau_0\sim \mathrm{Geo}({1}/{t_{a\levidx{l+1}}})]\,.
\end{align*} 

\item {\bf{Rewards}} $R^{l+1}: \cSAS\levidx{l+1}\rightarrow \mathbb{R}$ are set to $0$ when  $s\in\Sterm$, and to $r\levidx{l+1}$ when $a\levidx{l+1}\in (\cA\levidx{l+1})^\mathtt{end}$. In all other cases, they are defined to be the expected total discounted reward collected along trajectories associated to $(s,a\levidx{l+1})$, ending at $s'$ at time $\tau$, as described above: 
\beq \label{e:def-coarsened-reward}
R^{l+1}(s,a\levidx{l+1},s'):= \mathbb{E}^l_{\tau,S_{1:\tau},A_{0:\tau-1}}[R\levidx{l}_{0,\tau}  |S_0=s,S_\tau=s',A_{0:\tau-1}\sim a\levidx{l+1}]\,,
\eeq
where, for any two (random) times $T,T'$ satisfying $0\le T<T'<\infty (a.s.)$, the random variable $R\levidx{l}_{T,T'}$ is the discounted reward accumulated over the interval $[T,T']$ in $\MDP\levidx{l}$, defined similarly to \eqref{e:cumulative-rewards}: 
$$R\levidx{l}_{T,T'}:=R\levidx{l}(S_T,A_{T+1},S_{T+1})+\sum_{t=T+1}^{T'-1}[\Pi_{t'=T}^{t-1}\Gamma^l(S_{t'},A_{t'+1},S_{t'+1})]\mult R^l(S_{t},A_{t+1},S_{t+1})\,.$$
The conditional event in \eqref{e:def-coarsened-reward} may happen with probability $0$, in which case $R\levidx{l+1}$ is not well defined for such input triples $(s,a\levidx{l+1}, s')$; this is inconsequential as such values are not needed when solving $\MDP\levidx{l+1}$. This also applies to the discount factor $\Gamma\levidx{l+1}$.

\item {\bf{Discount factors}} $\Gamma^{l+1}: \cSAS\levidx{l+1}\rightarrow (0,1]$ are set, for $a\levidx{l+1}\in \Pi\levidx{l}$, to be the expected product of the discounts applied to rewards along trajectories associated to $(s,a\levidx{l+1})$, ending at $s'$ at time $\tau$, as described above: 
$$\Gamma^{l+1}(s,a\levidx{l+1},s'):= \mathbb{E}^l_{\tau,S_{1:\tau},A_{0:\tau-1}}[\Gamma\levidx{l}_{0,\tau}  |S_0=s,S_\tau=s',A_{0:\tau-1}\sim a\levidx{l+1}]\,,$$
where, for any two (random) times $T,T'$ satisfying $0\le T<T'<\infty (a.s.)$, the random variable $\Gamma\levidx{l}_{T,T'}$ is the cumulative discount applied to trajectories $(S_T,S_{T+1},\cdots,S_{T'})$ over the interval $[T,T']$ in $\MDP\levidx{l}$: $\Gamma\levidx{l}_{T,T'}:=\Pi_{t=T}^{T'-1}\Gamma^l(S_{t},A_{t+1},S_{t+1})$.
\end{enumerate}

With $\MDP\levidx{l+1}:=(\cS,\Sinit,\Sterm,\cA\levidx{l+1},P\levidx{l+1},R\levidx{l+1},\Gamma\levidx{l+1})$ constructed by the student as above, we iterate: if the optimal policy $\pi_{\ast}\levidx{l+1}$ of $\MDP\levidx{l+1}$ learned by the student is in $\Pi_{\cG\levidx{l+1}_{\text{test}}}$, then $\MDP$ is of {\bf{difficulty}} $L:=l+1$, and we are done; otherwise, we move to higher levels. 

This completes the construction of MMDP: given the inputs $\MDP=(\cS,\Sinit,\Sterm,\cA, P,R,\\ \Gamma)$, $\{\cG\levidx{l}\}_{l=1}^\infty$,  $\{\cG\levidx{l}_{\text{test}}\}_{l=1}^\infty, \{t_{\pi}\}_{\pi\in \cup_{l=1}^\infty (\Pi\levidx{l}\cup\Pi_{\cG\levidx{l}_{\text{test}}})}$, and $\{r\levidx{l}\}_{l=1}^\infty$, we have uniquely defined the difficulty $L$ of the original $\MDP=(\cS,\Sinit,\Sterm,\cA, P,R,\Gamma)$, as well as the MMDP $\{\MDP\levidx{l}\}_{l=1}^L$, where from the second level on, the compressed MDP$^{l+1}$ at each level is a self-standing compression and abstraction of the finer MDP$^{l}$, with rewards and discount factors consistent in expectation with the finer MDP. 
The compression process is reminiscent of other instances of conceptually related procedures: coarsening (applied to meshes and PDEs, for example), homogenization (of PDEs or SDEs), and lumping (of Markov chains). 
The general philosophy is to reduce a large problem to a smaller one constructed by ``locally averaging'' parts of the original problem. 
Here this compression is in terms of partial policies, in a way that, crucially, preserves the MDP structure, and maintains the semantic meanings of the original MDP; only as a possible byproduct, it might lead to a coarsening of spatial or temporal scales, in the sense that the effective state space is greatly reduced because we enable the agent to enact more powerful actions, with longer running times, at higher levels.

\subsubsection{Solving an MMDP}
To solve the original MDP utilizing an associated MMDP, we follow a bottom-up procedure followed by a top-bottom procedure. In the bottom-up phase, we construct $\MDP\levidx{l+1}$ from $\MDP\levidx{l}$; in the top-bottom phase we solve the MDP at the highest level $L$, and then iteratively refine the solution all the way down to the first level, the original $\MDP$. 
The refinement step is achieved as follows. 
Given the optimal policy $\pi_{\ast}\levidx{l+1}$ for $\MDP\levidx{l+1}$, we produce an initial policy $\pi\levidx{l}$ for $\MDP\levidx{l}$ by ``unpacking'' $\pi_{\ast}\levidx{l+1}$ in $\MDP\levidx{l}$: each action $a$ in the sequence of actions in $\pi_{\ast}\levidx{l+1}$ is, by the very construction of the MMDP, a policy at level $l$ and corresponds therefore to a sequence of actions at level $l$; concatenating over actions in $\pi_{\ast}\levidx{l+1}$ corresponds to concatenating such policies, and leads to a policy at level $l$, which we denote as $\pi_\ast\levidx{l+1}\circledast \cG\levidx{l}$. 
The ``convolution'' $\pi \circledast \cG$ between a generator set $\cG=\{g_I: \Theta\rightarrow \{\pi_I:\sum_{a\in \cA_I(s)}\pi_I(s,a)=1 \,\mathrm{for}\, \mathrm{any}\, s\in \cS\}\}$ and a function $\pi$ mapping from $\cS\times\overline{\Pi_{\cG}}$, is a policy on $\cSA$ defined by: 
\begin{align}   \label{e: convolution}
(\pi\circledast \cG)(s,a): =\sum_{g_{I_{m_1}},\cdots,g_{I_{m_J}}}\sum_{\theta_{m_1},\cdots,\theta_{m_J}} \pi\big(s,\otimes_{j=1}^J g_{I_{m_j}}(\theta_{m_j})\big)\mult \big[\otimes_{j=1}^J g_{I_{m_j}}(\theta_{m_j})\big](s,a)\,, 
\end{align}
where the outer sum is over all $g_{I_{m_1}},\dots, g_{I_{m_J}}\in \cG$ mapping from $\Theta_{m_1}\cup\{\endactionfactor\},\Theta_{m_2}\cup\{\endactionfactor\},\cdots,\Theta_{m_J}\cup\{\endactionfactor\}$ respectively,\footnote{With a slight abuse of notations, we actually allow $\theta_m=\endactionfactor$ in \eqref{e: convolution}, in which case $g_{I_m}(\theta_m)=\indic_{(\endaction)_{I_m}}(a_{I_m})$. On the other hand, this extra comment is only for the purpose of full mathematical rigor, because for almost all the cases of interest (and for all of our cases), the optimal $\pi$ (and also $\pi\circledast \cG$ defined here) will select actions in $\endactionset$ with probability $0$ except at $s\in \Sterm$, which is not of interest.} such that their active action factor sets $I_{m_1},I_{m_2},\cdots,\\ I_{m_J}$ are a partition of $[K]$, where $K$ is the number of action factors in $\cA$, and the inner sum is over all parameters in their domains.\footnote{Here, we assume the domains of all the partial policy generators in $\cG$ are discrete domains; the generalization to the case when some of the domains are continuous is straightforward.}
Iterating over levels (from top to bottom), the optimal policy $\pi_{\ast}\levidx{L}$ of $\MDP\levidx{L}$ is ``convolved'' with the {\policysequence} $\{\cG\levidx{l}\}_{l=1}^{L-1}$ to a policy $\tilde\pi_\ast$ for the original MDP:
\begin{align}   \label{e:approx-org-policy}
\tilde\pi_{\ast} := [\dots[(\pi_{\ast}\levidx{L}\circledast \cG\levidx{L-1})\circledast \cG\levidx{L-2}]\circledast\cdots]\circledast \cG\levidx{1}\,.
\end{align}
While $\tilde\pi_{\ast}$ will in general not be the optimal policy $\pi_\ast$ for the original MDP, we may hope that it is close to $\pi_\ast$, and a good start for value iteration. In fact, we may perform value iteration at each level, after each ``convolution'', in the top-bottom phase, and ideally only a few iterations will give the optimal policy at each level, which we then refine, as above, to the next finer level via ``convolution'', proceeding iteratively down in this fashion till we obtain the optimal policy for the original MDP.

The difficulty of $\MDP$ can then be viewed as follows: we are simplifying the optimal policy $\pi_\ast$ of the original $\MDP$ into a simple policy at the highest level $L$, and a sequence of ``convolutions'' with simple partial policy generators, with possibly some refinement via value iteration; the difficulty of $\MDP$ depends on how many ``convolution'' operators we need. 
Alternatively, it represents how many levels of abstraction we go through if we started from $\pi_\ast$ and compressed/abstracted/``un-convolved'' it with $\cG\levidx{1}$, $\cG\levidx{2},\dots$ until level $L$.

Perhaps surprisingly, there are explicit formulas for computing the higher-level transition probabilities, rewards, and discount factors, essentially as solutions of linear systems: see App.~\ref{s:MMDP-compress} for details.
This enables, for example, the explicit calculation of all the quantities in $\MDP_{2,2}\levidx{2}$ in the forthcoming MazeBase+ example, see \eqref{e:compressed-MDP22-MazeBase}.

\newcounter{boxcounter}

\subsection{The MazeBase+ example, part I}
\label{box:MDPs}
MazeBase+ is a more challenging version of the ``classical'' MazeBase, a well-known example in HRL, see \cite{Sukhbaatar2018LearningGE,sukhbaatar2018intrinsic} and references therein. 
We solve it with a curriculum containing a (stochastic) MDP of difficulty $3$, showcasing the role of the teacher, the assistant, and the student, and of curricula, abstractions and transfer learning (both within this example and from another example, ``navigation and transportation with traffic jams'', introduced later in Sec.~\ref{s:examples-MMDP-transfer-learning}). We show minimal examples to focus on the key ideas; on larger grid worlds the computational gains of our framework would be even larger due to the multi-scale compression.
At this point of the presentation we do not have quite introduced all the tools for tackling this problem: we give a high-level overview here, and detail its solution and important variations later. As a preview, in the follow-up paper where we introduce virtual policies, we will use an extended MazeBase+ example to illustrate the main ideas behind virtual policies. 

We visualize the MazeBase+ in Fig.~\ref{f:MazeBase}.
The states of the doors, open or not, are shown as they are in the initial state. 
We index doors and keys to show their correspondences. 
The goal of the agent is to travel in a grid world from an initial position to a known final one where it needs to pick up a goal (with large reward).
Because the whole grid world is separated into multiple rooms by blocks and doors, the agent often needs to open (possibly multiple) doors in order to reach the goal, by first picking up the corresponding keys.
We tackle this MDP by mapping it to an MMDP with $3$ levels, and using a curriculum with four MDPs of varying difficulty. 
More specifically, given the geometric configuration detailed in Sec.~\ref{s:MazeBaseGeometry}, the teacher provides a curriculum consisting of $\MDP_{1,1}$, $\MDP_{2,1}$, $\MDP_{2,2}$, and $\MDP_{3,1}$, of difficulties, $1,2,2,$ and $3$, respectively,
represented in the blue inset on the middle right, 
where a double-line arrow means the knowledge learned from one MDP (starting point of the arrow) is utilized when \textit{constructing} another MDP (endpoint of the arrow). 
The structure of the curriculum is reflected by the multi-level MDPs in the black boxes, also connected by corresponding arrows. 
Observe that there are many shared components between the target $\MDP_{3,1}$ and other MDPs, including the skill of navigation through doors, and the concatenation logic behind opening a door.
In order to efficiently solve this family of problems, we exploit the similarities between them to enable potential transfer:
\begin{enumerate}[leftmargin=0.5cm,itemsep=0pt]
\item[$\bullet$] $\MDP_{1,1}$ teaches the student to navigate within a single room; this is the same as $\MDP_{\text{dense}}^{\text{nav}}$ in the forthcoming example of navigation and transportation with traffic jams to be introduced in Sec.~\ref{s:examples-MMDP-transfer-learning}, with blocks and doors mapped to a subset of light traffic jams. 
In that MDP we will learn the skill $\overline\pi^{\mathrm{nav}}_{\obstacles}$ (a.k.a. $\overline\pi^{\mathrm{nav}}_{\text{dense}}$ in the context of $\MDP_{\text{dense}}^{\text{nav}}$) that teaches the student to navigate while avoiding traffic jams as much as possible.
\begin{figure}[t]
\centering
\includegraphics[width=1\textwidth]{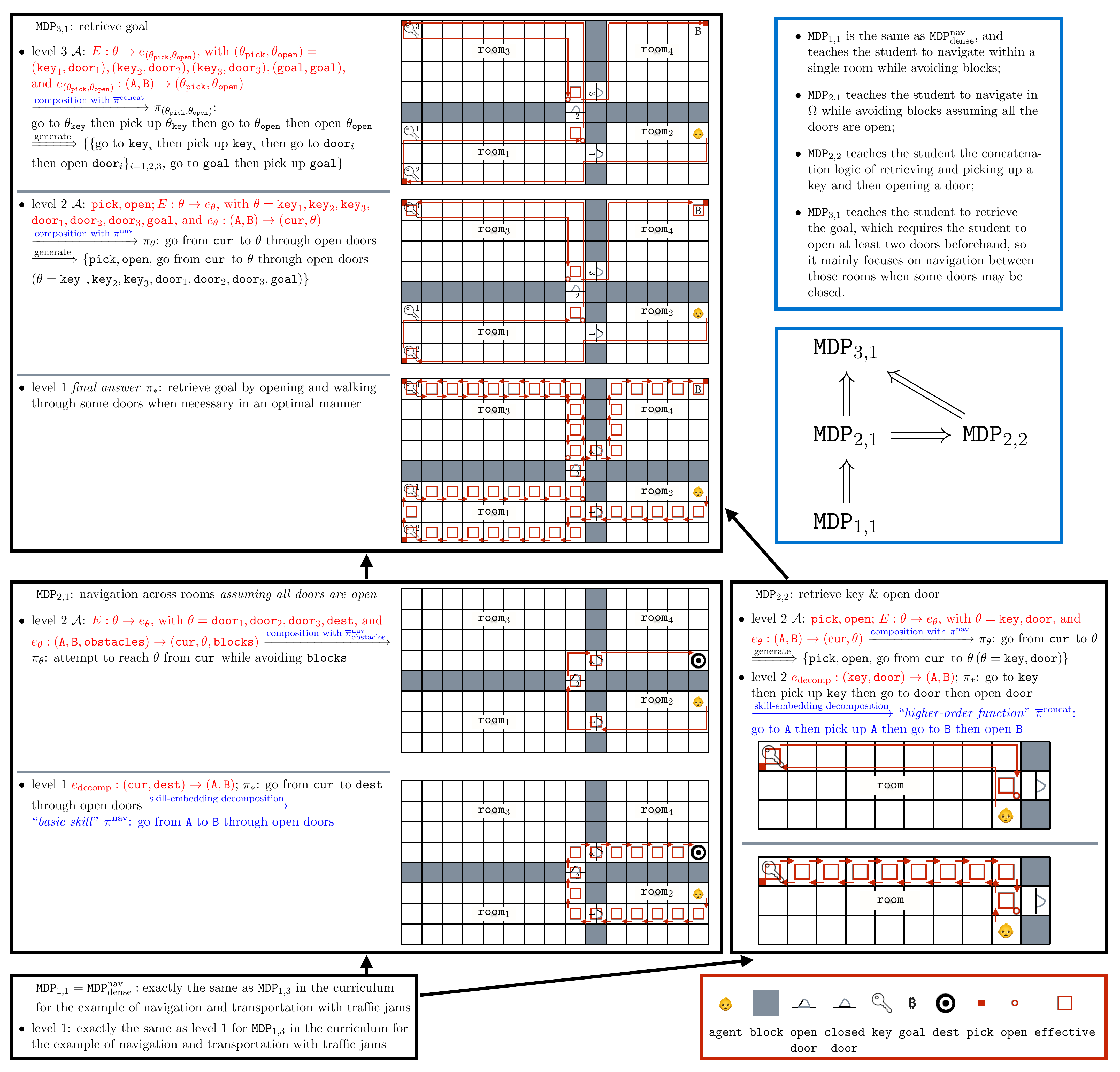}
\caption{A representation of MazeBase+, a more challenging version of the ``classical'' MazeBase, with curriculum in the blue inset (middle right), and the MMDPs in the curriculum, with their various levels, depicted and summarized in the black insets.} 
\label{f:MazeBase} 
\end{figure}
Therefore, we solve $\MDP_{1,1}$ by simply transferring that skill, now ready to be utilized as a single action in $\MDP_{2,1}$.

\item[$\bullet$] $\MDP_{2,1}$ teaches the student to navigate through all of $\gridworld$ while avoiding blocks, assuming all the doors are open. 
The agent will learn the optimal policy for this problem that becomes a higher-order function $\overline{\pi}^{\mathrm{nav}}$, to be utilized in $\MDP_{2,2}$ and $\MDP_{3,1}$.

\item[$\bullet$] $\MDP_{2,2}$ teaches the student the concatenation logic of retrieving and picking up a key and then opening the corresponding door: the optimal policy for $\MDP_{2,2}$ becomes a higher-order function $\overline{\pi}^{\mathrm{concat}}$, to be utilized in $\MDP_{3,1}$.

\item[$\bullet$] $\MDP_{3,1}$ teaches the student to pick up the goal, which may require navigating multiple rooms and opening multiple doors along the way.
\end{enumerate}
Also, in the first experiment we introduce here, $\door_1$ does not need to be opened, so it is just a ``confounder'' to confuse the agent (clearly, an optimal policy avoids picking up $\key_1$ and trying to apply the action of opening $\door_1$). See App.\ref{s:MazeBaseMDPs} for the formal definition of these MDPs.

A single action at higher level corresponds to an entire skill at the finer level beneath it, up to a composition/decomposition with family of maps that we call embeddings (represented by corresponding different types of arrows; this will be detailed later), and therefore leads to a long and complex path (long red arrow).
Roughly speaking, each skill is a parametric family of policies.
The effective state space (red squares) at higher levels, consisting of final states of actions at higher levels, is much smaller than that at level $1$. 
This leads to a significant speed-up in learning MDPs at higher levels, because we reduce the space of policies we search over. 
These actions and policies can be given names with semantic meanings, in the spirit of ``interpretable reinforcement learning''. 
The policy followed from a given initial position of the agent is obtained by following the red arrows we represent in the figure.
Optimal policies starting from certain initial states (which could be at any location) are represented by concatenations of red arrows (single actions at that level) and red squares (effective states at that level), again showing the significant reduction in the effective state space and in the number of actions at higher levels.
For example, in the inset representing level $2$ of $\MDP_{3,1}$ (middle row), the agent navigates from its initial position in $\room_2$ to $\room_1$ to the location of $\key_2$ using $\overline{\pi}^{\mathrm{nav}}$, then picks up $\key_2$ (small filled red square), then navigates to the location near $\door_2$ with $\overline{\pi}^{\mathrm{nav}}$, then opens it with $\key_2$ (small empty red circle) achieving a state that we call $s_2$, and so on.
In the inset representing level $3$ of $\MDP_{3,1}$ (top row), the actions are higher-level: with the first action the agent ends up in the state $s_2$.
This powerful action uses $\overline\pi^{\mathrm{concat}}$; we represent it by having the endpoints of the arrows corresponding to the use of $\overline{\pi}^{\mathrm{nav}}$ {\textit{touching}} the filled red square or the empty red circle.
Note how the effective state space, marked by the larger empty red squares, is reduced as we proceed from the lower to the higher levels of $\MDP_{3,1}$.
We do not plot the heat map for a value function because it is affected by changes in the states of the doors which we cannot represent compactly, as they occur along a trajectory. 
Different colors of text correspond to different roles in our framework, namely teacher (red), assistant (blue), and student (black).
To improve readability of the figure, and to emphasize the semantic meanings, some of the notations in the figure are slightly different from those in the text.

We are going to present three experiments in the MazeBase+ world. 
In the first one (Sec.\ref{s:curriculum}) we aim at solving MazeBase+ with a suitable curriculum. Please see Fig.~\ref{f:MazeBase} for a comprehensive summary of this example and the corresponding {\bf curriculum} (as defined in Sec.\ref{s:curriculum-definition}), which is a key component of our framework. 
A curriculum is a minimal ordered family of MDPs, whose solution helps the student agent develop skills, at different levels of ``abstraction'', useful for efficient solution of an original, difficult MDP.  
In the second one (Sec.\ref{s:MazeBase-transfer}), we apply our framework to perform efficient transfer to MazeBase+ worlds with different configurations of rooms, doors, keys, and the goal. Please see Fig.~\ref{f:MazeBase2} for a summary of this example.
In the third one (Sec.\ref{s:MazeBase-robust}), we consider the situation where the optimal policy at the highest level, when refined to a finer level, requires significant refinement in order to yield the optimal policy, while still outperforming na\"ive value iteration. This demonstrates the robustness of our optimization procedure.

\begin{figure}[t]
\centering
\includegraphics[width=1\textwidth]{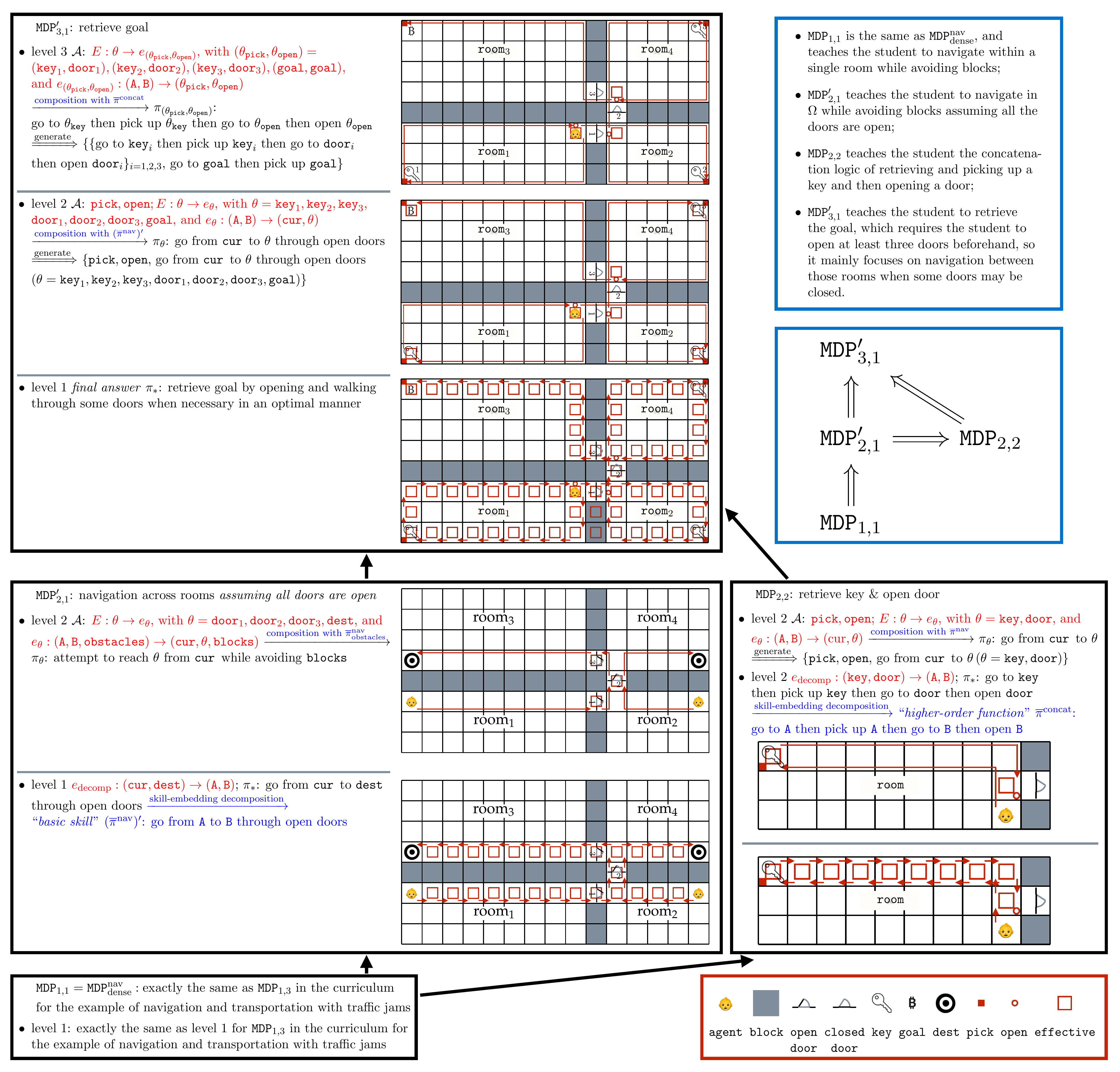}
\caption{A representation of the second experiment of the MazeBase+ example. The representation in this figure is similar to the one for the first experiment (Fig.~\ref{f:MazeBase}). In the second experiment, we apply our framework to perform efficient transfer to MazeBase+ worlds with different configurations of doors, keys, and the goal.}
\label{f:MazeBase2} 
\end{figure}

\subsubsection{Action factors and partial policy generators}
\label{box:ppg}
For some of the MDPs in this example, there is a single action factor, and to simplify the notation, we will sometimes omit the active action factor set $I$ in partial policies or partial policy generators when there is only a single action factor from now on.

As a consequence, most of the partial policy generators in the MazeBase+ example are policy generators, and one instance of policy generators is $(g_{2,2}\levidx{1})_{\alpha}$ used for $\MDP_{2,2}$. It is defined on $(\Theta_{2,2}\levidx{1})_{\alpha}:= \{\pick,\open\}$, indexing the two actions $\pick$ and $\open$, and $(g_{2,2}\levidx{1})_{\alpha}: (\Theta_{2,2}\levidx{1})_{\alpha}\rightarrow \{(\pi_{2,2}\levidx{1})^\alpha_\theta:\theta\in(\Theta_{2,2}\levidx{1})_{\alpha}\}$, mapping $\theta$ to $(\pi_{2,2}\levidx{1})^\alpha_\theta$.
It generates two partial policies $\{(\pi_{2,2}\levidx{1})^\alpha_\theta:\theta\in(\Theta_{2,2}\levidx{1})_{\alpha}\}$, with $(\pi_{2,2}\levidx{1})^\alpha_\theta:\cSA_{2,2}\levidx{1}\rightarrow \{0,1\} \,\,(\theta=\pick,\open)$, formally defined in \eqref{e:pialphatheta}, representing the partial policy of selecting the action $\theta$. 
Here, the role of the lower indices in the names will become apparent later as indicator of which MDP the generator is used on, 
and the role of the upper index $1$ in the names serves as indicator of the ``level'' of these MDPs within an MMDP.

The other partial policy generator $(g_{2,2}\levidx{1})_{\beta}$ used for $\MDP_{2,2}$ is more interesting. It is defined on $(\Theta_{2,2}\levidx{1})_{\beta}:=\{\key,\door\}$, indexing the two objects $\key$ and $\door$, and $(g_{2,2}\levidx{1})_{\beta}: (\Theta_{2,2}\levidx{1})_{\beta}\rightarrow \{(\pi_{2,2}\levidx{1})^\beta_\theta:\theta\in(\Theta_{2,2}\levidx{1})_{\beta}\}$, mapping $\theta$ to $(\pi_{2,2}\levidx{1})^\beta_\theta$. It generates two partial policies $\{(\pi_{2,2}\levidx{1})^\beta_\theta:\theta\in(\Theta_{2,2}\levidx{1})_{\beta}\}$, with
$(\pi_{2,2}\levidx{1})^\beta_\theta:\cSA_{2,2}\levidx{1}\rightarrow [0,1] \,\,(\theta=\key,\door)$ formally defined in \eqref{e:def-generator-mazebase}, representing the partial policy of going from the current location towards possible ``transfer hub'' ($\key$) or the destination ($\door$). A subtle comment here is that there is a little inconsistency between $(\pi_{2,2}\levidx{1})^\beta_{\door}$ and its semantic representation ``go from $\curs$ to $\door$'' in Fig.~\ref{f:MazeBase}, since $(\pi_{2,2}\levidx{1})^\beta_{\door}$ actually stops at $\neardoor(s_\door)$ (i.e., next to $\door$, as defined in App.~\ref{s:MazeBaseGeometry}) instead of stopping at $s_\door$, because the door may be closed initially.  
The way this was realized is that we ruled out the final step of reaching $s_\door$ from $(\pi_{2,2}\levidx{1})^\beta_{\door}$ generated from the policy generator $(g_{2,2}\levidx{1})_{\beta}$, and replaced it by taking the action $\endactionfactor$ with probability one at states satisfying $\agent \in \neardoor(s_\door)$, showcasing the need for ``$\endactionfactor$''. This also applies to similar ``going to a door'' policies throughout this example.

\subsubsection{Partial policy generator set}
\label{box:ppgs}
Take $\MDP_{2,2}$ for instance. 
For the first level, $\MDP_{2,2}\levidx{1}:=(\cS_{2,2}\levidx{1}$, $(\Sinit_{2,2})\levidx{1}$, $(\Sterm_{2,2})\levidx{1}$, $\cA_{2,2}\levidx{1}$, $P_{2,2}\levidx{1}$, $R_{2,2}\levidx{1}$, $\Gamma_{2,2}\levidx{1})$=$(\cS_{2,2}$, $\Sinit_{2,2}$, $\Sterm_{2,2},\cA_{2,2},P_{2,2},R_{2,2},\Gamma_{2,2})$. 
The teacher provides the partial policy generator set
$
\cG_{2,2}\levidx{1}:=\{(g_{2,2}\levidx{1})_{\alpha},(g_{2,2}\levidx{1})_{\beta}\}\,,
$
whose elements are defined in Sec.~\ref{box:ppg}, which generates the set of partial policies
$\widetilde{\Pi}_{\cG_{2,2}\levidx{1}} = 
\{(\pi_{2,2}\levidx{1})^\alpha_\theta:\theta\in (\Theta_{2,2}\levidx{1})_\alpha\}\cup\{(\pi_{2,2}\levidx{1})^\beta_\theta:\theta\in(\Theta_{2,2}\levidx{1})_\beta\}
$.
The student constructs the set of policies $\Pi_{2,2}\levidx{1}$ generated from the set of partial policies
$\widetilde{\Pi}_{\cG_{2,2}\levidx{1}}\,,
$
which is exactly the same as $\widetilde{\Pi}_{\cG_{2,2}\levidx{1}}$ in this scenario because there is only a single action factor in this example. More interesting examples utilizing $\eqref{e:def-otimes}$ can be seen in $\MDP_{3,1}$. 
Note that each policy in $\Pi_{2,2}\levidx{1}$ is represented by an element in the product set $((\Theta_{2,2}\levidx{1})_\alpha\cup\{\endactionfactor,\emptyvalue\})\times ((\Theta_{2,2}\levidx{1})_\beta\cup\{\endactionfactor,\emptyvalue\})$. 
For instance, $(\pi_{2,2}\levidx{1})^\alpha_{\pick}$ is represented by $(\pick,\emptyvalue)$. So, $\Pi_{2,2}\levidx{1}$ has two action factors.

These partial policy generators are ripe for being used to build higher-level actions for MMDPs, as we discuss next.

\subsubsection{Inputs for the construction of MMDPs}
\label{box:inputs}
The provided {\bf \policysequence} $\{\cG\levidx{l}_{2,2}\}_{l=1}^\infty$ for $\MDP_{2,2}$ in this example consists of $\cG\levidx{1}_{2,2}$ defined as in Sec.~\ref{box:ppgs}, and $\cG\levidx{l}_{2,2} := \varnothing$ for $l\geq 2$. There are multiple options for $\{\cG\levidx{l}_{2,2,\text{test}}\}_{l=1}^\infty$ here, with these restrictions: (1) $\pi_{2,2,\ast}\notin \cG\levidx{1}_{2,2,\text{test}}$; (2) $\pi\levidx{2}_{2,2,\ast}\in \cG\levidx{2}_{2,2,\text{test}}$, with $\pi_{2,2,\ast}$, $\pi\levidx{2}_{2,2,\ast}$ the optimal policies of $\MDP_{2,2}$, $\MDP\levidx{2}_{2,2}$ and provided in \eqref{e:first-level-policy-eqn-mazebase}, \eqref{e:second-level-policy-eqn_mazebase} respectively. These two conditions together guarantee that $\MDP_{2,2}$ is all of difficulty $2$ (See Sec.~\ref{s:MMDP-def}).
The timescale of $(g_{2,2}\levidx{1})_{\alpha}$ is $1$, and the timescale of $(g_{2,2}\levidx{1})_{\beta}$ is $+\infty$. We let $r\levidx{1} = -10$.

The student then constructs the second-level $\MDP\levidx{2}_{2,2}:=(\cS_{2,2},\Sinit_{2,2},\Sterm_{2,2},\overline{\Pi\levidx{1}_{2,2}},P\levidx{2}_{2,2}$, $R\levidx{2}_{2,2}$, $\Gamma\levidx{2}_{2,2})$, using the inputs above and the procedures we described in Sec.~\ref{s:MMDP-def}.

\subsubsection{Unpacking compressed policies}
\label{box:unpack}
Once the second-level MDP $\MDP\levidx{2}_{2,2}=(\cS_{2,2},\Sinit_{2,2},\Sterm_{2,2},\overline{\Pi\levidx{1}_{2,2}},P\levidx{2}_{2,2},R\levidx{2}_{2,2},\Gamma\levidx{2}_{2,2})$ is constructed, it can be solved to find the optimal policy $\pi_{2,2,\ast}\levidx{2}$. 
We construct stochastic trajectories starting from each state $s\in \cS$ by ``gluing'' together the actions in $\{(\pi_{2,2}\levidx{1})^\alpha_\theta:\theta\in (\Theta_{2,2}\levidx{1})_\alpha\}\cup\{(\pi_{2,2}\levidx{1})^\beta_\theta:\theta\in(\Theta_{2,2}\levidx{1})_\beta\}$ (defined in Sec.~\ref{box:ppg}) following an order of the form $(\pi_{2,2}\levidx{1})^{\lambda_0}_{\theta_0}, \cdots, (\pi_{2,2}\levidx{1})^{\lambda_t}_{\theta_t}, \cdots$: the agent starts from $S_0=s$ in $\MDP\levidx{2}_{2,2}$, and chooses actions in $A_t = (\pi_{2,2}\levidx{1})^{\lambda_t}_{\theta_t}$ for any $0\leq t\leq \tau-1$ with $\tau$ being the first time $t$ such that $S_t \in\Sterm$ (if such event does not occur, we set $\tau=+\infty$). 
For each state $s$, the sequence of actions needs to be optimized to maximize the expected cumulative rewards along the stochastic trajectories. 
Following the description here, the value $\pi_{2,2,\ast}\levidx{2}(s,a)$, for each $s=(\agent,s_\pick,s_\open)$, $a=(\pi_{2,2}\levidx{1})^{\lambda}_{\theta}\in \{(\pi_{2,2}\levidx{1})^\alpha_\theta:\theta\in (\Theta_{2,2}\levidx{1})_\alpha\}\cup\{(\pi_{2,2}\levidx{1})^\beta_\theta:\theta\in(\Theta_{2,2}\levidx{1})_\beta\}$, equals \eqref{e:second-level-policy-eqn-optimization-mazebase},
where recall that for any two (random) times $0\le T<T'<\infty$ (a.s.), $(R_{2,2}\levidx{2})_{T,T'}$ is defined as in \eqref{e:cumulative-rewards} for $\MDP_{2,2}\levidx{2}$. 
Solving this optimization problem (see \eqref{e:second-level-policy-eqn-optimization-second-mazebase}), the student derives the optimal policy $\pi_{2,2,\ast}\levidx{2}$ (see \eqref{e:second-level-policy-eqn_mazebase}),
representing the concatenation logic that the agent will go to the key, pick it up, go the door, and then open the door. 
The agent focuses on learning this higher-order function, which is the aim of $\MDP_{2,2}$, because the details of going from A to B are encapsulated in the navigation skill $\piex^{\mathrm{nav}}$. 

$\MDP_{2,2}\levidx{2}$ has at least two key advantages compared to the level-1 MDP: first, it has much shorter time horizon, as a single action moves the agent by multiple steps, until achieving a small subgoal such as going to $\key$, or arriving next to $\door$, or picking up $\key$ or opening $\door$; second, the stochasticity is greatly reduced as it is absorbed into each higher-level navigation policy, further simplifying the optimization to obtain the optimal policy.
Note that it is crucial here that the navigation policies used at this level terminate, by choosing $\endactionfactor$ at very precise times and locations, instead of relying on random stopping times.
 
Finally, the student solves the original $\MDP_{2,2}$, using the ``convolution'' in \eqref{e: convolution} to pass the optimal policy of $\MDP\levidx{2}_{2,2}$ down to level one and refine it, resulting in the policy $\pi_{2,2}$, detailed in \eqref{e:first-level-policy-eqn-mazebase}.
In this particular example, $\pi_{2,2}$ is in fact the optimal policy $\pi_{2,2,\ast}$ of the original $\MDP_{2,2}$, requiring no additional refinement by value iteration.

\section{Transfer learning with skills and embeddings}
\label{s:transfer learning}

\subsection{Basic definitions}

\label{s:transfer learning-definitions}

The objective of transfer learning is to reuse knowledge acquired in the solution of one learning problem to another new problem. 
In our context, knowledge is represented by policies, and we wish to transfer policies learned from one MDP to another new MDP, in such a way that the solution of the latter is significantly sped up.
Our MMDP and curriculum-based framework is particularly well-suited for transfer, as it provides the possibility of transferring not only the whole initial MDP, but any of the MDPs in the curriculum, or even just some of the levels of some of the MDP in a curriculum.

The key concept that will enable learning is that of a skill-embedding decomposition, which factors (in the sense of composition of functions) a policy into an embedding, which ``abstracts'' certain aspects of the state-action space on which the policy is defined, and a skill, which acts on the ``abstract'' output of the embedding, and becomes highly transferable. Indeed, several (new) MDPs may have embeddings that create ``abstracted'' state-action spaces to which the same skill may be applied. In this work, embeddings will be provided by the teacher, but it is of course of great interest to, at least partially, learn these skill-embedding decompositions from a family of MDPs.

\begin{definition} \label{def:pp-decompose}
$(\overline{\pi},e)$ is a {\bf skill-embedding decomposition} of a partial policy $\pi_I:\cSA_I\rightarrow [0,1]$ on $\cD$ if $\pi_I(s,a)=\overline{\pi}(e(s,a))$ for any $(s,a)\in \cD$, where $\overline{\pi}:\cE\rightarrow[0,1]$ is the {\bf skill}, $e:\cD\rightarrow\cE$ is the {\bf embedding}, with $\cD\subseteq \cSA_I$. When applicable, the timescale of the skill $\overline{\pi}$ is the same as the timescale of the original partial policy $\pi_I$.
\end{definition}

This definition aims at reducing the semantic and sample complexity of $\pi_I$, by introducing an embedding function $e$ that extracts ``features'' of the state-action space that are ``sufficient'' for $\pi_I$, up to possibly restricting the partial policy domain $\cSA_I$ to a subset $\cD$.\footnote{While using the subset $\cD$ here is not strictly necessary, it does simplify the demonstration of skills, embeddings, and embedding generators in both examples. This comment applies also to other definitions in this section.}
The skill $\overline\pi$ is a higher-order function,\footnote{a higher-order function is a function who has at least one input that is a function.} taking $e$ as an input; different $e$'s in different problems, defined on different state-action (sub)sets, may be provided as inputs to the same skill $\overline\pi$, yielding transferability: a skill represents a transferable abstraction of a policy. 
See Fig.~\ref{fig:commdiagskillembeddingdecomp} for the diagram representing the skill-embedding decomposition.

Now, we reverse the definition above to composition: 
\begin{definition} \label{def:pp-compose}
Given a partial policy domain $\cSA_I$, $\cD\subseteq \cSA_I$, an embedding $e:\cD\rightarrow\cE$, and a skill $\overline{\pi}:\cE\rightarrow[0,1]$, 
we define their {\bf composite partial policy} $\pi_I: \cSA_I\rightarrow [0,1]$ by function composition and normalization starting from $e$, $\overline{\pi}$:
\begin{equation*}
\begin{aligned}
\mathrm{composition:}\qquad
\widehat{\pi}_I(s,a)&:=
\begin{cases}   
\overline{\pi}(e(s,a))\, &, (s,a)\in \cD \\
0\, &,\text{otherwise,}
\end{cases}
\\
\mathrm{normalization:}\qquad
\pi_I(s,a)&:=
\begin{cases}   
\frac{\widehat{\pi}_I(s,a)}{\sum_{a\in \cA_I(s)}\widehat{\pi}_I(s,a)}\, &, \sum_{a\in \cA_I(s)}\widehat{\pi}_I(s,a)\neq 0 \\
\indic_{\{\endaction\}}(a)\, &,\text{otherwise,}
\end{cases}
\end{aligned}
\end{equation*} 
and the timescale of the composite partial policy $\pi_I$ is the same as the timescale of $\overline{\pi}$.
\end{definition}
The intuition behind the normalization in Def.~\ref{def:pp-compose} is that the action $\endaction$ provides the option to terminate when the agent reaches states completely new to it. With a slight abuse of notations, we will use $\overline{\pi}\circ e$ to represent the composite partial policy coming from the skill $\overline{\pi}$ and the embedding $e$, which by definition is not exactly the output of function composition, because there is a second extra step of normalization.
The diagram representing the composite partial policy is the same as for the skill-embedding decomposition in Fig.~\ref{fig:commdiagskillembeddingdecomp}, except that here $e$ and $\overline\pi$ are given, and $\pi_I$ is constructed; in the skill-embedding decomposition, $\pi_I$ is given, and $e$ and $\overline\pi$ are constructed so that the diagrams commute.
The generalization of Def.~\ref{def:pp-compose} to the composite partial policy generator is straightforward: 

\begin{definition} \label{def:ppg-compose}
Given a partial policy domain $\cSA_I$, an embedding generator $E:\Theta\rightarrow \{e:\cD_{e}\subseteq\cSA_I\rightarrow e(\cD_{e})\}$, and a skill $\overline{g}:\cE\rightarrow \overline{g}(\cE)$ satisfying $\cup_{e\in E(\Theta)}e(\cD_{e})\subseteq \cE$, we define their {\bf composite partial policy generator} $g_I: \Theta\rightarrow \{\pi_I:\sum_{a\in \cA_I(s)}\pi_I(s,a)=1 \,\mathrm{for}\, \mathrm{any}\, s\in \cS\}$ by function composition and normalization starting from $E,\overline{g}$: 
\begin{equation*}
\begin{aligned}
    \mathrm{composition}: \qquad
\widehat{g}_I(\theta)(s,a)&:= 
\begin{cases}
\overline{g}(E(\theta)(s,a))\, &, (s,a)\in \cD_{E(\theta)}  \\
0\, &,\text{otherwise,}
 \end{cases}\\
\mathrm{normalization}:\qquad
g_I(\theta)(s,a)&:= 
\begin{cases}
\frac{\widehat{g}_I(\theta)(s,a)}{\sum_{a\in \cA_I(s)}\widehat{g}_I(\theta)(s,a)}\, &, \sum_{a\in \cA_I(s)}\widehat{g}_I(\theta)(s,a)\neq 0 \\
 \indic_{\{\endaction\}}(a)\, &,\text{otherwise.}
 \end{cases}
\end{aligned}
\end{equation*} 
\end{definition}

Skills, embeddings, and embedding generators allow the simplification and abstraction of the functions appearing in \eqref{e:approx-org-policy} through operators such as function composition, allowing transfer learning of skills across different levels of different MDPs, with possibly different difficulties. 
The degree of abstraction of a skill increases with the level where it appears in the MMDP: the higher the level, the more abstract it is; if it appears at the first level, then it represents a {\bf basic skill}, otherwise, it is some {\bf higher-order function}. 
This incorporates important features of functional programming, including pure functions, function composition, and higher-order functions: when we finish a difficult programming task, we would decompose the original task into many subtasks, and focus on only one at each time (level). 
Then, if similar programming patterns arise across different subtasks, we abstract them into a function (skill), whose input parameters are absorbed into the embedding generators in our framework.

The transfer learning between different MDPs and the bridging between different compressed levels within the same MDP are done in the unit of policies, with extra ingredients such as embeddings, embedding generators, skills, or unpacking compressed policies. This is desired: the knowledge we acquire from each MDP consists of policies, rather than, and somewhat independently of, environments or problem settings. In this way, we avoid ``rote learning''.

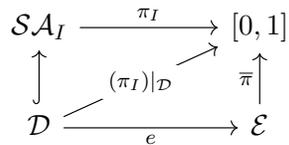
\begin{figure}[t]
\centering
\begin{tikzcd}
	{\mathcal{S}\mathcal{A}_I} && {[0,1]} \\
	{\mathcal{D}} && {\mathcal{E}}
	\arrow["{\pi_I}", from=1-1, to=1-3]
	\arrow[hook, from=2-1, to=1-1]
	\arrow["{(\pi_I)|_\mathcal{D}}"{description}, from=2-1, to=1-3]
	\arrow["e"', from=2-1, to=2-3]
	\arrow["{\overline{\pi}}", from=2-3, to=1-3]
\end{tikzcd}
\caption{Commuting diagram for a skill-embedding decomposition of $\pi_I$.}
\label{fig:commdiagskillembeddingdecomp}
\end{figure}

\subsection{The MazeBase+ example, part II}

\label{box:skill-compose}

\subsubsection{Skills and embeddings}
\label{box:skill} 
One instance of semantic meanings of skills is ``walking from A to B'' with inputs the starting location A and the destination B, which is some basic ``navigation'' skill $\overline{\pi}^{\mathrm{nav}}$, as appearing in the definition of $(\pi_{2,2}\levidx{1})^\beta_\theta$ in \eqref{e:def-generator-mazebase}. 

Alternatively, a skill could mean ``repeating a policy for multiple times'', which is a higher-order function $\overline{\pi}^{\mathrm{concat}}$ with $(\overline{\pi}^{\mathrm{concat}},(e_\decomp)_{2,2}\levidx{2})$ a skill-embedding decomposition of the (partial) policy $\pi_{2,2,\ast}\levidx{2}$, where  
the embedding $(e_\decomp)_{2,2}\levidx{2}: \cSA_{2,2}\levidx{2}\rightarrow \{(e_{\agent=s_\key}, $ $e_{\|\agent-s_\door\|_1=1}, s_\pick,s_\open,a)\}\subseteq$ $\{0,1\}^4\times\{\pi_{\pick},\pi_{\open},\pi_{\key},\pi_{\door},\endaction\}$, defined as in \eqref{e:embedding-concat}, takes out the important conditioning information, including whether the agent is at the location of $\key$, whether the agent is in $\neardoor(s_\door)$ with $\key$ in hand, whether $\key$ has been picked up, whether $\door$ has been open, as well as the semantic meaning of the subgoal;  
and the concatenation skill $\overline{\pi}^{\mathrm{concat}}:\{0,1\}^4\times\{\pi_{\pick}, \pi_{\open},\pi_{\key},\pi_{\door},\endaction\}\rightarrow [0,1]$ is defined as in \eqref{e:skill-concat}.
In general, $\overline{\pi}^{\mathrm{concat}}$ is ``repeating twice'', or in other words concatenation of two similar policies, such as $\pi_{\key}$ followed by $\pi_{\pick}$, and $\pi_{\door}$ followed by $\pi_{\open}$ here. In particular, it is a higher-order function providing a general logic, independent of the exact location of the agent, key, or door and only depending on the partial information of their relative locations and the states of objects: if the door has been open, the agent will stop; otherwise, if the key has not been picked up, the agent will pick up the key if it is already at the location of the key, and go to the key if not; otherwise, the agent will open the door if it is next to the door, and go towards the door if not. This demonstrates the advantages of our framework: we do not need the whole grid world in order to learn to open all the four doors, but only restrict to a much smaller domain $\gridworld_{\room_1}$ and learn to open a single door ($\door_1$) in order to learn the general logic behind opening a door, before applying it to opening more doors as in the following.

\subsubsection{Composing partial policy generators}
\label{box:compose}
In the MazeBase+ example, we now explain how the (partial) policy generators $(g_{2,2}\levidx{1})_{\alpha}$ and $(g_{2,2}\levidx{1})_{\beta}$ in $\cG\levidx{1}_{2,2}$ can be constructed by composing skills with two embedding generators $(E_{2,2}\levidx{1})_{\alpha}$ and $(E_{2,2}\levidx{1})_{\beta}$.

$(g_{2,2}\levidx{1})_{\alpha}$ is the composition of the degenerate skill $\mathrm{id}$, which is the identity map on $[0,1]$, and the embedding generator $(E_{2,2}\levidx{1})_{\alpha}$ defined on $(\Theta_{2,2}\levidx{1})_{\alpha}= \{\pick,\open\}$, indexing the two actions $\pick$ and $\open$, and $(E_{2,2}\levidx{1})_{\alpha}: (\Theta_{2,2}\levidx{1})_{\alpha}\rightarrow \{(e_{2,2}\levidx{1})^\alpha_\theta:\theta\in(\Theta_{2,2}\levidx{1})_{\alpha}\}$, mapping $\theta$ to $(e_{2,2}\levidx{1})^\alpha_\theta$.
It generates two embeddings $\{(e_{2,2}\levidx{1})^\alpha_\theta:\theta\in(\Theta_{2,2}\levidx{1})_{\alpha}\}$, with $(e_{2,2}\levidx{1})^\alpha_\theta:\cSA_{2,2}\levidx{1}\rightarrow \{0,1\} \,\,(\theta=\pick,\open)$ being defined as in \eqref{e:ealphatheta}, representing the embedding of selecting the action $\theta$, which is the $(\pi_{2,2}\levidx{1})^\alpha_\theta$ we introduced in \eqref{e:pialphatheta}.

$(g_{2,2}\levidx{1})_{\beta}$ is the composition of $\overline{\pi}^{\mathrm{nav}}$, and the embedding generator $(E_{2,2}\levidx{1})_{\beta}$ defined on $(\Theta_{2,2}\levidx{1})_{\beta}=\{\key,\door\}$, indexing the two objects $\key$ and $\door$, and
$(E_{2,2}\levidx{1})_{\beta}:(\Theta_{2,2}\levidx{1})_{\beta}\rightarrow \{(e_{2,2}\levidx{1})^\beta_\theta:\theta\in(\Theta_{2,2}\levidx{1})_{\beta}\}$, mapping $\theta$ to $(e_{2,2}\levidx{1})^\beta_\theta$. It generates two  embeddings $\{(e_{2,2}\levidx{1})^\beta_\theta:\theta\in(\Theta_{2,2}\levidx{1})_{\beta}\}$, with
$(e_{2,2}\levidx{1})^\beta_{\theta}:\{((\agent,s_\pick,s_\open),a)\in\cS_{2,2}\times(\cA_{\dir}\cup\{\endactionfactor\}):\agent+a\neq s_\door\}\subseteq \cSA_{2,2}\levidx{1}\rightarrow \{(\cur,\goal,a):\cur,\goal\in\noBlock,a\in\cA_{\dir}\cup\{\endactionfactor\}\} \,\,(\theta\in(\Theta_{2,2}\levidx{1})_{\beta})$ being defined as in \eqref{e:ebetatheta}, representing the embedding of taking out from the state-action pair the information of the current location of the agent, the location of the current subgoal, and the action considered. 
In particular, we ruled out the final step of reaching $s_\door$ from the domain of the embedding $(e^1_{2,2})^\beta_\door$ generated from $(E^1_{2,2})_{\beta}$, and replaced it by $\endactionfactor$, so that the policy generator $(g^1_{2,2})_{\beta}$, composed from $\piex^{\mathrm{nav}}$ and $(E^1_{2,2})_{\beta}$, generated the policy $(\pi_{2,2}\levidx{1})^\beta_{\door}$ stopping at $\neardoor(s_\door)$ instead of stopping at $s_\door$, because the door may be closed initially.

\section{Learning algorithms}
\subsection{Curricula}
\label{s:curriculum-definition}
A {\bf curriculum} is an ordered set of MDPs $\{\{\MDP_{L,n}\}_{n=1}^{n_L}\}_{L=1}^{L_{\max}}$ with $\MDP_{L,n}$ of difficulty $L$ as in \eqref{e:curriculum}:
\begin{equation}  \label{e:curriculum}
\begin{tikzcd}
	{\MDP_{L_{\max },1}} & \dots & {\MDP_{L_{\max },n_{L_{\max }}}} \\
	\vdots & \vdots & \vdots \\
	{\MDP_{2,1}} & {\MDP_{2,2}} & \dots & {\MDP_{2,n_2}} \\
	{\MDP_{1,1}} & {\MDP_{1,2}} & \dots & \dots & \dots & {\MDP_{1,n_1}}
	\arrow[from=1-1, to=1-2]
	\arrow[from=1-2, to=1-3]
	\arrow[dashed, from=2-3, to=1-1]
	\arrow[from=3-1, to=3-2]
	\arrow[from=3-2, to=3-3]
	\arrow[from=3-3, to=3-4]
	\arrow[dashed, from=3-4, to=2-1]
	\arrow[from=4-1, to=4-2]
	\arrow[from=4-2, to=4-3]
	\arrow[dashed, from=4-3, to=4-4]
	\arrow[dashed, from=4-4, to=4-5]
	\arrow[from=4-5, to=4-6]
	\arrow[from=4-6, to=3-1]
\end{tikzcd}
\end{equation}
We have strict lexicographic order ``$<$'' on the MDPs in a curriculum: for any $L,L'\in [L_{\max}], n\in [n_{L}], n'\in [n_{L'}]$, $(L,n)<(L',n')$ if and only if one of the following occurs: (1) $L<L'$; (2) $L=L',n<n'$. This is the order following which the student solves the MDPs.

\subsection{Teacher-student-assistant three-way cooperation}
\label{s:transfer-learning-algorithm}
To enable transfer learning and truly multi-level learning, we introduce 3 roles: a first actor, called {\textit{teacher}}, providing a curriculum of MDPs, each of which is an MMDP, and associated additional information, for example about transfer learning opportunities between these MDPs via common skills/embeddings/embedding generators; an agent, called {\textit{student}}, constructing and solving MDPs provided by the teacher following the given order; and a second actor, called {\textit{assistant}}, who extracts and records useful information from previously-solved MDPs using \textit{skill-embedding decompositions}.
These three roles use the curriculum to perform learning as we now describe, and as implemented in Algs.~\ref{algorithm:learn-MDPs}--\ref{algorithm:learn-MDP} in App.\ref{s:appendix:transferlearning}.

Having difficulty $L$, each $\MDP_{L,n}$ is an MMDP with $L$ levels; the teacher provides information for its solution to the student, as for an MMDP above, with the following caveats:
\begin{enumerate}[leftmargin=0.5cm,itemsep=0pt]
\item[$\bullet$] the teacher provides the  {\policysequence} $\{\cG_{L,n}\levidx{l}\}_{l=1}^{L-1}$, instead of $\{\Pi_{L,n}\levidx{l}\}_{l=1}^{L-1}$; each generator in $\cG_{L,n}\levidx{l}$ is provided directly or as a skill-embedding generator pair (Def.~\ref{def:ppg-compose});
\item[$\bullet$]  the initial policy for $\MDP_{L,n}\levidx{L}$ (the most compressed MDP) may be provided either directly or as a skill-embedding pair (Def.~\ref{def:pp-compose}).
\end{enumerate}
The embeddings or embedding generators are provided by the teacher, yielding an opportunity for transfer; skills are either directly provided by the teacher or hinted at by the teacher by providing the level and a previous MDP at that level from whose optimal policies the skill is extracted by the assistant, with the latter case yielding a direct transfer of the skill. 

The student solves $\MDP_{L,n} (1\leq L\leq L_{\max},1\leq n\leq n_L)$ using methods similar to Alg.~\ref{algorithm:MMDP-solver}, where the student first constructs the MDPs at higher levels in a bottom-up fashion, using the hints on the actions provided by the teacher: this is the first transfer opportunity, wherein those actions are essentially coming from the composition between skills, either directly provided by the teacher or extracted by the assistant, and embedding generators provided by the teacher. 
See lines~1--11 of Alg.~\ref{algorithm:learn-MDP}. 
The student then uses Props.~\ref{p:compress-transition}--\ref{p:compress-discount} for compression, described in Sec.~\ref{s:MMDP-compress}, to compute the transition probabilities, rewards, and discount factors respectively for those MDPs; see line~4 of Alg.~\ref{algorithm:MMDP-solver} using Alg.~\ref{algorithm:generate}.

Next, the student tries to solve these MDPs in a top-bottom fashion: the initial policy of the most compressed MDP (lines 12--16 of Alg.~\ref{algorithm:learn-MDP}), is either set as the degenerate \textit{diffusive policy}, with uniform distribution over actions at each state, or as some more informed policy, by composing a skill, either directly provided by the teacher or extracted by the assistant, with an embedding provided by the teacher. This is the second transfer opportunity, after which the student could solve those finer MDPs more easily thanks to the good initializations derived from the optimal policies at higher-levels, similar to warm starts in the domain of optimization. This is achieved in line 17 of Alg.~\ref{algorithm:learn-MDP}. 
During this process, we usually set the error tolerances at higher levels larger than those at lower levels, so that the higher-levels are providing more rough/general guidance before lower-levels fine-tune those compressed policies learned at higher-levels.  

Next, the assistant comes in: for each level $l\in [L]$, the student has learned an optimal policy $\pi_{L,n,\ast}\levidx{l}$ by solving $\MDP_{L,n}\levidx{l}$; some of these policies may be useful for later MDPs, as discussed above regarding the two transfer opportunities. 
More specifically, if the teacher has provided an embedding for that policy $\pi_{L,n,\ast}\levidx{l}$, then the assistant will implement a \textit{skill-embedding decomposition} according to Def.~\ref{def:pp-decompose}, to extract out a skill ${skill}\levidx{l}_{L,n}$ which is added to a set of public skills, called $\Skill$, that the assistant has acquired and the student can choose to use at any time. 
This is achieved as in lines 18--22 of Alg.~\ref{algorithm:learn-MDP}. 
The teacher does not know these skills, but sees the set $\Skill$ and can refer to its elements, for example to provide the students with hints about using a certain skill in an MDP.
The algorithms implementing the above are collected in App.~\ref{s:appendix:transferlearning}.

Our general framework can be combined with any existing MDP solvers, and without losing generality, we use value iteration in the above when we know the MDP, otherwise we use Q-learning to explore the environment: the adaptation of Q-learning to MMDPs is nontrivial and requires an extensive discussion, which we postpone to the follow-up paper.

\subsection{The MazeBase+ example, part III}
\label{s:curriculum}
In the MazeBase+ example, the teacher provides a curriculum containing four MDPs in the following order: $\MDP_{1,1}$, $\MDP_{2,1}$, $\MDP_{2,2}$, and $\MDP_{3,1}$, of difficulties, $1,2,2$, and $3$, respectively. See Sec.~\ref{box:MDPs} for the description of these MDPs and see App.\ref{s:MazeBaseMDPs} for the formal definition of these MDPs.

\subsubsection{Analytical run of the algorithm}
\label{s:curriculum}
The following goes through how Algs.~\ref{algorithm:learn-MDPs}--\ref{algorithm:learn-MDP} solve these MDPs analytically by constructing MMDPs, where $t_{\min} = t_{\max} = +\infty$. See App.~\ref{s:all-MazeBase} for the equations needed here as well as the mathematical notations mentioned here. 
\begin{enumerate}[leftmargin=0.5cm,itemsep=0pt]
\item[$\bullet$] $\MDP_{1,1}$ is $\MDP_{\text{dense}}^{\text{nav}}$ in the example of navigation and transportation with traffic jams (Sec.~\ref{s:examples-MMDP-transfer-learning}). 
We use the navigation skill $\overline{\pi}^{\mathrm{nav}}_{\text{dense}}$ extracted there (which we assume to be deterministic to simplify the calculations of compressed MDPs and better illustrate how the MMDP is constructed). This skill is then added to the set of public skills $\Skill$.

\item[$\bullet$] $\MDP_{2,1}:=(\cS_{2,1},\Sinit_{2,1},\Sterm_{2,1},\cA_{2,1},P_{2,1},R_{2,1},\Gamma_{2,1})$ models navigation through $\gridworld$ while avoiding blocks, assuming all the doors are open. See the definition in \eqref{e:MDP21}.
Since $\MDP_{2,1}$ is of difficulty $2$, the student constructs a two-level MMDP to solve it. 
For the first level, $\MDP_{2,1}\levidx{1}:=\MDP_{2,1}$.

For the second-level, the teacher provides $\cG_{2,1}\levidx{1}:=$ $\{(\overline{\pi}^{\mathrm{nav}}_{\text{dense}},$ $E_{2,1}\levidx{1})\}$,
where $E_{2,1}\levidx{1}$ is defined in \eqref{e:E211-mazebase}, generating all the embeddings extracting the current location $\agent$ and possible ``transfer hubs'' ($\door_1, \door_2, \door_3$) or the destination $\dest$.
For clarification, the teacher does not know the skill $\overline{\pi}^{\mathrm{nav}}_{\text{dense}}$, but sees the set $\Skill$ and can refer to its elements, for example here to provide the students with hints about using $\overline{\pi}^{\mathrm{nav}}_{\text{dense}}$ for constructing $\Pi_{2,1}\levidx{1}$. Similar considerations apply to all examples in this paper.

The student constructs the policy generators $g_{2,1}\levidx{1}$ (defined in \eqref{e:g21-mazebase}) from $(\overline{\pi}^{\mathrm{nav}}_{\text{dense}},E_{2,1}\levidx{1})$,
with timescale $t_{g_{2,1}\levidx{1}}=t_{\overline{\pi}^{\mathrm{nav}}_{\text{dense}}}=+\infty$, generating policies $(\pi_{2,1}\levidx{1})_\theta$ representing going from $\agent$ towards $\door_i$ or $\dest$, with timescales $t_{(\pi_{2,1}\levidx{1})_\theta}=t_{g_{2,1}\levidx{1}} =+\infty$. 
$(\pi_{2,1}\levidx{1})_\theta$ takes the action $\endactionfactor$ with probability 1 if $\agent = s_{\theta}$, telling the agent it could stop the current policy because it has reached the current subgoal and should move on to the next one; this shows why forcing a policy $\pi$ to satisfy $\sum_{a\in\endactionset}\pi(s,a)=1$ for any $s\in\Sterm$ is important in the construction of MMDPs and in policy transfer. 
With the help of the teacher, the student has concluded that $\Pi_{2,1}\levidx{1}=\{(\pi_{2,1}\levidx{1})_\theta:\theta\in \Theta_{2,1}\levidx{1}\}$,
where $\Theta_{2,1}\levidx{1}:=\{\door_1,\door_2,\door_3,\dest\}$ indexes the four objects $\door_1,\door_2,\door_3$, and $\door$, and $(\pi_{2,1}\levidx{1})_\theta$, $\theta\in \Theta_{2,1}\levidx{1}$, all have timescales $+\infty$.

The student then constructs the second level MDP $\MDP_{2,1}\levidx{2}=(\cS_{2,1},\Sinit_{2,1},\Sterm_{2,1},\overline{\Pi_{2,1}\levidx{1}}, P_{2,1}\levidx{2},$ $R_{2,1}\levidx{2},\Gamma_{2,1}\levidx{2})$, with $P_{2,1}\levidx{2}$, $R_{2,1}\levidx{2}$, $\Gamma_{2,1}\levidx{2}$ as in \eqref{e:compressed-MDP21-MazeBase}. If $\overline{\pi}^{\mathrm{nav}}_{\text{dense}}$ were not deterministic, the sequence defined here is stochastic, complicating explicit calculations.
The student solves $\MDP_{2,1}\levidx{2}$ and finds its optimal policy $\pi_{2,1,\ast}\levidx{2}$ with timescale $t_{\pi_{2,1,\ast}\levidx{2}}=+\infty$. 
The student typically needs to walk across several rooms when going from its initial location to the target location by avoiding the blocks around the route, but the navigation skill $\piex^{\mathrm{nav}}_{\text{dense}}$ only teaches the student to navigate within a single room while avoiding blocks and doors, so the student may need to stitch $\piex^{\mathrm{nav}}_{\text{dense}}$ multiple times, with the doors as intermediate locations: this is what $\pi_{2,1,\ast}\levidx{2}$ does. 
At this level the student can focus on learning it, easily, because the details of going from A to B within a single room while avoiding blocks and doors have been encapsulated in the navigation skill $\piex^{\mathrm{nav}}_{\text{dense}}$. 
The student derives the optimal policy $\pi_{2,1,\ast}=\pi_{2,1,\ast}\levidx{1}$ for $\MDP_{2,1}$ without any more iterations, since $\Block\cup\{\door_i:1\leq i\leq 3\}$ here is contained in $\Omega'_{\mathtt{jams}}$ in $\MDP_{\text{dense}}^{\text{nav}}$ (per our assumptions), and the light traffic will repel the student to not touch $\Block\cup\{\door_i:1\leq i\leq 3\}$ whenever the total distance the student needs to travel does not increase.

Next, if the teacher provides the embedding $(e_{\decomp})_{2,1}\levidx{1}$ as defined in \eqref{e:E211-decomp-mazebase}, which is essentially just an identity map, then the assistant extracts $\overline{\pi}^{\mathrm{nav}}$, a skill of navigation through $\gridworld$ while avoiding blocks assuming all the doors are open, as in \eqref{e:pi-nav-mazebase},
with timescale $t_{\overline{\pi}^{\mathrm{nav}}}=t_{\pi_{2,1,\ast}\levidx{1}}=+\infty$. Similar to $\overline{\pi}^{\mathrm{nav}}_{\text{dense}}$, we assume hereafter $\overline{\pi}^{\mathrm{nav}}$ is deterministic to simplify the calculations of compressed MDPs and better illustrate how the MMDP is constructed.

\item[$\bullet$] We move to $\MDP_{2,2}$, which represents picking up a key and opening the door, already discussed thoroughly in Secs.~\ref{box:MDPs} and~\ref{box:skill-compose}; from it, the assistant extracts a concatenation skill $\overline{\pi}^{\mathrm{concat}}$.

\item[$\bullet$] Now we consider the target MDP of difficulty $3$.  $\MDP_{3,1}:=(\cS_{3,1}$, $ \Sinit_{3,1}$, $\Sterm_{3,1}$, $\cA_{3,1}$, $P_{3,1}$, $R_{3,1}$, $\Gamma_{3,1})$ represents the task of picking up the goal, see the definition in \eqref{e:MDP31-mazebase}.
Since $\MDP_{3,1}$ is of difficulty $3$, the student constructs a three-level MMDP to solve it, and for the first level, $\MDP_{3,1}\levidx{1}:=\MDP_{3,1}$. To construct the second-level MDP, the teacher provides
\begin{equation*} 
\cG_{3,1}\levidx{1}:=\{(\mathrm{id},(E_{3,1}\levidx{1})_{\alpha}),(\overline{\pi}^{\mathrm{nav}},(E_{3,1}\levidx{1})_{\beta})\}\,,
\end{equation*}
where $(E_{3,1}\levidx{1})_{\alpha}$, defined in \eqref{e:E311alpha}, generates embeddings taking out action information, and $(E_{3,1}\levidx{1})_\beta$, defined in \eqref{e:E311beta} and similar to $E_{2,1}\levidx{1}$, generates all the embeddings extracting the current location as well as possible ``transfer hubs'' or the destination, together with the action information. 
The student then composes the policy generators $(g_{3,1}\levidx{1})_\alpha$, $(g_{3,1}\levidx{1})_\beta$ (defined in \eqref{e:g311alpha}, \eqref{e:g311beta}) coming from $(\mathrm{id}$, $(E_{3,1}\levidx{1})_\alpha)$, $(\overline{\pi}^{\mathrm{nav}},(E_{3,1}\levidx{1})_\beta)$ respectively. $(g_{3,1}\levidx{1})_\alpha$, defined in \eqref{e:g311alpha}, with timescale $t_{(g_{3,1}\levidx{1})_\alpha}=t_{\mathrm{id}}=1$ (given by the teacher), generates actions $\pick, \open$ as represented by $(\pi_{3,1}\levidx{1})^\alpha_\theta$,  $\theta\in (\Theta_{3,1}\levidx{1})_\alpha$,
with timescales $t_{(\pi_{3,1}\levidx{1})^\alpha_\theta}=t_{(g_{3,1}\levidx{1})_{\alpha}}=1$. Here, $ (\Theta_{3,1}\levidx{1})_\alpha := \{\pick,\open\}$ indexes the two actions $\pick$ and $\open$.
$(g_{3,1}\levidx{1})_\beta$, defined in \eqref{e:g311beta}, with timescale $t_{(g_{3,1}\levidx{1})_\beta}=t_{\overline{\pi}^{\mathrm{nav}}}=+\infty$, generates policies of going from the current location towards possible ``transfer hubs'' or the destination, as represented by $(\pi_{3,1}\levidx{1})^\beta_{\theta}$, $\theta\in (\Theta_{3,1}\levidx{1})_\beta$, with timescales $t_{(\pi_{3,1}\levidx{1})^\beta_{\theta}}=t_{\overline{\pi}^{\mathrm{nav}}}=+\infty$. Here, $(\Theta_{3,1}\levidx{1})_\beta :=\{\key_1,\key_2,\key_3,\door_1,\door_2,\door_3,\goals\}$ indexes the seven objects $\key_1,\key_2,\key_3,\door_1,\door_2,\door_3$, and $\goals$. As a reminder, here we use $\overline{\pi}^{\mathrm{nav}}\circ(e_{3,1}^1)^\beta_\theta$ to represent the composite partial policy coming from the skill $\overline{\pi}^{\mathrm{nav}}$ and the embedding $(e_{3,1}^1)^\beta_\theta$, which is a slight abuse of notations, because according to Def.~\ref{def:pp-compose}, it is not exactly the output of function composition: there is a second extra step of normalization. Similar notations apply hereafter.

To summarize, with the help of the teacher, the student concludes that 
\begin{align*} 
\Pi_{3,1}\levidx{1}=\{(\pi_{3,1}\levidx{1})^\alpha_{\theta}:\theta\in(\Theta_{3,1}\levidx{1})_\alpha\}\cup\{(\pi_{3,1}\levidx{1})^\beta_{\theta}:\theta\in(\Theta_{3,1}\levidx{1})_\beta\}\,,  
\end{align*}
where $(\pi_{3,1}\levidx{1})^\alpha_{\theta}\,\, (\theta\in(\Theta_{3,1}\levidx{1})_\alpha)$ are two degenerate extended policies with timescales $1$, and $(\pi_{3,1}\levidx{1})^\beta_{\theta}\,\, (\theta\in(\Theta_{3,1}\levidx{1})_\beta)$ are seven extended policies with timescales $+\infty$. In addition, all the policies in $\Pi_{3,1}\levidx{1}$ could be represented by an element in the product set $((\Theta_{3,1}\levidx{1})_\alpha\cup\{\endactionfactor,\emptyvalue\})\times ((\Theta_{3,1}\levidx{1})_\beta\cup\{\endactionfactor,\emptyvalue\})$: for instance, $(\pi_{3,1}\levidx{1})^\alpha_{\pick}$ is represented by $(\pick,\emptyvalue)$. So, $\Pi_{3,1}\levidx{1}$ has two action factors, which for simplicity we index with $\alpha, \beta$ respectively.

Then, the student constructs the second level MDP $\MDP_{3,1}\levidx{2}=(\cS_{3,1}$, $\Sinit_{3,1}$, $\Sterm_{3,1}$, $\overline{\Pi_{3,1}\levidx{1}}$, $P_{3,1}\levidx{2}$, $R_{3,1}\levidx{2}$, $\Gamma_{3,1}\levidx{2})$, where $P_{3,1}\levidx{2}$, $R_{3,1}\levidx{2}$, and $\Gamma_{3,1}\levidx{2}$ are as in \eqref{e:compressed-MDP312-MazeBase}.
The teacher provides $\cG_{3,1}\levidx{2}:=\{(\overline{\pi}^{\mathrm{concat}},E_{3,1}\levidx{2})\}$,
where $E_{3,1}\levidx{2}$ is as defined in \eqref{e:E312}, generating embeddings taking out the progress information needed for opening $\door_1$, $\door_2$, $\door_3$ or picking up $\goals$ respectively.
The student then derives that the policy generator $(g_{3,1}\levidx{2})_{\{\alpha, \beta\}}$ from $(\overline{\pi}^{\mathrm{concat}}$, $E_{3,1}\levidx{2})$, where $(g_{3,1}\levidx{2})_{\{\alpha, \beta\}}$ is as defined in \eqref{e:g312}, with timescale $t_{(g_{3,1}\levidx{2})_{\{\alpha, \beta\}}}=t_{\overline{\pi}^{\mathrm{concat}}}=+\infty$, generating policies for opening $\door_i$, $i\in[3]$, and picking up $\goals$, represented by $((\pi_{3,1}\levidx{2})_\theta)_{\{\alpha, \beta\}}$, and here for all $\theta\in\Theta_{3,1}\levidx{2}$, the timescales $t_{((\pi_{3,1}\levidx{2})_\theta)_{\{\alpha, \beta\}}}=t_{\overline{\pi}^{\mathrm{concat}}}=+\infty$. Here, $\Theta_{3,1}\levidx{2}:=\{(\key_1,\door_1),(\key_2,\door_2),(\key_3,\door_3),(\goals,\goals)\}$ indexes the objects needed for opening $\door_1$, $\door_2$, $\door_3$ or picking up $\goals$ respectively.

To summarize, with the help of the teacher, the student concludes that 
$
\Pi_{3,1}\levidx{2}=$\\$\{((\pi_{3,1}\levidx{2})_{\theta})_{\{\alpha,\beta\}}:\theta\in\Theta_{3,1}\levidx{2}\}  
$,
wherein all the four policies have timescales $+\infty$. Here, we use the two skills, the navigation skill $\overline{\pi}^{\mathrm{nav}}$ and the concatenation logic $\overline{\pi}^{\mathrm{concat}}$ originally used to learn the process of opening a door, to generate by their combinatorial combinations many new policies: this is the fundamental role of the third level.

The student now constructs the third-level MDP $\MDP_{3,1}\levidx{3}=(\cS_{3,1},\Sinit_{3,1}$, $\Sterm_{3,1}$, $\overline{\Pi_{3,1}\levidx{2}}$, $P_{3,1}\levidx{3}$, $R_{3,1}\levidx{3}$, $\Gamma_{3,1}\levidx{3})$, where $P_{3,1}\levidx{3},R_{3,1}\levidx{3},\Gamma_{3,1}\levidx{3}$ are as in \eqref{e:compressed-MDP313-MazeBase}, and finds its optimal policy $\pi_{3,1,\ast}\levidx{3}$ (see \eqref{e:pi313}), with timescale $+\infty$. This policy is a concatenation of opening $\door_i$'s and picking up $\goals$, depending on the agent's current locations. For instance, the agent will first open $\door_2$, then open $\door_3$ and eventually pick up $\goals$ if it starts from $\room_2$. 
For simplicity, we do not guarantee the functions (such as policies, and more generally, partial policies) written down are correct at states never explored unless specified, and in particular the normalization property may not be satisfied at those states. This applies to both $\pi_{3,1,\ast}\levidx{3}$ here and all the following such functions. 
Consequently, the student derives $\pi_{3,1,\ast}\levidx{2}$ as in \eqref{e:pi312}, which ``refines'' the actions in $\pi_{3,1,\ast}\levidx{3}$ to finer details about the concatenation logic of navigation policies or simple actions of $\pick, \open$.
Thus, the student concludes $\pi_{3,1,\ast}$ as in \eqref{e:pi311}, which further ``expands'' the actions in $\pi_{3,1,\ast}\levidx{2}$ to finer details about the concatenation logic of simple actions in $\cA_{\dir}$ and $\pick, \open$. 
\end{enumerate}

Merits of this are threefold. First, just like for a human, focusing on one problem at each time is more efficient than thinking about several problems all at once. 
Second, it can speed up the solution of the MDPs by leveraging transfer of skills within the curriculum, see Fig.~\ref{f:MazeBase_ours}, which generally comes from extracting the ``repetitions''/repeated patterns (shared skills/components) and learning them for only once, or extracting the patterns that could be merged together into a single one and learning them in parallel. In addition, more computational savings could be achieved by scaling up our current examples, which is not our main focus, since our aim is to use the simplest set of examples to convey all the important conceptual ingredients in this work. 
Third, it promotes further opportunities of transfer to/from other curricula; for instance, the next example of navigation and transportation with traffic jams provides one of the navigation skills utilized in this example. 

\subsubsection{Numerical run of the algorithm}
\paragraph{Numerical results.} 
In Fig.~\ref{f:MazeBase_ours} we report on the computational cost of our algorithm for the MazeBase+ example. 
These extra iterations correspond to learning how to stitch different steps including navigating to different objects, $\pick$ successfully or $\open$ successfully, well separated in semantic meanings, which contains the turning points at which the student may need to switch to a different strategy. This is illustrated by the forthcoming Thm.~\ref{t:complexity}.
The cost of solving the $\MDP_{2,n}$'s is amortized over solving just a single $\MDP_{3,1}$, showcasing the power of multi-level structure and transfer of skills in our framework. This is illustrated by the forthcoming Thm.~\ref{t:transfer-2}. 
One small extra comment is that when solving $\MDP_{2,2}$, albeit under the optimal policy each episode should be of length no more than $4$ at level $2$, there are slightly more iterations in blue using value iteration. 
 The reason is that the student may need to open $\door$ for multiple times before succeeding (recall $p_s=0.9$), so there is a self-loop on updating the value $V(\neardoor(\door))$ of the state at $\neardoor(\door)$: at each iteration, from the second one onwards, $V(\neardoor(\door))$  is updated to $V(\neardoor(\door))+(1-p_s)V(\neardoor(\door))$, until convergence. From the figure, we see near convergence is attained after 3 iterations.
Finally, note that cost per iteration in our algorithm is no larger than that in classical value iteration; it is in fact often much smaller because the compressed MDPs have typically a smaller action sets and effective state spaces. 
Similar considerations will apply to numerical experiments later on.

\begin{figure}[t]
\centering
\includegraphics[width=0.5\textwidth]{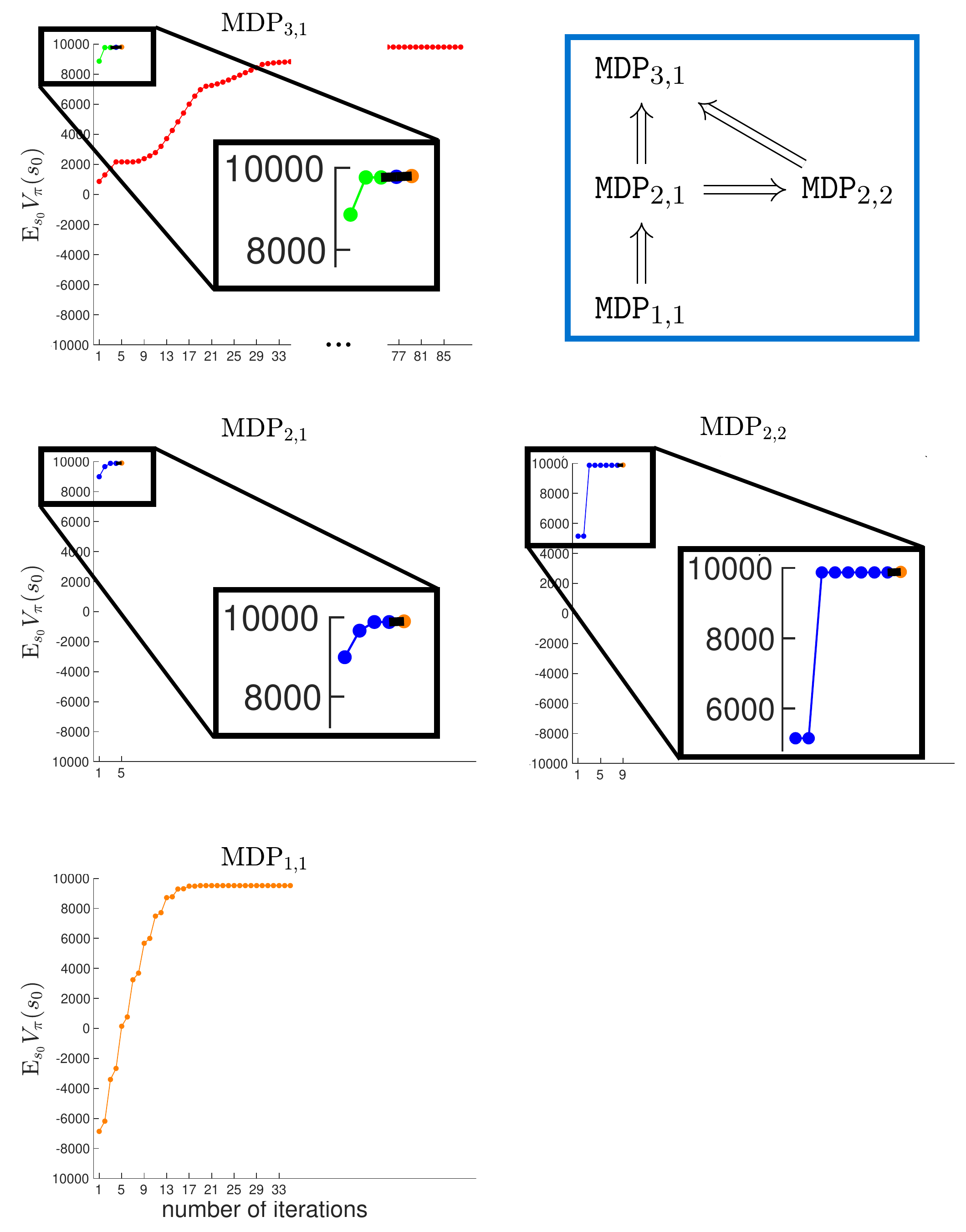}
\caption{We display $\mathbb{E}_{s_0}V_\pi(s_0):=\sum_{s\in\Sinit}V(s)/|\Sinit|$, the average (over all initial conditions $s_0$) of the $V_{\pi}(s_0)$, where $V_\pi$ is the value function for $\MDP_{1,1}$, $\MDP_{2,1}$, $\MDP_{2,2}$, and $\MDP_{3,1}$ respectively, as $\pi$ is optimized during iterations of classical value iteration (in red) and of value iteration within our algorithm, with iterations within $\MDP_{1,1}=\MDP_{\text{dense}}^{\text{nav}}$ in orange, iterations within $\MDP_{2,n} (1\leq n\leq 2)$ at level $2$ in blue followed by iterations at level $1$ in orange, and iterations within $\MDP_{3,1}$ at level $3$ in green followed by iterations at level $2$ in blue followed by iterations at level $1$ in orange. Although we spend extra effort in solving the $\MDP_{2,n}$'s (recall that $\MDP_{1,1}$'s comes from the example of navigation and transportation with traffic jams, so we do not spend extra effort there), they prepare us well enough so that we only need a few more iterations for solving the target $\MDP_{3,1}$ (green+blue+orange), much fewer than if we solved it from scratch using classical value iteration (red). 
}
\label{f:MazeBase_ours} 
\end{figure}

\paragraph{Why we do not report extensive wall-clock comparisons.}
We do not include broad running-time comparisons against many existing baselines because prior work naturally falls into two groups, and for the first group \emph{such comparisons are not meaningful, obviously dominated by incompatible overhead, or unfair/too expensive to compare.} Many HRL methods cap abstraction at one--two levels or tie abstraction to environment-specific goal heuristics, making them difficult to apply directly to our multi-level MazeBase+ and cross-geometry tasks without substantial redesign \citep{SuttonOptions, DietterichHRL, ParrRussell1997, DayanHinton1993, BartoHRLReview, Bacon2017OptionCritic, Vezhnevets2017FeUdal, Nachum2018HIRO, Levy2019HAC, Dwiel2019GoalSpace}. Likewise, much of the skill-discovery literature targets general-purpose primitives but does not construct higher-level compressed problems with reduced stochasticity or explicit cross-geometry transfer mechanisms \citep{Eysenbach2019DIAYN, Gregor2016VIC}. Even when a code path exists, these systems typically (i) do not compress state/action spaces into independent higher-level MDPs, (ii) intermix high-level decisions with low-level stochastic rollouts at every step, and (iii) lack mechanisms for transfer across different geometries---so the wall-clock cost scales with fine-level simulation rather than with the abstract problem. In such cases, any ``runtime'' head-to-head would mainly benchmark infrastructure and rollout budgets rather than the intended algorithmic question (safe compression and cross-level/cross-task reuse), and would not be an apples-to-apples comparison. We therefore emphasize qualitative comparisons 
that isolate how multi-level compression changes the effective planning and learning problem.

\paragraph{A comparator we hope to include, and why it is costly to extend.}
A closer conceptual neighbor is \citet{Sukhbaatar2018LearningGE}, where self-play learns goal embeddings to guide hierarchical behavior. Our MazeBase+ example realizes a similar high-level logic but generalizes it to many levels with \emph{independent, compressed MDPs} at higher scales. Extending \citet{Sukhbaatar2018LearningGE} na\"ively beyond two levels would (a) forfeit independence across levels---higher-level learning would still require traversing fine levels, so each abstract update pays for low-level exploration---and (b) remain unsupervised, whereas our construction is target-task--supervised. Moreover, their embeddings condition primarily on the goal (keeping the high-dimensional current state essentially fixed), whereas our embeddings take \emph{current state, action coordinates, and goal}, enabling stronger dimension reduction and better policy-structure reuse. Combined with partial policy/embedding generators and reusable skills, our compression lets higher levels avoid re-learning repeated navigation-like subtasks and yields deterministic planners at the top; even with Q-learning, we only estimate induced transitions/rewards once and then reuse them across tasks and levels, as demonstrated in the follow-up paper. For these reasons, a faithful multi-level extension (or optimized reimplementation) of \citet{Sukhbaatar2018LearningGE} would likely be substantially more expensive on our benchmarks today; we therefore defer direct wall-clock comparisons to future work, while providing detailed qualitative/ablation evidence and an extended methodological comparison in Appendix~\ref{s: MazeBase-literature}. We will also include connections to reinforcement-programming views of sorting \citep{5586457} in the follow-up paper, when we realize merge sort using our framework as a novel example of recursion.

\subsubsection{Transfer learning to a new problem}
\label{s:MazeBase-transfer}

Now we show how the decomposition of $\MDP_{3,1}$ into multiple levels and extraction of useful skills (both the basic skill $\overline\pi^{\mathrm{nav}}_{\text{dense}}$ about navigation within a single room while avoiding blocks and the higher-order function $\overline{\pi}^{\mathrm{concat}}$ about the concatenation logic of picking up a key and opening the door) help in learning new difficult problems much faster. 
We introduce a new problem $\MDP'_{3,1}$, similar to $\MDP_{3,1}$, with the main difference lying in the geometric configuration of the objects:
first of all, $\door_2$ is now connecting between $\room_2$ and $\room_4$ rather than $\room_1$ and $\room_3$ (i.e., the constraints for $s_{\door_2}$ change to $s_{\door_2}(1) > s_{\door_1}(1), s_{\door_1}(2)<s_{\door_2}(2)<s_{\door_3}(2)$), and secondly, initially, $s_{\key_2}\in\gridworld_{\room_2}, s_{\key_3}\in\gridworld_{\room_4}, \goal\in \gridworld_{\room_3}, (\cS^{\text{init}}_{3,1})'=\{(\agent,\keypickvec,\dooropenvec,\goalpick)\in \cS'_{3,1}: \goalpick=0\}$. 
The other objects and notations are defined as above, taking into consideration the updated configuration just described. 
Fig.~\ref{f:MazeBase2} depicts the new geometric configurations of rooms, doors, keys, and the goal.

In $\MDP'_{3,1}$, $\door_1$ is indeed useful in order to reach $\goals$, and even though $\goals$, $\door_2$, $\key_2$, and $\key_3$ now have significantly different positions, and all the other objects could have other fixed locations within the given constraints, in order to solve $\MDP'_{3,1}$ we only need a few iterations of value iteration inside our algorithm, with no need to re-solve $\MDP^{\text{nav}}_{\text{dense}}$ and $\MDP_{2,2}$. The reason is that we do not change the overall geometry, so the skill $\overline\pi^{\mathrm{nav}}_{\text{dense}}$ about navigation within a single room while avoiding blocks still applies; the logic of how to open a door is ubiquitous and universally applies everywhere; the only parts that need to be relearned include: ($\romannumeral1$) navigation through $\gridworld$ while avoiding blocks assuming all the doors are open (since the doors have significantly changed positions); ($\romannumeral2$) the policy of how to navigate between rooms when some doors may be closed at level $3$ (since doors and keys have significantly changed positions). 
For ($\romannumeral1$), the teacher just needs to construct a new MDP $\MDP'_{2,1}$, similar to $\MDP_{2,1}$; the student solves it and the assistant extracts a new skill $(\overline{\pi}^{\mathrm{nav}})'$ about navigation through the whole grid world $\gridworld$ while avoiding blocks assuming all the doors are open, by stitching together several policies coming from $\overline\pi^{\mathrm{nav}}_{\text{dense}}$. 
Only a few iterations of value iteration will suffice to achieve this, see Fig.~\ref{f:MazeBase_ours23} (left).
For ($\romannumeral2$), after the student constructs $(\MDP_{3,1}\levidx{3})'$, the student needs to solve it as if it is a completely new problem, but of tiny size. 
The optimal policy $(\pi_{3,1,\ast}\levidx{3})'$ for $(\MDP_{3,1}\levidx{3})'$ has length $4$, longer than the previous $\pi_{3,1,\ast}\levidx{3}$ for $\MDP_{3,1}\levidx{3}$, and can be discovered in very few iterations of value iteration. $(\pi_{3,1,\ast}\levidx{3})'$ is given by \eqref{e:policy-transfer-mazebase}, which is different from $\pi_{3,1,\ast}\levidx{3}$ because of the new geometry. For instance, the agent will first open $\door_1$, then open $\door_2$, then open $\door_3$ and eventually pick up $\goals$ if it starts from $\room_1$.

So, even if $(\overline{\pi}^{\mathrm{nav}})'$ and $(\pi_{3,1,\ast}\levidx{3})'$ need to be relearned, we still achieve a significant reduction in the effort to learn and in computations thanks to transfer learning and compression, achieving ``few-shot learning'', because ($\romannumeral1$) $\overline\pi^{\mathrm{nav}}_{\text{dense}}$ is still usable when building $(\MDP_{2,1}\levidx{2})'$, level $2$ of $\MDP'_{2,1}$, and thus essentially the process of traveling between any two locations within a single room is encapsulated in a single action. 
Given that in ``grid world type'' examples, the number of iterations in value iteration is proportional to the average number of actions needed before reaching the goal state, working at level $2$ of $\MDP'_{2,1}$ where each step/action is at level of a whole room, only a few iterations (proportional to the number of rooms the student needs to walk through before reaching $\dests$) are needed for solving $\MDP'_{2,1}$. ($\romannumeral2$) 
Similarly, at level $3$ for $\MDP'_{3,1}$, all the actions within a single room are encapsulated in a single action, so we only need a few iterations (proportional to the number of rooms the student needs to walk through before picking up $\goals$) for solving $\MDP'_{3,1}$. 

Without transfer and without our framework, solving $\MDP'_{3,1}$ directly would be significantly more expensive, requiring a number of iterations proportional to the average diameter of the rooms the student needs to walk through before picking up $\goals$ multiplied by the number of rooms the student needs to walk through before picking up $\goals$. 

Last but not least, this variant $\MDP'_{3,1}$ also further justifies why we need three levels in this MazeBase+ example: if we stopped at level $2$, we would still need at least three times more iterations when learning new problems, because we would not have encapsulated the learned knowledge of $\overline{\pi}^{\mathrm{concat}}$, so we would need to let the student learn the process represented by $\overline{\pi}^{\mathrm{concat}}$ repeatedly, while at the same time finding the correct navigation between rooms. See Fig.~\ref{f:MazeBase_ours23} (left) for the comparison and more detailed analysis.

\begin{figure}[t] 
    \centering
    \begin{subfigure}[t]{0.44\linewidth} 
        \centering
        \includegraphics[width=\textwidth]{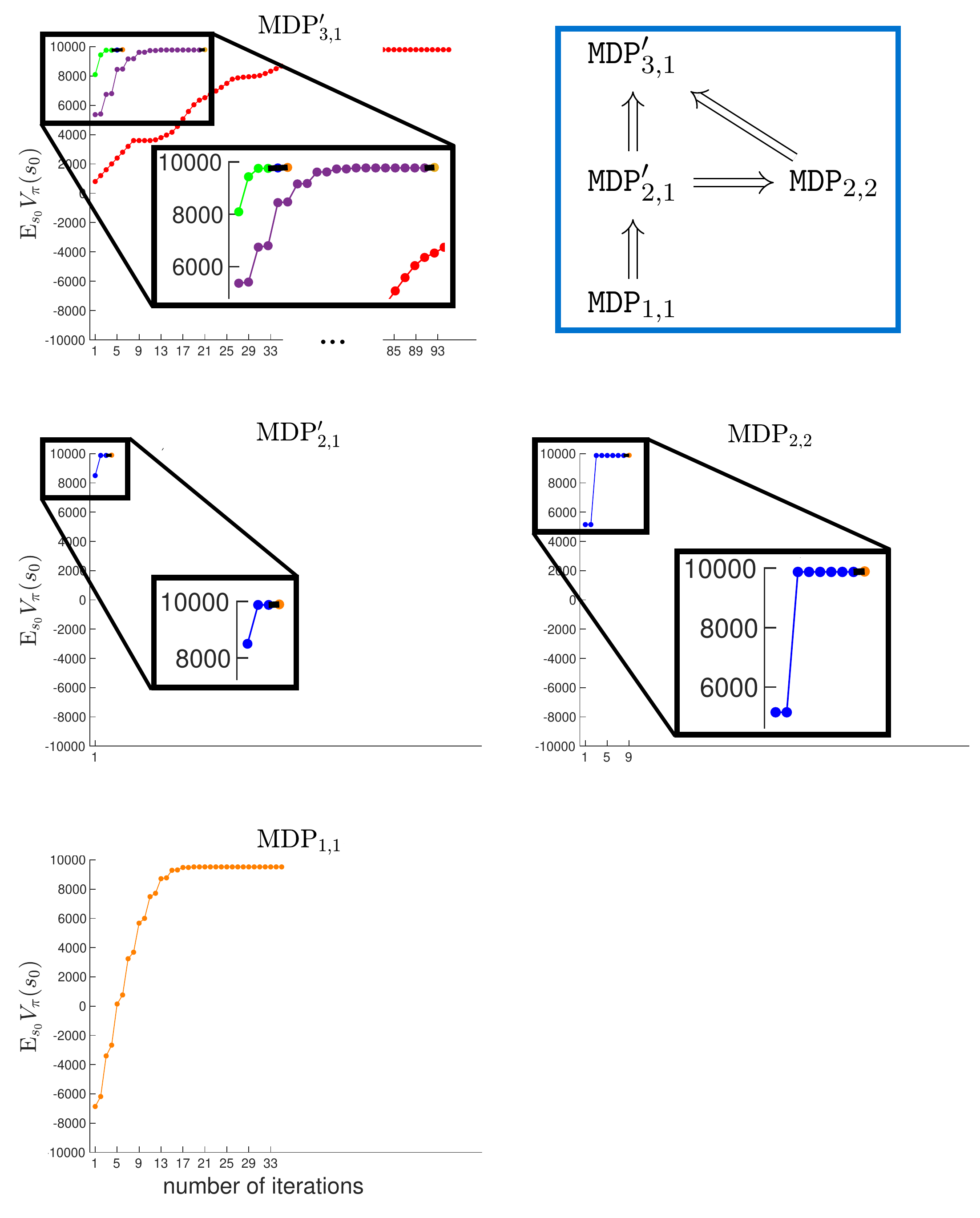} 
    \end{subfigure}
    \hfill 
    \begin{subfigure}[t]{0.49\linewidth}
        \centering
        \includegraphics[width=\textwidth]{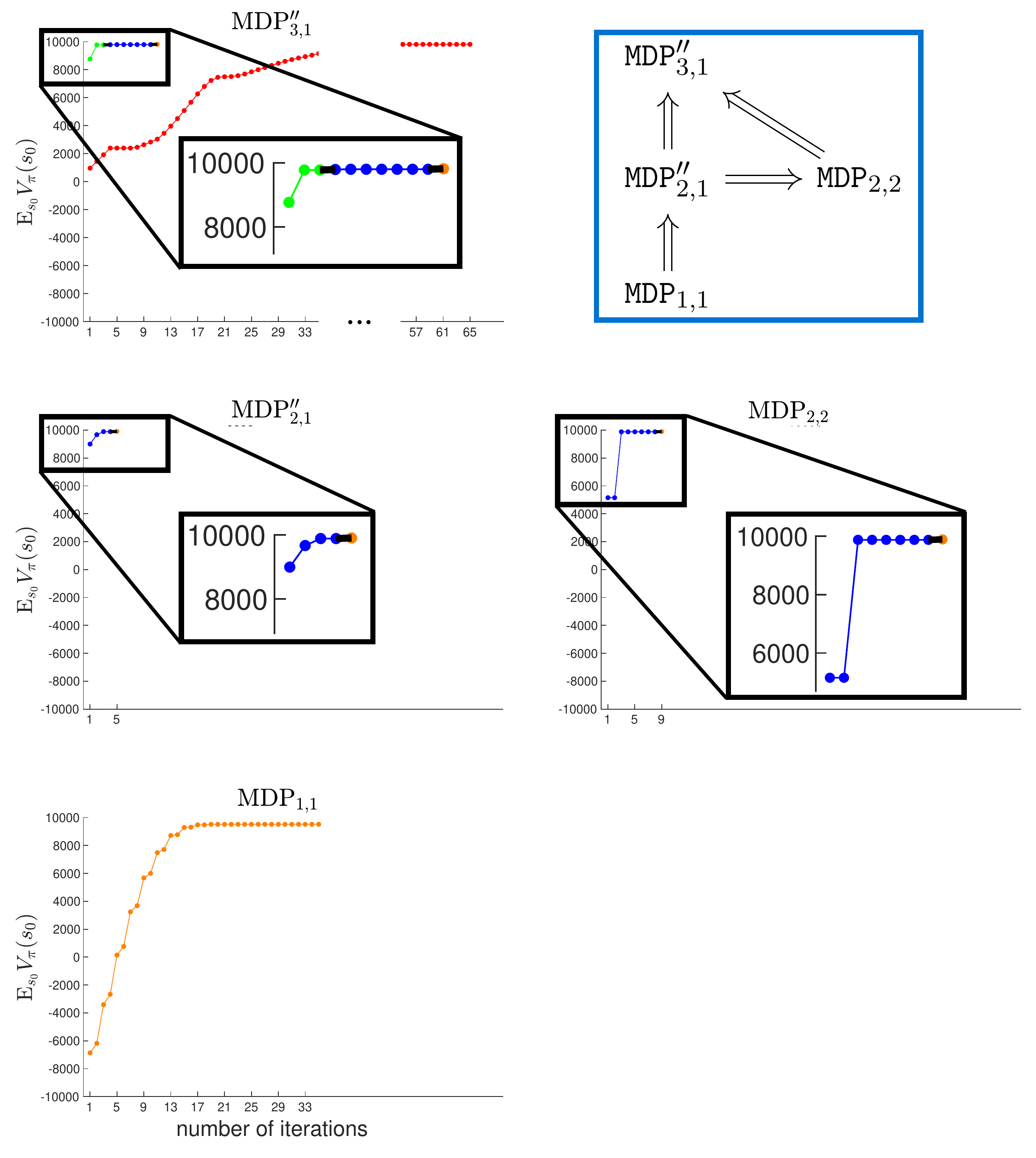}
    \end{subfigure}
    \caption{Left: we display $\mathbb{E}_{s_0}V_\pi(s_0)$, where $V_\pi$ is the value function for $\MDP_{1,1}$, $\MDP'_{2,1}$, $\MDP_{2,2}$, and $\MDP'_{3,1}$ respectively, as $\pi$ is optimized during iterations of classical value iteration and of value iteration within our algorithm. See Fig.~\ref{f:MazeBase_ours} for the detailed explanation, and Fig.~\ref{f:MazeBase2} for the representation of this second experiment of the MazeBase+ example. 
One addition here compared with Fig.~\ref{f:MazeBase_ours} is that here we also plot iterations within $\MDP'_{3,1}$ at level $2$ in purple followed by iterations at level $1$ in yellow, assuming the student did not solve $\MDP_{2,2}$ and consequently the assistant did not extract $\overline{\pi}^{\mathrm{concat}}$. 
This corresponds to treating $\MDP'_{3,1}$ as an MDP of difficulty $2$. 
The extra effort here (purple+yellow) compared with the original case (green+blue+orange) demonstrates the advantage of treating $\MDP'_{3,1}$ as an MDP of difficulty $3$ and of extracting the higher-order function $\overline{\pi}^{\mathrm{concat}}$. Right: we display $\mathbb{E}_{s_0}V_\pi(s_0)$, where $V_\pi$ is the value function for $\MDP_{1,1}$, $\MDP''_{2,1}$, $\MDP_{2,2}$, and $\MDP''_{3,1}$ respectively, as $\pi$ is optimized during iterations of classical value iteration and of value iteration within our algorithm. See Fig.~\ref{f:MazeBase_ours} for the detailed explanation, and see Fig.~\ref{f:MazeBase3} for the representation of this third experiment of the MazeBase+ example.}
    \label{f:MazeBase_ours23}
\end{figure}

\begin{figure}[t]
\centering
\includegraphics[width=1\textwidth]{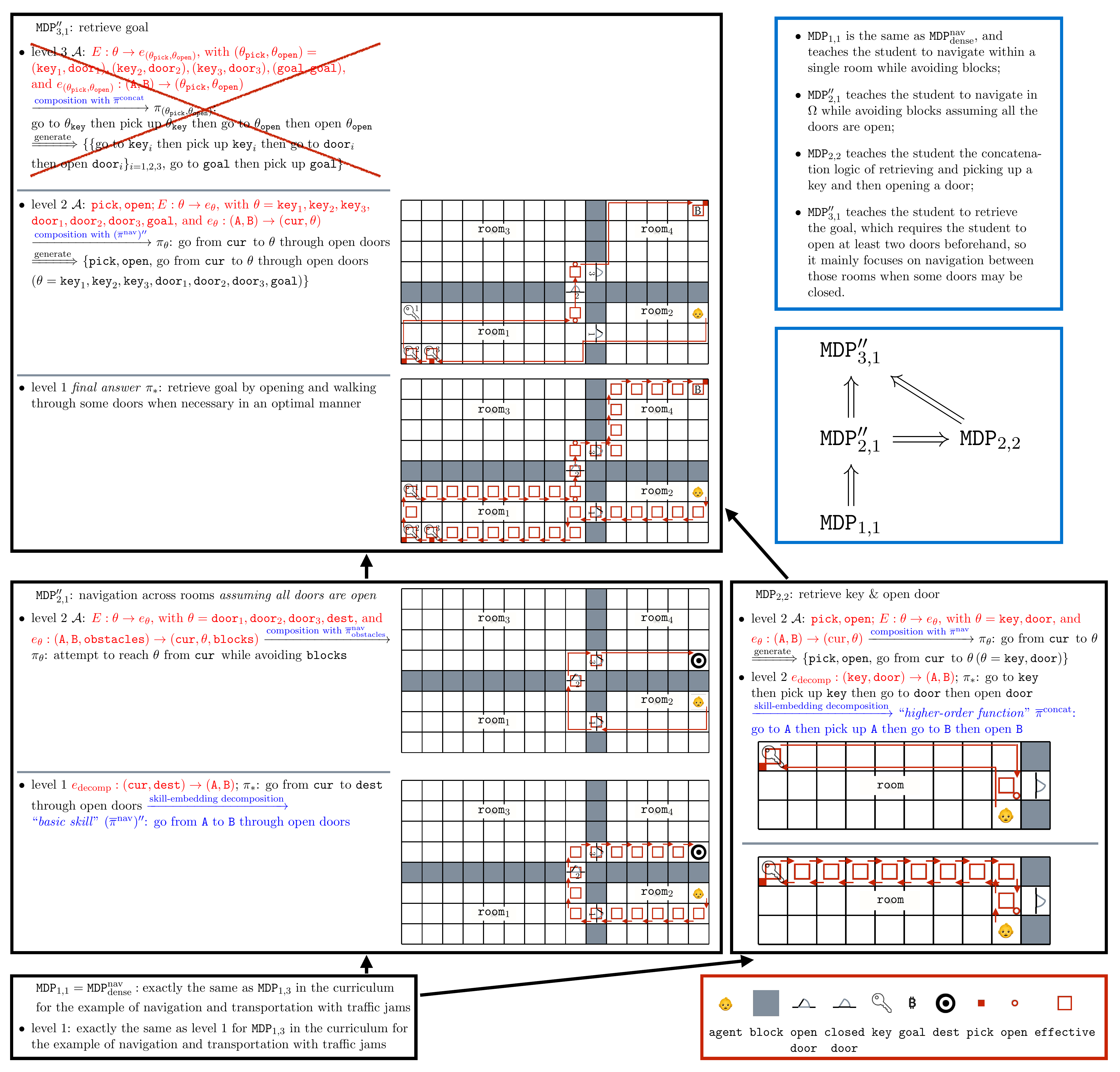}
\caption{
A representation of the third experiment of the MazeBase+ example. The representation in this figure is similar to the one for the first experiment (Fig.~\ref{f:MazeBase}). In the third experiment, we consider the situation where the optimal policy at the highest level, when refined to a finer level, yields a suboptimal policy where the student collects $\key_2$, opens $\door_2$, and then collects $\key_3$ (very close to $\key_2$) to then open $\door_3$. Within our algorithm, this suboptimal policy gets refined in order to yield the optimal policy, demonstrating the robustness of our optimization procedure; in this case it is also the case that our algorithm still outperforms na\"ive value iteration (see Fig.~\ref{f:MazeBase_ours23} (right)). 
}
\label{f:MazeBase3} 
\end{figure}

\subsubsection{Online learning}
\label{box:MazeBase-online}

When new MDPs come into the curriculum, such as $\MDP'_{2,1}$ and $\MDP'_{3,1}$ here, the student can solve them following the strict lexicographic order defined on the MDPs as in \eqref{e:curriculum}, while the assistant extracts out new skills, such as $(\overline{\pi}^{\mathrm{nav}})'$ and adds to the public skill set $\Skill$. The student can also utilize all the skills in the skill set, both the ones already there and ones newly added, when solving the MDPs. For instance, the student utilizes $\overline{\pi}^{\mathrm{concat}}$ which are already there in the skill set, and $(\overline{\pi}^{\mathrm{nav}})'$, which is newly added when solving $\MDP'_{3,1}$, in order to solve it faster. In this way, all the current learned knowledge does not need to be relearned again, so this whole process is in a fully online fashion.

\subsubsection{Robustness in the refinement procedure}
\label{s:MazeBase-robust}
We introduce a new problem $\MDP''_{3,1}$, similar to $\MDP_{3,1}$, with the main difference lying in the geometric configuration of the objects. We omit the details of the changes, but hope to emphasize that one main change is that now both $\key_2$ and $\key_3$ are in $\room_1$, and they are next to each other, so the optimal policy is to pick up both keys, before opening $\door_2$, rather than pick up $\key_2$ then open $\door_2$ and go back to pick up $\door_3$ then open $\door_3$ as hinted by $\overline{\pi}^{\mathrm{concat}}$. 
In this situation, the optimal policy at the highest level, when refined to a finer level, is suboptimal and requires, within our algorithm, significant refinement in order to yield the optimal policy, demonstrating the robustness of our optimization procedure: see Fig.~\ref{f:MazeBase3}; Fig.~\ref{f:MazeBase_ours23} (right) shows that even in this setting our algorithm outperforms na\"ive value iteration.

\section{Navigation and transportation with traffic jams: an example MazeBase+ is based on with multiple action factors}
\label{s:examples-MMDP-transfer-learning}

In this section, we provide another example curricula: navigation and transportation with traffic jams, an example, with multiple action factors, that the first example MazeBase+ is based on. 

\subsection{The example of navigation and transportation with traffic jams, part I}

\label{box:MDPs-traffic}

\begin{figure}[t]
\begin{center}
\includegraphics[width=\textwidth]{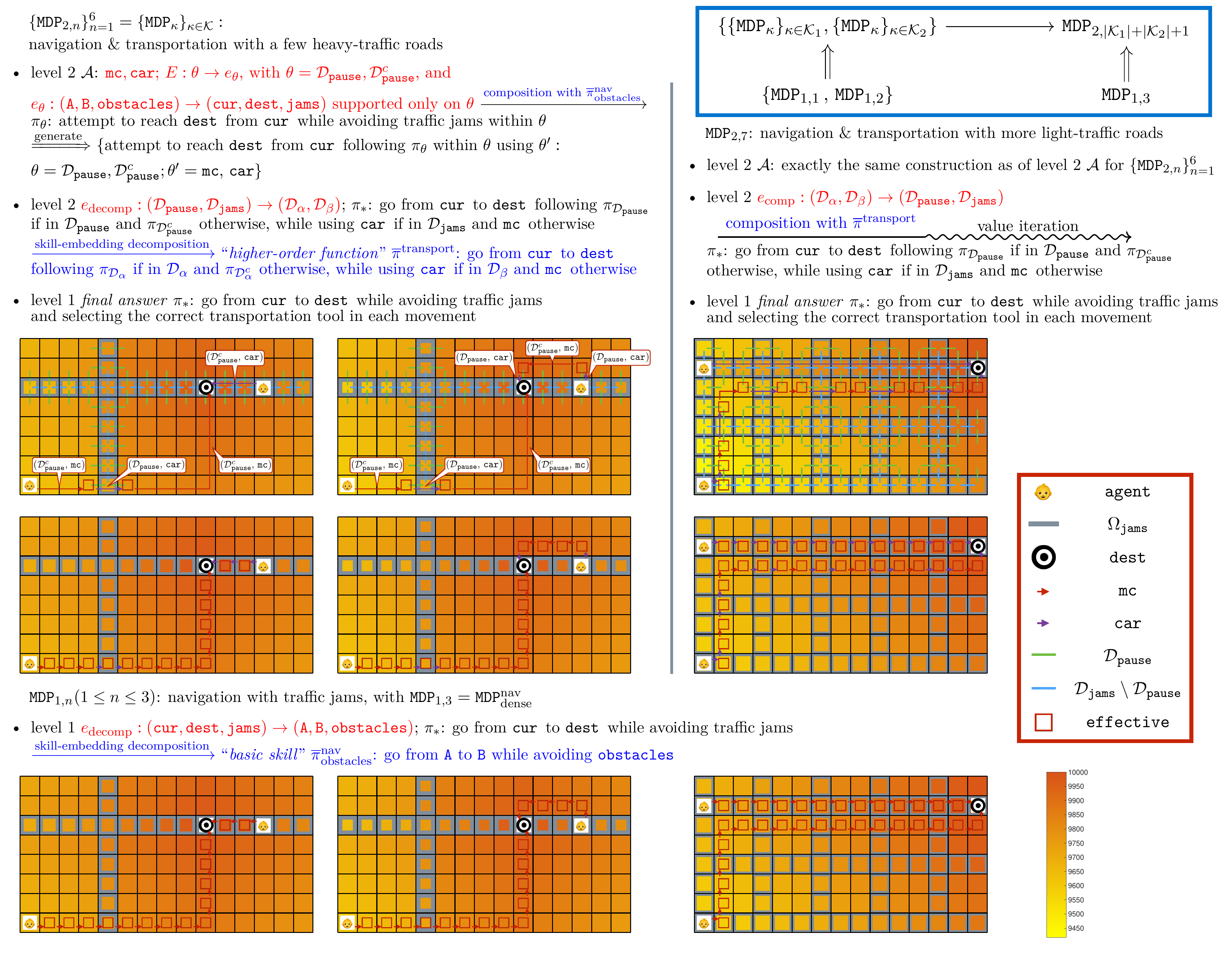}
\captionof{figure}{A representation of the example of navigation and transportation with traffic jams, with curriculum in the blue inset (top right), and the MMDPs in the curriculum, with their various levels, depicted and summarized in the black insets. The representation in this figure is similar to the one for the first experiment (Fig.~\ref{f:MazeBase}). 
The goal of the agent is to learn optimal policies to travel from an initial position in a grid world to a known final state in any position in the grid world.
Some regions, $\jams$, in the grid world have traffic jams (of different degrees of severity). 
The agent will have at its disposal two means of transportation, $\motor$ and $\car$, with $\car$ having different velocities in $\jams$ and $\nojams$.
}
\label{f:traffic} 
\end{center}
\end{figure}

We now introduce this pedagogical ``navigation and transportation with traffic jams'' example to one again walk the reader through the concepts and notations we introduce. Please see Fig.~\ref{f:traffic} for a comprehensive summary of this example and the corresponding curriculum, 
represented in the blue inset on the top right, 
where a double-line arrow means the knowledge learned from one MDP (starting point of the arrow) is utilized when \textit{constructing} solving another MDP (endpoint of the arrow), and a single-line arrow means the knowledge learned from one MDP (starting point of the arrow) is utilized when \textit{solving} another MDP (endpoint of the arrow). 
The structure of the curriculum is reflected by the multi-level MDPs in the black boxes, also connected by corresponding arrows. 
The curriculum has two sets of problems, corresponding to two difficulties.
The problems $\{\MDP_{1,n}\}_{n=1}^3$ of difficulty $1$, at the bottom of Fig.~\ref{f:traffic}, are designed to teach the agent how to navigate, without a choice of means of transportation, and only through basic up/down/left/right moves (red arrows).
Different $n$'s correspond to different velocities in the $\jams$ regions with $\car$, and, for $n=3$, a different choice of $\jams$.
The agent will learn optimal policies for these problems that become basic skills $\overline{\pi}^{\mathrm{nav}}_{\obstacles}$, to be utilized in higher level MDPs (upper part of the figure).
The problems $\{\MDP_{\kappa}\}_{\kappa\in\mathcal{K}}$ of difficulty $2$ involve both navigation and means of transportation; different $\kappa$'s correspond to different velocities in the $\jams$ regions with $\car$, separated into two classes $\mathcal{K}_1$ (first column) and $\mathcal{K}_2$ (second column), requiring two different optimal navigation rules because of the different velocities in $\jams$ vs. $\nojams$. 
A single action at level $2$ corresponds to an entire skill at level $1$, up to a composition/decomposition with family of maps that we call embeddings (represented by corresponding different types of arrows), and therefore leads to a long and complex path (long red arrow).
Roughly speaking, each skill is a parametric family of policies, and in some of the figures above we display the actual parameter for each action (at level $2$); for example $(\nochangemeans,\motor)$ denotes the action at level $2$ consisting of navigating with $\motor$ in a region not requiring changes in the means of transportation; $(\changemeans,\car)$ denotes the action at level 2 consisting of navigating with $\car$ in a region that does require a change in the means of transportation. 
The effective state space (red squares) at level $2$, consisting of final states of level $2$ actions, is much smaller than at level $1$. 
This leads to a significant speed-up in learning MDPs at level $2$, because we reduce the space of policies we search over. 
  The gray lines between adjacent grid points represent that the roads between them have certain traffic conditions: the thicker the lines, the heavier the traffic. 
The value functions are represented by heat maps; 
optimal policies starting from certain initial states (which could be at any location) are represented by concatenations of red arrows (single actions at that level) and red squares (effective states at that level), again showing the significant reduction in the effective state space and in the number of actions at level $2$.
The third column shows an MDP with a different configuration of traffic jam regions, often requiring many changes in the means of transportation, increasing the difficulty of learning: we will use transfer learning of higher-order functions to speed up the learning process.

Now we detail the target family of MDPs we consider. Given a two-dimensional grid world $\gridworld:=[x_1,x_2]\times [y_1,y_2]\subseteq\mathbb{N}\times \mathbb{N}$, a first action factor $\cA_1:=\cA_{\dir}\cup\{\endactionfactor\}$ with the set of moves in different directions given by $\cA_{\dir}:=\{(1,0),(0,1),(-1,0),(0,-1)\}$, and a second action factor $\cA_2:=\cA_{\means}\cup\{\endactionfactor\}$, with the set of uses of different means of transportation given by $\cA_\means:=\{\motor,\car\}$, motorcycle and car respectively, having different velocities (and corresponding rewards, which we detail momentarily). We try to teach an agent (student) how to travel from point $s$ to point $s'$ while selecting the most efficient means of transportation for each step along the way.
To mimic a variety of road conditions, we add two roads of heavy traffic: $[x_1,x_2]\times \{y_3\}$ and $\{x_3\}\times[y_1,y_2]$, where $x_1<x_3<x_2, y_1<y_3<y_2$ (here we take $x_1=1,x_2=15,x_3=5,y_1=1,y_2=8,y_3=6$), and we denote their union as $\jams$. 
The traffic rules are: generally, the agent could use either the motorcycle or the car, with corresponding speeds $v_{\motor},v_{\car}$, satisfying $v_{\motor}=1> v_{\car}=0.6$; if the agent is traveling along, entering, or leaving any of these two roads with $\motor$, it will receive a large negative penalty in the reward; with $\car$ its speed will be $\kappa$ for some $\kappa< v_{\car}$. The reward at each step is $\frac{r_0}{v}$, with $r_0<0$ (set to $-10$ here), and $v$ the agent's current speed, i.e., $v=v_{\motor}$ if the agent has chosen $\motor$ in the current action, $v=\kappa$ if the agent is traveling along, entering, or leaving any of the two roads of heavy traffic with $\car$, and otherwise $v=v_{\car}$.

We consider a family $\{\MDP_\kappa\}_{\kappa\in\mathcal{K}}$ of MDPs as above, with varying $\kappa$, which can and will affect the optimal policy. 
The formal definition of these MDPs is postponed to \eqref{e:MDPs-traffic} in App.~\ref{s:MDPs-traffic}; here we specify the state space and action set as follows:
\begin{equation*} 
\cS:=\{(\cur,\dests):\cur,\dests\in\gridworld\}=\gridworld\times\gridworld\quad,\quad
\cA:=(\cA_{\dir}\cup\{\endactionfactor\})\times (\cA_{\means}\cup\{\endactionfactor\})\,.
\end{equation*}
Here the teacher has the role to set all the parameters of the problems, including the reward at $\dest$ or the negative reward for traveling along/entering/leaving $\jams$ using $\motor$. 
The value of $\kappa$ affects the speed of the car in traffic regions.

\subsubsection{Action factors and partial policy generators}

\label{box:ppg-traffic}

For the previously defined $\MDP_\kappa$, there are two action factors $\cA_\dir\cup\{\endactionfactor\}$ and $\cA_\means\cup\{\endactionfactor\}$, with active action factor sets $I_1:=\dir$ and $I_2:=\means$, respectively.

Observe that there will be entanglement between the two action factors, modeling navigation and means of transportation, along typical optimal routes.
In order to efficiently solve this target family of problems, we exploit the similarities between them to enable potential transfer; we start by disentangling the action factors and introducing partial policies used in this example, and we will then combine these partial policies across action factors to obtain full optimal policies.

We define the transition region $\changemeans:=\{((\cur,\dests), (a_\dir,0))\in\cSA_\dir:\cur\in\jams \text{ and } \cur+a_\dir\in\nojams, \text{ or } \cur\in\nojams \text{ and } \cur+a_\dir\in\jams\}$ as the set of all the state-action pairs that either enter or leave the region $\jams$ with traffic jams, $\nochangemeans := \cSA_\dir-\changemeans$.  
See Fig.~\ref{f:traffic} for an illustration of these important subsets of $\cSA_\dir = \{(\cur,\dests),(a_{\dir},0))\}$, where we use an undirected edge to represent two directed elements, for the sake of not cluttering the image.

Now, we will have two types of partial policy generators, corresponding to the two action factors $\cA_\dir\cup\{\endactionfactor\}$ and $\cA_\means\cup\{\endactionfactor\}$, which we {\it{name}} $(g\levidx{1}_\kappa)_{\dir} (\kappa\in\mathcal{K})$ and $(g\levidx{1})_{\means}$, with active action factor sets $I_1:=\dir$ and $I_2:=\means$, respectively.

For each $\kappa\in\mathcal{K}$, $(g\levidx{1}_\kappa)_{\dir}: \Theta_\dir\rightarrow \{(\pi^{\mathrm{nav}_{\kappa}}_{\theta})_{\dir}\}_{\theta\in\Theta_\dir}$ is defined by $(g\levidx{1}_\kappa)_{\dir}(\theta):=(\pi^{\mathrm{nav}_{\kappa}}_{\theta})_{\dir}$, where $\Theta_{\dir}:=\{\changemeans,\nochangemeans\}$ indexes partial policies $(\pi^{\mathrm{nav}_{\kappa}}_{\theta})_{\dir}:\cSA\levidx{1}_\dir \rightarrow [0,1]$, such that each $(\pi^{\mathrm{nav}_{\kappa}}_{\theta})_{\dir}$ takes positive values only in $\nochangemeans$ or only in $\changemeans\cup\{(\cur,\goal),(\endactionfactor,0)\}$.
This partial policy is oblivious of the means of transportation, and represents only navigation along paths.
This partial policy generator conveys information: $\changemeans$ suggests where the agent may change the means of transportation, and $\nochangemeans$ suggests that there might not be a benefit in changing the means of transportation, so the agent may focus on navigation only.

The other partial policy generator $(g\levidx{1})_{\means}$ is simpler. It is defined on $\Theta_\means:=\cA_\means$,
and $(g\levidx{1})_{\means}:\cA_{\means}\rightarrow \{(\pi_{\theta'})_{\means}\}_{\theta'\in\cA_{\means}}$, mapping $\theta'$ to $(\pi_{\theta'})_{\means}$.
It generates two partial policies $\{(\pi_{\theta'})_{\means}:\theta'\in\cA_{\means}\}$, with $(\pi_{\theta'})_{\means}:\cSA^{1}_{\means} \rightarrow \{0,1\}$ (defined in \eqref{eq:pithetameans}) representing the partial policy of choosing a fixed means of transportation $\theta'$.

\subsubsection{Partial policy generator set}
\label{box:ppgs-traffic}
For the first level, $\MDP\levidx{1}_{\kappa}:=(\cS\levidx{1},(\Sinit)\levidx{1},(\Sterm)\levidx{1},\cA\levidx{1},P\levidx{1},R\levidx{1}_{\kappa},\Gamma\levidx{1})=(\cS,{\Sinit},{\Sterm},\cA,P,R_{\kappa},\Gamma)$. 
The teacher provides 
\begin{equation}  \label{e:G1-traffic}
\cG\levidx{1}:=\{\{(g\levidx{1}_\kappa)_{\dir}\}_{\kappa\in\mathcal{K}},(g\levidx{1})_{\means}\}\,,
\end{equation}
whose elements were defined as in Sec.~\ref{box:ppg-traffic}. Notice that it may happen that $(g\levidx{1}_\kappa)_{\dir}=(g\levidx{1}_{\kappa'})_{\dir}$ for $\kappa\neq\kappa'$.
The student constructs the set of policies $\Pi\levidx{1}$ generated from the set of partial policies $\widetilde{\Pi}_{\cG\levidx{1}}=\{(\pi^{\mathrm{nav}_{\kappa}}_{\theta})_{\dir}\}_{\theta\in\Theta_{\dir},\kappa\in\mathcal{K}}\cup \{(\pi_{\theta'})_{\means}\}_{\theta'\in\cA_{\means}}$, which itself is generated from $\cG\levidx{1}$: 
\begin{align*} 
\Pi\levidx{1}=\{(\pi^{\mathrm{nav}_{\kappa}}_{\theta})_{\dir}\otimes (\pi_{\theta'})_{\means}:\theta\in\Theta_{\dir}, \theta'\in\cA_{\means}, \kappa\in \mathcal{K}\}\,, 
\end{align*}
where, as in \eqref{e:def-otimes}, $(\pi^{\mathrm{nav}_{\kappa}}_{\theta})_{\dir}\otimes (\pi_{\theta'})_{\means}:\cSA\levidx{1}\rightarrow [0,1]$ is, using the definition of $(\pi_{\theta'})_\means$ in \eqref{eq:pithetameans},
\begin{align*} 
&((\pi^{\mathrm{nav}_{\kappa}}_{\theta})_{\dir}\otimes (\pi_{\theta'})_{\means})((\cur,\dests),(a_\dir,a_\means))\\=&(\pi^{\mathrm{nav}_{\kappa}}_{\theta})_{\dir}((\cur,\dests),(a_\dir,0))\mult\indic_{\{\theta'\}}(a_\means)\,.
\end{align*}
The right hand side, once $(\pi^{\mathrm{nav}_{\kappa}}_{\theta})_{\dir}$ will be learned, represents a policy for navigating from $\cur$ to $\dests$ within the domain $\mathrm{dom}_{\theta}\in\Theta_{\dir}$ using only $a_\means=\theta'$.
Note that each policy in $\Pi\levidx{1}$ is represented by an element in the product set $(\Theta_1\cup\{\emptyvalue\})^{|\mathcal{K}|}\times (\Theta_2\cup\{\emptyvalue\})$, with $\Theta_1 := \Theta_{\dir}\cup\{\endactionfactor\}$, and $\Theta_2 :=\cA_{\means}\cup\{\endactionfactor\}$. 
For instance, $(\pi^{\mathrm{nav}_{\kappa}}_{\theta})_{\dir}\otimes (\pi_{\car})_{\means}$ is represented by a vector with $\theta$ for the entry corresponding to $\kappa$ and $\car$ for the last entry, and $\emptyvalue$ for all the other $|\mathcal{K}|-1$ entries; it corresponds to a policy of navigating inside $\mathrm{dom}_\theta$ from $\cur$ to $\dests$ with $\car$; if $\mathrm{dom}_{\theta}=\nochangemeans$, note that this policy will lead to $\dests$ only when $\cur$ and $\dests$ are not separated by $\jams$.

These partial policy generators are ripe for being used to build higher-level actions for MMDPs, as we now discuss.

\subsubsection{Inputs for the construction of MMDPs}
\label{box:inputs-traffic}
The provided {\bf \policysequence} $\{\cG\levidx{l}\}_{l=1}^\infty$ (same for any $\MDP_{\kappa}$) for this example consists of $\cG\levidx{1}$ defined as in \eqref{e:G1-traffic}, and $\cG\levidx{l} := \varnothing$ for $l\geq 2$. There are multiple options for $\{\cG\levidx{l}_{\kappa,\text{test}}\}_{l=1}^\infty$ here, with these restrictions: (1) $\pi_{\kappa,\ast}\notin \cG\levidx{1}_{\kappa,\text{test}}$; (2) $\pi\levidx{2}_{\kappa,\ast}\in \cG\levidx{2}_{\kappa,\text{test}}$, with $\pi_{\kappa,\ast}=\pi_{\kappa}$, $\pi\levidx{2}_{\kappa,\ast}$ the optimal policies of $\MDP_{\kappa}$, $\MDP\levidx{2}_{\kappa}$ and provided in \eqref{e:first-level-policy-eqn}, \eqref{e:second-level-policy-eqn} respectively. These two conditions together guarantee that MDPs in $\{\MDP_{\kappa}\}_{\kappa\in\mathcal{K}}$ are all of difficulty $2$ (See Sec.~\ref{s:MMDP-def}).
The timescales of both $(g\levidx{1}_\kappa)_{\dir}$ and $(g\levidx{1})_{\means}$ are $+\infty$. We let $r\levidx{1} = -10$.

The student then constructs the second-level $\MDP\levidx{2}_{\kappa}:=(\cS,\Sinit,\Sterm,\overline{\Pi\levidx{1}},P\levidx{2},R\levidx{2}_{\kappa},\Gamma\levidx{2})$, using the inputs above and the procedures we described in Sec.~\ref{s:MMDP-def}.

\subsubsection{Unpacking compressed policies}
\label{box:unpack-traffic}
Once the second-level MDP $\MDP\levidx{2}_{\kappa}=(\cS,\Sinit,\Sterm,\overline{\Pi\levidx{1}},P\levidx{2},R\levidx{2}_{\kappa},\Gamma\levidx{2})$, for each $\kappa\in\mathcal{K}$, is constructed, it can be solved to find the optimal policy $\pi_{\kappa,\ast}\levidx{2}$. 
We construct stochastic trajectories starting from each state $s\in \cS$ by ``gluing'' together the actions $(\pi^{\mathrm{nav}_{\kappa'}}_{\theta})_{\dir}\otimes (\pi_{\theta'})_{\means}$ (defined in Sec.~\ref{box:ppgs-traffic}) following an order of the form $(\pi^{\mathrm{nav}_{\kappa_0}}_{\theta_0})_{\dir}\otimes (\pi_{\theta'_0})_{\means}, \cdots, (\pi^{\mathrm{nav}_{\kappa_t}}_{\theta_t})_{\dir}\otimes (\pi_{\theta'_t})_{\means}, \cdots$: the agent starts from $S_0=s$ in $\MDP_{\kappa}\levidx{2}$, and chooses actions in $A_t = (\pi^{\mathrm{nav}_{\kappa_t}}_{\theta_t})_{\dir}\otimes (\pi_{\theta'_t})_{\means}$ for any $0\leq t\leq \tau-1$ with $\tau$ being the first time $t$ such that $S_t \in\Sterm$ (if such event does not occur, we set $\tau=+\infty$). 
For each state $s$, the sequence of actions needs to be optimized to maximize the expected cumulative rewards along the stochastic trajectories. 
Following the description here, the value $\pi_{\kappa,\ast}\levidx{2}(s,(\pi^{\mathrm{nav}_{\kappa}}_{\theta})_{\dir}\otimes (\pi_{\theta'})_{\means})$, for each $\theta\in\Theta_\dir, \theta'\in\cA_{\means}$, $s=(\cur,\dests)$, equals \eqref{e:second-level-policy-eqn}.

Note that the minimization of $\theta'_t$ is trivial at any state $s$.
Here $\djams:=\{(s,(\dir,0))\in\cSA_{\dir}:\cur\in\jams \text{ or } \cur+a_\dir\in\jams\}$ is the set of state-action pairs such that either the agent's current location is in $\jams$, or the agent intends to move to $\jams$; $\dnojams := \cSA_{\dir}\setminus\djams$. See Fig.~\ref{f:traffic} for an illustration of these important subsets of $\cSA_{\dir} = \{(\cur,\dests),(a_{\dir},0))\}$ appearing in this higher-order policy $\pi_{\kappa,\ast}\levidx{2}$, where again we use an undirected edge to represent two directed elements, for the sake of not cluttering the image. In addition, recall that for any two (random) times $0\le T<T'<\infty$ (a.s.), $(R_{\kappa}\levidx{2})_{T,T'}$ is defined as in \eqref{e:cumulative-rewards} for $\MDP_{\kappa}\levidx{2}$. 
$\MDP_{\kappa}\levidx{2}$ has at least two key advantages compared to the level 1 MDP: first, it has much shorter time horizon, as a single action moves the agent by multiple steps all within a connected region of $\changemeans$ or $\nochangemeans\setminus\{(\cur,\dests),(\endactionfactor,0)\}$; second, the stochasticity is greatly reduced as it is absorbed into each higher-level navigation policy, further simplifying the optimization in \eqref{e:second-level-policy}.

Note how \eqref{e:second-level-policy-eqn} and \eqref{e:second-level-policy} cause an interplay between the tensor-product structure of $A_t$ and the geometry of $\jams$: the student needs to stop and reconsider which means of transportation to use every time it enters or leaves $\jams$.
Also note that it is crucial here that the navigation policies used at this level terminate, by choosing $\endactionfactor$ at very precise times and locations, instead of relying on random stopping times.
 
Finally, the student solves the original MDPs, by using \eqref{e: convolution} to pass the optimal policy of each $\MDP\levidx{2}_{\kappa}$ down to level one and refine it, resulting in the policy $\pi_{\kappa}$ as in \eqref{e:first-level-policy-eqn}, where $({\pi}^{\mathrm{nav}_{\kappa}})_\dir$ is defined in the forthcoming Sec.~\ref{box:compose-traffic}.
In this particular example, $\pi_{\kappa}$ is in fact the optimal policy $\pi_{\kappa,\ast}$ of the original MDP, requiring no additional refinement by value iteration.

\subsection{The example of navigation and transportation with traffic jams, part II}

\label{box:compose-traffic}
In the example of navigation and transportation with traffic jams, we now explain how the partial policy generators $(g\levidx{1}_\kappa)_{\dir} (\kappa\in\mathcal{K})$, $(g\levidx{1})_{\means}$ in $\cG\levidx{1}$ can be constructed by composing skills with two embedding generators $E\levidx{1}_\alpha$ and $E\levidx{1}_\beta$ respectively.

For each $\kappa\in\mathcal{K}$, $(g\levidx{1}_\kappa)_{\dir}$ is the composition of $\overline{\pi}^{\mathrm{nav}_{\kappa}}_{\mathrm{obstacles}}$ and $E\levidx{1}_{\alpha}$, which we now define. 
We let the navigation skill $\overline{\pi}^{\mathrm{nav}_{\kappa}}_{\mathrm{obstacles}}:\{(\cur,\dests,a_{\dir}):\cur,\dests\in\gridworld,a_{\dir}\in\cA_{\dir}\cup\{\endactionfactor\}\}\rightarrow [0,1]$, with timescale $+\infty$, be essentially the same as $({\pi}^{\mathrm{nav}_{\kappa}})_\dir$: $\overline{\pi}^{\mathrm{nav}_{\kappa}}_{\mathrm{obstacles}}(\cur,\dests,a_\dir):=({\pi}^{\mathrm{nav}_{\kappa}})_\dir((\cur,\dests),(a_\dir,0))$, which coincides with $({\pi}_{\changemeans}^{\mathrm{nav}_{\kappa}})_\dir$ on $\changemeans$ and $({\pi}_{\nochangemeans}^{\mathrm{nav}_{\kappa}})_\dir$ on $\nochangemeans$, introduced in Sec.~\ref{box:ppg-traffic}.
Such a skill $\overline{\pi}^{\mathrm{nav}_{\kappa}}_{\mathrm{obstacles}}$ will be learned from some other auxiliary MDPs focusing on navigation only in a curriculum as discussed momentarily. 

$E\levidx{1}_{\alpha}$ is defined as in \eqref{e:E1alpha}, generating embeddings $(e^1)\levidx{\dir}_{\theta}$. While $(e^1)\levidx{\dir}_{\theta}$ may look like a glorified identity, the crucial point here is that $E\levidx{1}_\alpha$ lets the teacher provide the student with the information that $\changemeans$ and $\nochangemeans$ is a partition of $\cSA\levidx{1}_\dir$ important for learning an optimal policy, since $(e^1)\levidx{\dir}_{\theta}$ is supported on $\theta$, where $\theta\in\Theta_{\dir}$.

$(g\levidx{1})_{\means}$ is the composition of the degenerate skill, which is the identity map on $[0,1]$, and $E\levidx{1}_{\beta}$ as defined in \eqref{e:E1beta}, generating embeddings taking out action information of the means of transportation.

\subsection{The example of navigation and transportation with traffic jams, part III}

\label{s:curriculum-traffic}
In the example of navigation and transportation with traffic jams, the teacher provides a curriculum containing two types of MDPs: $\MDP_{1,n} (1\leq n\leq n_1:=2)$ and $\MDP_{2,n} (1\leq n\leq n_2:=|\mathcal{K}|=6)$. 
$\MDP_{1,n}$, of difficulty $1$ and with detailed definition in \eqref{e:MDP-basic-traffic}, teaches the student to navigate through $\gridworld$, with no choice of means of transportation, but with $n_1=2$ different values of the parameter $\kappa$, corresponding to light and heavy traffic jams, and inducing different optimal policies. The MDPs $\{\MDP_{2,n}\}_{1\leq n\leq n_2}$, of difficulty $2$, are $\{\MDP_{\kappa}\}_{\kappa\in \mathcal{K}}$ as previously mentioned (with detailed definition in \eqref{e:MDPs-traffic}), with $n_2=|\mathcal{K}|=6$; these are our main objectives and require combining navigation with transportation.

\subsubsection{Analytical run of the algorithm}
\label{s:analytical-traffic}
The following goes through how Algs.~\ref{algorithm:learn-MDPs}--\ref{algorithm:learn-MDP} solve these MDPs analytically by constructing MMDPs, where $t_{\min} = t_{\max} = +\infty$. See App.~\ref{s:all-traffic} for the equations needed here as well as the mathematical notations mentioned here.

\begin{enumerate}[leftmargin=0.5cm,itemsep=0pt]
\item[$\bullet$] The student finds an optimal policy $\pi_{1,n,\ast}\levidx{1}$ for $\MDP_{1,n}\levidx{1}:=\MDP_{1,n} (1\leq n\leq n_1)$,
and when the teacher provides the trivial embedding $(e_{\decomp})_{1}\levidx{1}$ as in \eqref{e:e11decomp-traffic}, the assistant extracts two navigation skills 
$\overline{\pi}^{\mathrm{nav}_{n}}_{\mathrm{obstacles}}$ as in \eqref{e:navskill-traffic}, with timescales $+\infty$.
Both navigation skills are about finding the shortest path between $\cur$ and $\dests$, where each grid point is a vertex in the graph, and there are edges between two grid points if the agent can use one step in $\cA_{\dir}$ to move between them, with edge weights reflecting the time the agent needs to take for that single step as shown in the reward functions. 

\item[$\bullet$] Let $\MDP_{2,n} := \MDP_{\kappa_n}$, $1\leq n \leq n_2$, $\kappa_n\in\mathcal{K}$, be an MDP of difficulty $2$, which we have discussed thoroughly in the previous sections. 
The student constructs, for each $\MDP_{2,n}$ a two-level MMDP, as in Sec.~\ref{box:inputs-traffic}, to solve it; for $\theta\in\Theta_{\dir}, \theta'\in\cA_{\means}$, the timescales for $(\pi^{\mathrm{nav}_{\kappa}}_{\theta})_{\dir}\otimes (\pi_{\theta'})_{\means}$ in $\Pi\levidx{1}$ are $t_{(\pi^{\mathrm{nav}_{\kappa}}_{\theta})_{\dir}\otimes (\pi_{\theta'})_{\means}}=\min\{t_{(\pi^{\mathrm{nav}_{\kappa}}_{\theta})_{\dir}},t_{(\pi_{\theta'})_{\means}}\}=+\infty$. Then, the student solves level 2 $\MDP_{2,n}$ and finally solves level 1 of $\MDP_{2,n}$ (See Sec.~\ref{box:unpack-traffic} for details).
\end{enumerate}

\subsubsection{Merits of action factors and numerical run of the algorithm}
\label{box:Merits-action-factors-traffic}

In each $\MDP_{2,n}$, the action set consists of two action factors $\cA_{\dir}\cup\{\endactionfactor\}$, $\cA_{\means}\cup\{\endactionfactor\}$ independent of each other, which takes charge of navigation and transportation respectively. 
This tensor product structure enables transfer to $\MDP_{2,n}$, for $1\le n\leq n_2$, of the skills $\overline{\pi}^{\mathrm{nav}_{n}}_{\mathrm{obstacles}}(1\leq n\leq n_1)$ in the action factor $\cA_{\dir}\cup\{\endactionfactor\}$, with $\{\overline{\pi}^{\mathrm{nav}_{n}}_{\mathrm{obstacles}}:1\leq n\leq n_1\}$ exactly the same as $\{\overline{\pi}^{\mathrm{nav}_{\kappa}}_{\mathrm{obstacles}}:\kappa\in\mathcal{K}\}$, the set of skills we need for deriving $\cG\levidx{1}$ defined as in \eqref{e:G1-traffic} and constructing $\MDP\levidx{2}_{2,n}$. 
The optimization over $\cA_{\means}$ occurs at level 2 of $\MDP_{2,n}$, and the combination with $\cA_{\dir}$ is optimized at level 1 of $\MDP_{2,n}$.
In words: before going to a certain destination, the student first solves $\MDP_{1,n}$ to find virtual routes to the destination, then solving level 2 $\MDP_{2,n}$ yields the optimal means for those routes, based on traffic conditions, and finally solving level 1 of $\MDP_{2,n}$ yields the optimal combination of routes and means of transportation. 
The embedding generator $E_\alpha^1$, provided by the teacher, allows the student to break routes at locations in $\changemeans$, where it can then optimally change the means of transportation.

Reflecting on what we accomplished in this example using action factors, we have again achieved the three merits discussed in the end of Sec.~\ref{s:curriculum}. In particular, it can speed up the solution of the MDPs by leveraging transfer of skills within the curriculum, see Fig.~\ref{f:navigation_ours}. These extra iterations correspond to learning how to stitch different pieces of routes separated by $\changemeans$, which contains the turning points at which the student may need to switch the means of transportation. This is illustrated by the forthcoming Thm.~\ref{t:complexity}.
Of course, the cost of solving the $\MDP_{1,n}$'s (which for our algorithm is the same as for classical value iteration) is amortized over solving possibly many $\MDP_{2,n}$'s, showcasing the power of transfer in our framework. This is also analyzed in general by the forthcoming Thm.~\ref{t:transfer-2}. 

\begin{figure}[t]
\begin{center}
\includegraphics[width=\textwidth]{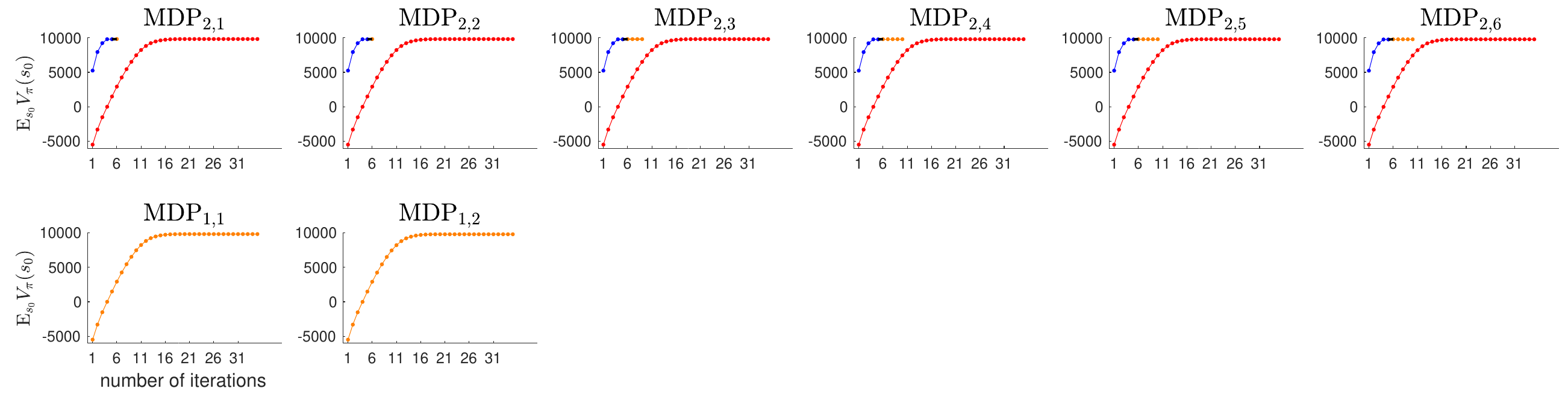}
\caption{Similar to Fig.~\ref{f:MazeBase_ours}, we display $\mathbb{E}_{s_0}V_\pi(s_0)$, where $V_\pi$ is the value function for $\MDP_{1,n}$ and $\MDP_{2,n}$, as $\pi$ is optimized during iterations of classical value iteration (in red) and of value iteration within our algorithm, 
with iterations within $\MDP_{1,n}$ in orange, and iterations within $\MDP_{2,n}$ at level $2$ in blue followed by iterations at level $1$ in orange. 
Although we spend extra effort in solving the $\MDP_{1,n}$'s, they prepare us well-enough so that we only need a few more iterations for solving the $\MDP_{2,n}$'s (blue+orange) -- much fewer than if we solved them from scratch using classical value iteration (red). 
}
\label{f:navigation_ours} 
\end{center}
\end{figure}

\subsubsection{Transfer learning of the higher-order function}
\label{box:transfer-traffic}
When there are many roads of traffic (say, $\Omega'_{\mathtt{jams}}:=[x_1,x_2]\times \{y_1,y_1+3,\cdots\}\cup\{x_1,x_1+3,\cdots\}\times [y_1,y_2]$), the agent if forced to switch the means of transportation every few steps. 
This problem can be resolved in our framework by using transfer learning to greatly speed up the learning of the second-level policy. 

The main idea is that if we take out one element of $\mathcal{K}$, say $\kappa_1$, we observe that the second-level policy $\pi\levidx{2}_{\kappa_1,\ast}$, given by \eqref{e:second-level-policy-eqn}, depends on the exact choice of $\Omega_{\mathtt{jams}}$, but only mildly, and such dependence could be decomposed using an embedding. Semantically, $\pi\levidx{2}_{\kappa_1,\ast}$ selects the index of the navigation partial policy, which determines both the support of the navigation policy and which navigation policy, and the means of transportation, independently.
(Here for simplicity we restrict the embedding in the decomposition only to a single navigation policy and focus on selecting onto which support it is restricted.)  
We use the following logic: the agent selects the index $\theta$ if $((\cur,\dests),(a_\dir,0))\in \theta$ for some $\theta\in \Theta_\dir$, and selects the means of transportation $\theta'=\car$ if $((\cur,\dests),(a_\dir,0))\in \djams$ and $\theta'=\motor$ otherwise. Here, $\cur$, $\dests$ are the agent's current location and destination respectively, and $a_\dir$ is the intended direction according to $({\pi}^{\mathrm{nav}_{\kappa_1}})_\dir$, as introduced in Sec.~\ref{box:ppg-traffic}. Notice here $\djams$ or $\changemeans$ are different when we change $\Omega_{\mathtt{jams}}$ to $\Omega'_{\mathtt{jams}}$, so we need an embedding to encapsulate the ``if-conditions'' in the logic above, after which a higher-order function could be extracted from the second-level policy $\pi\levidx{2}_{\kappa_1,\ast}$, which is purely about the ``if-then'' logic, the core part in $\pi\levidx{2}_{\kappa_1,\ast}$ of interest to us, and also transferrable across different geometries, in some sense achieving ``few-shot learning''.

Formally, after learning $\pi\levidx{2}_{\kappa_1,\ast}$, if the teacher provides the embedding, defined as in \eqref{e:embedding-transfer-traffic}, taking out from the state-action pair the information of which $\theta$-region $((\cur,\dests),(a_\dir,0))$ lies in, whether $((\cur,\dests),(a_\dir,0))$ is in $\djams$ or not, the agent's intended direction according to $({\pi}^{\mathrm{nav}_{\kappa_1}})_\dir$, as well as the means of transportation considered, then the assistant extracts a higher-order function for selecting the index of the navigation policy and the means of transportation $\overline{\pi}^{\mathrm{transport}}$ based on the ``if-then'' logic described above, defined in \eqref{e:transport-skill} and with timescale $t_{\overline{\pi}^{\mathrm{transport}}}=+\infty$.

Next, we show how the transfer learning of this higher-order function could be achieved in this example. Given the new region of traffic jams $\Omega'_{\mathtt{jams}}=[x_1,x_2]\times \{y_1,y_1+3,\cdots\}\cup \{x_1,x_1+3,\cdots\}\times [y_1,y_2]$ containing many more roads, displayed in the insets on the right column in Fig.~\ref{f:traffic}, the teacher adds to the curriculum $\MDP_{1,n_1+1}$, similar to $\MDP_{1,n}$, and $\MDP_{2,n_2+1}$, similar to $\MDP_{2,n}$, except that $1/\kappa=1.1$, and correspondingly $1/v_{\car} = 1.05$, for both these new MDPs, describing a new scenario where the traffic jams are light. Sometimes we may also refer to $\MDP_{1,n_1+1}$ as $\MDP_{\text{dense}}^{\text{nav}}$.

The process of solving the two new MDPs in the curriculum uses the same Algs.~\ref{algorithm:learn-MDPs}--\ref{algorithm:learn-MDP}, but we also extract a navigation skill $\overline{\pi}^{\mathrm{nav}}_\mathrm{dense}$, and we also have in hand the higher-order function $\overline{\pi}^{\mathrm{transport}}$ extracted previously for selecting the means of transportation. Therefore, after the student constructs $\MDP_{2,n_2+1}\levidx{2}$, the teacher provides the hint to use the skill $\overline{\pi}^{\mathrm{transport}}$ and the embedding $(e_{\comp})_{2,n_2+1}$ for solving it, where $(e_{\comp})_{2,n_2+1}$ is defined similarly to $(e_{\decomp})_{2,1}\levidx{2}$. Then, the student uses the composition of skill $\overline{\pi}^{\mathrm{transport}}$ and $(e_{\comp})_{2,n_2+1}$, which is similar to $\pi_{2,1,\ast}\levidx{2}$, as the initial policy for solving $\MDP_{2,n_2+1}\levidx{2}$. This leads to a very fast learning of $\MDP_{2,n_2+1}\levidx{2}$.
On the other hand, if we do not utilize this opportunity of transfer learning when following our framework to solve $\MDP_{2,n_2+1}\levidx{2}$, we will need many more iterations, because there are many pieces of routes separated by $\cD'_{\chng}$ when selecting means of transportation given that the roads with traffic are much more densely-distributed in $\gridworld$ now. See Fig.~\ref{f:navigation_ours_denser_logic} for a comparison between the two options.

\subsubsection{Online learning}
\label{box:traffic-online}

Another comment is that the extra experiment described in Sec.~\ref{box:transfer-traffic} also shows that the student can achieve online learning in the current framework. When new MDPs come into the curriculum, such as $\MDP_{1,n_1+1}\levidx{1}$ and $\MDP_{2,n_2+1}\levidx{2}$ here, the student can solve them following the strict lexicographic order defined on the MDPs as in \eqref{e:curriculum}, while the assistant extracts out new skills, such as the new navigation skill $\overline{\pi}^{\mathrm{nav}}_\mathrm{dense}$, and adds to the public skill set $\Skill$. The student can also utilize all the skills in the skill set, both the ones already there and ones newly added, when solving the MDPs. For instance, in order to solve $\MDP_{2,n_2+1}\levidx{2}$ faster, the student utilizes the higher-order function $\overline{\pi}^{\mathrm{transport}}$ for selecting the means of transportation, which is already there in the skill set, and $\overline{\pi}^{\mathrm{nav}}_\mathrm{dense}$, which is the navigation skill newly added when solving $\MDP_{1,n_1+1}\levidx{1}$. In this way, all the current learned knowledge does not need to be relearned again, so this whole process is in a fully online fashion.

\begin{figure}[t]
\centering
\includegraphics[width=0.4\textwidth]{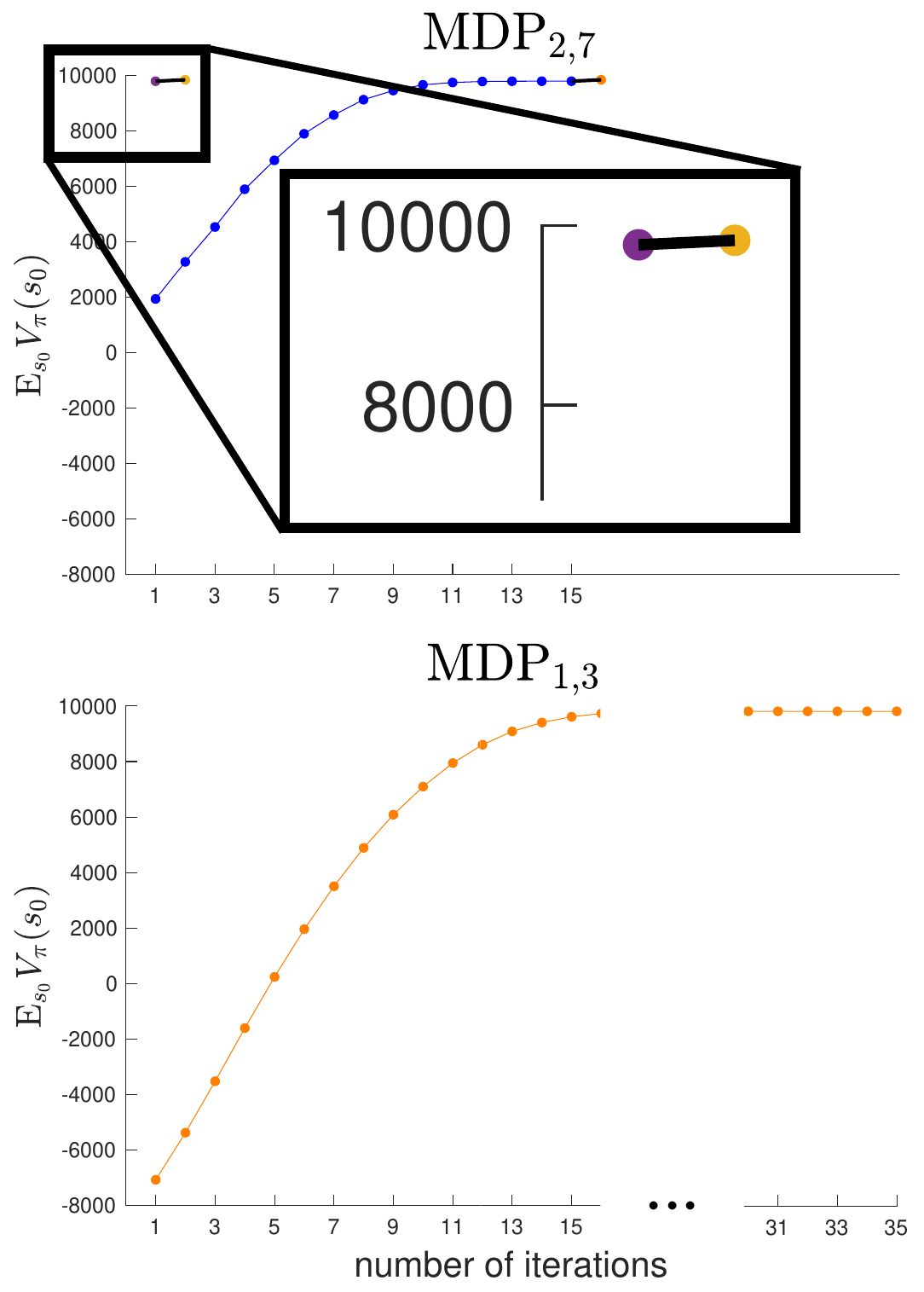}
\caption{Similar to Fig.~\ref{f:navigation_ours}, we display $\mathbb{E}_{s_0}V_\pi(s_0)$, where $V_\pi$ is the value function for $\MDP_{1,n_1+1}$ (in orange) and $\MDP_{2,n_2+1}$, as $\pi$ is optimized during iterations of value iteration within our algorithm. More specifically, within $\MDP_{2,n_2+1}$ we have iterations at level $2$ in blue followed by iterations at level $1$ in orange, where $\pi$ is optimized during iterations of value iteration within our algorithm without utilizing $\overline{\pi}^{\mathrm{transport}}$; we also have iterations at level $2$ in purple followed by iterations at level $1$ in yellow, where $\pi$ is optimized within our algorithm utilizing $\overline{\pi}^{\mathrm{transport}}$.
Although we spend extra effort in extracting $\overline{\pi}^{\mathrm{transport}}$, it prepares us well enough so that we could solve $\MDP_{2,n_2+1}$ (purple+yellow) almost instantly, with much fewer iterations than if we solved it from scratch within our algorithm (blue+orange). This is also analyzed in general by the forthcoming Thm.~\ref{t:transfer-1}.
}
\label{f:navigation_ours_denser_logic} 
\end{figure}

\section{Theoretical analysis: MMDP solver and transfer learning}
\label{s:theory}

In this section, we implement case analysis for the correctness of the solution of the MMDP (Sec.~\ref{s:correctness}), and then provide theoretical guarantees for the savings in computational cost brought by the multi-level structure (Thm.~\ref{t:complexity}) and transfer learning (Thms.~\ref{t:transfer-1}--\ref{t:transfer-3}).

We start by defining what we mean in general by an \textbf{MMDP solver}.

\begin{definition} Given an MMDP $\{\MDP\levidx{l}\}_{l=1}^L$ with $L$ levels, an {\bf MMDP solver} is the solution process which starts from $\MDP\levidx{L}$, the $L$-th level MDP, solves it and refines its solution within $\MDP\levidx{L-1}$, the $(L-1)$-st level MDP, solves $\MDP\levidx{L-1}$, and so on, until in the end it solves the original $\MDP^1:=\MDP$.
\end{definition}

Given a single $\MDP$ of difficulty $L$, our framework constructs an MMDP and uses an MMDP solver to solve it: Props.~\ref{p:compress-transition}--\ref{p:compress-discount} are used to construct the MMDP, after which we apply value iteration/Q-learning to solve each $\MDP\levidx{l}$ in the MMDP, for $l$ decreasing from $L$ to $1$, while using \eqref{e: convolution} to refine the solution at level $l+1$ to level $l$, for $l=L-1,L-2,\dots,1$. The motivation behind constructing an MMDP $\{\MDP\levidx{l}\}_{l=1}^L$ with $\MDP^1 := \MDP$ and solving it instead of solving $\MDP$ directly is that we expect the multi-level compression to be such that the solution at each compressed level provides a good initialization to the refinement/solution at its next finer level, leading to computational gains that we analyze in this section.

Recall that $\pi\levidx{l}_{\ast}$ denotes the optimal policy of $\MDP\levidx{l}$ ($l\in [L]$), with $\pi\levidx{1}_{\ast}=\pi_{\ast}$, the optimal policy of $\MDP\levidx{1}=\MDP$, which we assume unique to simplify the discussion. 
The corresponding value functions are $V\levidx{l}_{\ast}: = V_{\pi\levidx{l}_{\ast}}$, with $V\levidx{1}_{\ast}$ also denoted by $V_{\ast}$. We also denote the initial policy and value function for $\MDP\levidx{l}$ ($l\in [L]$) as $\pi\levidx{l}_{\mathrm{init}}$ and $V\levidx{l}_{\mathrm{init}}$ respectively; for $l\in [L-1]$, they are obtained from the refinement step applied to $\pi\levidx{l+1}_{\ast}$ and $V\levidx{l+1}_{\ast}$ respectively. 
The \textit{Bellman optimality operator} used for updating the value function at each iteration of value iteration for solving the $l$-th level MDP $\MDP\levidx{l}$ is denoted by $\mathcal{T}\levidx{l}$. In particular, note that we continue to use the superscript $l$ as an index, not an exponent -- this is to maintain simplicity and consistency of notation. Throughout this section, we use the $\infty$-norm unless otherwise specified, and we denote it by $\|\cdot\|$ for simplicity.

\subsection{Correctness of an MMDP solver}
\label{s:correctness}

It is clear that if an MDP solver converges to the optimal policy for the original $\MDP\levidx{1} = \MDP$ with an arbitrary initialization, then the new framework, constructing an MMDP and using an MMDP solver to solve it, would also solve the MDP correctly. However, very often we do not fall into this ideal situation, and the original MDP could not be solved successfully directly with a chosen MDP solver (e.g., policy gradient, which may only converge when started from a suitably well-initialized policy).

We begin with the benign regime for the MMDP solver, in which the higher-level MDPs can be thought of as ``smoother'', more ``convex'' versions of the original MDP, with large regions of convergence for the MDP solvers, and, with refined policies $\pi\levidx{l}_{\mathrm{init}} (l\in [L-1])$ that are good initializations for the MDP solvers.
This is analogous to, e.g. smoothing out the non-convex part when optimizing a non-convex function or coarse-grained optimization \cite{Pozharskiy2020ManifoldLF}, but in a structure-preserving way that preserves the MDP structure. 
In this regime, the original MDP may not be solvable with a chosen MDP solver, but may be solved successfully in the multi-level framework. 

Now we consider non-benign regimes. We may encounter scenarios in which the policy learned at a higher-level leads to a bad initialization when refined to the next finer level, requiring significant refinement at the next finer level to obtain the optimal policy. 
Even in such a scenario, the MMDP solver could still converge to the optimal policy of the original MDP as long as the MDP solver applied to the next finer level, started from this bad initialization, converges to the optimal policy as the number of iterations goes to infinity. 
In Sec.~\ref{s:MazeBase-robust}, we run one numerical experiment in the MazeBase+ example, where we deliberately design a situation such that the action set in the third level is incomplete, leading to a poor policy initialization at the next finer level, requiring significant refinement in order to yield the optimal policy. However, our method still converges, and even outperforms na\"ive value iteration on the original MDP, demonstrating the robustness of our optimization procedure. 

A second even less-benign regime is that the learning at the higher-level does not stabilize when we stop. In such a case, the algorithm still continues by setting the initial policy at the next finer level to be the degenerate \textit{diffusive policy}. This allows the MMDP solver to still converge to the optimal policy of the original MDP as long as, eventually, at the finest level $1$, the MDP solver, started from the initialization derived by the MMDP solver or from the degenerate \textit{diffusive policy}, converges to the optimal policy.

\subsection{Complexity of MMDP solver (with value iteration)}

For convenience, we start by naming an assumption, which we may refer to for multiple times in the remaining part of this Sec.~\ref{s:theory}. 
\begin{myassum}
\label{a:regularity}
(1) {\bf Uniqueness of the optimal policies:} For each $l\in [L]$, the optimal policy $\pi\levidx{l}_\ast$ of $\MDP\levidx{l}$ is unique.
(2) {\bf Convergence of the MDP solvers within the MMDP solver:} For each $l\in [L]$, the MDP solver applied to $\MDP\levidx{l}$, started from the $\pi\levidx{l}_\mathrm{init}$ provided by the MMDP solver, converges  to $\pi\levidx{l}_{\ast}$ as the number of iterations goes to infinity. 
\end{myassum}
This assumption guarantees that the MMDP solver converges to the unique optimal policy $\pi_\ast$, i.e., the MMDP solver is correct, as the numbers of iterations used by the MDP solvers within the MMDP solver all go to infinity, as discussed in the previous section.

We now discuss the computational savings brought by the multi-level structure and, in particular, focus on the case where all the MDP solvers within the MMDP solver are value iteration, since value iteration has well-understood theoretical guarantees, notably {the contraction property of the Bellman optimality operator} \cite{10.5555/528623,10.5555/1396348}; we still allow an arbitrary MMDP construction and refinement step for generality. The results could be easily extended to other MDP solvers with theoretical guarantees.
To characterize the reduction of computational cost brought by our framework, we first introduce the following lemma, which strictly improves on the original contraction property of the Bellman optimality operator, and may be of independent interest in the domain of reinforcement learning: 

\begin{lemma} \label{l:error-bound}
For a general $\MDP = (\cS,\cA,P,R,\Gamma)$ with optimal policy $\pi_\ast$ and corresponding value function $V_{\ast}:=V_{\pi_\ast}$, let the value iteration algorithm applied to $\MDP$ return the estimated value function $V_i$ at the $i$-th iteration. 
If, at iteration $i$, $\epsilon^{\min}(s)\leq V_i(s)-V_{\ast}(s)\leq \epsilon^{\max}(s)$ holds, then
\begin{equation}
\Err^{\min}(s)\leq V_{i+1}(s)-V_{\ast}(s)\leq \Err^{\max}(s)\,,
\label{e:boundValueFunctionError}
\end{equation}
where
\begin{align*}
\Err^{\min}(s)&:=\sum_{s'\in\cS}P(s,a_\ast(s),s')\Gamma(s,a_\ast(s),s')\epsilon^{\min}(s')\,,\\
\Err^{\max}(s)&:=\sum_{s'\in\cS}P(s,a_{i+1}(s),s')\Gamma(s,a_{i+1}(s),s')\epsilon^{\max}(s')\,,
\end{align*}
$a_\ast(s)$ is the action selected at state $s$ in the Bellman update when the value function is optimal, and $a_{i+1}(s)$ is the action selected at state $s$ at the $(i+1)$-st iteration of value iteration. 
\end{lemma}
Lemma~\ref{l:error-bound} yields error bounds as functions of states, while existing bounds are in the $\infty$-norm (worst-case over states). This enables us to study how the values at different states converge asynchronously to the corresponding optimal values.
Another novelty is that instead of assuming constant discount factors, we allow them to be a function of the current state, action, and next state, and analyze the role of these discount factors in convergence. This is crucial for analyzing our compressed MDPs, since even when the discount factors are constant in the original MDP, the action sets of the compressed MDPs consist of policies at different timescales, yielding non-constant discount factors in the compressed MDPs.

\begin{corollary} \label{c:error-bound}
With the assumption and notations of Lemma~\ref{l:error-bound}, \eqref{e:boundValueFunctionError} also holds with 
\begin{align*}
\Err^{\min}(s):=& \min_{\substack{s'\in\cS\\ P(s,a_\ast(s),s')>0}}\Big(\Gamma(s,a_\ast(s),s')\epsilon^{\min}(s')\Big)\,,\\ \Err^{\max}(s):=&\max_{\substack{s'\in\cS\\ P(s,a_{i+1}(s),s')>0}}\Big(\Gamma(s,a_{i+1}(s),s')\epsilon^{\max}(s')\Big)\,.
\end{align*}
\end{corollary}
Given an $\MDP = (\cS,\cA,P,R,\Gamma)$ and two error functions $\Err^{\min}_{\mathrm{init}},\Err^{\max}_{\mathrm{init}}:\cS\rightarrow \mathbb{R}$, bounding below and above the error of the initial value function $V_{\mathrm{init}}$ respectively, we can derive two sequences of error functions $\{\Err^{\min}_i\}_{i\geq 0}$ and $\{\Err^{\max}_i\}_{i\geq 0}$ starting from $\Err^{\min}_0:=\Err^{\min}_{\mathrm{init}}$ and $\Err^{\max}_0:=\Err^{\max}_{\mathrm{init}}$ respectively, by applying repeatedly Lemma~\ref{l:error-bound}. 
Then, if for any error tolerance $\epsilon>0$, we define $N(\MDP, \Err^{\min}_{\mathrm{init}}, \Err^{\max}_{\mathrm{init}},\epsilon)$ to be the smallest $i$ such that $\max\{|\Err^{\min}_i|,|\Err^{\max}_i|\}<\epsilon$, we deduce the following:

\begin{corollary} \label{c:number-iteration}
With the assumption and notations of Lemma~\ref{l:error-bound}, at the $N(\MDP, \Err^{\min}_{\mathrm{init}},$\\ $\Err^{\max}_{\mathrm{init}},\epsilon)$-th iteration of value iteration, the value function is $\epsilon$-close to $V_\ast$ in the $\infty$-norm.
\end{corollary}

Cor.~\ref{c:error-bound} implies the following bound on $N(\MDP, \Err^{\min}_{\mathrm{init}}, \Err^{\max}_{\mathrm{init}},\epsilon)$:
\begin{proposition} \label{p:number-iteration}
With the assumption and notations of Lemma~\ref{l:error-bound}, we have 
\begin{align*} 
N(\MDP, \Err^{\min}_{\mathrm{init}}, \Err^{\max}_{\mathrm{init}},\epsilon)\leq N'(\MDP, \Err^{\min}_{\mathrm{init}}, \Err^{\max}_{\mathrm{init}},\epsilon)\,,
\end{align*}
where $
N'(\MDP, \Err^{\min}_{\mathrm{init}}, \Err^{\max}_{\mathrm{init}},\epsilon) := \max\left\{\left \lceil \log_{\gamma_\ast}\frac{\epsilon}{\|\Err^{\min}_{\mathrm{init}}\|} \right \rceil, N_{\text{VI}}(\MDP, \Err^{\max}_{\mathrm{init}},\epsilon)\right\}$, $\left \lceil x \right \rceil$ is the smallest integer greater than or equal to $x$,
$\gamma_\ast:=\max_{s,s'\in \cS, P(s,a_\ast(s),s')>0} \Gamma(s,a_\ast(s),s')$, and
$N_{\text{VI}}(\MDP, \Err^{\max}_{\mathrm{init}},\epsilon)$ is the smallest natural number $N$ such that 
$$\prod_{i=1}^{N} \max_{s,s'\in \cS, P(s,a_i(s),s')>0} \{\Gamma(s,a_i(s),s')\}<\frac{\epsilon}{\|\Err^{\max}_{\mathrm{init}}\|}\,,$$
if there exists one, $+\infty$ otherwise.
\end{proposition}

Now, we are ready to state the following main result:
\begin{theorem}\label{t:complexity}
Assume that Assump.~\ref{a:regularity} holds and, for $l\in [L]$, we set the stopping criterion for solving $\MDP\levidx{l} $as in our algorithm: the student stops learning if the value functions obtained from two consecutive iterations are $\epsilon\levidx{l}$-close to each other in the $\infty$-norm. Furthermore, assume that:
\begin{enumerate}[leftmargin=0.5cm,itemsep=0pt]
\item {\bf MMDP construction:} There exists a sequence of positive real numbers $\{\delta\levidx{l}\}_{l=1}^{L-1}$, such that $\|V\levidx{l+1}_{\ast}-V\levidx{l}_{\ast}\|_{\infty}<\delta\levidx{l}$ for each $l\in [L-1]$.
\item {\bf MDP solvers within the MMDP solver:} All the MDP solvers within the MMDP solver are value iteration.
\item {\bf Refinement steps within the MMDP solver:} There exists a sequence of positive real numbers $\{\delta\levidx{l}_{\text{refine}}\}_{l=1}^{L-1}$, such that for any policy $\pi\levidx{l+1}$ of $\MDP\levidx{l+1} (l\in [L-1])$ whose value function $V_{\pi\levidx{l+1}}$ is $\epsilon\levidx{l+1}/2$-close to the optimal value function $V\levidx{l+1}_{\ast}$ of $\MDP\levidx{l+1}$ in the $\infty$-norm, the refined policy of $\pi\levidx{l+1}$ within $\MDP\levidx{l}$ is ``close'' to $\pi\levidx{l+1}$ within $\MDP\levidx{l+1}$, where here the closeness between the two policies means that the derived value functions of the two policies, within the MDPs they are of, are $\delta\levidx{l}_{\text{refine}}$-close to each other in the $\infty$-norm.
\end{enumerate}
Then, to solve $\MDP$ by solving the MMDP $\{\MDP\levidx{l}\}_{l=1}^L$ using the MMDP solver, the student needs at most $N(\MDP\levidx{l},\Err^{l,\min}_{\mathrm{init}}, \Err^{l,\max}_{\mathrm{init}}, \epsilon\levidx{l}/2)+1$ iterations at level $l$ for $l\in [L]$, where 
\begin{equation} \label{e:error-define-L}
\Err^{L,\min}_{\mathrm{init}}(s) = \Err^{L,\max}_{\mathrm{init}}(s):=V\levidx{L}_{\mathrm{init}}(s)-V\levidx{L}_{\ast}(s)\,,
\end{equation}
and 
\begin{equation} \label{e:error-define}
\Err^{l,\max}_{\mathrm{init}}(s)=-\Err^{l,\min}_{\mathrm{init}}(s):=
\begin{cases}   
0 &,  s\in \Sterm\\
\frac{\epsilon\levidx{l+1}}{2}+\delta\levidx{l}+\delta\levidx{l}_{\text{refine}} &,\text{otherwise}
\end{cases}\,
\end{equation}
for $l\in [L-1]$. 
\end{theorem}

Therefore, the total number of iterations needed by the MMDP solver is at most $\sum_{l=1}^L [N(\MDP\levidx{l},\Err^{l,\min}_{\mathrm{init}}, \Err^{l,\max}_{\mathrm{init}}, \epsilon\levidx{l}/2)+1]$, and usually, $N(\MDP\levidx{L},\Err^{L,\min}_{\mathrm{init}}, \Err^{L,\max}_{\mathrm{init}}, \epsilon\levidx{L}/2)+1$ is much smaller than the number of iterations needed for solving the original MDP $\MDP\levidx{1} = \MDP$ directly, because of Prop.~\ref{p:number-iteration}, and the observation that usually $\cA\levidx{L}$ consists of policies at much larger timescales, and thus the discount factor function $\Gamma\levidx{L}$ takes values much smaller than the ones taken by $\Gamma$ in $\MDP$, guaranteeing a much faster convergence. As for the number of iterations at the remaining levels, they are usually negligible considering the fact that $V\levidx{l}_{\text{init}}$ is already very close to $V\levidx{l}_{\ast}$ ($l\in [L-1]$) thanks to the good initializations derived from the optimal policies at higher levels.

Another comment is about the assumption over the refinement steps within the MMDP solver in Thm.~\ref{t:complexity}. We often observe that when we refine, to the finer level, any policy of the higher-level MDP that is close to the optimal policy, the behavior at the higher-level (in terms of both transition between states and selections of actions) is similar to the behavior at the finer level, but this correspondence is not exact, for example because  the higher-level discount factors are averages of the finer-level discount factors, and also because the translation of higher-level policies to the lower level forces the introduction of additional actions inside the $\endactionset$.
We can choose the upper bound $\delta\levidx{l}_{\text{refine}}$ to be sufficiently not small in order to let the assumption over the refinement step be satisfied. 
If we desired to make the correspondence closer to exact, we could set the penalty for selecting actions in $\endactionset$ to be small.
We do not do this in our numerical examples, in order to let the influence of the actions in $\endactionset$ be more apparent.

Last but not least, albeit both the assumption over the MMDP construction and the assumption over the refinement steps within the MMDP solver in Thm.~\ref{t:complexity} are written in terms of closeness between value functions, essentially they could be relaxed to closeness between policies, which are equivalent to closeness between the ``gradients'' of value functions and usually easier to satisfy.

\subsection{Computational gains with transfer learning}

\label{s:transfer-learning-computation}

In our framework, we also incorporate transfer learning utilizing function composition/\\decomposition (see Sec.~\ref{s:transfer learning}), and recall that there are two such opportunities: ($\romannumeral1$) the first one is the skill-embedding composition to provide a suitably well-initialized policy for the most compressed MDP $\MDP\levidx{L}$; ($\romannumeral2$) the second one is the skill-embedding generator composition to generate the action set $\cA\levidx{l}$ for some $l\in [L]\setminus\{1\}$. We provide results for both opportunities in this section. For transfer learning opportunity ($\romannumeral1$), we state the following result: 

\begin{theorem} \label{t:transfer-1}
Assume that Assump.~\ref{a:regularity} holds and, for $l\in [L]$, we set the stopping criterion for solving $\MDP\levidx{l} $as in our algorithm: the student stops learning if the value functions obtained from two consecutive iterations are $\epsilon\levidx{l}$-close to each other in the $\infty$-norm. 

If transfer learning opportunity ($\romannumeral1$) occurs, then, to finish solving such an MDP of difficulty $L$ using an MMDP solver, within which the MDP solver for the highest-level MDP $\MDP\levidx{L}$ is value iteration, we have the reduction of iterations at least $N\levidx{L}_{\text{org}}-N(\MDP\levidx{L},-\|(V\levidx{L}_{\mathrm{init}})'-V\levidx{L}_{\ast}\|,\|(V\levidx{L}_{\mathrm{init}})'-V\levidx{L}_{\ast}\|, \epsilon\levidx{L}/2)-1$ compared with the case without any transfer learning when solving $\MDP\levidx{L}$, where $N\levidx{L}_{\text{org}}$ is the number of iterations it takes for value iteration to converge to the optimal policy in the most compressed MDP $\MDP\levidx{L}$ (e.g. $\pi\levidx{L}_{\text{init}}$ is set to be the diffusive policy), and $(V\levidx{L}_{\mathrm{init}})'$ is the value function derived from the new initial policy $(\pi\levidx{L}_{\text{init}})'$ after transfer learning. 
\end{theorem}
This is a direct consequence Thm.~\ref{t:complexity}. Thanks to transfer learning, the new initial policy $(\pi\levidx{L}_{\text{init}})'$ is usually very close to the optimal policy $\pi\levidx{L}_{\ast}$, so $\|(V\levidx{L}_{\text{init}})'-V\levidx{L}_{\ast}\|$ is usually small, leading to $N(\MDP\levidx{L},-\|(V\levidx{L}_{\mathrm{init}})'-V\levidx{L}_{\ast}\|,\|(V\levidx{L}_{\mathrm{init}})'-V\levidx{L}_{\ast}\|, \epsilon\levidx{L}/2)+1\ll N\levidx{L}_{\text{org}}$. Therefore, great reduction could be achieved when solving $\MDP\levidx{L}$. On the other hand, the changes in the number of iterations when solving the other MDPs are usually negligible, because those MDPs are not changed, and the value functions/policies the MDP solvers started from are usually suitably well-initialized.

For transfer learning opportunity ($\romannumeral2$), we state the following result: 
\begin{theorem} \label{t:transfer-2}
With the assumptions and notations of Thm.~\ref{t:complexity}, and with the extra assumption that transfer learning opportunity ($\romannumeral2$) occurs when constructing $\MDP\levidx{l}$ for some $l\in [L]\setminus\{1\}$,
we have the reduction of iterations at least $N\levidx{l}_{\text{org}}-N((\MDP\levidx{l})',(\Err^{l,\min}_{\mathrm{init}})', (\Err^{l,\max}_{\mathrm{init}})',$ $\epsilon\levidx{l}/2)-1$ compared with the case without any transfer learning when solving $\MDP\levidx{l}$, where $N\levidx{l}_{\text{org}}$ is the number of iterations it takes for value iteration to converge to the optimal policy in $\MDP\levidx{l}$, $(\MDP\levidx{l})'$ is the new MDP with transfer learning occurs when generating the action set, and $(\Err^{l,\min}_{\mathrm{init}})', (\Err^{l,\max}_{\mathrm{init}})'$ are defined in \eqref{e:error-define-L} if $l = L$, and in \eqref{e:error-define} if $l<L$, for the new MDP $(\MDP\levidx{l})'$.
\end{theorem}

This is a direct consequence of Thm.~\ref{t:complexity}. By using the skill-embedding generator composition for generating the new action set $(\cA\levidx{l})'$, the actions inside $(\cA\levidx{l})'$ typically have larger timescales, and thus the discount factors in $(\MDP\levidx{l})'$ will be smaller, so usually $N((\MDP\levidx{l})',(\Err^{l,\min}_{\mathrm{init}})', (\Err^{l,\max}_{\mathrm{init}})', \epsilon\levidx{l}/2)+1\leq N'((\MDP\levidx{l})',(\Err^{l,\min}_{\mathrm{init}})', (\Err^{l,\max}_{\mathrm{init}})', \epsilon\levidx{l}/2)+1\ll N\levidx{l}_{\text{org}}$ according to Prop.~\ref{p:number-iteration}. On the other hand, the changes in the number of iterations when solving the other MDPs may only occur when solving finer MDPs, which are usually negligible, because those MDPs are not changed, and the value functions/policies the MDP solvers started from are usually suitably well-initialized.

Besides this, we also have a more accurate description of the local updates reflecting the computational savings brought by transfer learning opportunity ($\romannumeral2$) as the following Thm.~\ref{t:transfer-3}, which is a direct consequence of Lemma~\ref{l:error-bound}, by taking $\MDP=\MDP\levidx{l}$, and $\epsilon^{\min} = \epsilon^{\max}=\Err\levidx{l}_{i}$. 
\begin{theorem} \label{t:transfer-3}
Assume that Assump.~\ref{a:regularity} holds, and we assume that the MDP solver for $\MDP\levidx{l}:=(\cS\levidx{l},(S^{\text{init}})\levidx{l},(S^{\text{end}})\levidx{l},\cA\levidx{l},P\levidx{l},R\levidx{l},\Gamma\levidx{l})$ is value iteration for some $l\in [L]\setminus\{1\}$. Then, 
\begin{align*}
& \sum_{s'\in\cS\levidx{l}}P\levidx{l}(s,a\levidx{l}_\ast(s),s')\Gamma\levidx{l}(s,a\levidx{l}_\ast(s),s')\Err\levidx{l}_{i}(s') \\ & \leq \Err\levidx{l}_{i+1}(s)\leq \sum_{s'\in\cS\levidx{l}}P\levidx{l}(s,a\levidx{l}_{i+1}(s),s')\Gamma\levidx{l}(s,a\levidx{l}_{i+1}(s),s')\Err\levidx{l}_{i}(s')\,,
\end{align*}
where $\Err\levidx{l}_{i}:=V\levidx{l}_{i}-V\levidx{l}_{\ast}$, and $V\levidx{l}_{i}$, $a\levidx{l}_\ast(s)$, $a\levidx{l}_{i+1}(s)$ are defined for $\MDP\levidx{l}$, in a similar manner to how $V_{i}$, $a_\ast(s)$, $a_{i+1}(s)$ respectively are defined for $\MDP$ as in Lemma~\ref{l:error-bound}.
\end{theorem}

Thm.~\ref{t:transfer-3} demonstrates the power of transfer learning: thanks to the skill-embedding generator composition, many policies with larger timescales may appear in the action set $\cA\levidx{l}$, and thus if $a\levidx{l}_\ast(s), a\levidx{l}_{i+1}(s)$ are among them, then probably $\Gamma\levidx{l}(s,a\levidx{l}_\ast(s),s')$ are small for most $s'\in\cS\levidx{l}$ such that $P\levidx{l}(s,a\levidx{l}_\ast(s),s')>0$, and $\Gamma\levidx{l}(s,a\levidx{l}_{i+1}(s),s')$ are small for most $s'\in\cS\levidx{l}$ such that $P\levidx{l}(s,a\levidx{l}_{i+1}(s),s')>0$, bounding the absolute value of $\Err\levidx{l}_{i+1}(s)$. Therefore, the convergence of $V\levidx{l}_{i}(s)$ towards $V\levidx{l}_{\ast}(s)$ is sped up.

A second factor that speeds up convergence is that, at some states (e.g., those ``around'' the terminal states), the values are already pretty close to the true ones. With many policies with larger timescales appearing in the action set $\cA\levidx{l}$, such accurate value information could be propagated pretty far to many other states within one single iteration, and thus the same level of accuracy may cover the whole state space much faster. For instance, in the example of navigation and transportation with traffic jams, thanks to the transfer of the two navigation skills $\overline{\pi}^{\mathrm{nav}_{n}}_{\mathrm{obstacles}} (n\in [2])$, the actions in the actions sets $\overline{\Pi\levidx{1}}$ of $\MDP\levidx{2}_{\kappa}$ at level $2$ can easily move the agent from $\cur$ to $\cur'$ as long as they are not separated by roads with traffic jams, even if they are pretty far away from each other. Therefore, the accurate value information can be quickly spread around. This would not be possible without transfer learning opportunity ($\romannumeral2$): in such a case, an action can only move the agent from its current location to its neighboring location or even still at its current location. 

To illustrate these two factors respectively, we visualize the discount factors at the actions selected by the optimal policies, as well as the times of convergence to the optimal value functions for $\{\MDP\levidx{2}_{2,n}\}_{n=1}^{n_2}$, both represented by heat maps, in the example of navigation and transportation with traffic jams as in Figs.~\ref{f:navigation_discount}--\ref{f:navigation_time} respectively. We also do a sanity check, by examining the actual errors of the value functions returned by the value iteration algorithm at each iteration as well as the lower and upper bounds for them we provided in Thm.~\ref{t:transfer-3}, and see whether the actual errors always fall within the ranges, which is the case, and the smallest observed differences between the error and the upper and lower bounds are of the order $10^{-9}$. We also check the tightness of the bounds in Thm.~\ref{t:transfer-3}, by plotting the relative error in the logarithmic scale as in Fig.~\ref{f:sanity_check}.

We conclude this section by providing the following corollary, which is a direct consequence of Thm.~\ref{t:transfer-3} and formalizes the analysis above: 
\begin{corollary} \label{c:transfer-3}
With the assumptions and notations of Thm.~\ref{t:transfer-3}, we have
\begin{align*}
|\Err\levidx{l}_{i+1}(s)|\leq & \max\left\{\max_{s'\in\cS\levidx{l}, P\levidx{l}(s,a\levidx{l}_\ast(s),s')>0}\Gamma\levidx{l}(s,a\levidx{l}_\ast(s),s'), \max_{s'\in\cS\levidx{l}, P\levidx{l}(s,a\levidx{l}_{i+1}(s),s')>0}\Gamma\levidx{l}(s,a\levidx{l}_{i+1}(s),s')\right\} \\ \times &\max_{s'\in\cS\levidx{l}, P\levidx{l}(s,a\levidx{l}_{i+1}(s),s')>0}|\Err\levidx{l}_{i}(s')|\,.
\end{align*}
\end{corollary}

\section{Related work}
\label{s:related-work}

We expand our discussion of related work in the introduction -- of course the related literature is vast, and we restrict ourselves here to mentioning and comparing with some of the techniques we feel are most related to our proposal in this work.

\paragraph{Hierarchical RL (classic and deep).}
Classic HRL formalizes temporal abstraction via options and SMDPs \citep{SuttonOptions}, value decomposition in MAXQ \citep{DietterichHRL}, and hierarchies of abstract machines \citep{ParrRussell1997}; see \citet{BartoHRLReview} for a survey. Early “feudal” ideas cast hierarchy as managers setting subgoals for workers \citep{DayanHinton1993}. Deep HRL automates parts of this stack: option-critic learns intra-option policies and termination end-to-end \citep{Bacon2017OptionCritic}, FeUdal Networks separate goal setting from control in a learned latent space \citep{Vezhnevets2017FeUdal}, HIRO corrects off-policy bias for hierarchical goal relabeling \citep{Nachum2018HIRO}, and HAC combines hindsight with multi-level goals \citep{Levy2019HAC}. \emph{These systems improve sample efficiency in sparse-reward tasks}, yet many fix hierarchical depth or entangle low-level stochasticity with high-level planning; goal-space specification can be brittle \citep{Dwiel2019GoalSpace}. Our approach repeatedly compresses families of lower-level policies into single abstract actions at higher levels, producing MDPs with less stochasticity that are easier to solve.

\paragraph{Skill discovery and pretraining.}
Unsupervised methods learn versatile primitives without task rewards, e.g., DIAYN maximizes skill–state mutual information \citep{Eysenbach2019DIAYN} and VIC pursues empowerment-like objectives \citep{Gregor2016VIC}. When composed for long horizons, such skills can reintroduce stochasticity and complicate credit assignment; our compression reduces variance at higher levels and aligns the skill library with planning on specific objectives.

\paragraph{Abstraction, representation, and safe compression.}
Abstraction further motivates multi-level compression. Spectral state representations capture environment geometry and enable multiscale reasoning \citep{MahadevanMaggioni:JMLR:07, Machado2017Laplacian}. Model minimization and homomorphisms formalize when two models or states are ``equivalent'' for decision making, giving principled conditions for safe compression \citep{DeanGivan1997, LiWalshLittman2006, RavindranBarto2003}. Our method leverages these ideas operationally: each compression step \emph{preserves the semantics} of the original MDP while shrinking variance and branching, so that solving a long-horizon task reduces to solving a stack of smaller, cleaner MDPs with existing algorithms \citep{Puterman}.

\paragraph{Curriculum learning.}
Curriculum learning provides another pillar. Classical curricula gradually increase task difficulty to ease optimization \citep{Bengio2009Curriculum}, with modern variants such as teacher--student curricula \citep{Matiisen2017TSCL}, reverse curricula from goal states \citep{Florensa2017RCG}, and (a-)symmetric self-play that automatically generates tasks \citep{Sukhbaatar2017ASP, Sukhbaatar2018HSP}. However, many curricula define difficulty by \emph{time to solve} rather than \emph{time to learn}, and often restrict the final task to a concatenation of subtasks. In our framework, \emph{difficulty emerges from the compression}: higher-level problems are coarsened in space/time but more global in scope, producing a curriculum that mirrors how humans solve complex tasks---learn sub-skills, compress them into abstract actions, then plan at the higher scale, improving both optimization and transfer on sparse-reward domains.

\paragraph{Transfer and modularity.}
Our transfer perspective is equally explicit. Successor-feature methods decouple dynamics from rewards and enable generalized policy improvement across tasks with shared structure \citep{Barreto2017SF, Barreto2018DeepSF}; policy sketches encourage modular sub-policies that recombine across tasks \citep{Andreas2017PolicySketches}; and progressive nets avoid catastrophic forgetting by lateral connections \citep{Rusu2016ProgressiveNets}. We factor policies into \emph{embeddings} (problem-specific perception/featurization) and \emph{skills} (including reusable higher-order functions), and then compress families of such policies into single abstract actions at higher levels. This creates \emph{transfer across levels and across MDPs}, even when state spaces differ, without resorting to rote replay of previously seen trajectories. As a cautionary contrast, Go-Explore achieves remarkable exploration by explicitly remembering cells and returning before exploring \citep{Ecoffet2021GoExplore}; while highly effective, this mechanism can resemble \emph{memorization of the state space} and requires additional robustification for stochasticity, whereas our abstraction focuses on reusing \emph{semantics} rather than stored states.

\paragraph{Meta-reinforcement learning.}
Because multi-level structure is a prior over \emph{families} of problems, our work connects to meta-RL: RL$^2$ trains recurrent agents to adapt quickly \citep{Duan2016RL2}, MAML learns initializations for fast policy updates \citep{Finn2017MAML}, PEARL infers latent task variables for rapid adaptation \citep{Rakelly2019PEARL}, and MLSH co-trains shared sub-policies with task-specific controllers \citep{Frans2018MLSH}. Our compression-and-skill factorization gives a constructive route to meta-generalization: higher levels expose slower, more stable dynamics, while lower levels encapsulate fast feedback, allowing efficient cross-task reuse with fewer iterations and cheaper per-iteration computation.

\paragraph{Factored MDPs and factored action structure.}
Factored Markov decision processes (FMDPs) are \emph{flat} MDPs whose dynamics and rewards admit a compact structured representation: the state is a tuple of variables, transitions are encoded by a two-slice dynamic Bayesian network (DBN), and rewards often decompose additively over small scopes \citep{BoutilierDeardenGoldszmidt1995,BoutilierDeardenGoldszmidt2000SDPFactored,Degris2013FMDPChapter}. 
This conditional-independence structure---including context-specific independence \citep{BoutilierFriedmanGoldszmidtKoller1996CSI}---can yield exponential reductions in representation size and supports specialized solution methods: symbolic dynamic programming with algebraic decision diagrams and related decision-diagram representations \citep{BaharFrohmHachtelMaciiPardoSomenzi1997ADD,Hoey1999SPUDD,StaubinHoeyBoutilier2000APRICODD}; structured value determination and policy iteration using compact policy representations \citep{KollerParr1999FactoredValue,KollerParr2000PolicyIterationFactored,KimDean2003NonHomogeneousPartitions}; and approximate methods that combine local basis functions with variable elimination and optimization \citep{GuestrinKollerParrVenkataraman2003JAIR,GuestrinKollerParr2001MaxNorm,deFariasVanRoy2003LPADP,deFariasVanRoy2004ConstraintSampling,SchuurmansPatrascu2001DirectValueApprox,PoupartBoutilierPatrascuSchuurmans2002PiecewiseLinear}. 
Decision-theoretic planning work further connects DBN-style representations to DP backups/regression \citep{BoutilierDeanHanks1999DTP}, and online planning/heuristic-search variants exploit symbolic structure (e.g., symbolic LAO$^\ast$ and sRTDP) \citep{FengHansen2002sLAO,FengHansenZilberstein2003sRTDP}. 
Finally, beyond planning-centric assumptions of a known factorization, there is an RL literature on learning and exploiting factorization (including learning unknown structure and regret guarantees) \citep{GivanLeachDean2000BPMDP,KearnsKoller1999EfficientRLFMDP,GuestrinPatrascuSchuurmans2002Exploration,ChitnisSilverVeness2020CAMPs,XuTewari2020RLFactoredUnknown,StrehlDiukLittman2007StructureLearning,OsbandVanRoy2014NearOptimal,TianQianSra2020MinimaxFactoredRL,XuTewari2020OracleEfficient,GuoWeiLuo2021OracleEfficientUnknownStructure,ChenWangJiang2021FMDPBF}.
Importantly for positioning, \emph{``factored'' is not ``hierarchical''}: FMDPs exploit \emph{structural} independence within a single-timescale model, whereas HRL and our multi-level framework primarily exploit \emph{temporal} abstraction. 
Our use of factorization is also different: 
rather than assuming a pre-specified factored (two-slice DBN) model of the environment as in FMDPs, we impose structure on the action/policy side: we decompose \emph{policies} into composable partial policies together with skill/embedding generators, enabling reusable skill extraction and transfer across tasks and abstraction levels.
This distinction is especially sharp on the action side. Classic FMDP work sometimes assumes factored actions (or action variables) and develops structured backups under correlated action effects \citep{Boutilier1997CorrelatedActionEffects}, whereas our framework uses a tensor-product structure of the action set (when present) mainly to \emph{compose} learned partial policies and to support modular transfer; see also related work on compositional / embedded / structured action representations \citep{Khardon2012Learning,Raghavan2013Embedded,Cui2015CompositionalActions}. 
Moreover, for generalization we allow higher-level action sets to be constructed using \emph{subcomponents} of action factors, and in particular to form a \emph{proper subset} of the full Cartesian product; this is crucial in our multi-level setting because higher-level action sets produced by compression are typically not full products.

\paragraph{Inverse reinforcement learning and imitation.}
Finally, our framework is compatible with inverse reinforcement learning (IRL) and imitation. Foundational methods (IRL by \citealt{NgRussell2000IRL}; apprenticeship learning by \citealt{AbbeelNg2004Apprenticeship}; Maximum-Entropy IRL by \citealt{Ziebart2008MaxEntIRL}) recover rewards from demonstrations; modern deep variants learn cost/reward implicitly (GCL \citep{Finn2016GCL}), through adversarial imitation (GAIL \citep{HoErmon2016GAIL}) or adversarial IRL (AIRL \citep{Fu2018AIRL}). Hierarchical IRL further segments demonstrations to recover sub-task rewards \citep{Krishnan2016HIRL}. Because each compression step yields an independent, semantically preserved MDP, we can invert the process: estimate rewards or subgoal structures at appropriate levels, then learn skills and curricula consistent with demonstrations, improving sample-efficiency and interpretability relative to flat IRL \citep{Adams2022IRLSurvey}.

\paragraph{Extended discussion and comparison of our work with \cite{Sukhbaatar2018LearningGE}.}
The MazeBase+ example is a more complex variant of MazeBase that appears in much of the literature, e.g., \cite{Sukhbaatar2018LearningGE}, which is of difficulty $2$ in our regime. In MazeBase+, we require an additional level of abstraction/complexity, because we add multiple rooms, with different geometries and arrangements/order of doors and keys. This MazeBase+ example shows sufficiently how planning like humans is realized by our framework, and our framework is in some sense a mathematical realization of the idea in \cite{Sukhbaatar2018LearningGE}, and actually a generalization because \cite{Sukhbaatar2018LearningGE} only considers MDPs of difficulties up to $2$, where the student Bob and the manager Charlie are analogous to the student at level $1$ and $2$ respectively in our framework. Another main advantage of our framework is that thanks to the introduction of partial policy generators, embedding generators and skills, we allow for very flexible transfer learning opportunities, not restricted to only concatenation as in \cite{Sukhbaatar2018LearningGE}. See Sec.~\ref{s: MazeBase-literature} for detailed discussions.

\section{Concluding remarks and future directions}

\paragraph{Summary of the main message.}
This paper argues that \emph{multi-level structure} can be exploited in a principled and operational way by repeatedly compressing an MDP so that a \emph{parametric family of policies at one level becomes single abstract actions at the next}. The resulting stack of compressed MDPs preserves the semantics of the original problem while progressively reducing stochasticity and effective branching at higher scales. Together with a factorization of policies into embeddings and reusable skills (implemented via partial policies/embedding generators/skills), this yields a concrete mechanism for constructing skill-based curricula and enabling transfer \emph{across tasks and across levels} without relying on rote memorization.

A key observation is that the teacher's initial supervision can be viewed as providing \emph{only the connections}---i.e., which auxiliary MDPs, skills, or pieces of knowledge should compose toward the target task, and how information should flow across levels. Most of the computational reduction in our framework stems from exploiting this connectivity through multi-level compression; this is expected in light of ``no free lunch'' principles, since any systematic speedup must ultimately come from additional structure rather than from universally improving all problems \citep{WolpertMacready1997NFL}. Empirically, these design choices improve both the number of iterations and the per-iteration cost when solving sparse-reward problems, while keeping the framework compatible with standard dynamic programming and reinforcement learning solvers.

In the exploration-centric regime, an additional source of reduction can come from Cartesian-product structure in the action sets: when actions admit factorization (or when the feasible action sets are \emph{proper subsets} of Cartesian products), one can exploit this structure to reduce exploration and generalize across related actions \citep{KearnsKoller1999FactoredRL, Chandak2019ActionRepr}. See the follow-up paper.

\paragraph{Directions for further research.}
A first, immediate direction is to develop algorithms for environments that require \emph{exploration} within our framework. While our current experiments highlight settings where compressed MDPs can be solved efficiently (e.g., via planning/value iteration), many real domains require learning transition and reward structure from interaction. A natural next step is to design multi-level exploration strategies that estimate the induced higher-level transitions/rewards, combining model-free temporal-difference updates \citep{WatkinsDayan1992Qlearning} with model-based planning and optimistic exploration principles \citep{Sutton1990Dyna, BrafmanTennenholtz2002Rmax}. Our follow-up paper develops algorithmic instantiations of these ideas (e.g., under Q-learning).

A second direction is to \emph{learn the connections} rather than receiving them from the teacher. In the follow-up paper, we plan to formalize \emph{virtual policies} that allow the student to discover which auxiliary MDPs or previously learned skills are useful for a target MDP, thereby constructing minimal curricula in a semi-supervised manner and enabling further context-dependent transfer when the teacher does not provide a well-ordered curriculum upfront. Relatedly, it is natural to consider two-staged (and for multiple levels, multi-staged) learning schemes in which policies and the induced abstract actions of compressed MDPs are learned \emph{alternately}, akin in spirit to classical alternating-latent-variable estimation procedures (e.g., EM) \citep{Dempster1977EM}. Such alternation raises interesting questions about identifiability, stability, and when to refine abstractions online.

A third direction is to broaden the set of domains beyond grid-world navigation to \emph{algorithmic} and \emph{program-like} tasks that naturally admit recursive decompositions, such as sorting an array recursively or evaluating arithmetic expressions recursively \citep{WhiteMartinezRudolph2010RP, ReedDeFreitas2016NPI}. In particular, character drawing offers a compelling testbed for multi-level abstraction: classical work shows that human-like concept learning and drawing can be modeled via probabilistic program induction \citep{LakeSalakhutdinovTenenbaum2015PPI}, and a growing line of research learns reusable program abstractions and motor programs for drawing from data \citep{EllisRitchieSolarLezamaTenenbaum2018Graphics, TianEllisTenenbaum2020Drawing, HewittTenenbaum2020MWS, Ellis2021DreamCoder}. These lines suggest that the ``skills'' required for drawing are naturally compositional and reusable, aligning closely with our emphasis on multi-level compression, function composition, and cross-task transfer.

Beyond these, several longer-horizon directions appear particularly promising. (i) \emph{Inverse RL and imitation.} Since each compression step yields an independent semantically preserved MDP, one can infer rewards or subgoal structure at the most appropriate level and then learn transferrable skills consistent with demonstrations, connecting to classical and modern IRL/imitation paradigms \citep{NgRussell2000IRL, AbbeelNg2004Apprenticeship, Ziebart2008MaxEntIRL, HoErmon2016GAIL, Fu2018AIRL}. (ii) \emph{Natural language interfaces.} Many of our abstractions (hierarchy, curriculum, transfer via skill composition) admit crisp descriptions in natural language, suggesting strong synergy with language-conditioned planning/control and language-augmented robotic decision making \citep{Ichter2023SayCan}. (iii) \emph{Multi-agent RL and mean-field limits.} Multi-level structure and the notion of ``neighbors'' can be integrated with mean-field perspectives, potentially yielding scalable decompositions for large-population decision making \citep{HuangMalhameCaines2006MFG, LasryLions2007MFG, Yang2018MeanFieldMARL}. (iv) \emph{Broader applications.} The same abstraction/transfer/curriculum principles may benefit robotics and autonomy \citep{KoberBagnellPeters2013IJRR, Grigorescu2020AutonomousDrivingSurvey} as well as automated/formal theorem proving \citep{PoluSutskever2020GPTf, Bansal2019HOList, Yang2023LeanDojo, Wu2021INT}. Taken together, these directions reinforce our main message: multi-level compression plus compositional skills provide a general, semantically grounded route to scalable planning, reasoning, learning, and transfer across tasks.

\bigskip
\begin{acks}
We thank Y. Kevrekidis and F. Lu, A. Basu, Y. Sire for helpful discussions related to this work. MM is grateful for partial support from DOE-255223, FA9550-20-1-0288, FA9550-23-1-0445, NSF-1837991, NSF-1913243, and the Simons Fellowship. SY is grateful for partial support from AMS-Simons travel grant. Prisma Analytics, Inc. provided computing equipment and support. SY is grateful for B. Dong's host while visiting Peking University, as part of the research was conducted during the visit. 
\end{acks}

\newpage
\let\normalsize\small
\appendix
\small

\section{Theoretical results and proofs}
\label{s:appendixB}

\subsection{Analytical results for multi-level compression} 
\label{s:MMDP-compress}
The results in this section are analytical and allow the efficient construction of higher-level transition probabilities, rewards and discount factors from a current level, essentially by solving linear systems.
More specifically, we derive the closed-form formulas for $P\levidx{l+1}(s,\pi\levidx{l},s')$, $R\levidx{l+1}(s,\pi\levidx{l},s')$, $\Gamma\levidx{l+1}(s,\pi\levidx{l},s')$ given the finer MDP $\MDP\levidx{l}=(\cS,\Sinit,\Sterm,\cA\levidx{l},P\levidx{l},R\levidx{l},\Gamma\levidx{l})$ and a policy $\pi\levidx{l}$ on $\cSA\levidx{l}$. Before stating the three propositions, we first introduce the following notations. 

Given a policy $\pi\levidx{l}$, we compute its corresponding policy-specific Markov transition matrix by averaging out the actions in $\cA\levidx{l}$ according to $\pi\levidx{l}$:
\begin{align}   \label{e:Markov-matrix}
P^{\pi\levidx{l}}(s,s'): =\sum_{a\in \cA\levidx{l}(s)} P\levidx{l}(s,a,s')\pi\levidx{l}(s,a)\,, 
\end{align}
and if we restrict the domain in the sum above to $(\cA^l)^{\mathtt{end}}$, then we have:
\begin{align}  \label{e:Markov-matrix-restrict}
P^{\pi\levidx{l}}_{(\cA^l)^{\mathtt{end}}}(s,s'): =\sum_{a\in (\cA^l)^{\mathtt{end}}} P\levidx{l}(s,a,s')\pi\levidx{l}(s,a) = \sum_{a\in (\cA^l)^{\mathtt{end}}}\indic_{\{s\}}(s')\pi\levidx{l}(s,a)\,, 
\end{align}
where the second equality follows from $P\levidx{l}(s,a,s')=\indic_{\{s\}}(s')$ for any $a\in (\cA^l)^{\mathtt{end}}$. 

For $X=X(s,a,s')$, $Y=Y(s,a,s')$, indexed by $s,s'\in \cS$, $a\in \cA\levidx{l}(s)$, we define the matrix $(X\circ Y)^{\pi\levidx{l}}$ to be the expectation w.r.t. the extended policy $\pi\levidx{l}$ of the Hadamard product between $X$ and $Y$: 
\begin{equation*} 
[(X\circ Y)^{\pi\levidx{l}}]_{s,s'}:=\sum_{a\in \cA\levidx{l}(s)} X(s,a,s')Y(s,a,s')\pi\levidx{l}(s,a)\,.
\end{equation*}
Note that $(X\circ Y)^{\pi\levidx{l}}=(Y\circ X)^{\pi\levidx{l}}$.
Now we can state the three propositions.

\begin{proposition} \label{p:compress-transition}
For the finer MDP $\MDP\levidx{l}=(\cS,\Sinit,\Sterm,\cA\levidx{l},P\levidx{l}, R\levidx{l},\Gamma\levidx{l})$ and the policy $\pi\levidx{l}$ on $\cSA\levidx{l}$, we have $P\levidx{l+1}(s,\pi\levidx{l},s')=H_{s,s'}$, for all $s,s'\in \cS$, where $H$ is the minimal non-negative solution to the linear system
\begin{align} \label{e:compress-transition-matrix}
[1-(1-\frac{1}{t_{\pi\levidx{l}}})(P^{\pi\levidx{l}}-P^{\pi\levidx{l}}_{(\cA^l)^{\mathtt{end}}})]H=(1-\frac{1}{t_{\pi\levidx{l}}})P^{\pi\levidx{l}}_{(\cA^l)^{\mathtt{end}}}+\frac{P^{\pi\levidx{l}}}{t_{\pi\levidx{l}}}\,.
\end{align}
\end{proposition}

\begin{proposition} \label{p:compress-reward}
With the same settings as in Prop.~\ref{p:compress-transition}, we have $R\levidx{l+1}(\cdot,\pi\levidx{l},s')=h_{s'}$, for all $s'\in \cS$, where $h$ is the (unique, bounded) solution, for each $s'\in \cS$, to the linear system
\begin{equation}
\label{e:compress-reward-matrix}
\begin{aligned}  
[1+(1-\frac{1}{t_{\pi\levidx{l}}})(P^{\pi\levidx{l}}_{(\cA^l)^{\mathtt{end}}}-(P_{s'}^{\pi\levidx{l}}\circ\Gamma\levidx{l})^{\pi\levidx{l}})]h_{s'}
=&(1-\frac{1}{t_{\pi\levidx{l}}})[(P_{s'}^{\pi\levidx{l}}\circ R\levidx{l})^{\pi\levidx{l}}\underline{1}-r\levidx{l}P^{\pi\levidx{l}}_{(\cA^l)^{\mathtt{end}}}\underline{1}\\ 
&+r\levidx{l}[P\levidx{l+1}(\cdot,\pi\levidx{l},s')]_{\diag}^{-1}P^{\pi\levidx{l}}_{(\cA^l)^{\mathtt{end}}}v_{s'}]+\frac{(\overline{P}_{s'}^{\pi\levidx{l}}\circ R\levidx{l})^{\pi\levidx{l}}v_{s'}}{t_{\pi\levidx{l}}}\,,
\end{aligned}
\end{equation}
where 
$
P_{s'}^{\pi\levidx{l}}(s,a,s''):=\frac{P\levidx{l}(s,a,s'')P\levidx{l+1}(s'',\pi\levidx{l},s')}{P\levidx{l+1}(s,\pi\levidx{l},s')}\,,\quad\text{and}\quad\overline{P}_{s'}^{\pi\levidx{l}}(s,a,s''):=\frac{P\levidx{l}(s,a,s'')}{P\levidx{l+1}(s,\pi\levidx{l},s')}\,,
$
$\underline{1}$ is a vector whose $|\{s\in\cS:P\levidx{l+1}(s,\pi\levidx{l},s')>0\}|$ coordinates are all ones, $v_{\diag}$ is a diagonal matrix whose diagonal elements are in $v$, $v_s$ is a vector whose $|\cS|$ coordinates are all zeros except for the position corresponding to the state $s$, whose value is one. 
\end{proposition}

\begin{proposition} \label{p:compress-discount}
With the same settings as in Prop.~\ref{p:compress-transition}, we have $\Gamma\levidx{l+1}(\cdot,\pi\levidx{l},s')=h_{s'}$, for all $s'\in \cS$, where $h$ is the minimal non-negative solution, for each $s'\in \cS$, to the linear system
\begin{align}  \label{e:compress-discount-matrix}
[1+(1-\frac{1}{t_{\pi\levidx{l}}})(P^{\pi\levidx{l}}_{(\cA^l)^{\mathtt{end}}}-(P_{s'}^{\pi\levidx{l}}\circ\Gamma\levidx{l})^{\pi\levidx{l}})]h_{s'}=(1-\frac{1}{t_{\pi\levidx{l}}})[P\levidx{l+1}(\cdot,\pi\levidx{l},s')]^{-1}_{\diag}P^{\pi\levidx{l}}_{(\cA^l)^{\mathtt{end}}}v_{s'}+\frac{(\overline{P}_{s'}^{\pi\levidx{l}}\circ\Gamma\levidx{l})^{\pi\levidx{l}}v_{s'}}{t_{\pi\levidx{l}}}\,.
\end{align}
\end{proposition}

\noindent\textbf{Proof of Prop.~\ref{p:compress-transition}.} 
First, according to the Markov property, we have $P\levidx{l+1}(s,\pi\levidx{l},s')=H_{s,s'}$, for all $s,s'\in \cS$, where $H$ is the minimal non-negative solution, for each $s'\in \cS$, to the linear system
\begin{align*} 
&[(1-\frac{1}{t_{\pi\levidx{l}}})\sum_{a\in(\cA^l)^{\mathtt{end}}}\pi\levidx{l}(s,a)+1]H_{s,s'}\\&=\quad(1-\frac{1}{t_{\pi\levidx{l}}})\left(\sum_{s''\in \cS}P^{\pi\levidx{l}}(s,s'')H_{s'',s'}+\sum_{a\in(\cA^l)^{\mathtt{end}}}\pi\levidx{l}(s,a)\indic_{\{s\}}(s')\right)+\frac{P^{\pi\levidx{l}}(s,s')}{t_{\pi\levidx{l}}}\,.
\end{align*}
Transforming this to matrix-vector form, we obtain \eqref{e:compress-transition-matrix}.

\noindent\textbf{Proof of Prop.~\ref{p:compress-reward}.} 
First, according to the Markov property, we have $R\levidx{l+1}(s,\pi\levidx{l},s')=H_{s,s'}$, for all $s,s'\in \cS$, where $H$ is the (unique, bounded) solution, for each $s'\in \cS$, to the linear system
\begin{align*} \nonumber
&[1+(1-\frac{1}{t_{\pi\levidx{l}}})\sum_{a\in (\cA^l)^{\mathtt{end}}}\pi\levidx{l}(s,a)]H_{s,s'}\\ \label{compress-reward-scalar}&=(1-\frac{1}{t_{\pi\levidx{l}}})[\sum_{s''\in\cS}\sum_{a\in\cA\levidx{l}(s)}P_{s'}^{\pi\levidx{l}}(s,a,s'')\Gamma\levidx{l}(s,a,s'')\pi\levidx{l}(s,a)H_{s'',s'}\\&+\sum_{s''\in\cS}\sum_{a\in\cA\levidx{l}(s)}P_{s'}^{\pi\levidx{l}}(s,a,s'')R\levidx{l}(s,a,s'')\pi\levidx{l}(s,a)\\&-r\levidx{l}\sum_{a\in(\cA^l)^{\mathtt{end}}}\pi\levidx{l}(s,a)+\frac{\r\levidx{l}}{P\levidx{l+1}(s,\pi\levidx{l},s')}\sum_{a\in(\cA^l)^{\mathtt{end}}}\pi\levidx{l}(s,a)\indic_{\{s\}}(s')]\\&+\frac{1}{t_{\pi\levidx{l}}}\sum_{a\in\cA\levidx{l}(s)}\overline{P}_{s'}^{\pi\levidx{l}}(s,a,s')R\levidx{l}(s,a,s')\pi\levidx{l}(s,a)\,.
\end{align*}
Transforming this to matrix-vector form, we have \eqref{e:compress-reward-matrix}.

\noindent\textbf{Proof of Prop.~\ref{p:compress-discount}.} 
First, according to the Markov property, we have $\Gamma\levidx{l+1}(s,\pi\levidx{l},s')=H_{s,s'}$, for all $s,s'\in \cS$, where $H$ is the minimal non-negative solution, for each $s'\in \cS$, to the linear system
\begin{align*} 
&[1+(1-\frac{1}{t_{\pi\levidx{l}}})\sum_{a\in (\cA^l)^{\mathtt{end}}}\pi\levidx{l}(s,a)]H_{s,s'}\\ &\nonumber=(1-\frac{1}{t_{\pi\levidx{l}}})[\sum_{s''\in\cS}\sum_{a\in\cA\levidx{l}(s)}P_{s'}^{\pi\levidx{l}}(s,a,s'')\Gamma\levidx{l}(s,a,s'')\pi\levidx{l}(s,a)H_{s'',s'}\\&+\frac{1}{P\levidx{l+1}(s,\pi\levidx{l},s')}\sum_{a\in(\cA^l)^{\mathtt{end}}}\pi\levidx{l}(s,a)\indic_{\{s\}}(s')]\\&+\frac{1}{t_{\pi\levidx{l}}}\sum_{a\in\cA\levidx{l}(s)}\overline{P}_{s'}^{\pi\levidx{l}}(s,a,s')\Gamma\levidx{l}(s,a,s')\pi\levidx{l}(s,a)\,.
\end{align*}
Transforming this to matrix-vector form, we have \eqref{e:compress-discount-matrix}.

\subsection{Proofs of the theoretical results in Sec.~\ref{s:theory}}
This section proves the theoretical results of in Sec.~\ref{s:theory}, including Lemma~\ref{l:error-bound} and Thm.~\ref{t:complexity}.

\noindent\textbf{Proof of Lemma~\ref{l:error-bound}.} 
The proof follows the main idea behind proving the contraction property of the Bellman optimality operator \cite{10.5555/528623,10.5555/1396348}:
\begin{align*} 
V_{i+1}(s)-V_{\ast}(s)
=&\sum_{s'\in\cS}P(s,a_{i+1}(s),s')[R(s,a_{i+1}(s),s')+\Gamma\levidx{l}(s,a_{i+1}(s),s')V_i(s')]\\
&-\sum_{s'\in\cS}P(s,a_\ast(s),s')[R(s,a_\ast(s),s')+\Gamma\levidx{l}(s,a_\ast(s),s')V_\ast(s')]\\ 
\geq & \sum_{s'\in\cS}P(s,a_\ast(s),s')[R(s,a_\ast(s),s')+\Gamma\levidx{l}(s,a_\ast(s),s')V_i(s')]\\
&-\sum_{s'\in\cS}P(s,a_\ast(s),s')[R(s,a_\ast(s),s')+\Gamma\levidx{l}(s,a_\ast(s),s')V_\ast(s')]\\ 
\geq &\sum_{s'\in\cS}P(s,a_\ast(s),s')\Gamma(s,a_\ast(s),s')\epsilon^{\min}(s') = \Err^{\min}(s)\,,
\end{align*}
where the first equality follows from the Bellman updates behind the value iteration and the optimal value function respectively, and the next two inequalities follow from the optimality of the Bellman update used in value iteration and the definition of $\Err^{\min}$ respectively.
Similarly, we prove the other side of the inequality:
\begin{align*} 
V_{i+1}(s)-V_{\ast}(s)
=&\sum_{s'\in\cS}P(s,a_{i+1}(s),s')[R(s,a_{i+1}(s),s')+\Gamma\levidx{l}(s,a_{i+1}(s),s')V_i(s')]\\
&-\sum_{s'\in\cS}P(s,a_\ast(s),s')[R(s,a_\ast(s),s')+\Gamma\levidx{l}(s,a_\ast(s),s')V_\ast(s')]\\ 
\leq & \sum_{s'\in\cS}P(s,a_{i+1}(s),s')[R(s,a_{i+1}(s),s')+\Gamma\levidx{l}(s,a_{i+1}(s),s')V_i(s')]\\
&-\sum_{s'\in\cS}P(s,a_{i+1}(s),s')[R(s,a_{i+1}(s),s')+\Gamma\levidx{l}(s,a_{i+1}(s),s')V_\ast(s')]\\ 
\leq &\sum_{s'\in\cS}P(s,a_{i+1}(s),s')\Gamma(s,a_{i+1}(s),s')\epsilon^{\max}(s') = \Err^{\max}(s)\,.
\end{align*}

\noindent\textbf{Proof of Thm.~\ref{t:complexity}.}
The proof follows directly from Cor.~\ref{c:number-iteration}, as well as from applying the triangular inequalities to bound above $\|V\levidx{l}_{i+1}-V\levidx{l}_{i}\|$ for $l\in [L]$ and $i=N(\MDP\levidx{l},\Err^{l,\min}_{\mathrm{init}}, \Err^{l,\max}_{\mathrm{init}}, \epsilon\levidx{l}/2)$, and $\|V\levidx{l}_{\text{init}}-V\levidx{l}_{\ast}\|$ for $l\in [L-1]$ respectively:

For $l\in [L]$,
\begin{align*}
\|V\levidx{l}_{i+1}-V\levidx{l}_{i}\| \leq & \|V\levidx{l}_{i+1}-V\levidx{l}_{\ast}\| + \|V\levidx{l}_{i}-V\levidx{l}_{\ast}\| =  \|\mathcal{T}\levidx{l}(V\levidx{l}_{i})-\mathcal{T}\levidx{l}(V\levidx{l}_{\ast})\| + \|V\levidx{l}_{i}-V\levidx{l}_{\ast}\|\\ \leq & \|V\levidx{l}_{i}-V\levidx{l}_{\ast}\| + \|V\levidx{l}_{i}-V\levidx{l}_{\ast}\| \leq \frac{\epsilon\levidx{l}}{2}+\frac{\epsilon\levidx{l}}{2} = \epsilon\levidx{l}\,,
\end{align*}
where $V\levidx{l}_{i}$ is the value function at the $i$-th iteration of the value iteration when solving $\MDP\levidx{l}$, the second inequality follows from generalized proofs of the contraction property of the Bellman optimality operator \cite{10.5555/528623,10.5555/1396348} (generalized in the sense that $\Gamma\levidx{l}$ is not a constant function), and the last inequality follows from Cor.~\ref{c:number-iteration}.

Similarly, for $l\in [L-1]$, 
\begin{align*}
\|V\levidx{l}_{\text{init}}-V\levidx{l}_{\ast}\| \leq & \|V\levidx{l}_{\text{init}}-V\levidx{l+1}_{\text{VI}}\|+\|V\levidx{l+1}_{\text{VI}}-V\levidx{l+1}_{\ast}\|+\|V\levidx{l+1}_{\ast}-V\levidx{l}_{\ast}\| \leq \delta\levidx{l}_{\text{refine}}+\frac{\epsilon\levidx{l+1}}{2}+\delta\levidx{l}\,,
\end{align*}
where $V\levidx{l}_{\text{VI}}$ is the value function obtained from the value iteration when solving $\MDP\levidx{l}$, and in our setting it equals $V\levidx{l}_{i}$ with $i = N(\MDP\levidx{l},\Err^{l,\min}_{\mathrm{init}}, \Err^{l,\max}_{\mathrm{init}}, \epsilon\levidx{l}/2)+1$, and the second inequality follows from the new assumptions (the first and third ones) listed in this theorem plus Cor.~\ref{c:number-iteration}.

\section{Algorithmic realization}

\subsection{Learning with MMDPs}
\label{s:MMDP-algorithm}

We have the following Alg.~\ref{algorithm:MMDP-solver} as well as the auxiliary Algs.~\ref{algorithm:generate}--\ref{algorithm:MDP-solver}. 
The inputs of Alg.~\ref{algorithm:MMDP-solver} are mostly as discussed in Sec.~\ref{s:MMDP-def}. There are two exceptions: (1) for simplification, the teacher does not provide the sequence of finite partial policy generator sets $\{\cG\levidx{l}_{\text{test}}\}_{l=1}^\infty$, and provides instead the difficulty $L$ of $\MDP$ directly, which is set as an extra property of $\MDP$. Consequently, the teacher only needs to provide $\{\cG\levidx{l}\}_{l=1}^{L-1}$, the first $L-1$ elements of the infinite sequence $\{\cG\levidx{l}\}_{l=1}^\infty$. (2) The timescales of policies could be reduced to the timescales of partial policies and then further reduced to partial policy generators by assuming: ($\romannumeral1$), the timescale of a policy equals the minimum of the timescales of partial policies producing it using outer-product as in \eqref{e:def-otimes}; ($\romannumeral2$) the timescale of a partial policy equals the timescale of the partial policy generator generating it. In this way, the timescales of partial policy generators are provided instead, which are set as an extra property of partial policy generators. 
Optionally, the teacher also provides $t_{\min}$ and $t_{\max}$, a lower bound and an upper bound of such timescales in a curriculum respectively, which are attributed to $t_{\bound}$ as in Alg.~\ref{algorithm:MMDP-solver}.

The algorithm consists of two main parts: constructing the MMDP following the recursive definition in Sec.~\ref{s:MMDP-def} and solving each MDP in it from the top most compressed MDP down til the bottom finest original MDP. 

For completeness, we also add the error detection mechanism in Alg.~\ref{algorithm:MMDP-solver} according to the criteria $\{\thresh\levidx{l}\}_{l=1}\levidx{L} = \{(N_{\max}\levidx{l},T_{\max}\levidx{l},v_{\min}\levidx{l})\}_{l=1}\levidx{L}$, where for the $l$-th level MDP $\MDP\levidx{l}$, $N_{\max}\levidx{l}$ is the upper bound on the number of allowed iterations before convergence; $T_{\max}\levidx{l}$ is the maximum number of steps per episode upon convergence, averaged across episodes upon convergence from possibly multiple initial states, and $T_{\max}\levidx{l}$ is set to $+\infty$ for non-episodic $\MDP\levidx{l}$ by default; $v_{\min}\levidx{l}$ is the lower bound for the value of initial states, possibly averaged across multiple initial states. Meanwhile, we record the corresponding actual values as $\mathrm{stats}\levidx{l}= (N\levidx{l},T\levidx{l},v\levidx{l})$ when the student solves $\MDP\levidx{l}$ for $l\in[L]$. If for some $l\in [L]$, one of the actual values are out of the three thresholds in $\thresh\levidx{l}$, then the output error message $\err$ is $\{\mathrm{exist}:1,\mathrm{level}:l\}$, indicating that there is an error occurring at level $l$, so the algorithm stops there. Otherwise, all levels of MDPs are solved successfully, and the output error message $\err$ is $\{\mathrm{exist}:0,\mathrm{level}:1\}$, indicating that there is no error across all levels and the algorithm finishes at level one.

\begin{algorithm}
\begin{algorithmic}[1]
\Require $\MDP$: the original MDP of difficulty $L$; $\{\cG\levidx{l}\}_{l=1}^{L-1}$: the first $L-1$ elements of {\policysequence}; $t_{\bound}$: bounds for the timescale of any policy; $\{r\levidx{l}\}_{l=1}^{L}$: negative rewards on choosing actions in $(\cA^l)^{\mathtt{end}}$ for $l\in[L]$; $\mathrm{solver}\_\mathrm{init}\levidx{L}$: initial value function and policy for the most compressed MDP $\MDP\levidx{L}$; $\{\mathrm{solver}\_\mathrm{end}\levidx{l}\}_{l=1}\levidx{L}$: stopping criteria for solver of $\MDP\levidx{l}$ for $l\in[L]$, possibly including error tolerance of solved value function or maximum number of iterations; $\{\thresh\levidx{l}\}_{l=1}\levidx{L}$: thresholds for error detection on $\MDP\levidx{l}$ for $l\in[L]$.
\Ensure $\{\MDP\levidx{l}\}_{l=1}^{L},\{\mathrm{solution}\levidx{l}\}_{l=1}^{L}$, $\{\mathrm{stats}\levidx{l}\}_{l=1}\levidx{L}$: the MMDP, the solutions of the MDPs in it, and summary of statistics when solving the MDPs; $\err$: error information.
\Statex
\State \textbf{Initialize:} $\MDP\levidx{1} = \MDP$, $\mathrm{solution}\levidx{l} = \NA$ for $l\in [L]$, $\mathrm{stats}\levidx{l} = \NA$ for $l\in [L]$
\For{$l=1,2,\cdots, L-1$}
    \State $\cA\levidx{l+1}=\mathtt{generate}(\MDP\levidx{l},\cG\levidx{l})$
    \State $\MDP\levidx{l+1}=\mathtt{compress}(\MDP\levidx{l},\cA\levidx{l+1},r\levidx{l+1})$
\EndFor
\For{$l=L,L-1,\cdots, 1$}
    \State $(\mathrm{solution}\levidx{l},\mathrm{stats}\levidx{l},\err\levidx{l}) = \mathtt{Solve\_MDP}(\MDP\levidx{l},t_{\bound},\mathrm{solver}\_\mathrm{init}\levidx{l}, \mathrm{solver}\_\mathrm{end}\levidx{l},\thresh\levidx{l})$
    \If{$\err\levidx{l}.\mathrm{exist} = 1$}
        \State \Return $(\{\MDP\levidx{l'}\}_{l'=1}^{L},\{\mathrm{solution}\levidx{l'}\}_{l'=1}^{L},\{\mathrm{stats}\levidx{l'}\}_{l'=1}\levidx{L}, \{\mathrm{exist}:1,\mathrm{level}:l\})$
    \EndIf
    \If{$l>1$}
        \State Compute $\mathrm{solver}\_\mathrm{init}\levidx{l-1}$ from $\mathrm{solution}\levidx{l}$ and $\cG\levidx{l-1}$ using \eqref{e: convolution}
    \EndIf
\EndFor
\State \Return $(\{\MDP\levidx{l}\}_{l=1}^{L},\{\mathrm{solution}\levidx{l}\}_{l=1}^{L},\{\mathrm{stats}\levidx{l}\}_{l=1}\levidx{L}, \{\mathrm{exist}:0,\mathrm{level}:1\})$
\end{algorithmic}
\caption[Solve an MDP from top to bottom]{Solve an MDP from top to bottom:\,\, $(\{\MDP\levidx{l}\}_{l=1}^{L},\{\mathrm{solution}\levidx{l}\}_{l=1}^{L},\{\mathrm{stats}\levidx{l}\}_{l=1}\levidx{L},\err)=\mathtt{solve\_MMDP}(\MDP,\{\cG\levidx{l}\}_{l=1}^{L-1}, t_{\bound}, \{r\levidx{l}\}_{l=1}\levidx{L}, \mathrm{solver}\_\mathrm{init}\levidx{L},\{\mathrm{solver}\_\mathrm{end}\levidx{l}\}_{l=1}\levidx{L},\{\thresh\levidx{l}\}_{l=1}\levidx{L})$}
\label{algorithm:MMDP-solver}
\end{algorithm}

\begin{algorithm}
\begin{algorithmic}[1]
\Require $\MDP$: the MDP; $\cG$: the partial policy generator set, with domains of generators in it being $\Theta_1,\Theta_2,\cdots, \Theta_M$.
\Ensure $\widetilde{\cA}$: the action set.
\Statex
\If{$\cG = \varnothing$}
    \State $\cG = \MDP.\cA$
\EndIf
\State $\mathrm{partitions\_list}=\mathtt{list\_all\_partitions}(\MDP.\cA, \cG)$
\State $\Pi = \varnothing$
\For{each $\mathrm{partition}\in\mathrm{partitions\_list}$}
    \State Generate policies $\Pi_{\mathrm{new}}$ using outer product as in \eqref{e:def-otimes}
    \State Set timescales of policies in $\Pi_{\mathrm{new}}$ as $\min_{g|_I\in \mathrm{partition}}\{g|_I.\mathrm{timescale}\}$
    \State Add $\Pi_{\mathrm{new}}$ to $\Pi$
\EndFor
\State $\widetilde{\cA} = \Pi\cup\big(\Pi_{m=1}^M (\Theta_m\cup \endactionfactor)-\Pi_{m=1}^M \Theta_m\big)$
\end{algorithmic}
\caption{Generate the action set $\widetilde{\cA}=\mathtt{generate}(\MDP,\cG)$}
\label{algorithm:generate}
\end{algorithm}

\begin{algorithm}
\begin{algorithmic}[1]
\Require $\MDP$: the finer MDP; $\widetilde{\cA}$: the action set for the compressed MDP; $\widetilde{r}$: negative reward for choosing actions in $(\widetilde{\cA})^{\mathtt{end}}$ within the compressed MDP.
\Ensure $\widetilde{\MDP}$: the compressed MDP.
\Statex
\State $\widetilde{\MDP}.\cS=\MDP.\cS$
\State $\widetilde{\MDP}.\Sinit=\MDP.\Sinit$
\State $\widetilde{\MDP}.\Sterm=\MDP.\Sterm$
\State $\widetilde{\MDP}.\cA=\widetilde{\cA}$
\State $\widetilde{\MDP}.\cSA=\widetilde{\MDP}.\cS\times\widetilde{\cA}$
\For{each $\pi\in\widetilde{\cA}-(\widetilde{\cA})^{\mathtt{end}}$}
    \State Compute matrix $P^{\pi}$,$P^{\pi}_{\endactionset}$ using \eqref{e:Markov-matrix}, \eqref{e:Markov-matrix-restrict}
    \State Solve for matrix $\widetilde{\MDP}.P(\cdot,\pi,\cdot)$ using \eqref{e:compress-transition-matrix}
    \For{each $s'\in\cS$}
        \State Solve for vector $\widetilde{\MDP}.R(\cdot,\pi,s')$ using \eqref{e:compress-reward-matrix}, with $r\levidx{l} = \MDP.R(\cdot,a,\cdot)$ for any $a\in\endactionset$
        \State Solve for vector $\widetilde{\MDP}.\Gamma(\cdot,\pi,s')$ using \eqref{e:compress-discount-matrix}
    \EndFor
\EndFor
\State Set $P(s,a,s') = \indic_{\{s\}}(s'), R(s,a,s') = \widetilde{r}, \Gamma(s,a,s') = 1$ for \textbf{any} $a\in(\widetilde{\cA})^{\mathtt{end}}$
\end{algorithmic}
\caption{Compress an MDP $\widetilde{\MDP}=\mathtt{compress}(\MDP,\widetilde{\cA},\widetilde{r})$}
\label{algorithm:compress}
\end{algorithm}

\begin{algorithm}
\begin{algorithmic}[1]
\Require $\MDP$: the MDP to be solved; $t_{\bound}$: bounds for the timescale of any policy; $\mathrm{solver}\_\mathrm{init}$: initial value function and policy for $\MDP$; $\mathrm{solver}\_\mathrm{end}$: stopping criteria for solver of $\MDP$; $\thresh$: thresholds for error detection on $\MDP$.
\Ensure $\mathrm{solution}$: the solution of $\MDP$, possibly containing a learned policy and a value function associated to a learned policy; $\mathrm{stats}$: summary of statistics when solving $\MDP$; $\err$: error information.
\Statex
\State \textbf{Initialize}: $N=0,\mathrm{solution}=\mathrm{solver}\_\mathrm{init},\mathrm{solution}_{\prev}=\mathrm{solver}\_\mathrm{init},\err.\mathrm{exist}=1$ 
\While{$0<N<\mathrm{solver}\_\mathrm{end}.N_{\max}$ \textbf{and} $\|\mathrm{solution}.V-\mathrm{solution}_{\prev}.V\|_{\infty}>\mathrm{solver}\_\mathrm{end}.\eps$}
    \State $\mathrm{solution}_{\prev}=\mathrm{solution}$
    \State $\mathrm{solution}=\mathtt{MDP\_solve\_update}(\MDP,\mathrm{solution})$
    \State $N = N+1$
\EndWhile
\State \textbf{Update} the distributions of $\mathrm{solution}.\pi$ at $\MDP.\Sterm$ to be $\indic_{\{\endaction\}}(\cdot)$
\State $\mathrm{stats} = \{N:N, T:+\infty, v:\frac{\sum_{s\in\Sinit}(\mathrm{solution}.V)(s)}{|\Sinit|}\}$
\If{$N\leq\mathrm{thresh}.N_{\max}$ \textbf{and} $\|(\mathrm{solution}.V-\mathrm{solution}_{\prev}.V)\mid_{\Sinit}\|_{\infty}\leq\mathrm{solver}\_\mathrm{end}.\eps$ \textbf{and} $\mathrm{stats}.V\geq \thresh.v_{\min}$ \textbf{and} $t_{\bound}.{\max} > t_{\bound}.{\min}$}
    \State $((\mathrm{solution}.\pi).T,\err)=\mathtt{compute\_timescale}(\MDP,\mathrm{solution}.\pi,t_{\bound}, \thresh.T_{\max})$
    \State $\mathrm{stats}.T = (\mathrm{solution}.\pi).T$
\EndIf
\If{$\err.\mathrm{exist} = 1$}
    \State $\mathrm{solution} = \NA$
\EndIf
\end{algorithmic}
\caption{Solve an MDP $(\mathrm{solution},\mathrm{stats},\err) =$\\$ \mathtt{Solve\_MDP}(\MDP,t_{\bound},\mathrm{solver}\_\mathrm{init},\mathrm{solver}\_\mathrm{end},\thresh)$}
\label{algorithm:MDP-solver}
\end{algorithm}

\subsection{Transfer learning}
\label{s:appendix:transferlearning}
We collect here Algs.~\ref{algorithm:learn-MDPs}--\ref{algorithm:learn-MDP} described at high level in Sec.~\ref{s:transfer-learning-algorithm}, where the teacher, student and assistant cooperate in solving an MMDP.

\begin{algorithm}
\begin{algorithmic}[1]
\Require $\{\MDP_{L,n}\}_{L=1,n=1}^{L_{\max},n_L}$, $\{\hint_{L,n}\}_{L=1,n=1}^{L_{\max},n_L}$: a series of MDPs ordered by difficulty as well as correspondingly a series of hints provided by the teacher with detailed descriptions as given in Alg.~\ref{algorithm:learn-MDP}; $t_{\bound}$: bounds for the timescale of any policy and skill; $\{\underline{\mathrm{solver\_end}_{L,n}}\}_{L=1,n=1}^{L_{\max},n_L}$: a series of stopping criteria for solvers of MDPs in the MMDP constructed from $\MDP_{L,n}$ for $L\in [L_{\max}], n\in [n_L]$; $\{\underline{\mathrm{thresh}_{L,n}}\}_{L=1,n=1}^{L_{\max},n_L}$: a series of thresholds for error detection on MDPs in the MMDP constructed from $\MDP_{L,n}$ for $L\in [L_{\max}], n\in [n_L]$.
\Ensure $\{\underline{\mathrm{solution}_{L,n}}\}_{L=1,n=1}^{L_{\max},n_L},\{\err_{L,n}\}_{L=1,n=1}^{L_{\max},n_L}$: the solutions of the MDPs in the series of MMDPs constructed from $\{\MDP_{L,n}\}_{L=1,n=1}^{L_{\max},n_L}$ and the series of error information.
\Statex
\State \textbf{Initialize:} $\Skill=\{\mathrm{id}\}$
\For{$L=1,2,\cdots, L_{\max}$}
    \For{$n=1,2,\cdots, n_L$}
        \State $(\underline{\mathrm{solution}_{L,n}},\err_{L,n})$
        \Statex \hspace*{3em}$=\mathtt{learn\_MDP}(\MDP_{L,n},\hint_{L,n},t_{\bound},\underline{\mathrm{solver\_end}_{L,n}},\underline{\mathrm{thresh}_{L,n}})$
    \EndFor
\EndFor
\end{algorithmic}
\caption[Learn a curriculum]{Learn a curriculum $(\{\underline{\mathrm{solution}_{L,n}}\}_{L=1,n=1}^{L_{\max},n_L},\{\err_{L,n}\}_{L=1,n=1}^{L_{\max},n_L})=\mathtt{learn\_curriculum}\\(\{\MDP_{L,n}\}_{L=1,n=1}^{L_{\max},n_L},\{\hint_{L,n}\}_{L=1,n=1}^{L_{\max},n_L},t_{\bound},\{\underline{\mathrm{solver\_end}_{L,n}}\}_{L=1,n=1}^{L_{\max},n_L}, \{\underline{\mathrm{thresh}_{L,n}}\}_{L=1,n=1}^{L_{\max},n_L})$}
\label{algorithm:learn-MDPs}
\end{algorithm}

\begin{algorithm}[t]
\begin{algorithmic}[1]
\Require $\MDP$: the original MDP of difficulty $L$ and number $n$; $\hint$: hints provided by the teacher with five fields: the first field $\cG$ of $\hint$ helps the student derive action sets in compressed MDPs, which contains a sequence with length $L-1$ of skill-embedding generator pair sets; a second field $r$ contains a sequence with length $L$ of negative rewards for choosing actions in $(\cA^l)^{\mathtt{end}}$ for $l\in [L]$; a third field skill $\overline{\pi}$ and a fourth field embedding $e_{\comp}$ help the student compose the initial policy for the most compressed MDP $\MDP\levidx{L}$; the last field $e_{\decomp}$ helps the assistant extract out new skills from the optimal policies for the MDPs, which contains a sequence with length $L$ of embeddings; $t_{\bound}$: bounds for the timescale of any skill; $\{\mathrm{solver\_end}\levidx{l}\}_{l=1}^{L}$: stopping criteria for solver of $\MDP\levidx{l}$ for $l\in[L]$; $\{\thresh\levidx{l}\}_{l=1}^{L}$: thresholds for error detection on $\MDP\levidx{l}$ for $l\in[L]$.
\Ensure the solutions $\{\mathrm{solution}\}_{l=1}^{L}$ of the MDPs in the MMDP constructed from the original MDP and the error information $\err$.
\Statex
\State \textbf{Initialize:} $\MDP\levidx{1}=\MDP$
\For{$l=1,2,\cdots, L-1$}
    \State $\cG\levidx{l}=\varnothing$
    \For{each $(\overline{g},E)\in \hint.\cG(l)$}
        \If{$\overline{g} \neq \NA$}
            \State $\cG\levidx{l}=\cG\levidx{l}\cup\mathtt{compose}((\overline{g},E),\MDP\levidx{l}.\cSA)$
        \EndIf
    \EndFor
    \State $\cA\levidx{l+1}= \mathtt{generate}(\MDP\levidx{l},\cG\levidx{l})$
    \State $\MDP\levidx{l+1}.\cSA=\MDP.\cS\times \cA\levidx{l+1}$, $\MDP\levidx{l+1}.\mathrm{difficulty}=L-l$
\EndFor
\If{$\hint.\overline{\pi} \neq \NA$ \textbf{and} $\hint.e_{\comp} \neq \NA$}
    \State $\mathrm{solver\_init}.\pi = \mathtt{compose}(\hint.\overline{\pi},\hint.e_{\comp},\MDP\levidx{L}.\cSA)$
\Else
    \State \textbf{Set} $\mathrm{solver\_init}.\pi$ to be the diffusive policy with uniform distribution at each state
\EndIf
\State $(\{\mathrm{solution}\}_{l=1}^{L},\sim,\sim,\err)$
\Statex $=\mathtt{solve\_MMDP}(\MDP,\{\cG\levidx{l}\}_{l=1}^{L-1},\hint.r, \mathrm{solver\_init},\{\mathrm{solver\_end}\levidx{l}\}_{l=1}^{L},\{\thresh\levidx{l}\}_{l=1}^{L})$
\For{$l=1,2,\cdots, L$}
    \If{$(\hint.e_{\decomp})(l)\neq \NA$ \textbf{and} $\mathrm{solution}\levidx{l}\neq\NA$}
        \State $\Skill = \Skill\cup\{\mathtt{decompose}(\mathrm{solution}\levidx{l}.\pi,(\hint.e_{\decomp})(l),t_{\bound})\}$
    \EndIf
\EndFor
\end{algorithmic}
\caption[Learn an MDP]{Learn an MDP $ (\{\mathrm{solution}\}_{l=1}^{L},\err)=\\ \mathtt{learn\_MDP}(\MDP,\hint,t_{\bound},\{\mathrm{solver\_end}\levidx{l}\}_{l=1}^{L},\{\thresh\levidx{l}\}_{l=1}^{L})$}
\label{algorithm:learn-MDP}
\end{algorithm}

\section{Analytical realization of algorithms applied to our examples}
\label{s:appendixB}

In this section, we provide the analytical realization of Algs.\ref{algorithm:learn-MDPs}--\ref{algorithm:learn-MDP} collected in App.~\ref{s:appendix:transferlearning} applied to our examples. We organize this section in the following way: for each example, we first provide inputs (problem settings and MDP definitions etc.) and how different objects (policies, embeddings, skills, etc.) should be constructed when using Algs.\ref{algorithm:learn-MDPs}--\ref{algorithm:learn-MDP}. This section serves as a complement to Sec.~\ref{s:curriculum} and Sec.~\ref{s:analytical-traffic}, in which we describe how the algorithms construct and discover these objects as the algorithms progress, demonstrating their correctness.

\subsection{MazeBase+}
\label{s:all-MazeBase} 
\subsubsection{Geometric configuration and object states}
\label{s:MazeBaseGeometry}
We start with the geometric configuration of this example. With notations as in the example of navigation and transportation with traffic jams, we have a two-dimensional grid world $\gridworld=[x_1,x_2]\times [y_1,y_2]\subseteq\mathbb{N}\times \mathbb{N}$ and a set of actions in different directions, $\cA_{\dir}=$ $\{(1,0)$, $(0,1)$, $(-1,0)$, $(0,-1)\}$. 
This grid world contains blocks, three doors, three keys, and a goal. 
To show the effects of transfer learning in Sec.~\ref{s:MazeBase-transfer} and Fig.~\ref{f:MazeBase2}, we assume the initial locations of all the objects (including the blocks, keys, doors, and the goal) are fixed to certain locations, satisfying the following constraints. 
All the state spaces of the MDPs in Sec.~\ref{s:MazeBase-transfer} and Sec.~\ref{s:MazeBase-robust} also satisfy the constraints given here unless specified:
\begin{enumerate}[leftmargin=0.25cm,itemsep=0pt]
\item[$\cdot$] $s_{\door_1},s_{\door_2},s_{\door_3}\in \gridworld$, the locations of $\door_1$, $\door_2$, and $\door_3$ respectively, satisfy the following conditions: $s_{\door_2}(1)<s_{\door_1}(1)=s_{\door_3}(1),s_{\door_1}(2)<s_{\door_2}(2)<s_{\door_3}(2)$, and $ s_{\door_1}(1) \in\{x_1,x_1+3,\cdots\},s_{\door_2}(2) \in \{y_1,y_1+3,\cdots\}$. 
\item[$\cdot$] the set of locations of the blocks is $\Block=(([x_1,x_2]\times\{s_{\door_2}(2)\})\cup (\{s_{\door_1}(1)\}\times[y_1,y_2]))\setminus\{s_{\door_1},\dots,s_{\door_3}\}$. 
\item[$\cdot$] the blocks together with the doors separate the remaining grid points $\gridworld\setminus(\Block\cup\cup_{i=1}^3 s_{\door_i})$ into four rooms: $\room_1$, $\room_2$, $\cdots$, $\room_4$, with corresponding regions 
$\gridworld_{\room_1}:=\{\loc\in\gridworld:\loc(1)<s_{\door_1}(1),\loc(2)<s_{\door_2}(2)\},
\gridworld_{\room_2}:=\{\loc\in\gridworld:\loc(1)>s_{\door_1}(1),\loc(2)<s_{\door_2}(2)\},
\gridworld_{\room_3}:=\{\loc\in\gridworld:\loc(1)<s_{\door_1}(1),\loc(2)>s_{\door_2}(2)\},
\gridworld_{\room_4}:=\{\loc\in\gridworld:\loc(1)>s_{\door_1}(1),\loc(2)>s_{\door_2}(2)\}$.
\item[$\cdot$] $\agent\in\noblock:=\Omega\setminus\Block$ is the current location of the agent; the agent cannot be at the location of a door unless the door is open. 
\item[$\cdot$] $s_{\key_1},s_{\key_2},s_{\key_3}\in\noblock$ are the locations of $\key_1$, $\key_2$, and $\key_3$ respectively; $\key_i$ opens only $\door_i$, for $i=1,2,3$. 
\item[$\cdot$] $\goal\in\noblock\setminus\cup_{i=1}^3 s_{\door_i}$ is the location of the goal object. 
\end{enumerate}
Our method would allow the locations of the keys and goals to vary, within the above constraints, using an embedding to extract a more general skill which abstracts out the logic of navigation within a single room while avoiding obstacles. 
Even the geometry of the grid world or the locations of blocks could vary; the only necessary and sufficient condition is that each room contains at least one shortest path between any two points. 
We are considering here a simpler setting which suffices to convey our message.

For $i\in [3]$, $s_{{\open}_i}\in \{\sclosed,\sopen\}$ indicates if $\door_i$ is open; $s_{{\pick}_i}\in \{0,1\}$ indicates if the agent has $\key_i$; $\goalpick\in \{0,1\}$ indicates if the agent has in hand the goal object. 
In particular, if $s_{{\pick}_i} = 1$, then $\cur=s_{\key_i}$, because once the agent has picked up the key, the agent will always have the key in hand; similarly, if $\goalpick= 1$, then $\cur=\goal$. 
For the objects in $\{\key_1,\key_2,\key_3, \goals\}$ not initially picked up by the agent, we may assume initially $s_{\key_1},s_{\key_2}\in \gridworld_{\room_1}, s_{\key_3}\in \gridworld_{\room_3}$, and $\goal\in\gridworld_{\room_4}$.

The agent can open $\door_i$ with the action $\open$ if the agent has in hand $\key_i$ and the agent is next to $\door_i$, i.e., $\cur\in\neardoor(s_{\door_i}):=\{s:||s-s_{\door_i}||_1=1\}$; the agent can pick up $\key_i$ with the action $\pick$ if $\cur=s_{\key_i}$; the agent can pick up the goal with the action $\pick$ if $\cur=\goal$; the agent can pick up multiple objects in $\{\key_1,\key_2,\key_3, \goals\}$ in a single time step. 

For ease of notation, we let $\keyvec := (s_{\key_1},s_{\key_2},s_{\key_3})$, $\doorvec := (s_{\door_1},s_{\door_2},s_{\door_3})$, $\keypickvec := (s_{{\pick}_1},s_{{\pick}_2},s_{{\pick}_3})$, $\dooropenvec := (s_{{\open}_1},s_{{\open}_2},s_{{\open}_3})$,
$\gridworld_{\room}:=\gridworld_{\room_1}$, $s_{\key}:= s_{\key_1}$, $s_{\door} := s_{\door_1}$, $s_{\pick} := s_{{\pick}_1}$, $s_{\open} := s_{{\open}_1}$.
In the definition of $\MDP_{3,1}$ in \eqref{e:MDP31-mazebase}, we let $\Omega(s)$ denote the set of all the possible locations of the agent given the objects' locations and states determined by the current state $s$, i.e., for $s=(\agent,\keypickvec,\dooropenvec,\goalpick)\in \cS_{3,1}$, $\Omega(s):=\noBlock\setminus\{s_{\door_i}:s_{\open_i}=0, 1\leq i\leq 3\}$, so that $\agent\in\Omega(s)$ incorporates the restriction that the agent cannot be at the location of a closed door (nor of a block, of course).  
Table~\ref{t:param-MazeBase} summarizes all the parameter values in the highest-difficulty MDPs throughout this example; $\MDP_{1,1}$ is exactly the same as $\MDP_{\text{dense}}^{\text{nav}}$ in the example of navigation and transportation with traffic jams (Sec.\ref{s:examples-MMDP-transfer-learning}); all the other MDPs use the same parameter values as in the corresponding highest-difficulty MDP: for instance, $\MDP_{2,1}$ and $\MDP_{2,2}$ use the same parameter values as in $\MDP_{3,1}$. 
Fig.~\ref{f:MazeBase} represents the geometric configurations of the grid world and the objects in it, see also Sec.~\ref{box:MDPs}. 
Finally, recall that in the definition of an MDP, the set of initial states of the agent, $\Sinit$, is set to be the set of states starting from which the agent can reach the goal; in this example starting the agent in $\Sinit$ means that the location of the agent, and the status of the keys, doors, and the goal are such that the puzzle is indeed solvable.

\setlength{\tabcolsep}{2pt}
\begin{table}[b]
\caption{Parameters in the MazeBase+ example}
\label{t:param-MazeBase}
\begin{center}
\begin{tabular}{|c|c|c|c|c|c|c|c|c|c|c|c|c|c|c|c|c|}
\cline{2-17}
\multicolumn{1}{l|}{}                                                                                                          &  $x_1$     & $x_2$ & $x_3$  & $y_1$ &$y_2$ &$y_3$  & $s_{\key_1}$& $s_{\key_2}$& $s_{\key_3}$& $s_{\door_1}$& $s_{\door_2}$& $s_{\door_3}$& $s_{\goals}$& $p_s$ & $R_0$ & $r_0$ \\ \hline
\multicolumn{1}{|c|}{\begin{tabular}[c]{@{}c@{}}$\MDP_{3,1}$\end{tabular}}       &$1$      &$15$   & $10$ & $1$   &  $8$ &   $4$ & $(1,3)$ & $(1,1)$ & $(1,8)$ & $(10,2)$  & $(9,4)$ & $(10,5)$  & $(15,8)$ & $0.9$ & $10^4$ & $-10$\\ \hline
\multicolumn{1}{|c|}{\begin{tabular}[c]{@{}c@{}}$\MDP'_{3,1}$\end{tabular}}       &$1$      &$15$   & $10$ & $1$   &  $8$ &   $4$ & $(1,1)$ & $(15,1)$ & $(15,8)$ & $(10,3)$  & $(11,4)$ & $(10,5)$  & $(1,8)$ & $0.9$ & $10^4$ & $-10$\\ \hline
\multicolumn{1}{|c|}{\begin{tabular}[c]{@{}c@{}}$\MDP''_{3,1}$\end{tabular}}       &$1$      &$15$   & $10$ & $1$   &  $8$ &   $4$ & $(1,3)$ & $(1,1)$ & $(2,1)$ & $(10,2)$  & $(9,4)$ & $(10,5)$  & $(15,8)$ & $0.9$ & $10^4$ & $-10$\\ \hline
\end{tabular}
\end{center}
\end{table}

\subsubsection{Formal definitions of the MDPs in the curriculum}

\label{s:MazeBaseMDPs}
The formal definitions of these MDPs are as follows, with the definition of $\MDP_{1,1}$, which is exactly $\MDP_{\text{dense}}^{\text{nav}}$ in the example of navigation and transportation with traffic jams, postponed to the forthcoming App.~\ref{s:MDPs-traffic} when we introduce the example of navigation and transportation with traffic jams.
$\MDP_{2,1}:=(\cS_{2,1},\Sinit_{2,1},\Sterm_{2,1},\cA_{2,1},P_{2,1},R_{2,1},\Gamma_{2,1})$ models navigation through $\gridworld$ while avoiding blocks assuming all the doors are open:
\begin{equation}
\begin{aligned}
\label{e:MDP21} 
&\cS_{2,1}		:=\{(\cur,s_{\dest}):\cur,s_{\dest}\in\noBlock\}=\noBlock\times\noBlock\,,\\
&\Sinit_{2,1}	:=\{(\cur,s_{\dest})\in \cS_{2,1}:\cur\neq\dests\}\,,\\
&\Sterm_{2,1}	:=\{(\cur,s_{\dest})\in \cS_{2,1}:\cur=\dests\}\,,\\
&\cA_{2,1}		:=\cA_{\dir}\cup\{\endactionfactor\}\,,\\
& R_{2,1}((\cur,s_{\dest}),a,(\cur',s'_{\dest}))	 :=R_0\mult\indic_{\{s'_{\dest}\}}(\cur')+r_0\mult[1-\indic_{\{s'_{\dest}\}}(\cur')]\,,\\
&\Gamma_{2,1}((\cur,s_{\dest}),a,(\cur',s'_{\dest})):=1\,,
\end{aligned}
\end{equation}
\vspace{-\abovedisplayskip}
\vspace{-\abovedisplayskip}
\begin{align*}
P_{2,1}((\cur,s_{\dest}),a,(\cur',s'_{\dest}))	& :=\indic_{\{s_{\dest}\}}(s'_{\dest}) \mult\big[[1-\indic_{\noblock}(\cur+a)]\mult\indic_{\{\cur\}}(\cur')\\
							&+\indic_{\noblock}(\cur+a)\mult[p_s\mult\indic_{\{\cur+a\}}(\cur')+(1-p_s)\mult\indic_{\{\cur\}}(\cur')]\big]\,.
\end{align*}
In the above, the teacher sets the following quantities: $0<p_s\leq 1$ (set here to $0.9$), the probability for an action to succeed in the intended movement; $R_0>0$ (set here to $10000$), a large positive reward; $r_0<0$ (set here to $-10$), a small negative reward. Also, for brevity, we do not guarantee the functions (such as transition probabilities, rewards, discount factors, policies, and their variants) written down are defined and correct at $s\in\Sterm$ or $s\notin\Sterm, a\in \endactionset$ unless specified. These apply to both examples. 
$\MDP_{2,2}:=(\cS_{2,2},\Sinit_{2,2},\Sterm_{2,2},\cA_{2,2},P_{2,2},R_{2,2},\Gamma_{2,2})$ models picking up a key and opening the door: 
\begin{equation*}
\begin{aligned} 
\cS_{2,2} 		&:= \{(\agent,s_\pick,s_\open): \agent\in\gridworld_{\room}, s_\pick,s_\open\in\{0,1\}\}=\gridworld_{\room}\times\{0,1\}^2\,,\\
\Sinit_{2,2} 	&:= \{(\agent,s_\pick,s_\open)\in \cS_{2,2}:s_\open=0\}\,,\quad
\Sterm_{2,2} 	:= \{(\agent, s_\pick,s_\open)\in \cS_{2,2}:s_\open=1\}\,,\\
\cA_{2,2} 		&:= \cA_{\dir}\cup\{\pick,\open,\endactionfactor\}\,,\\
R_{2,2}&((\agent,s_\pick,s_\open),a,(\agentnext,s'_\pick,s'_\open)) := R_0\mult\indic_{\{1\}}(s'_\open)+r_0\mult\indic_{\{0\}}(s'_\open)\,,\\
\Gamma_{2,2}&((\agent,s_\pick,s_\open),a,(\agentnext,s'_\pick,s'_\open)):=1\,,
\end{aligned} 
\end{equation*}
\vspace{-\abovedisplayskip}
\vspace{-\abovedisplayskip}
\begin{align*}
P_{2,2}&((\agent,s_\pick,s_\open),a,(\agentnext,s'_\pick,s'_\open))\\&:=
\begin{dcases}   
\begin{aligned}
&\indic_{\{s_\pick\}}(s'_\pick)\mult \indic_{\{s_\open\}}(s'_\open)\mult \{p_s\mult\indic_{\gridworld_{\room}}(\agent+a) \mult\indic_{\{\agent+a\}}(\agentnext)\\
&\quad+[1-p_s\mult\indic_{\gridworld_{\room}}(\agent+a)]\mult \indic_{\{\agent\}}(\agentnext)\} 
\end{aligned} &, \text{if}\  a\in\cA_{\dir}\\
\begin{aligned}
& \indic_{\{\agent\}}(\agentnext)  \mult \{\indic_{\{\open\}}(a)\mult\indic_{\{s_\pick\}}(s'_\pick)\\
&\quad\mult \big[p_s\mult\indic_{\neardoor(s_\door)}(\agent)\mult\indic_{\{1\}}(s_\pick)\mult\indic_{\{1\}}(s'_\open)\\
&\quad+[1-p_s\mult\indic_{\neardoor(s_\door)}(\agent)\mult\indic_{\{1\}}(s_\pick)]\mult\indic_{\{s_\open\}}(s'_\open)\big]\\
&\quad+\indic_{\{\pick\}}(a)\mult\indic_{\{s_\open\}}(s'_\open) \mult\big[p_s\mult\indic_{\{s_\key\}}(\agent)\mult\indic_{\{1\}}(s'_\pick)\\
&\quad+[1-p_s\mult\indic_{\{s_\key\}}(\agent)]\mult\indic_{\{s_\pick\}}(s'_\pick)\big]\} 
\end{aligned} &, \text{otherwise.}
\end{dcases}
\end{align*}
Now we move to the target MDP of difficulty $3$. $\MDP_{3,1}:=(\cS_{3,1},\Sinit_{3,1},\Sterm_{3,1},\cA_{3,1},P_{3,1},R_{3,1},\Gamma_{3,1})$ models picking up the goal ($s=(\agent,\keypickvec,\dooropenvec,\goalpick)$, and similarly for $s'$):
\begin{equation}  \label{e:MDP31-mazebase}
\begin{aligned}
\cS_{3,1}&:=\{s\}\,,\\
\Sinit_{3,1}&:=\{s\in \cS_{3,1}: \agent\notin\gridworld_{\room_2}, \goalpick=0\} \cup \{s\in \cS_{3,1}: s_{\open_1} = 1, \goalpick=0\}\,,\\
\Sterm_{3,1}&:=\{s\in \cS_{3,1}:\goalpick=1\}\,,\\
\cA_{3,1}&:=\cA_{\dir}\cup\{\pick,\open, \endactionfactor\}\,,\\
&R_{3,1}(s,a,s'):=R_0\mult\indic_{\{1\}}(\goalpicknext)+r_0\mult\indic_{\{0\}}(\goalpicknext)\,,\\
&\Gamma_{3,1}(s,a,s'):=1\,,
\end{aligned}
\end{equation}
\vspace{-\abovedisplayskip}
\vspace{-\abovedisplayskip}
\begin{align*}
P_{3,1}(s,a,s')&:=
\begin{dcases}
\begin{aligned}   
&\indic_{\{\keypickvec\}}(\keypickvecnext)\mult\indic_{\{\dooropenvec\}}(\dooropenvecnext)\mult\indic_{\{\goalpick\}}(\goalpicknext)\\ 
&\quad\mult [p_s\mult\indic_{\Omega(s)}(\agent+a)\mult\indic_{\{\agent+a\}}(\agentnext)\\ 
&\quad+\big(1-p_s\mult\indic_{\Omega(s)}(\agent+a)\big)\mult\indic_{\{\agent\}}(\agentnext)] 
\end{aligned} &,\text{if}\  a\in \cA_{\dir}\\
\begin{aligned}
&\indic_{\{\agent\}}(\agentnext)\mult\indic_{\{\keypickvec\}}(\keypickvecnext)\mult\indic_{\{\goalpick\}}(\goalpicknext)\mult \{(1-p_s)\mult\indic_{\{\dooropenvec\}}(\dooropenvecnext)\\
&\quad+p_s\mult\prod_{i=1}^3\big[\indic_{\neardoor(s_{\door_i})}(\agent)\mult\indic_{\{1\}}(s_{\pick_i})\mult\indic_{\{1\}}(s'_{\open_i})\\
&\quad+[1-\indic_{\neardoor(s_{\door_i})}(\agent)\mult\indic_{\{1\}}(s_{\pick_i})]\mult\indic_{\{s_{\open_i}\}}(s'_{\open_i})\big]\} 
\end{aligned} &,\text{if}\  a=\open\\
\begin{aligned}
&\indic_{\{\agent\}}(\agentnext)\mult\indic_{\{\dooropenvec\}}(\dooropenvecnext) \mult \{(1-p_s)\mult\indic_{\{\keypickvec\}}(\keypickvecnext)\mult\indic_{\{\goalpick\}}(\goalpicknext)\\
&\quad+p_s\mult\prod_{i=1}^3\big[\indic_{\{s_{\key_i}\}}(\agent)\mult\indic_{\{1\}}(s'_{\pick_i})\\
&\quad+\indic_{\{s_{\key_i}\}^c}(\agent)\mult\indic_{\{s_{\pick_i}\}}(s'_{\pick_i})\big]\\ 
&\quad\mult \big[\indic_{\{\goal\}}(\agent)\mult\indic_{\{1\}}(\goalpicknext)+\indic_{\{\goal\}^c}(\agent)\mult\indic_{\{\goalpick\}}(\goalpicknext)\big]\} 
\end{aligned} &,\text{if}\  a=\pick.
\end{dcases}
\end{align*}
We use $\curs,\dest$ in Fig.~\ref{f:MazeBase} instead of $\cur,\dests$ for simplicity, and similarly for other objects, here and elsewhere.
The teacher has the role to set all the parameters of the problems, including the reward for retrieving $\goals$ or the negative reward for taking one extra step without making progress.

\subsubsection{Compressed MDPs}
\label{s:compressedMDPs-mazebase}

For $\MDP_{2,1}$, the student constructs the second level MDP $\MDP_{2,1}\levidx{2}=(\cS_{2,1},\Sinit_{2,1},\Sterm_{2,1},\overline{\Pi_{2,1}\levidx{1}},P_{2,1}\levidx{2},R_{2,1}\levidx{2},$ $\Gamma_{2,1}\levidx{2})$, with $P_{2,1}\levidx{2}$, $R_{2,1}\levidx{2}$, $\Gamma_{2,1}\levidx{2}$ as follows: 
\begin{equation}  \label{e:compressed-MDP21-MazeBase} 
\begin{aligned}
R_{2,1}\levidx{2}(s,(\pi_{2,1}\levidx{1})_\theta,s') &=
r_{2,1}\levidx{1}+r_0\mult\big[\mathbb{E}[X\mid X \sim \mathrm{NB} ((d_{2,1}\levidx{1})_{\theta}(s,s'),p_s)]+(d_{2,1}\levidx{1})_{\theta}(s,s')\big]\,,\\
\Gamma_{2,1}\levidx{2}(s,(\pi_{2,1}\levidx{1})_\theta,s')&=1\,,\\
P_{2,1}\levidx{2}(s,(\pi_{2,1}\levidx{1})_\theta,s') &= \indic_{\{(d_{2,1}\levidx{1})_{\theta}(s,0)\}}((d_{2,1}\levidx{1})_{\theta}(s,s'))\,,
\end{aligned}
\end{equation}
where $r_{2,1}\levidx{1}<0$ is a negative reward set by the teacher (here equal to $-10$); $(d_{2,1}\levidx{1})_{\theta}:\{(s,s'):s,s'\in\cS_{2,1}\}\rightarrow \mathbb{N}\cup\{+\infty\}$, defined similarly to $(d_{2,2}\levidx{1})_{\theta}$, is a parametric family of distance functions parametrized by $\theta\in \Theta_{2,1}\levidx{1}=\{\door_1,\door_2,\door_3,\dest\}$. 
Given the parameter $\theta$, the agent's starting state $s$ and current state $s'$, the value of $(d_{2,1}\levidx{1})_{\theta}(s,s')$ incorporates the information of $(\pi_{2,1}\levidx{1})_{\theta}$ by defining the infinite sequence $\{s_i\}_{i=0}^\infty$ inductively starting from $s_0 = s$: for $i\in\mathbb{N}$, there exists a unique action $a\in\cA_{2,1}\levidx{1}$ such that $(\pi_{2,1}\levidx{1})_{\theta}(s_i,a)=1$ because $\overline{\pi}^{\mathrm{nav}}_{\text{dense}}$ is assumed to be deterministic, and we let 
\begin{equation*}
s_{i+1}:=
\begin{cases}   
\sup_{s\in\cS_{2,1}-\{s_i\}}P_{2,1}(s_i,a,s) &,\text{if}\  P_{2,1}(s_i,a,s_i)<1\\[0.3cm]
s_i &,\text{otherwise.}
\end{cases}
\end{equation*}
Then, we set $(d_{2,1}\levidx{1})_{\theta}(s,s')$ to be the smallest $i\in\mathbb{N}$ such that $s_i=s'$ if there exists such an $i$, and set it to be $+\infty$ otherwise. Consequently, we denote $(d_{2,1}\levidx{1})_{\theta}(s,0):=\sup\{(d_{2,1}\levidx{1})_{\theta}(s,s'):s'\in\cS_{2,1}, (d_{2,1}\levidx{1})_{\theta}(s,s')<+\infty\}$. 

For $\MDP_{2,2}$, the student constructs the second level MDP $\MDP_{2,2}\levidx{2}=(\cS_{2,2},\Sinit_{2,2}$, $\Sterm_{2,2}$, $\overline{\Pi_{2,2}\levidx{1}}$, $P_{2,2}\levidx{2}$, $R_{2,2}\levidx{2}$, $\Gamma_{2,2}\levidx{2})$, with $P_{2,2}\levidx{2}$, $R_{2,2}\levidx{2}$, $\Gamma_{2,2}\levidx{2}$ as follows: 
\begin{equation} \label{e:compressed-MDP22-MazeBase} 
\begin{aligned}
R_{2,2}\levidx{2}(s,(\pi_{2,2}\levidx{1})^\lambda_\theta,s')=&
\begin{dcases}   
R_0\mult\indic_{\{1\}}(s'_\open)+r_0\mult[1-\indic_{\{1\}}(s'_\open)] &\,,\text{if}\  \lambda=\alpha
\\
r_{2,2}\levidx{1}+r_0\mult\big[\mathbb{E}[X\mid X \sim \mathrm{NB} ((d_{2,2}\levidx{1})_{\theta}(s,s'),p_s)]+(d_{2,2}\levidx{1})_{\theta}(s,s')\big] &\,, \text{if}\  \lambda=\beta,
\end{dcases}\\
\Gamma_{2,2}\levidx{2}(s,(\pi_{2,2}\levidx{1})^\lambda_\theta,s')&=1\,,
\end{aligned}
\end{equation}
\vspace{-\abovedisplayskip}
\vspace{-\abovedisplayskip}
\begin{align*}
P_{2,2}\levidx{2}(s,(\pi_{2,2}\levidx{1})^\lambda_\theta,s')=&
\begin{dcases}
\begin{aligned}   
&\indic_{\{\agent\}}(\agentnext) \mult \{\indic_{\{\open\}}(\theta)\mult\indic_{\{s_\pick\}}(s'_\pick) \\ 
&\quad \mult \big[p_s\mult\indic_{\neardoor(s_\door)}(\agent)\mult\indic_{\{1\}}(s_\pick)\mult\indic_{\{1\}}(s'_\open)\\
&\quad+[1-p_s\mult\indic_{\neardoor(s_\door)}(\agent)\mult\indic_{\{1\}}(s_\pick)]\mult\indic_{\{s_\open\}}(s'_\open)\big]\\
&\quad +\indic_{\{\pick\}}(\theta)\mult\indic_{\{s_\open\}}(s'_\open) \mult\big[p_s\mult\indic_{\{s_\key\}}(\agent)\mult\indic_{\{1\}}(s'_\pick)\\
&\quad +[1-p_s\mult\indic_{\{s_\key\}}(\agent)]\mult\indic_{\{s_\pick\}}(s'_\pick)\big]\} 
\end{aligned}
&\,, \text{if}\  \lambda=\alpha\\
\begin{aligned}
 &\indic_{\{(d_{2,2}\levidx{1})_{\theta}(s,0)\}}((d_{2,2}\levidx{1})_{\theta}(s,s')) 
 \end{aligned}
 &\,,\text{if}\  \lambda=\beta,
\end{dcases}
\end{align*}
where $s=(\agent,s_\pick,s_\open), s'=(\agentnext,s'_\pick,s'_\open)$; $r_{2,2}\levidx{1}<0$ is a negative reward set by the teacher, and here it is set to $-10$; $X \sim \mathrm{NB} (r,p)(r\in\mathbb{Z}^+,0<p\leq 1)$ means $X$ follows negative binomial (or Pascal) distribution with $r$ successes and success probability $p$, with the probability mass function $f_{nb}(k;r,p):=\big(\begin{matrix}k+r-1\\ k\end{matrix}\big)(1-p)^{k}p^r (k\geq 0)$, so $ \mathbb{E}[X\mid X \sim \mathrm{NB} (r,p)]= \frac{(1-p)\mult r}{p}$; $(d_{2,2}\levidx{1})_{\theta}:\{(s,s'):s,s'\in\cS_{2,2}\}\rightarrow \mathbb{N}\cup\{+\infty\}$, is a parametric family of distance functions parametrized by $\theta\in (\Theta_{2,2}\levidx{1})_{\beta}=\{\key,\door\}$. Given the parameter $\theta$, the agent's starting state $s$ and current state $s'$, the value of $(d_{2,2}\levidx{1})_{\theta}(s,s')$ incorporates the information of $(\pi_{2,2}\levidx{1})^\beta_{\theta}$ by defining the infinite sequence $\{s_i\}_{i=0}^\infty$ inductively starting from $s_0 = s$: for $i\in\mathbb{N}$, there exists a unique action $a\in\cA_{2,2}\levidx{1}$ such that $(\pi_{2,2}\levidx{1})_{\theta}(s_i,a)=1$ because $\overline{\pi}^{\mathrm{nav}}$ is assumed to be deterministic, and we let  $s_{i+1}:=\sup_{s\in\cS_{2,2}-\{s_i\}}P_{2,2}(s_i,a,s)$ when $P_{2,2}(s_i,a,s_i)<1$, and $s_{i+1}:=s_i$ otherwise.
Then, we set $(d_{2,2}\levidx{1})_{\theta}(s,s')$ to be the smallest $i\in\mathbb{N}$ such that $s_i=s'$ if there exists such an $i$, and set it to be $+\infty$ otherwise. Consequently, we denote $(d_{2,2}\levidx{1})_{\theta}(s,0):=\sup\{(d_{2,2}\levidx{1})_{\theta}(s,s'):s'\in\cS_{2,2}, (d_{2,2}\levidx{1})_{\theta}(s,s')<+\infty\}$. Notice that if $\overline{\pi}^{\mathrm{nav}}$ is not deterministic, then the sequence defined here is stochastic, making the calculations slightly more complicated.

The student then constructs $\MDP_{3,1}\levidx{2}=(\cS_{3,1}$, $\Sinit_{3,1}$, $\Sterm_{3,1}$, $\overline{\Pi_{3,1}\levidx{1}}$, $P_{3,1}\levidx{2}$, $R_{3,1}\levidx{2}$, $\Gamma_{3,1}\levidx{2})$, with
\begin{equation}
\label{e:compressed-MDP312-MazeBase} 
\begin{aligned}
R_{3,1}\levidx{2}(s,(\pi_{3,1}\levidx{1})^\lambda_{\theta},s')&=
\begin{dcases}    
R_0\mult\indic_{\{1\}}(\goalpicknext)+r_0\mult\indic_{\{0\}}(\goalpicknext) &\,,\text{if}\  \lambda=\alpha\\
r_{3,1}\levidx{1}+r_0\mult\big[\mathbb{E}[X\mid X \sim \mathrm{NB} ((d_{3,1}\levidx{1})_{\theta}(s,s'),p_s)]+(d_{3,1}\levidx{1})_{\theta}(s,s')\big] &\,,\text{if}\  \lambda=\beta,
\end{dcases}
\\
\Gamma_{3,1}\levidx{2}(s,(\pi_{3,1}\levidx{1})^\lambda_{\theta},s')&=1\,,
\end{aligned}
\end{equation}
\vspace{-1.5\abovedisplayskip}
\begin{equation*}
\begin{aligned}  
P_{3,1}\levidx{2}(s,(\pi_{3,1}\levidx{1})^\lambda_{\theta},s')&=
\begin{dcases}
\begin{aligned}   
& \indic_{\{\agent\}}(\agentnext) \mult\indic_{\{\keypickvec\}}(\keypickvecnext)\mult\indic_{\{\goalpick\}}(\goalpicknext)\\
&\quad\mult \{(1-p_s)\mult\indic_{\{\dooropenvec\}}(\dooropenvecnext)\\
&\quad+p_s\mult\prod_{i=1}^3\big[\indic_{\neardoor(s_{\door_i})}(\agent)\mult\indic_{\{1\}}(s_{\pick_i})\mult\indic_{\{1\}}(s'_{\open_i})\\
&\quad+[1-\indic_{\neardoor(s_{\door_i})}(\agent)\mult\indic_{\{1\}}(s_{\pick_i})]\mult\indic_{\{s_{\open_i}\}}(s'_{\open_i})\big]\} 
\end{aligned} &\,,\text{if}\  \theta=\open\\
\begin{aligned}
& \indic_{\{\agent\}}(\agentnext)\mult\indic_{\{\dooropenvec\}}(\dooropenvecnext) \\
&\quad\mult \{(1-p_s)\mult\indic_{\{\keypickvec\}}(\keypickvecnext)\mult\indic_{\{\goalpick\}}(\goalpicknext)\\
&\quad + p_s\mult\prod_{i=1}^3\big[\indic_{\{s_{\key_i}\}}(\agent)\mult\indic_{\{1\}}(s'_{\pick_i}) \\
&\quad+\indic_{\{s_{\key_i}\}^c}(\agent)\mult\indic_{\{s_{\pick_i}\}}(s'_{\pick_i})\big]\\ 
&\quad\mult \big[\indic_{\{\goal\}}(\agent)\mult\indic_{\{1\}}(\goalpicknext)\\
&\quad+\indic_{\{\goal\}^c}(\agent)\mult\indic_{\{\goalpick\}}(\goalpicknext)\big]\} 
\end{aligned} &\,,\text{if}\  \theta=\pick\\
\begin{aligned}
&   \indic_{\{(d_{3,1}\levidx{1})_{\theta}(s,0)\}}((d_{3,1}\levidx{1})_{\theta}(s,s')) 
\end{aligned}
&\,,\text{if}\  \lambda = \beta,
\end{dcases}
\end{aligned}
\end{equation*}
where $(d_{3,1}\levidx{1})_{\theta}$ is defined similarly to $(d_{2,2}\levidx{1})$; $r_{3,1}\levidx{1}<0$ is some negative reward set by the teacher and here it is set to $-10$.

Then, the student constructs the third-level MDP $\MDP_{3,1}\levidx{3}=(\cS_{3,1},\Sinit_{3,1},\Sterm_{3,1},\overline{\Pi_{3,1}\levidx{2}},P_{3,1}\levidx{3},R_{3,1}\levidx{3},\Gamma_{3,1}\levidx{3})$, where $P_{3,1}\levidx{3},R_{3,1}\levidx{3},\Gamma_{3,1}\levidx{3}$ are as follows:
\begin{equation} \label{e:compressed-MDP313-MazeBase} 
\begin{aligned} 
R_{3,1}\levidx{3}(s,((\pi_{3,1}\levidx{2})_{\theta})_{\{\alpha,\beta\}},s')&= R_0\mult\indic_{\{1\}}(\goalpicknext) +r_{3,1}\levidx{2}\\& \quad+r_{3,1}\levidx{1}\mult[\indic_{(0,+\infty)}(((d_{3,1}\levidx{2})_{\theta}(s,s'))_1)+\indic_{(0,+\infty)}(((d_{3,1}\levidx{2})_{\theta}(s,s'))_3)]\\
&\quad+r_0\mult\big[\mathbb{E}[X\mid X \sim \mathrm{NB} (\sum_{i=1}^4((d_{3,1}\levidx{2})_{\theta}(s,s'))_i,p_s)]+\sum_{i=1}^4((d_{3,1}\levidx{2})_{\theta}(s,s'))_i\big]\,,\\
\Gamma_{3,1}\levidx{3}(s,((\pi_{3,1}\levidx{2})_{\theta})_{\{\alpha,\beta\}},s')&=1\,,\\
\end{aligned}
\end{equation}
\vspace{-1.5\abovedisplayskip}
\begin{align*}
P_{3,1}\levidx{3}(s,((\pi_{3,1}\levidx{2})_{\theta})_{\{\alpha,\beta\}},s')&= \indic_{\{(d_{3,1}\levidx{2})_{\theta}(s,0)\}}((d_{3,1}\levidx{2})_{\theta}(s,s'))\,,
\end{align*}
where $r_{3,1}\levidx{2}<0$ is some negative reward set by the teacher, and here it is set to $-10$; $(d_{3,1}\levidx{2})_{\theta}:\{(s,s'):s,s'\in\cS_{3,1}\}\rightarrow \mathbb{N}^4\cup\{+\infty\}$ generalizing $(d_{3,1}\levidx{1})_{\theta}$ is a parametric family of vector distance functions parametrized by $\theta\in \Theta_{3,1}\levidx{2}=\{(\key_1,\door_1),(\key_2,\door_2),(\key_3,\door_3),(\goals,\goals)\}$. Given the parameter $\theta$, the agent's starting state $s$ and current state $s'$, the value of $(d_{3,1}\levidx{2})_{\theta}(s,s')$ incorporates the information of $((\pi_{3,1}\levidx{2})_{\theta})_{\{\alpha,\beta\}}\,\, (\theta\in \Theta_{3,1}\levidx{2})$ by defining the infinite sequence $\{s_i\}_{i=0}^\infty$ inductively starting from $s_0 := s$: for $i\in\mathbb{N}$, there exists a unique action $\pi_i\in\cA_{3,1}\levidx{2}$ such that $((\pi_{3,1}\levidx{2})_{\theta})_{\{\alpha,\beta\}}(s_i,\pi_i)=1$, and then there exists a unique action $a_i\in\cA_{3,1}\levidx{1}$ such that $\pi_i(s_i,a_i)=1$ because $\overline{\pi}^{\mathrm{nav}}$ is assumed to be deterministic, ($a_i=\endactionfactor$ if $\pi_i=\endaction$), and we let 
\begin{equation*}
s_{i+1}:=
\begin{cases}   
\sup_{s\in\cS_{3,1}-\{s_i\}}P_{3,1}(s_i,a_i,s) &,  \text{if}\  P_{3,1}(s_i,a_i,s)<1\\
s_i &,\text{otherwise.}
\end{cases}
\end{equation*}
Then, we find the smallest $i_0\in\mathbb{N}$ such that $s_{i_0}=s'$ (we set $(d_{3,1}\levidx{2})_{\theta}(s,s')$ to be $+\infty$ if there does not exist such an $i_0$), and we set $(d_{3,1}\levidx{2})_{\theta}(s,s')$ to be $(\sum_{i=0}^{i_0-1}\indic_{\{\pi_{\key_{i'}}\}}(\pi_i),\sum_{i=0}^{i_0-1}\indic_{\{\pi_{\pick}\}}(\pi_i),$\\$\sum_{i=0}^{i_0-1}\indic_{\{\pi_{\door_{i'}}\}}(\pi_i),\sum_{i=0}^{i_0-1}\indic_{\{\pi_{\open}\}}(\pi_i))$ for $\theta =(\key_{i'},\door_{i'}) (i'=1,2,3)$ and we set $(d_{3,1}\levidx{2})_{\theta}(s,s')$ to be $(\sum_{i=0}^{i_0-1}\indic_{\{\pi_{\goals}\}}(\pi_i),\sum_{i=0}^{i_0-1}\indic_{\{\pi_{\pick}\}}(\pi_i),0,\sum_{i=0}^{i_0-1}\indic_{\{\pi_{\open}\}}(\pi_i))$ for $\theta =(\goals,\goals)$. Then, since for any fixed $s$, there is a natural ordering between any two elements in $\{(d_{3,1}\levidx{2})_{\theta}(s,s'):s'\in\cS_{3,1}, (d_{3,1}\levidx{2})_{\theta}(s,s')<+\infty\}$ induced by ordering on $\mathbb{N}$, we still denote $(d_{3,1}\levidx{2})_{\theta}(s,0):=\sup\{(d_{3,1}\levidx{2})_{\theta}(s,s'):s'\in\cS_{3,1}, (d_{3,1}\levidx{2})_{\theta}(s,s')\neq +\infty\}$.

\subsubsection{(Partial) policies and (partial) policy generators}
\label{s:policies-mazebase}

\begin{enumerate}[leftmargin=0cm,itemsep=0pt]
\item[$\bullet$] For $\MDP_{2,1}$, with $\Theta_{2,1}\levidx{1}=\{\door_1,\door_2,\door_3,\dest\}$, $(e_{2,1}^1)_\theta$ defined in \eqref{e:E211-mazebase}, and with the basic navigation skill $\overline{\pi}^{\text{nav}}_{\text{dense}}$ extracted from $\MDP_{1,1}$ as in $(\pi_{2,1}\levidx{1})_\theta((\agent,\dests),a)$\\=$(\overline{\pi}^{\text{nav}}_{\text{dense}}\circ(e_{2,1}^1)_\theta)((\agent,\dests),a))=\overline{\pi}^{\mathrm{nav}}_{\text{dense}}(\agent,s_{\theta},a)$, 
\begin{equation}
\label{e:g21-mazebase} 
\begin{aligned} 
g_{2,1}\levidx{1}: \Theta_{2,1}\levidx{1}&\rightarrow \{(\pi_{2,1}\levidx{1})_\theta:\theta\in\Theta_{2,1}\levidx{1}\}\\
\theta&\mapsto(\pi_{2,1}\levidx{1})_\theta\quad:\cSA_{2,1}\levidx{1}\rightarrow [0,1] \,\,(\theta\in\Theta_{2,1}\levidx{1})\,.
\end{aligned} 
\end{equation}

\item[$\bullet$] For $\MDP_{2,2}$, with $(\Theta_{2,2}\levidx{1})_{\alpha}= \{\pick,\open\}$, and with
\begin{equation} \label{e:pialphatheta}
(\pi_{2,2}\levidx{1})^\alpha_\theta((\agent,s_\pick,s_\open),a):=\indic_{\{\theta\}}(a)\,,
\end{equation}
\begin{equation*}
\begin{aligned} 
(g_{2,2}\levidx{1})_{\alpha}: (\Theta_{2,2}\levidx{1})_{\alpha}&\rightarrow \{(\pi_{2,2}\levidx{1})^{\alpha}_\theta:\theta\in(\Theta_{2,2}\levidx{1})_{\alpha}\}\\
\theta&\mapsto(\pi_{2,2}\levidx{1})^{\alpha}_\theta\quad:\cSA_{2,2}\levidx{1}\rightarrow [0,1] \,\,(\theta\in(\Theta_{2,2}\levidx{1})_{\alpha})\,.
\end{aligned} 
\end{equation*}
With $(\Theta_{2,2}\levidx{1})_{\beta}=\{\key,\door\}$, and with the navigation skill $\overline{\pi}^{\mathrm{nav}}$ extracted from $\MDP_{2,1}$ as in
\begin{align} \label{e:def-generator-mazebase} 
& \nonumber (\pi_{2,2}\levidx{1})^\beta_\theta((\agent,s_\pick,s_\open),a)\\
&\qquad=\begin{cases}   
\overline{\pi}^{\mathrm{nav}}(\agent,s_{\theta},a) &, \text{if}\  \overline{\pi}^{\mathrm{nav}}(\agent,s_\theta,s_\door-\agent) = 0\ \mathrm{and}\ a\in\cA_{\dir}\cup\{\endactionfactor\} \\
  \indic_{\{\endactionfactor\}}(a) &,\text{otherwise,}
\end{cases}
\end{align}
\begin{equation*}
\begin{aligned} 
(g_{2,2}\levidx{1})_{\beta}: (\Theta_{2,2}\levidx{1})_{\beta}&\rightarrow \{(\pi_{2,2}\levidx{1})^{\beta}_\theta:\theta\in(\Theta_{2,2}\levidx{1})_{\beta}\}\\
\theta&\mapsto(\pi_{2,2}\levidx{1})^{\beta}_\theta\quad:\cSA_{2,2}\levidx{1}\rightarrow [0,1] \,\,(\theta\in(\Theta_{2,2}\levidx{1})_{\beta})\,.
\end{aligned} 
\end{equation*}

The optimal policy is
\begin{align}   \label{e:second-level-policy-eqn-optimization-mazebase}
\pi_{2,2,\ast}\levidx{2}(s,(\pi_{2,2}\levidx{1})_{\theta}^{\lambda})=&
\indic_{\{(\theta_{0,\ast}(s),\lambda_{0,\ast}(s))\}}((\theta,\lambda))\,,
\end{align}
with 
\begin{align} \label{e:second-level-policy-eqn-optimization-second-mazebase}
(\theta_{0,\ast}(s), \lambda_{0,\ast}(s))&\in \argmax_{(\theta_0,\lambda_0)}\max_{(\theta_1,\lambda_1),\cdots, (\theta_{\tau-1},\lambda_{\tau-1})}\\ 
\nonumber&\qquad\mathbb{E}_{\tau,S_{1:\tau}}[(R_{2,2}\levidx{2})_{0,\tau}  | S_0=s,A_t = (\pi_{2,2}\levidx{1})^{\lambda_t}_{\theta_t} \text{ for any } 0\leq t\leq \tau-1]\,.
\end{align}
Therefore,
\begin{align}  \label{e:second-level-policy-eqn_mazebase}
\pi_{2,2,\ast}\levidx{2}&((\agent,s_\pick,s_\open),a) = \nonumber\\  \nonumber
&\quad\indic_{\{1\}}(s_\open)\mult\indic_{\{\endaction\}}(a)+\indic_{\{0\}}(s_\open)\mult\indic_{\{0\}}(s_\pick)\mult\big[\indic_{\{s_\key\}}(\agent)\mult\indic_{\{(\pi^1_{2,2})^\alpha_{\pick}\}}(a)\\
&\quad+\indic_{\{s_\key\}^c}(\agent)\mult\indic_{\{(\pi^1_{2,2})^\beta_{\key}\}}(a)\big]+\indic_{\{0\}}(s_\open)\mult\indic_{\{1\}}(s_\pick)\\  
\nonumber&\qquad\mult\big[\indic_{\neardoor(s_\door)}(\agent)\mult\indic_{\{(\pi^1_{2,2})^\alpha_{\open}\}}(a)+\indic_{\neardoor(s_\door)^c}(\agent)\mult\indic_{\{(\pi^1_{2,2})^\beta_{\door}\}}(a)\big]\,.
\end{align}

\begin{equation}
\label{e:first-level-policy-eqn-mazebase}
\begin{aligned}  
\pi_{2,2,\ast}&((\agent,s_\pick,s_\open),a)\\&=
\begin{dcases}
\begin{aligned}   
&\indic_{\{0\}}(s_\open)\mult\big[\indic_{\{0\}}(s_\pick)\mult\indic_{\{s_\key\}^c}(\agent)\mult\overline{\pi}^{\mathrm{nav}}(\agent,s_\key,a)\\&\quad+\indic_{\{1\}}(s_\pick)\mult\indic_{\neardoor(s_\door)^c}(\agent)\mult\overline{\pi}^{\mathrm{nav}}(\agent,s_\door,a)\big] 
\end{aligned}
&,\text{if}\  a\in\cA_{\dir}\\
\begin{aligned}
&\indic_{\{1\}}(s_\open)\mult\indic_{\{\endactionfactor\}}(a)+\indic_{\{0\}}(s_\open)\\ &\quad\mult\big[\indic_{\{0\}}(s_\pick)\mult\indic_{\{s_\key\}}(\agent)\mult\indic_{\{\pick\}}(a)\\&\quad+\indic_{\{1\}}(s_\pick)\mult\indic_{\neardoor(s_\door)}(\agent)\mult\indic_{\{\open\}}(a)\big] 
\end{aligned} &,\text{otherwise.}
\end{dcases}
\end{aligned}
\end{equation}

\item[$\bullet$]  For $\MDP_{3,1}$,
with $(\Theta_{3,1}\levidx{1})_\alpha=\{\pick,\open\}$, $(e_{3,1}\levidx{1})^\alpha_{\theta}$ defined in \eqref{e:E311alpha}, $(\Theta_{3,1}\levidx{1})_\beta =\{\key_1,\key_2,\key_3,$\\$\door_1,\door_2,\door_3,\goals\}$,
\begin{equation}
\label{e:g311alpha} 
\begin{aligned} 
(g_{3,1}\levidx{1})_\alpha: (\Theta_{3,1}\levidx{1})_\alpha&\rightarrow \{(\pi_{3,1}\levidx{1})^\alpha_\theta:\theta\in(\Theta_{3,1}\levidx{1})_\alpha\}\\
\theta&\mapsto(\pi_{3,1}\levidx{1})^\alpha_\theta=\mathrm{id}\circ (e_{3,1}\levidx{1})^\alpha_{\theta}=(e_{3,1}\levidx{1})^\alpha_{\theta}\,,
\end{aligned}
\end{equation}
\begin{equation}
\label{e:g311beta} 
\begin{aligned}  
(g_{3,1}\levidx{1})_\beta: (\Theta_{3,1}\levidx{1})_\beta &\rightarrow \{(\pi_{3,1}\levidx{1})^\beta_\theta:\theta\in(\Theta_{3,1}\levidx{1})_\beta\} \\
\theta &\mapsto(\pi_{3,1}\levidx{1})^\beta_\theta\quad:\quad \cSA_{3,1}\levidx{1}\rightarrow [0,1]\,,
\end{aligned}
\end{equation}
where with $(e_{3,1}^1)^\beta_\theta$ defined in \eqref{e:E311beta}, for $\theta \in \{\key_1,\key_2,\key_3,\goals\}$, 
\begin{align*} 
(\pi_{3,1}\levidx{1})^\beta_{\theta}((\agent,\keypickvec,\dooropenvec,\goalpick),a)=&(\overline{\pi}^{\mathrm{nav}}\circ(e_{3,1}^1)^\beta_\theta)((\agent,s_\pick,s_\open),a)\\=&
\begin{cases}   
\overline{\pi}^{\mathrm{nav}}(\agent,s_\theta,a) &,\text{if}\  a\in \cA_{\dir}\cup\{\endactionfactor\} \\
 0 &,\text{otherwise,}
\end{cases}
\end{align*} 
and for $\theta \in \{\door_1,\door_2,\door_3\}$,
\begin{align*} 
(\pi_{3,1}\levidx{1})^\beta_{\theta}&((\agent,\keypickvec,\dooropenvec,\goalpick),a) = (\overline{\pi}^{\mathrm{nav}}\circ(e_{3,1}^1)^\beta_\theta)((\agent,s_\pick,s_\open),a)\\
&=\begin{cases} 
\overline{\pi}^{\mathrm{nav}}(\agent,s_\theta,a) &,\text{if}\  \overline{\pi}^{\mathrm{nav}}(\agent,s_\theta,s_\theta-\agent) = 0\ \mathrm{and}\ a\in\cA_{\dir}\cup\{\endactionfactor\} \\
\indic_{\{\endactionfactor\}}(a) &,\text{otherwise.}
\end{cases}
\end{align*}
At level 2, with $\Theta_{3,1}\levidx{2}=\{(\key_1,\door_1),(\key_2,\door_2),(\key_3,\door_3),(\goals,\goals)\}$,
\begin{equation}
\label{e:g312}
\begin{aligned}  
(g_{3,1}\levidx{2})_{\{\alpha, \beta\}}: \Theta_{3,1}\levidx{2}&\rightarrow \{((\pi_{3,1}\levidx{2})_\theta)_{\{\alpha, \beta\}}:\theta\in\Theta_{3,1}\levidx{2}\} \\
\theta&\mapsto ((\pi_{3,1}\levidx{2})_\theta)_{\{\alpha, \beta\}}\quad:\quad\cSA_{3,1}\levidx{2}\rightarrow [0,1]\,,
\end{aligned} 
\end{equation}
where with $(e_{3,1}^2)_\theta$ defined in \eqref{e:E312}, and with the concatenation skill $\overline{\pi}^{\mathrm{concat}}$ extracted from $\MDP_{2,2}$, for $\theta = (\key_i,\door_i)$, $i = 1,2,3$,
\begin{align*} 
&((\pi_{3,1}\levidx{2})_{\theta})_{\{\alpha,\beta\}}((\agent,\keypickvec,\dooropenvec,\goalpick),a)=(\overline{\pi}^{\mathrm{concat}}\circ(e_{3,1}^2)_\theta)((\agent,\keypickvec,\dooropenvec,\goalpick),a)\\ 
&=\begin{dcases}
\overline{\pi}^{\mathrm{concat}}(\indic_{\{s_{\key_i}\}}(\agent),\indic_{\{1\}}(s_{\pick_i})\mult\indic_{\neardoor(s_{\door_i})}(\agent),s_{\pick_i},s_{\open_i},\pi_\theta) &\,,\text{if}\  a=(\pi_{3,1}\levidx{1})^\alpha_{\theta} \\
\overline{\pi}^{\mathrm{concat}}(\indic_{\{s_{\key_i}\}}(\agent),\indic_{\{1\}}(s_{\pick_i})\mult\indic_{\neardoor(s_{\door_i})}(\agent),s_{\pick_i},s_{\open_i},\pi_{\key}) &\,,\text{if}\  a=(\pi_{3,1}\levidx{1})^\beta_{{\key}_{i}}\\
\overline{\pi}^{\mathrm{concat}}(\indic_{\{s_{\key_i}\}}(\agent),\indic_{\{1\}}(s_{\pick_i})\mult\indic_{\neardoor(s_{\door_i})}(\agent),s_{\pick_i},s_{\open_i},\pi_{\door}) &\,,\text{if}\  a=(\pi_{3,1}\levidx{1})^\beta_{{\door}_{i}} \\
\overline{\pi}^{\mathrm{concat}}(\indic_{\{s_{\key_i}\}}(\agent),\indic_{\{1\}}(s_{\pick_i})\mult\indic_{\neardoor(s_{\door_i})}(\agent),s_{\pick_i},s_{\open_i},a) &\,,\text{if}\  a=\endaction \\
 0 &\,,\text{otherwise,}
\end{dcases}
\end{align*}
and for $\theta = (\goals,\goals)$
\begin{align*} 
&((\pi_{3,1}\levidx{2})_\theta)_{\{\alpha, \beta\}}((\agent,\keypickvec,\dooropenvec,\goalpick),a)=(\overline{\pi}^{\mathrm{concat}}\circ(e_{3,1}^2)_\theta)((\agent,\keypickvec,\dooropenvec,\goalpick),a)\\ 
&=\begin{cases}
\overline{\pi}^{\mathrm{concat}}(\indic_{\{\goal\}}(\agent),\indic_{\{1\}}(\goalpick)\mult\indic_{\{\goal\}}(\agent),\goalpick,\goalpick,\pi_{\theta}) &\,,\text{if}\  a=(\pi_{3,1}\levidx{1})^\alpha_{\theta} \\
\overline{\pi}^{\mathrm{concat}}(\indic_{\{\goal\}}(\agent),\indic_{\{1\}}(\goalpick)\mult\indic_{\{\goal\}}(\agent),\goalpick,\goalpick,\pi_{\key}) &\,,\text{if}\  a=(\pi_{3,1}\levidx{1})^\beta_{\goals}\\
\overline{\pi}^{\mathrm{concat}}(\indic_{\{\goal\}}(\agent),\indic_{\{1\}}(\goalpick)\mult\indic_{\{\goal\}}(\agent),\goalpick,\goalpick,a) &\,,\text{if}\  a=\endaction \\
 0 &\,,\text{otherwise.}
\end{cases}
\end{align*}
At level $3$,
\begin{align} \label{e:pi313}  \nonumber
\pi_{3,1,\ast}\levidx{3}&((\agent,\keypickvec,\dooropenvec,\goalpick),a)=\indic_{\{1\}}(\goalpick)\mult\indic_{\{\endactionfactor\}}(a)\\ 
\nonumber&+[\indic_{\gridworld_{\room_4}\cup\{\door_3\}}(\agent)+\indic_{\gridworld_{\room_3}\cup\{\door_2\}}(\agent)\mult\indic_{\{1\}}(s_{\open_3})\\
&+\indic_{\gridworld_{\room_1}\cup \gridworld_{\room_2}\cup\{\door_1\}}(\agent)\mult\indic_{\{1\}}(s_{\open_2})\mult\indic_{\{1\}}(s_{\open_3})]\mult\indic_{\{(\pi\levidx{2}_{3,1})_{\goals}\}}(a)\\
\nonumber &+[\indic_{\gridworld_{\room_3}\cup\{\door_2\}}(\agent)+\indic_{\gridworld_{\room_1}\cup \gridworld_{\room_2}\cup\{\door_1\}}(\agent)\mult\indic_{\{1\}}(s_{\open_2})]\mult\indic_{\{0\}}(s_{\open_3})\mult\indic_{\{(\pi\levidx{2}_{3,1})_{\door_3}\}}(a)\\
\nonumber &+\indic_{\gridworld_{\room_1}\cup \gridworld_{\room_2}\cup\{\door_1\}}(\agent)\mult\indic_{\{0\}}(s_{\open_2})\mult\indic_{\{(\pi\levidx{2}_{3,1})_{\door_2}\}}(a)\,.
\end{align}
\begin{align}  \label{e:pi312} \nonumber
\pi_{3,1,\ast}\levidx{2}&(s,a)=\indic_{\{1\}}(\goalpick)\mult\indic_{\{\endaction\}}(a)\\ \nonumber
&\quad+[\indic_{\gridworld_{\room_4}\cup\{\door_3\}}(\agent)+\indic_{\gridworld_{\room_3}\cup\{\door_2\}}(\agent)\mult\indic_{\{1\}}(s_{\open_3})\\ 
&\quad+\indic_{\gridworld_{\room_1}\cup \gridworld_{\room_2}\cup\{\door_1\}}(\agent)\mult\indic_{\{1\}}(s_{\open_2})\mult\indic_{\{1\}}(s_{\open_3})]\mult(\pi\levidx{2}_{3,1})_{\goals}(s,a)\\ 
\nonumber &\quad+[\indic_{\gridworld_{\room_3}\cup\{\door_2\}}(\agent)+\indic_{\gridworld_{\room_1}\cup \gridworld_{\room_2}\cup\{\door_1\}}(\agent)\mult\indic_{\{1\}}(s_{\open_2})]\\ \nonumber
&\quad\mult\indic_{\{0\}}(s_{\open_3})\mult(\pi\levidx{2}_{3,1})_{\door_3}(s,a)\\ \nonumber&\quad+\indic_{\gridworld_{\room_1}\cup \gridworld_{\room_2}\cup\{\door_1\}}(\agent)\mult\indic_{\{0\}}(s_{\open_2})\mult(\pi\levidx{2}_{3,1})_{\door_2}(s,a)\,.
\end{align}
\begin{align}  \label{e:pi311} \nonumber
\pi_{3,1,\ast}&(s,a)=\indic_{\{1\}}(\goalpick)\mult\indic_{\{\endactionfactor\}}(a)\\   
\nonumber &\quad+[\indic_{\gridworld_{\room_4}\cup\{\door_3\}}(\agent)+\indic_{\gridworld_{\room_3}\cup\{\door_2\}}(\agent)\mult\indic_{\{1\}}(s_{\open_3})\\ 
&\quad+\indic_{\gridworld_{\room_1}\cup \gridworld_{\room_2}\cup\{\door_1\}}(\agent)\mult\indic_{\{1\}}(s_{\open_2})\mult\indic_{\{1\}}(s_{\open_3})]\mult\sum_{\pi\in \Pi_{3,1}\levidx{1}}(\pi\levidx{2}_{3,1})_{\goals}(s,\pi)\mult\pi(s,a)\\ 
\nonumber&\quad+[\indic_{\gridworld_{\room_3}\cup\{\door_2\}}(\agent)+\indic_{\gridworld_{\room_1}\cup \gridworld_{\room_2}\cup\{\door_1\}}(\agent)\mult\indic_{\{1\}}(s_{\open_2})]\\ 
\nonumber &\quad\quad\mult\indic_{\{0\}}(s_{\open_3})\mult\sum_{\pi\in \Pi_{3,1}\levidx{1}}(\pi\levidx{2}_{3,1})_{\door_3}(s,\pi)\mult\pi(s,a)\\ 
\nonumber &\quad+\indic_{\gridworld_{\room_1}\cup \gridworld_{\room_2}\cup\{\door_1\}}(\agent)\mult\indic_{\{0\}}(s_{\open_2})\mult\sum_{\pi\in \Pi_{3,1}\levidx{1}}(\pi\levidx{2}_{3,1})_{\door_2}(s,\pi)\mult\pi(s,a)\,.
\end{align}

\item[$\bullet$] For $\MDP'_{3,1}$:
\begin{align} \label{e:policy-transfer-mazebase}  \nonumber
(\pi_{3,1,\ast}\levidx{3})'&((\agent,\keypickvec,\dooropenvec,\goalpick),a)=\indic_{\{1\}}(\goalpick)\mult\indic_{\{\endactionfactor\}}(a)\\ 
\nonumber&\quad+[\indic_{\gridworld_{\room_3}\cup\{\door_3\}}(\agent)+\indic_{\gridworld_{\room_4}\cup\{\door_2\}}(\agent)\mult\indic_{\{1\}}(s_{\open_3})\\ 
\nonumber&\quad+\indic_{\gridworld_{\room_2}\cup\{\door_1\}}(\agent)\mult\indic_{\{1\}}(s_{\open_2})\mult\indic_{\{1\}}(s_{\open_3})\\
&\quad+\indic_{\gridworld_{\room_1}}(\agent)\mult\indic_{\{1\}}(s_{\open_1})\mult\indic_{\{1\}}(s_{\open_2})\mult\indic_{\{1\}}(s_{\open_3})]\mult\indic_{\{(\pi\levidx{2}_{3,1})'_{\goals}\}}(a)\\ 
\nonumber&\quad+[\indic_{\gridworld_{\room_4}\cup\{\door_2\}}(\agent)+\indic_{\gridworld_{\room_2}\cup\{\door_1\}}(\agent)\mult\indic_{\{1\}}(s_{\open_2})\\ 
\nonumber&\quad+\indic_{\gridworld_{\room_1}}(\agent)\mult\indic_{\{1\}}(s_{\open_1})\mult\indic_{\{1\}}(s_{\open_2})]\mult\indic_{\{0\}}(s_{\open_3})\mult\indic_{\{(\pi\levidx{2}_{3,1})'_{\door_3}\}}(a)\\ 
\nonumber&\quad+[\indic_{\gridworld_{\room_2}\cup\{\door_1\}}(\agent)+\indic_{\gridworld_{\room_1}}(\agent)\mult\indic_{\{1\}}(s_{\open_1})]\mult\indic_{\{0\}}(s_{\open_2})\mult\indic_{\{(\pi\levidx{2}_{3,1})'_{\door_2}\}}(a)\\ 
\nonumber&\quad+\indic_{\gridworld_{\room_1}}(\agent)\mult\indic_{\{0\}}(s_{\open_1})\mult\indic_{\{(\pi\levidx{2}_{3,1})'_{\door_1}\}}(a)\,.
\end{align}
\end{enumerate}

\subsubsection{Embeddings, embedding generators and skills}

\label{s:skills-mazebase}

\begin{enumerate}[leftmargin=0cm,itemsep=0pt]
\item[$\bullet$] For $\MDP_{2,1}$,
$E_{2,1}\levidx{1}:\Theta_{2,1}\levidx{1}\rightarrow \{(e_{2,1}\levidx{1})_\theta:\theta\in\Theta_{2,1}\levidx{1}\}$ is defined as
\begin{equation} \label{e:E211-mazebase} 
\begin{aligned}
E_{2,1}\levidx{1}(\theta):=(e_{2,1}\levidx{1})_\theta\quad:\quad &\cSA_{2,1}\levidx{1}&\rightarrow&\, \noBlock\times\noBlock\times(\cA_{\dir}\cup\{\endactionfactor\}) \\
&((\agent,\dests),a)&\mapsto &\,(\agent,s_{\theta}, a)\,.
\end{aligned}
\end{equation}

\begin{equation} \label{e:E211-decomp-mazebase} 
\begin{aligned}
(e_{\decomp})_{2,1}\levidx{1}: &\cSA_{2,1}\levidx{1}&\rightarrow&\, \noBlock\times\noBlock\times(\cA_{\dir}\cup\{\endactionfactor\})\\
&((\agent,\dests),a)&\mapsto&\,(\agent,\dests,a)\,.
\end{aligned}
\end{equation} 

\begin{equation} \label{e:pi-nav-mazebase} 
\begin{aligned}
\overline{\pi}^{\mathrm{nav}}:&(e_{\decomp})_{2,1}\levidx{1}(\cSA_{2,1}\levidx{1})&\rightarrow&\, [0,1]\\
&(\agent,\dests,a)&\mapsto&\,\pi_{2,1,\ast}\levidx{1}((\agent,\dests),a)\,.
\end{aligned}
\end{equation}

\item[$\bullet$] 
For $\MDP_{2,2}$, with $(e_{2,2}\levidx{1})^\alpha_\theta:\cSA_{2,2}\levidx{1}\rightarrow \{0,1\} \,\,(\theta\in(\Theta_{2,2}\levidx{1})_{\alpha})$ being defined as
\begin{equation} \label{e:ealphatheta}
(e_{2,2}\levidx{1})^\alpha_\theta((\agent,s_\pick,s_\open),a):=\indic_{\{\theta\}}(a)\,,
\end{equation}
\begin{equation*}
\begin{aligned} 
(E_{2,2}\levidx{1})_{\alpha}: (\Theta_{2,2}\levidx{1})_{\alpha}&\rightarrow \{(e_{2,2}\levidx{1})^{\alpha}_\theta:\theta\in(\Theta_{2,2}\levidx{1})_{\alpha}\}\\
\theta&\mapsto(e_{2,2}\levidx{1})^{\alpha}_\theta\,,
\end{aligned} 
\end{equation*}
and with $(e_{2,2}\levidx{1})^\beta_{\theta}:\{((\agent,s_\pick,s_\open),a)\in\cS_{2,2}\times(\cA_{\dir}\cup\{\endactionfactor\}):\agent+a\neq s_\door\}\subseteq \cSA_{2,2}\levidx{1}\rightarrow \{(\cur,\goal,a):\cur,\goal\in\noBlock,a\in\cA_{\dir}\cup\{\endactionfactor\}\} \,\,(\theta\in(\Theta_{2,2}\levidx{1})_{\beta})$ being defined as 
\begin{equation}  \label{e:ebetatheta}
(e_{2,2}\levidx{1})^\beta_{\theta}((\agent,s_\pick,s_\open),a):=(\agent,s_{\theta},a)\,,
\end{equation}
\begin{equation*}
\begin{aligned} 
(E_{2,2}\levidx{1})_{\beta}: (\Theta_{2,2}\levidx{1})_{\beta}&\rightarrow \{(e_{2,2}\levidx{1})^{\beta}_\theta:\theta\in(\Theta_{2,2}\levidx{1})_{\beta}\}\\
\theta&\mapsto(e_{2,2}\levidx{1})^{\beta}_\theta\,.
\end{aligned} 
\end{equation*}

$(e_\decomp)_{2,2}\levidx{2}: \cSA_{2,2}\levidx{2}\rightarrow \{(e_{\agent=s_\key}, $ $e_{\|\agent-s_\door\|_1=1}, s_\pick,s_\open,a)\}\subseteq$ $\{0,1\}^4\times\{\pi_{\pick},\pi_{\open},\pi_{\key},\pi_{\door},$\\$\endaction\}$ 
is defined as:
\begin{align}  \label{e:embedding-concat} \nonumber
(e_\decomp)_{2,2}\levidx{2}&((\agent,s_\key,s_\door,s_\pick,s_\open),a)\\&:=
\begin{dcases}
(\indic_{\{s_\key\}}(\agent), \indic_{\{1\}}(s_\pick)\mult\indic_{\neardoor(s_\door)}(\agent), s_\pick,s_\open,\pi_{\theta}) &,\text{if}\  a=(\pi\levidx{1}_{2,2})^\lambda_{\theta} \\[0.3cm]
(\indic_{\{s_\key\}}(\agent), \indic_{\{1\}}(s_\pick)\mult\indic_{\neardoor(s_\door)}(\agent), s_\pick,s_\open,a) &,\text{otherwise,}  
\end{dcases}
\end{align}

and the concatenation skill $\overline{\pi}^{\mathrm{concat}}:\{0,1\}^4\times\{\pi_{\pick}, \pi_{\open},\pi_{\key},\pi_{\door},\endaction\}\rightarrow [0,1]$ is
\begin{align} \label{e:skill-concat} 
\overline{\pi}^{\mathrm{concat}}&(e_{\agent=s_\key}, e_{\|\agent-s_\door\|_1=1}, s_\pick,s_\open,a)=\indic_{\{1\}}(s_\open)\mult\indic_{\{\endaction\}}(a)+\indic_{\{0\}}(s_\open)\mult\indic_{\{0\}}(s_\pick)\\  
\nonumber&\mult\big[e_{\agent=s_\key}\mult\indic_{\{\pi_{\pick}\}}(a)+(1-e_{\agent=s_\key})\mult\indic_{\{\pi_{\key}\}}(a)\big]\\  \nonumber&+\indic_{\{0\}}(s_\open)\mult\big[e_{\|\agent-s_\door\|_1=1}\mult\indic_{\{\pi_{\open}\}}(a)+(\indic_{\{1\}}(s_\pick)-e_{\|\agent-s_\door\|_1=1})\mult\indic_{\{\pi_{\door}\}}(a)\big]\,.
\end{align}

\item[$\bullet$] 
For $\MDP_{3,1}$, with $(e_{3,1}\levidx{1})^\alpha_\theta:\cSA_{3,1}\levidx{1}\rightarrow \{0,1\} \,\,(\theta\in(\Theta_{3,1}\levidx{1})_\alpha)$ being defined as
\begin{equation*}
(e_{3,1}\levidx{1})^\alpha_\theta((\agent,\keypickvec,\dooropenvec,\goalpick),a):=\indic_{\{\theta\}}(a)\,,
\end{equation*}
\begin{equation}
\label{e:E311alpha} 
\begin{aligned} 
(E_{3,1}\levidx{1})_{\alpha}:(\Theta_{3,1}\levidx{1})_\alpha &\rightarrow \{(e_{3,1}\levidx{1})^\alpha_\theta:\theta\in(\Theta_{3,1}\levidx{1})_\alpha\} \\
\theta&\mapsto(e_{3,1}\levidx{1})^\alpha_\theta\,, 
\end{aligned}
\end{equation}
and
\begin{equation} \label{e:E311beta}  
\begin{aligned}
(E_{3,1}\levidx{1})_\beta:(\Theta_{3,1}\levidx{1})_\beta&\rightarrow \{(e_{3,1}\levidx{1})^\beta_\theta:\theta\in(\Theta_{3,1}\levidx{1})_\beta\} \\
\theta &\mapsto (e_{3,1}\levidx{1})^\beta_{\theta}\,,
\end{aligned}
\end{equation} 
where for $\theta=\key_1,\key_2,\key_3,\goals$, $(e_{3,1}\levidx{1})^\beta_{\theta}: \cS_{3,1}\times(\cA_{\dir}\cup\{\endactionfactor\})\subseteq \cSA_{3,1}\levidx{1}\rightarrow \noBlock\times\noBlock\times(\cA_{\dir}\cup\{\endactionfactor\})$ is defined as
\begin{equation*}
(e_{3,1}\levidx{1})^\beta_{\theta}((\agent,\keypickvec,\dooropenvec,\goalpick),a):=(\agent,s_\theta,a)\,,
\end{equation*} 
and for $\theta=\door_1,\door_2,\door_3$, 
$(e_{3,1}\levidx{1})^\beta_{\theta}:\{((\agent,\keypickvec,\dooropenvec,\goalpick),a)\in\cS_{3,1}\times(\cA_{\dir}\cup\{\endactionfactor\}):\agent+a\neq s_\theta\}\subseteq \cSA_{3,1}\levidx{1}
\rightarrow \noBlock\times\noBlock\times(\cA_{\dir}\cup\{\endactionfactor\})$ is defined as
\begin{equation*}
(e_{3,1}\levidx{1})^\beta_{\theta}((\agent,\keypickvec,\dooropenvec,\goalpick),a):=(\agent,s_\theta,a)\,.
\end{equation*}

At level $2$,
\begin{equation} \label{e:E312} 
\begin{aligned}
E_{3,1}\levidx{2}:\Theta_{3,1}\levidx{2}&\rightarrow \{(e_{3,1}\levidx{2})_\theta:\theta=(\theta_{\key},\theta_{\door})\in\Theta_{3,1}\levidx{2}\}\\
\theta&\mapsto(e_{3,1}\levidx{2})_\theta\,,
\end{aligned}
\end{equation}
where for $\theta = (\key_i,\door_i)\,\, (i = 1,2,3)$, $(e_{3,1}\levidx{2})_\theta:\cS_{3,1}\times(\{(\pi_{3,1}\levidx{1})^\alpha_{\pick},(\pi_{3,1}\levidx{1})^\alpha_{\open},(\pi_{3,1}\levidx{1})^\beta_{\key_i},(\pi_{3,1}\levidx{1})^\beta_{\door_i},$ $\endaction\})$ $\subseteq \cSA_{3,1}\levidx{2}\rightarrow \{0,1\}^4\times\{\pi_{\pick},\pi_{\open},\pi_{\key},\pi_{\door},\endaction\}\}$ is defined as 
\begin{align*} 
(e_{3,1}\levidx{2})_\theta&((\agent,\keypickvec,\dooropenvec,\goalpick),a)\\ 
&:=
\begin{cases}   
(\indic_{\{s_{\key_i}\}}(\agent),\indic_{\{1\}}(s_{\pick_i})\mult\indic_{\neardoor(s_{\door_i})}(\agent),s_{\pick_i},s_{\open_i},\pi_{\theta}) &\,,\text{if}\  a=(\pi_{3,1}\levidx{1})^\alpha_{\theta}\\ 
(\indic_{\{s_{\key_i}\}}(\agent),\indic_{\{1\}}(s_{\pick_i})\mult\indic_{\neardoor(s_{\door_i})}(\agent),s_{\pick_i},s_{\open_i},\pi_{\key}) &\,,\text{if}\  a=(\pi_{3,1}\levidx{1})^\beta_{{\key}_{i}} \\
(\indic_{\{s_{\key_i}\}}(\agent),\indic_{\{1\}}(s_{\pick_i})\mult\indic_{\neardoor(s_{\door_i})}(\agent),s_{\pick_i},s_{\open_i},\pi_{\door}) &\,,\text{if}\  a=(\pi_{3,1}\levidx{1})^\beta_{{\door}_{i}}\\ 
(\indic_{\{s_{\key_i}\}}(\agent),\indic_{\{1\}}(s_{\pick_i})\mult\indic_{\neardoor(s_{\door_i})}(\agent),s_{\pick_i},s_{\open_i},a) &\,,\text{otherwise,} 
\end{cases}
\end{align*}
and for $\theta = (\goals,\goals)$, $(e_{3,1}\levidx{2})_\theta:\cS_{3,1}\times(\{(\pi_{3,1}\levidx{1})^\alpha_{\pick},(\pi_{3,1}\levidx{1})^\alpha_{\open},(\pi_{3,1}\levidx{1})^\beta_{\goals},\endaction\})\subseteq \cSA_{3,1}\levidx{2}\rightarrow \{0,1\}^4\times\{\pi_{\pick},\pi_{\open},\pi_{\key},\endaction\}\}$ is defined as 
\begin{align*} 
&(e_{3,1}\levidx{2})_\theta((\agent,\keypickvec,\dooropenvec,\goalpick),a)\\ 
&:=
\begin{cases}   
(\indic_{\{\goal\}}(\agent), \indic_{\{1\}}(\goalpick)\mult\indic_{\{\goal\}}(\agent),\goalpick,\goalpick,\pi_\theta) &\,,\text{if}\  a=(\pi_{3,1}\levidx{1})^\alpha_{\theta} \\
(\indic_{\{\goal\}}(\agent), \indic_{\{1\}}(\goalpick)\mult\indic_{\{\goal\}}(\agent),\goalpick,\goalpick,\pi_{\key}) &\,,\text{if}\  a=(\pi_{3,1}\levidx{1})^\beta_{\goals}\\
(\indic_{\{\goal\}}(\agent), \indic_{\{1\}}(\goalpick)\mult\indic_{\{\goal\}}(\agent),\goalpick,\goalpick,a) &\,,\text{otherwise.}\
\end{cases}
\end{align*}
\end{enumerate}

\subsubsection{Extended discussion and comparison of our work with \cite{Sukhbaatar2018LearningGE}}

\label{s: MazeBase-literature}

This MazeBase+ example shows sufficiently how planning like humans is realized by our framework. 
When trying to solve a difficult problem of picking up a goal, we first consider it at high level: we need to open door A, open door B, and then we pick up the goal. Then, we focus on each subtask and consider it more carefully: to open a door, we need to travel to a key, pick it up, travel to the door, and then open it; to pick up the goal, we need to travel to the goal, and then pick it up. 
These two logics are similar, so we extract out a single higher-order function $\overline{\pi}^{\mathrm{concat}}$.
Then, we also notice that we need to go to another location in $\gridworld$ for multiple times while avoiding blocks assuming all the doors are open, so we extract out a single navigation skill $\overline{\pi}^{\mathrm{nav}}$. Finally, we go down to the finest level: to travel to a certain destination while avoiding the blocks, we should first go to door A, then to doors B,C, ... before going directly to that destination. All these processes are similarly about navigation within a single room while avoiding blocks, so we extract a single basic navigation skill $\overline{\pi}^{\mathrm{nav}}_{\mathrm{dense}}$, which we may also transfer from solved problems in a different curriculum, such as here in the example of navigation and transportation with traffic jams.  

This is in some sense a mathematical realization of the idea in \cite{Sukhbaatar2018LearningGE}, and actually a generalization because \cite{Sukhbaatar2018LearningGE} only considers MDPs of difficulty up to $2$, where the student Bob and the manager Charlie are analogous to the student at level $1$ and $2$ respectively.
However, note that if \cite{Sukhbaatar2018LearningGE} were generalized na\"ively to multiple levels, the higher levels would not consist of independent MDPs, and the agent would still need to go through all the finer levels when learning at each level, which is computationally heavy.
This is not the case for our framework, at least in the case of value iteration; in the case of Q-learning we do need to estimate transition probabilities, rewards and discount factors from lower-level MDPs, albeit even in this case our framework can still enable us to do so with significant savings in terms of exploration and computational cost.

Another main advantage of our framework is that thanks to the introduction of partial policy generators, embedding generators and skills, we allow for very flexible transfer learning opportunities, not restricted to only concatenation as in \cite{Sukhbaatar2018LearningGE}. 
In particular, our embeddings take the current state, which often contains the information of goal state, as well as action coordinates as inputs, rather than only taking the goal state as input while keeping fixed the high-dimensional current state as in \cite{Sukhbaatar2018LearningGE}, giving us the opportunity to reduce the dimension and generate abstractions to a greater extent, as well as utilize the policy structure better. 
In this way, we encapsulate the knowledge from the policies we have already learned within the skills and focus on learning the unknown parts of the policies in new MDPs. 
Therefore, we avoid learning repeatedly similar subtasks for multiple times, such as traveling to certain destinations while avoiding blocks assuming all the doors are open, which greatly reduces computational cost. 
Last but not least, the framework in \cite{Sukhbaatar2018LearningGE} is unsupervised, whereas our framework is constructed for solving the target MDPs, such as $\MDP_{3,1}$ here, so it is supervised.

\subsection{Navigation and transportation with traffic jams}
\label{s:all-traffic}

\subsubsection{Formal definitions of the MDPs in the curriculum}
\label{s:MDPs-traffic}

The formal definition of the family of target MDPs is as follows: $\MDP_\kappa :=(\cS,\Sinit,\Sterm,\cA,P,R_\kappa,\Gamma)$, for $\kappa$ in some finite set $\mathcal{K}$ ($|\mathcal{K}|=6$), with 
\begin{equation}  \label{e:MDPs-traffic}
\begin{aligned}
\cS&:=\{(\cur,\dests):\cur,\dests\in\gridworld\}=\gridworld\times\gridworld\,,\\
\Sinit&:=\{(\cur,\dests)\in \cS:\cur\neq\dests\}\,\,,\,\,
\Sterm:=\{(\cur,\dests)\in \cS:\cur=\dests\}\,,\\
\cA&:=\{(a_\dir,a_\means):a_\dir\in \cA_{\dir}\cup\{\endactionfactor\}\,,\, a_\means\in \cA_{\means}\cup\{\endactionfactor\}\}\\&\,=(\cA_{\dir}\cup\{\endactionfactor\})\times (\cA_{\means}\cup\{\endactionfactor\})\,,\\
R_\kappa&((\cur,\dests),(a_\dir,a_\means),(\cur',\destnext))&\\
&:= R_{\dest}\mult\indic_{\{\destnext\}}(\cur')
+[(R_{\jam}+\frac{r_0}{v_{\motor}})\mult\indic_{\{\motor\}}(a_\means)+\frac{r_0}{\kappa}\mult\indic_{\{\car\}}(a_\means)]\\ 
&\quad\quad\mult[1-\indic_{\nojams}(\cur)\mult\indic_{\nojams}(\cur')]\\
&\quad\quad+[\frac{r_0}{v_{\motor}}\mult\indic_{\{\motor\}}(a_\means)+\frac{r_0}{v_{\car}}\mult\indic_{\{\car\}}(a_\means)]\mult\indic_{\nojams}(\cur)\mult\indic_{\nojams}(\cur')\,,\\
\Gamma&((\cur,\dests),(a_\dir,a_\means),(\cur',\destnext)):=\gamma_0\,,\\
P&((\cur,\dests),(a_\dir,a_\means),(\cur',\destnext)) \\&:= \indic_{\{\dests\}}(\destnext) \mult  \big[[1-\indic_{\gridworld}(\cur+a_{\dir})]\mult\indic_{\{\cur\}}(\cur')\\
&\quad+\indic_{\gridworld}(\cur+a_{\dir})\mult[p_s\mult\indic_{\{\cur+a_{\dir}\}}(\cur')+(1-p_s)\mult\indic_{\{\cur\}}(\cur')]\big]\,.
\end{aligned}
\end{equation}
For clarification, the ``+'' signs here in expressions such as $\cur+a_{\dir}$ are just vector sums in the grid world.  
Also, for the purpose of simplicity only, we set $x + \endactionfactor = \endactionfactor$ for any numerical value $x$ throughout, such as in the transition probabilities here. The teacher has the role to set all the parameters of the problems: $0<p_s\leq 1$, is the probability for any action to succeed, and here we set $p_s=0.9$; $R_{\dest}>0$ is some large positive reward (here $R_{\dest}=10^4$) set for reaching $\dests$; $R_{\jam}<\frac{r_0}{\kappa}-\frac{r_0}{v_{\motor}}$ is some large negative reward (here $R_{\jam}=-10^3$) if the agent drives along, enters, or leaves $\jams$ using $\motor$; $0<\gamma_0<1$ is some large discount factor (here $\gamma_0=0.999$). $1/\kappa = 2.4,2.8,3.2,3.6,4,4.4$ respectively for $\{\MDP_{\kappa}\}_{\kappa\in \mathcal{K}}$, modeling the dependence of the speed of the car in traffic regions as a function of the heaviness of traffic. As a summary, Table~\ref{t:param-nav} lists all the values of the parameters of the MDPs in this example.

\setlength{\tabcolsep}{2pt}
\begin{table}[H]
\caption{Parameters in the example of navigation and transportation with traffic jams}
\centering
\begin{tabular}{|c|c|c|c|c|c|c|c|c|c|}
\cline{2-10}
\multicolumn{1}{l|}{}                                                                                                        & $\jams$& $p_s$ & $v_{\motor}$ & $v_{\car}$ & $\kappa$ & $R_{\dest}$ & $R_{\jam}$ & $r_0$ & $\gamma_0$\\ \hline
\multicolumn{1}{|c|}{\begin{tabular}[c]{@{}c@{}}$\MDP_{2,1}$\end{tabular}}       &$[1,15]\times\{6\}\cup\{5\}\times [1,8]$ & $0.9$ & $1$ & $0.6$  & $1/2.4$ & $10^4$  & $-10^3$ & $-10$ & $0.999$\\ \hline
\multicolumn{1}{|c|}{\begin{tabular}[c]{@{}c@{}}$\MDP_{2,2}$\end{tabular}}       &$[1,15]\times\{6\}\cup\{5\}\times [1,8]$ & $0.9$ & $1$ & $0.6$  & $1/2.8$ & $10^4$  & $-10^3$ & $-10$ & $0.999$\\ \hline
\multicolumn{1}{|c|}{\begin{tabular}[c]{@{}c@{}}$\MDP_{2,3}$\end{tabular}}      &$[1,15]\times\{6\}\cup\{5\}\times [1,8]$ & $0.9$ & $1$ & $0.6$  & $1/3.2$ & $10^4$  & $-10^3$ & $-10$ & $0.999$\\ \hline
\multicolumn{1}{|c|}{\begin{tabular}[c]{@{}c@{}}$\MDP_{2,4}$\end{tabular}}       &$[1,15]\times\{6\}\cup\{5\}\times [1,8]$ & $0.9$ & $1$ & $0.6$  & $1/3.6$ & $10^4$  & $-10^3$ & $-10$ & $0.999$\\ \hline
\multicolumn{1}{|c|}{\begin{tabular}[c]{@{}c@{}}$\MDP_{2,5}$\end{tabular}}      &$[1,15]\times\{6\}\cup\{5\}\times [1,8]$ & $0.9$ & $1$ & $0.6$  & $1/4$ & $10^4$  & $-10^3$ & $-10$ & $0.999$\\ \hline
\multicolumn{1}{|c|}{\begin{tabular}[c]{@{}c@{}}$\MDP_{2,6}$\end{tabular}}       &$[1,15]\times\{6\}\cup\{5\}\times [1,8]$ & $0.9$ & $1$ & $0.6$  & $1/4.4$ & $10^4$  & $-10^3$ & $-10$ & $0.999$\\ \hline
\multicolumn{1}{|c|}{\begin{tabular}[c]{@{}c@{}}$\MDP_{2,7}$\end{tabular}}  &$[1,15]\times\{1,4,7\}\cup\{1,4,7,10,13\}\times [1,8]$ & $0.9$ & $1$ & $1/1.05$  & $1/1.1$ & $10^4$  & $-10^3$ & $-10$ & $0.999$\\ \hline
\multicolumn{1}{|c|}{\begin{tabular}[c]{@{}c@{}}$\MDP_{1,1}$\end{tabular}}      &$[1,15]\times\{6\}\cup\{5\}\times [1,8]$ & $0.9$ & $1$ & $0.6$  & $1/2.5$ & $10^4$  & $-10^3$ & $-10$ & $0.999$\\ \hline
\multicolumn{1}{|c|}{\begin{tabular}[c]{@{}c@{}}$\MDP_{1,2}$\end{tabular}}     &$[1,15]\times\{6\}\cup\{5\}\times [1,8]$ & $0.9$ & $1$ & $0.6$  & $1/4$ & $10^4$  & $-10^3$ & $-10$ & $0.999$\\ \hline
\multicolumn{1}{|c|}{\begin{tabular}[c]{@{}c@{}}$\MDP_{1,3}$\end{tabular}}      &$[1,15]\times\{1,4,7\}\cup\{1,4,7,10,13\}\times [1,8]$ & $0.9$ & $1$ & $1/1.05$  & $1/1.1$ & $10^4$  & $-10^3$ & $-10$ & $0.999$\\ \hline
\end{tabular}
\label{t:param-nav}
\end{table}

The following is the detailed definition of $\MDP_{1,n}:=(\cS,\Sinit,\Sterm,\cA_1,P_1,R_{1,n},\Gamma_{1}) (1\leq n\leq n_1=2)$, with $\cS,\Sinit,\Sterm$ defined in the same manner as in $\MDP_\kappa$: 
\begin{equation} \label{e:MDP-basic-traffic}
\begin{aligned}
\cA_{1}&:=\cA_{\dir}\cup\{{\endactionfactor}\}\,, \\
R_{1,n}((\cur,\dests),a,(\cur',\destnext))&:=R_{\dest}\mult\indic_{\{\destnext\}}(\cur')+\frac{r_0}{\kappa}\mult[1-\indic_{\nojams}(\cur)\mult\indic_{\nojams}(\cur')]\\
&\qquad+\frac{r_0}{v_{\motor}}\mult\indic_{\nojams}(\cur)\mult\indic_{\nojams}(\cur')\,,\\
\Gamma_{1}((\cur,\dests),a,(\cur',\destnext))&:=\gamma_0\,,\\
P_{1}((\cur,\dests),a,(\cur',\destnext))&:=P((\cur,\dests),(a,\motor),(\cur',\destnext))\,,
\end{aligned}
\end{equation}
where $1/\kappa = 2.5,4$ in $\MDP_{1,1}, \MDP_{1,2}$, respectively.

\subsubsection{(Partial) policies for target MDPs}
\label{s:policies-traffic}

\begin{equation} 
(\pi_{\theta'})_{\means}((\cur,\dests),(0,a_{\means})) := \indic_{\{\theta'\}}(a_{\means})\,.
\label{eq:pithetameans}
\end{equation}

\begin{align}   \label{e:second-level-policy-eqn}
\pi_{\kappa,\ast}\levidx{2}&(s,(\pi^{\mathrm{nav}_{\kappa'}}_{\theta})_{\dir}\otimes (\pi_{\theta'})_{\means})=
\indic_{\{\theta_{0,\ast}(s)\}}(\theta)\mult\indic_{\{\kappa\}}(\kappa')\mult\sum_{a_\dir\in\cA_{\dir}}(\pi^{\mathrm{nav}_{\kappa}}_{\theta_{0,\ast}(s)})_{\dir}(s,(a_\dir,0))\\ \nonumber
&\quad \mult [\indic_{\{\car\}}(\theta')\mult\indic_{\djams}(s,(a_\dir,0))+\indic_{\{\motor\}}(\theta')\mult\indic_{\dnojams}(s,(a_\dir,0))]\,,
\end{align}
with 
\begin{align} \label{e:second-level-policy}
\theta_{0,\ast}(s)\in \argmin_{\theta_0}\min_{\theta_1,\cdots, \theta_{\tau-1},,\theta_0',\theta'_1,\cdots, \theta'_{\tau-1}}\mathbb{E}_{\tau,S_{1:\tau}}[(R_{\kappa}\levidx{2})_{0,\tau}  |&S_0=s,A_t = (\pi^{\mathrm{nav}_{\kappa}}_{\theta_t})_{\dir}\otimes (\pi_{\theta'_t})_{\means} \\ \nonumber &\text{for any } 0\leq t\leq \tau-1]\,.
\end{align}

\begin{align}  \label{e:first-level-policy-eqn}
\pi_{\kappa}(s,(a_\dir,a_\means))=&({\pi}^{\mathrm{nav}_{\kappa}})_\dir(s,(a_\dir,0))\mult [\indic_{\{\car\}}(a_\means)\mult\indic_{\djams}((s,(a_\dir,0)))\\ \nonumber&+\indic_{\{\motor\}}(a_\means)\mult\indic_{\dnojams}((s,(a_\dir,0)))]\,.
\end{align}

\subsubsection{Embeddings, embedding generators and skills}
\label{s:skills-traffic}

\begin{enumerate}[leftmargin=0cm,itemsep=0pt]
\item[$\bullet$] For $\MDP_{1,n}\, (1\leq n\leq n_1=2)$ of difficulty 1:
\begin{equation}
\label{e:e11decomp-traffic}
\begin{aligned}
(e_{\decomp})_{1}\levidx{1}:\cSA_{1}\levidx{1}&\rightarrow \{(\cur,\dests,a_{\dir}):\cur,\dests\in\gridworld,a_{\dir}\in\cA_{\dir}\cup\{\endactionfactor\}\}\\
((\cur,\dests),a_{\dir})&\mapsto(\cur,\dests,a_{\dir})\,.
\end{aligned}
\end{equation}
\begin{equation}
\label{e:navskill-traffic}
\begin{aligned}
\overline{\pi}^{\mathrm{nav}_{n}}_{\mathrm{obstacles}}:(e_{\decomp})_{1}\levidx{1}(\cSA_{1}\levidx{1})&\rightarrow [0,1]\\
(\cur,\dests,a_{\dir})&\mapsto\pi_{1,n,\ast}\levidx{1}((\cur,\dests),a_{\dir})\,.
\end{aligned}
\end{equation}

\item[$\bullet$] For target $\MDP_\kappa \, (\kappa\in \mathcal{K})$ of difficulty 2, with $\Theta_{\dir}=\{\changemeans,\nochangemeans\}$, and with $(e^1)\levidx{\dir}_{\theta}:\theta \rightarrow \gridworld\times\gridworld\times(\cA_{\dir}\cup\{\endactionfactor\}) (\theta\in\Theta_{\dir})$ being defined as
\begin{equation*}
(e^1)\levidx{\dir}_{\theta}(((\cur,\dests),(a_\dir,0))):= (\cur,\dests,a_\dir)\,,
\end{equation*}
\begin{equation}
\label{e:E1alpha}
\begin{aligned}
E\levidx{1}_{\alpha}:\Theta_{\dir} &\rightarrow \{(e^1)\levidx{\dir}_{\theta}:\theta\in\Theta_{\dir}\}\\
\theta&\mapsto (e^1)\levidx{\dir}_{\theta}\,.
\end{aligned}
\end{equation}
With $(e^1)\levidx{\means}_{\theta}:\cSA\levidx{1}_{\means}\rightarrow \{0,1\} (\theta\in\cA_{\means})$ being defined as
\begin{equation*}
(e^1)\levidx{\means}_{\theta}(((\cur,\dests),(0,a_\means))):= \indic_{\{\theta\}}(a_\means)\,,
\end{equation*}
which is $(\pi_{\theta})_{\means}$ as in \eqref{eq:pithetameans},
we have 
\begin{equation}
\label{e:E1beta}
\begin{aligned}
E\levidx{1}_{\beta}:\cA_{\means}&\rightarrow \{(e^1)\levidx{\means}_{\theta}:\theta\in\cA_{\means}\} \\
\theta&\mapsto (e^1)\levidx{\means}_{\theta}\,.
\end{aligned}
\end{equation}
And finally,
\begin{align} \label{e:embedding-transfer-traffic} \nonumber
&(e_{\decomp})_{2,1}\levidx{2}((\cur,\dests),((\pi^{\mathrm{nav}_{\kappa_1}}_{\theta})_{\dir}\otimes (\pi_{\theta'})_{\means}))\\
&\quad:=(\{\indic_{\theta}(((\cur,\dests),(a_\dir,0)))\}_{a_\dir\in\cA_\dir}, \{\indic_{\djams}(((\cur,\dests),(a_\dir,0)))\}_{a_{\dir}\in\cA_{\dir}}\,,
\\ \nonumber &\qquad\quad\{({\pi}^{\mathrm{nav}_{\kappa_1}})_\dir(((\cur,\dests),a_\dir))\}_{a_{\dir}\in\cA_{\dir}}\,,\quad\theta')\,.
\end{align}

The higher-order function for selecting the index of the navigation policy and the means of transportation $\overline{\pi}^{\mathrm{transport}}:\{(e_{\mathtt{pause}},e_{\mathtt{jams}}, {e}^{\mathrm{nav}},a_{\means})\}\rightarrow [0,1]$ is as follows:
\begin{align} \label{e:transport-skill} \nonumber
\overline{\pi}^{\mathrm{transport}}&(e_{\mathtt{pause}},e_{\mathtt{jams}}, {e}^{\mathrm{nav}},a_{\means})\\
=&\sum_{a_\dir\in\cA_{\dir}}e^{\mathrm{nav}}(a_\dir)\mult e_{\mathtt{pause}}(a_\dir) \mult [\indic_{\{\car\}}(a_{\means})\mult e_{\mathtt{jams}}(a_\dir)\\ 
\nonumber &\quad+\indic_{\{\motor\}}(a_{\means})\mult \big(1-e_{\mathtt{jams}}(a_\dir)\big)]\,.
\end{align}
\end{enumerate}

\section{Numerical experiments on the theoretical analysis in Sec.~\ref{s:theory}}
\label{s:numexptightnessbounds}

We conclude the appendix with numerical experiments on the theoretical analysis in Sec.~\ref{s:theory}. We consider the example of navigation and transportation with traffic jams for illustrating our theoretical results, in particular Thm.~\ref{t:transfer-3}.

\begin{figure}
\centering
\includegraphics[width=\textwidth]{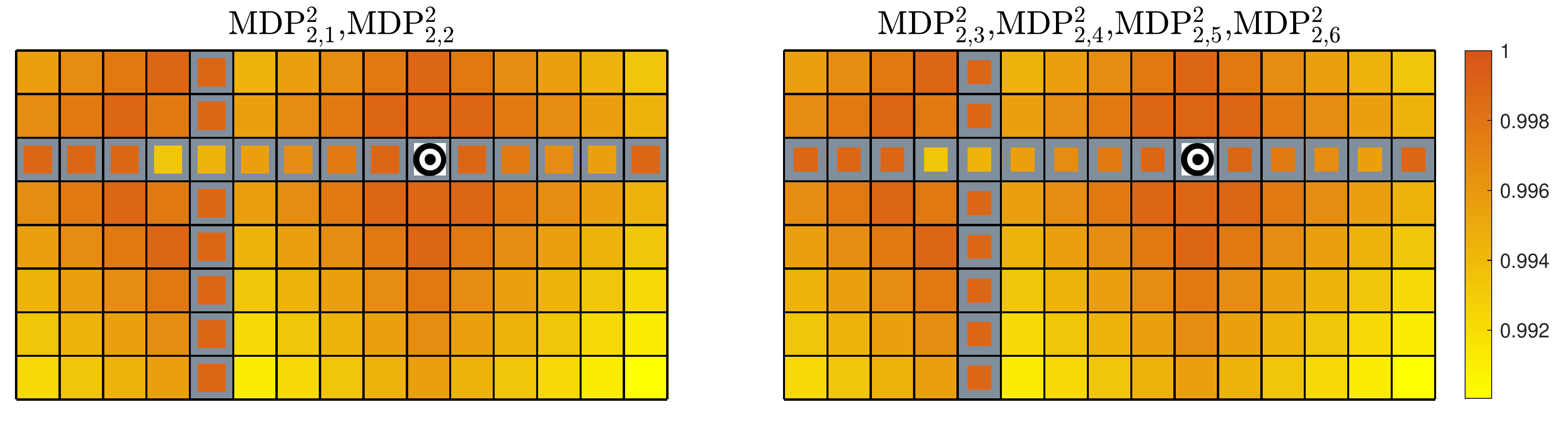}
\caption{The discount factors, represented by heat maps, at the actions selected by the optimal policy for $\{\MDP\levidx{2}_{2,n}\}_{n=1}^{n_2}$ in the example of navigation and transportation with traffic jams (It happens to be the case that the agent will reach the next state with probability one at all the states when selecting the actions according to the optimal policy). The regions with traffic jams and the destination are marked in this figure similar to Fig.~\ref{f:traffic}. These heat maps share similar color patterns with the ones representing value functions in Fig.~\ref{f:traffic}, as explained by Thm.~\ref{t:transfer-3}. The main message of this figure is on the relative values of discount factors at different locations, not the absolute magnitude, because the discount factors are set to be really close to $1$ in the original MDPs. In examples where the discount factors are smaller, they will have a more important role in convergence of the value function during value iteration.
} \label{f:navigation_discount} 
\end{figure}

\begin{figure}
\centering
\includegraphics[width=\textwidth]{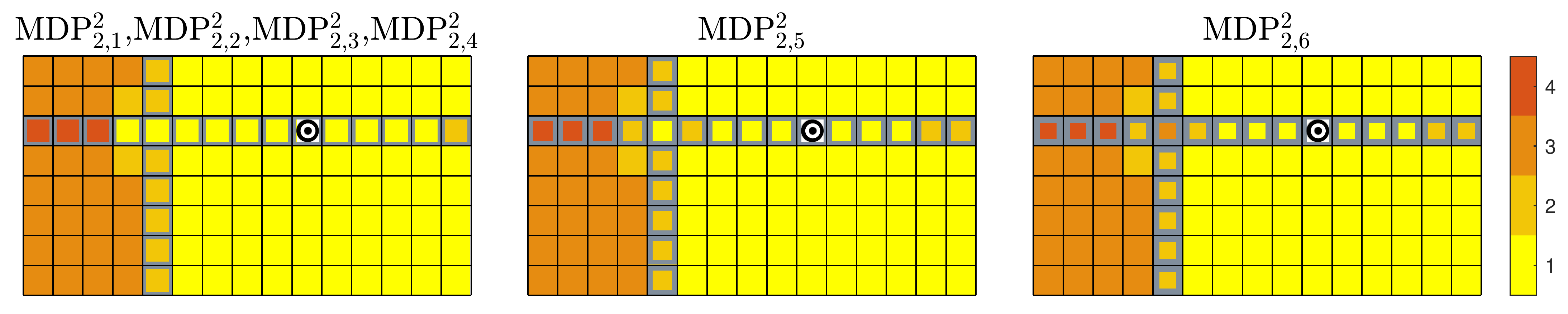}
\caption{The times of convergence, represented by heat maps, to the optimal value functions for $\{\MDP\levidx{2}_{2,n}\}_{n=1}^{n_2}$ in the example of navigation and transportation with traffic jams. The regions with traffic jams and the destination are marked in this figure similar to Fig.~\ref{f:traffic}. As expected, the more traffic roads the agent needs to pass, and the heavier the traffic, the longer the convergence time. These heat maps share similar color patterns, upon reversing the color bar, with the ones representing value functions in Fig.~\ref{f:traffic}, as explained by Thm.~\ref{t:transfer-3}. 
} \label{f:navigation_time} 
\end{figure}

\begin{figure}
\centering
\includegraphics[width=\textwidth]{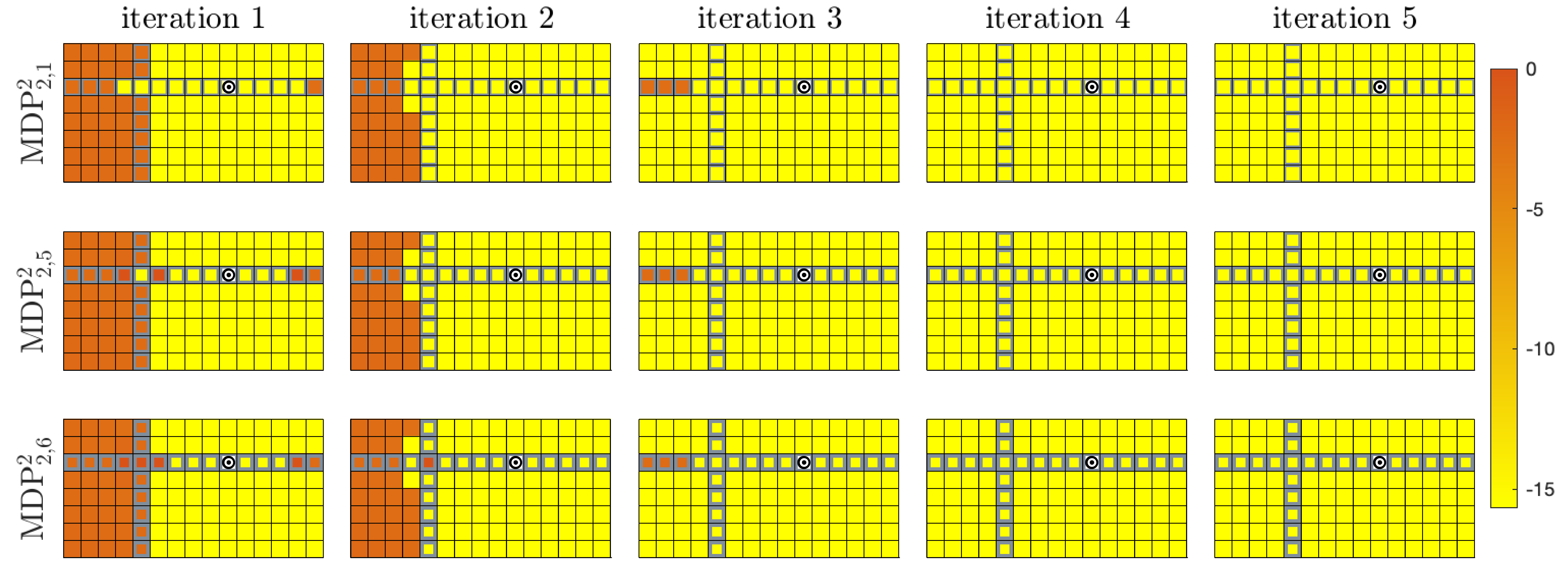}
\caption{Examination of the tightness of the bounds in Thm.~\ref{t:transfer-3} using $\{\MDP\levidx{2}_{2,n}\}_{n=1}^{n_2}$ in the example of navigation and transportation with traffic jams, by plotting the relative errors in the logarithmic scale $\log(\frac{\mathrm{upper\_bound}-\mathrm{lower\_bound}}{\mathrm{true\_value}})$, represented by heat maps. Here, for each location, and for each iteration of the value iteration algorithm, $\mathrm{true\_value}$ is the value of the value function at the current iteration. The heat maps of this quantity for $\MDP\levidx{2}_{2,1}, \MDP\levidx{2}_{2,2}, \MDP\levidx{2}_{2,3}$, and $\MDP\levidx{2}_{2,4}$ are similar, so we only plot the relative errors for $\MDP\levidx{2}_{2,1}, \MDP\levidx{2}_{2,5}$, and $\MDP\levidx{2}_{2,6}$ across different iterations until convergence to avoid redundancy. The regions with traffic jams and the destination are marked in this figure similar to Fig.~\ref{f:traffic}. As expected, the relative error quickly decays to zero as soon as the value function at the current iteration is reasonably ``close'' to the true value function. Actually, the color patterns of these heat maps at beginning iterations match pretty nicely with the ones representing convergence times in Fig.~\ref{f:navigation_time}.
} \label{f:sanity_check} 
\end{figure}

\let\oldbibliography\thebibliography
\renewcommand{\thebibliography}[1]{
  \oldbibliography{#1}
  \setlength{\itemsep}{0pt}
}
\bibliography{all_refs}

@inproceedings{AbbeelNg2004Apprenticeship,
  author    = {Pieter Abbeel and Andrew Y. Ng},
  title     = {Apprenticeship Learning via Inverse Reinforcement Learning},
  booktitle = {Proceedings of the 21st International Conference on Machine Learning},
  year      = {2004},
  publisher = {ACM},
  doi       = {10.1145/1015330.1015430},
  url       = {https://dl.acm.org/doi/10.1145/1015330.1015430}
}

@article{Adams2022IRLSurvey,
  author    = {Stephen Adams and Tyler Cody and Philip A. Beling},
  title     = {A Survey of Inverse Reinforcement Learning},
  journal   = {Artificial Intelligence Review},
  year      = {2022},
  volume    = {55},
  number    = {6},
  pages     = {4307--4346},
  doi       = {10.1007/s10462-021-10108-x},
  url       = {https://link.springer.com/article/10.1007/s10462-021-10108-x}
}

@inproceedings{Andreas2017PolicySketches,
  author    = {Jacob Andreas and Dan Klein and Sergey Levine},
  title     = {Modular Multitask Reinforcement Learning with Policy Sketches},
  booktitle = {Proceedings of the 34th International Conference on Machine Learning},
  year      = {2017},
  volume    = {70},
  series    = {Proceedings of Machine Learning Research},
  pages     = {166--175},
  publisher = {PMLR},
  url       = {https://proceedings.mlr.press/v70/andreas17a/andreas17a.pdf}
}

@article{BaharFrohmHachtelMaciiPardoSomenzi1997ADD,
  author  = {R. Iris Bahar and Erica A. Frohm and Charles M. Gaona and Gary D. Hachtel and Enrico Macii and Abelardo Pardo and Fabio Somenzi},
  title   = {Algebraic Decision Diagrams and Their Applications},
  journal = {Formal Methods in System Design},
  year    = {1997},
  volume  = {10},
  number  = {2-3},
  pages   = {171--206},
  doi     = {10.1023/A:1008699807402}
}

@inproceedings{Bansal2019HOList,
  author    = {Bansal, Kshitij and Loos, Sarah and Rabe, Markus and Szegedy, Christian and Wilcox, Stewart},
  title     = {{HOList}: An Environment for Machine Learning of Higher Order Logic Theorem Proving},
  booktitle = {Proceedings of the 36th International Conference on Machine Learning (ICML)},
  year      = {2019},
  series    = {Proceedings of Machine Learning Research},
  volume    = {97}
}

@inproceedings{Barreto2017SF,
  author    = {Andr{\'e} Barreto and Will Dabney and R{\'e}mi Munos and Jonathan J. Hunt and Tom Schaul and Hado van Hasselt and David Silver},
  title     = {Successor Features for Transfer in Reinforcement Learning},
  booktitle = {Advances in Neural Information Processing Systems 30},
  year      = {2017},
  url       = {https://papers.neurips.cc/paper/6994-successor-features-for-transfer-in-reinforcement-learning.pdf}
}

@inproceedings{Bengio2009Curriculum,
  author    = {Yoshua Bengio and J{\'e}r{\^o}me Louradour and Ronan Collobert and Jason Weston},
  title     = {Curriculum Learning},
  booktitle = {Proceedings of the 26th International Conference on Machine Learning},
  year      = {2009},
  pages     = {41--48},
  publisher = {ACM},
  url       = {https://dl.acm.org/doi/10.1145/1553374.1553380}
}

@inproceedings{Boutilier1997CorrelatedActionEffects,
  author    = {Craig Boutilier},
  title     = {Correlated Action Effects in Decision Theoretic Regression},
  booktitle = {Proceedings of the Thirteenth Conference on Uncertainty in Artificial Intelligence (UAI)},
  year      = {1997},
  pages     = {30--37}
}

@article{BoutilierDeanHanks1999DTP,
  author  = {Craig Boutilier and Thomas Dean and Steve Hanks},
  title   = {Decision-Theoretic Planning: Structural Assumptions and Computational Leverage},
  journal = {Journal of Artificial Intelligence Research},
  year    = {1999},
  volume  = {11},
  pages   = {1--94},
  doi     = {10.1613/jair.555}
}

@inproceedings{BoutilierDeardenGoldszmidt1995,
  author    = {Craig Boutilier and Richard Dearden and Mois{\'e}s Goldszmidt},
  title     = {Exploiting Structure in Policy Construction},
  booktitle = {Proceedings of the Fourteenth International Joint Conference on Artificial Intelligence (IJCAI)},
  year      = {1995},
  pages     = {1104--1111}
}

@article{BoutilierDeardenGoldszmidt2000SDPFactored,
  author  = {Craig Boutilier and Richard Dearden and Mois{\'e}s Goldszmidt},
  title   = {Stochastic Dynamic Programming with Factored Representations},
  journal = {Artificial Intelligence},
  year    = {2000},
  volume  = {121},
  number  = {1-2},
  pages   = {49--107},
  doi     = {10.1016/S0004-3702(00)00033-9}
}

@inproceedings{BoutilierFriedmanGoldszmidtKoller1996CSI,
  author    = {Craig Boutilier and Nir Friedman and Mois{\'e}s Goldszmidt and Daphne Koller},
  title     = {Context-Specific Independence in {B}ayesian Networks},
  booktitle = {Proceedings of the Twelfth Conference on Uncertainty in Artificial Intelligence (UAI)},
  year      = {1996},
  pages     = {115--123}
}

@article{BrafmanTennenholtz2002Rmax,
  author  = {Brafman, Ronen I. and Tennenholtz, Moshe},
  title   = {R-MAX---A General Polynomial Time Algorithm for Near-Optimal Reinforcement Learning},
  journal = {Journal of Machine Learning Research},
  year    = {2002},
  volume  = {3},
  pages   = {213--231}
}

@inproceedings{Chandak2019ActionRepr,
  author    = {Chandak, Yash and Theocharous, Georgios and Kostas, James and Jordan, Scott and Thomas, Philip},
  title     = {Learning Action Representations for Reinforcement Learning},
  booktitle = {Proceedings of the 36th International Conference on Machine Learning},
  series    = {Proceedings of Machine Learning Research},
  volume    = {97},
  pages     = {941--950},
  year      = {2019},
  editor    = {Chaudhuri, Kamalika and Salakhutdinov, Ruslan},
  publisher = {PMLR},
  url       = {https://proceedings.mlr.press/v97/chandak19a.html}
}

@inproceedings{ChenWangJiang2021FMDPBF,
  author    = {Yudong Chen and Rui Wang and Nan Jiang},
  title     = {Model-Free Reinforcement Learning in Factored {MDP}s with Unknown Structure},
  booktitle = {International Conference on Learning Representations (ICLR)},
  year      = {2021},
  note      = {Poster},
  url       = {https://openreview.net/forum?id=fmtSg8591Q}
}

@inproceedings{ChitnisSilverVeness2020CAMPs,
  author    = {Rohan Chitnis and Tom Silver and Joel Veness},
  title     = {{CAMPs}: Learning Context-Specific Abstractions for Efficient Planning in Factored {MDP}s},
  booktitle = {Conference on Robot Learning (CoRL)},
  year      = {2020}
}

@inproceedings{Cui2015CompositionalActions,
  author    = {Yiran Cui and Chunlin Wu and Qiang Liu},
  title     = {Compositional Reinforcement Learning with Factored Action Spaces},
  booktitle = {Proceedings of the Twenty-Ninth Conference on Neural Information Processing Systems (NeurIPS)},
  year      = {2015}
}

@inproceedings{DayanHinton1993,
  author    = {Peter Dayan and Geoffrey E. Hinton},
  title     = {Feudal Reinforcement Learning},
  booktitle = {Advances in Neural Information Processing Systems 5},
  year      = {1993},
  pages     = {271--278},
  publisher = {Morgan Kaufmann},
  url       = {https://papers.nips.cc/paper/1992/file/d14220ee66aeec73c49038385428ec4c-Paper.pdf}
}

@inproceedings{DeanGivan1997,
  author    = {Thomas Dean and Robert Givan},
  title     = {Model Minimization in Markov Decision Processes},
  booktitle = {Proceedings of the AAAI Conference on Artificial Intelligence},
  year      = {1997},
  pages     = {106--111},
  url       = {https://cdn.aaai.org/AAAI/1997/AAAI97-017.pdf}
}

@incollection{Degris2013FMDPChapter,
  author    = {Thomas Degris and Olivier Sigaud and Pierrick Wuillemin},
  title     = {Factored Markov Decision Processes},
  booktitle = {Markov Decision Processes in Artificial Intelligence},
  publisher = {Wiley},
  year      = {2013},
  doi       = {10.1002/9781118557426.ch4}
}

@article{Dempster1977EM,
  author  = {Dempster, A. P. and Laird, N. M. and Rubin, D. B.},
  title   = {Maximum Likelihood from Incomplete Data via the {EM} Algorithm},
  journal = {Journal of the Royal Statistical Society: Series B (Methodological)},
  year    = {1977},
  volume  = {39},
  number  = {1},
  pages   = {1--38},
  doi     = {10.1111/j.2517-6161.1977.tb01600.x}
}

@article{Duan2016RL2,
  author    = {Yan Duan and John Schulman and Xi Chen and Peter L. Bartlett and Ilya Sutskever and Pieter Abbeel},
  title     = {{RL}$^{2}$: Fast Reinforcement Learning via Slow Reinforcement Learning},
  journal   = {arXiv preprint arXiv:1611.02779},
  year      = {2016},
  url       = {https://arxiv.org/abs/1611.02779}
}

@article{Dwiel2019GoalSpace,
  author    = {Zach Dwiel and Madhavun Candadai and Mariano Phielipp and Arjun K. Bansal},
  title     = {Hierarchical Policy Learning is Sensitive to Goal Space Design},
  journal   = {arXiv preprint arXiv:1905.01537},
  year      = {2019},
  url       = {https://arxiv.org/abs/1905.01537}
}

@article{Ecoffet2021GoExplore,
  author    = {Adrien Ecoffet and Joost Huizinga and Joel Lehman and Kenneth O. Stanley and Jeff Clune},
  title     = {First return, then explore},
  journal   = {Nature},
  year      = {2021},
  volume    = {590},
  number    = {7847},
  pages     = {580--586},
  doi       = {10.1038/s41586-020-03157-9},
  url       = {https://www.nature.com/articles/s41586-020-03157-9}
}

@inproceedings{Ellis2021DreamCoder,
  author    = {Ellis, Kevin and Wong, Catherine and Nye, Maxwell and Sabl{\'e}-Meyer, Mathias and Cary, Lucas and Morales, Lucy and Hewitt, Luke and Solar-Lezama, Armando and Tenenbaum, Joshua B.},
  title     = {DreamCoder: Bootstrapping Inductive Program Synthesis with Wake-Sleep Library Learning},
  booktitle = {Proceedings of the 42nd ACM SIGPLAN International Conference on Programming Language Design and Implementation (PLDI)},
  year      = {2021},
  doi       = {10.1145/3453483.3454080}
}

@inproceedings{EllisRitchieSolarLezamaTenenbaum2018Graphics,
  author    = {Ellis, Kevin and Ritchie, Daniel and Solar-Lezama, Armando and Tenenbaum, Joshua B.},
  title     = {Learning to Infer Graphics Programs from Hand-Drawn Images},
  booktitle = {Advances in Neural Information Processing Systems (NeurIPS)},
  year      = {2018}
}

@inproceedings{Eysenbach2019DIAYN,
  author    = {Benjamin Eysenbach and Abhishek Gupta and Julian Ibarz and Sergey Levine},
  title     = {Diversity is All You Need: Learning Skills without a Reward Function},
  booktitle = {International Conference on Learning Representations},
  year      = {2019},
  url       = {https://openreview.net/forum?id=SJx63jRqFm}
}

@inproceedings{FengHansen2002sLAO,
  author    = {Zhe Feng and Eric A. Hansen},
  title     = {Symbolic {LAO}$^\ast$ Search for Factored Markov Decision Processes},
  booktitle = {Proceedings of the Eighteenth National Conference on Artificial Intelligence (AAAI)},
  year      = {2002},
  pages     = {123--128}
}

@inproceedings{FengHansenZilberstein2003sRTDP,
  author    = {Zhe Feng and Eric A. Hansen and Shlomo Zilberstein},
  title     = {Symbolic Generalization for On-line Planning},
  booktitle = {Proceedings of the Nineteenth Conference on Uncertainty in Artificial Intelligence (UAI)},
  year      = {2003},
  pages     = {210--217}
}

@inproceedings{Finn2016GCL,
  author    = {Chelsea Finn and Sergey Levine and Pieter Abbeel},
  title     = {Guided Cost Learning: Deep Inverse Optimal Control via Policy Optimization},
  booktitle = {Proceedings of the 33rd International Conference on Machine Learning},
  year      = {2016},
  volume    = {48},
  series    = {Proceedings of Machine Learning Research},
  pages     = {49--58},
  publisher = {PMLR},
  url       = {https://proceedings.mlr.press/v48/finn16.pdf}
}

@inproceedings{Finn2017MAML,
  author    = {Chelsea Finn and Pieter Abbeel and Sergey Levine},
  title     = {Model{-}Agnostic Meta{-}Learning for Fast Adaptation of Deep Networks},
  booktitle = {Proceedings of the 34th International Conference on Machine Learning},
  year      = {2017},
  volume    = {70},
  series    = {Proceedings of Machine Learning Research},
  pages     = {1126--1135},
  publisher = {PMLR},
  url       = {https://proceedings.mlr.press/v70/finn17a/finn17a.pdf}
}

@inproceedings{Florensa2017RCG,
  author    = {Carlos Florensa and David Held and Markus Wulfmeier and Michael Zhang and Pieter Abbeel},
  title     = {Reverse Curriculum Generation for Reinforcement Learning},
  booktitle = {Conference on Robot Learning},
  series    = {Proceedings of Machine Learning Research},
  volume    = {78},
  year      = {2017},
  pages     = {482--495},
  publisher = {PMLR},
  url       = {https://proceedings.mlr.press/v78/florensa17a/florensa17a.pdf}
}

@inproceedings{Frans2018MLSH,
  author    = {Kevin Frans and Jonathan Ho and Xi Chen and Pieter Abbeel and John Schulman},
  title     = {Meta Learning Shared Hierarchies},
  booktitle = {International Conference on Learning Representations},
  year      = {2018},
  url       = {https://openreview.net/forum?id=SyX0IeWAW}
}

@inproceedings{Fu2018AIRL,
  author    = {Justin Fu and Katie Luo and Sergey Levine},
  title     = {Learning Robust Rewards with Adversarial Inverse Reinforcement Learning},
  booktitle = {International Conference on Learning Representations},
  year      = {2018},
  url       = {https://openreview.net/forum?id=rkHywl-A-}
}

@article{GivanLeachDean2000BPMDP,
  author  = {Robert Givan and Sonia M. Leach and Thomas Dean},
  title   = {Bounded-Parameter Markov Decision Processes},
  journal = {Artificial Intelligence},
  year    = {2000},
  volume  = {122},
  number  = {1-2},
  pages   = {71--109},
  doi     = {10.1016/S0004-3702(00)00048-0}
}

@article{Gregor2016VIC,
  author    = {Karol Gregor and Danilo J. Rezende and Daan Wierstra},
  title     = {Variational Intrinsic Control},
  journal   = {arXiv preprint arXiv:1611.07507},
  year      = {2016},
  url       = {https://arxiv.org/abs/1611.07507}
}

@article{Grigorescu2020AutonomousDrivingSurvey,
  author  = {Grigorescu, Sorin Mihai and Tr\u{a}snea, Bogdan and Cocias, Tiberiu and M\u{a}ce\c{s}anu, Gigel},
  title   = {A survey of deep learning techniques for autonomous driving},
  journal = {Journal of Field Robotics},
  year    = {2020},
  volume  = {37},
  number  = {3},
  pages   = {362--386},
  doi     = {10.1002/rob.21918}
}

@inproceedings{GuestrinKollerParr2001MaxNorm,
  author    = {Carlos Guestrin and Daphne Koller and Ronald Parr},
  title     = {Max-norm Projections for Factored {MDP}s},
  booktitle = {Proceedings of the Seventeenth International Joint Conference on Artificial Intelligence (IJCAI)},
  year      = {2001},
  pages     = {673--682}
}

@article{GuestrinKollerParrVenkataraman2003JAIR,
  author  = {Carlos Guestrin and Daphne Koller and Ronald Parr and Shobha Venkataraman},
  title   = {Efficient Solution Algorithms for Factored {MDP}s},
  journal = {Journal of Artificial Intelligence Research},
  year    = {2003},
  volume  = {19},
  pages   = {399--468},
  doi     = {10.1613/jair.1000}
}

@inproceedings{GuestrinPatrascuSchuurmans2002Exploration,
  author    = {Carlos Guestrin and Relu Patrascu and Dale Schuurmans},
  title     = {Algorithm-Directed Exploration for Model-Based Reinforcement Learning in Factored {MDP}s},
  booktitle = {Proceedings of the Nineteenth International Conference on Machine Learning (ICML)},
  year      = {2002},
  pages     = {235--242}
}

@inproceedings{GuoWeiLuo2021OracleEfficientUnknownStructure,
  author    = {Siyuan Guo and Chen-Yu Wei and Haipeng Luo},
  title     = {Oracle-Efficient Reinforcement Learning in Factored {MDP}s with Unknown Structure},
  booktitle = {Advances in Neural Information Processing Systems (NeurIPS)},
  year      = {2021}
}

@inproceedings{HewittTenenbaum2020MWS,
  author    = {Hewitt, Luke B. and Tenenbaum, Joshua B.},
  title     = {A New Approach to Learning Compact, Interpretable Representations},
  booktitle = {Proceedings of the Thirty-Sixth Conference on Uncertainty in Artificial Intelligence (UAI)},
  year      = {2020},
  series    = {Proceedings of Machine Learning Research},
  volume    = {124}
}

@inproceedings{HoErmon2016GAIL,
  author    = {Jonathan Ho and Stefano Ermon},
  title     = {Generative Adversarial Imitation Learning},
  booktitle = {Advances in Neural Information Processing Systems 29},
  year      = {2016},
  pages     = {4565--4573},
  url       = {https://papers.neurips.cc/paper/6391-generative-adversarial-imitation-learning.pdf}
}

@inproceedings{Hoey1999SPUDD,
  author    = {Jesse Hoey and Robert St-Aubin and Alan Hu and Craig Boutilier},
  title     = {{SPUDD}: Stochastic Planning using Decision Diagrams},
  booktitle = {Proceedings of the Fifteenth Conference on Uncertainty in Artificial Intelligence (UAI)},
  year      = {1999},
  pages     = {279--288}
}

@article{HuangMalhameCaines2006MFG,
  author  = {Huang, Minyi and Malham{\'e}, Roland P. and Caines, Peter E.},
  title   = {Large Population Stochastic Dynamic Games: Closed-Loop {M}c{K}ean--{V}lasov Systems and the Nash Certainty Equivalence Principle},
  journal = {Communications in Information and Systems},
  year    = {2006},
  volume  = {6},
  number  = {3},
  pages   = {221--252}
}

@inproceedings{Ichter2023SayCan,
  author    = {Ichter, Brian and Brohan, Anthony and Chebotar, Yevgen and Finn, Chelsea and Kalashnikov, Dmitry and Levine, Sergey and Quillen, Deirdre and Schneider, Jonas and Shah, Avinash and Sermanet, Pierre and others},
  title     = {Do As {I} Can, Not As {I} Say: Grounding Language in Robotic Affordances},
  booktitle = {Conference on Robot Learning (CoRL)},
  year      = {2022},
  series    = {Proceedings of Machine Learning Research},
  volume    = {205},
  pages     = {287--318}
}

@inproceedings{KearnsKoller1999EfficientRLFMDP,
  author    = {Michael Kearns and Daphne Koller},
  title     = {Efficient Reinforcement Learning in Factored {MDP}s},
  booktitle = {Proceedings of the Sixteenth International Conference on Machine Learning (ICML)},
  year      = {1999},
  pages     = {240--247}
}

@inproceedings{KearnsKoller1999FactoredRL,
  author    = {Kearns, Michael and Koller, Daphne},
  title     = {Efficient Reinforcement Learning in Factored {MDP}s},
  booktitle = {Proceedings of the Sixteenth International Joint Conference on Artificial Intelligence (IJCAI)},
  year      = {1999},
  pages     = {740--747}
}

@inproceedings{Khardon2012Learning,
  author    = {Roni Khardon and Hema Raghavan and Saeed Salman and Yaron Schapire},
  title     = {Learning to Take Actions: From Factored Action Representations to Factored Policies},
  booktitle = {Proceedings of the Twenty-Eighth Conference on Uncertainty in Artificial Intelligence (UAI)},
  year      = {2012}
}

@article{KimDean2003NonHomogeneousPartitions,
  author  = {Keeyoung Kim and Thomas Dean},
  title   = {Solving Factored {MDP}s Using Non-Homogeneous Partitions},
  journal = {Artificial Intelligence},
  year    = {2003},
  volume  = {147},
  number  = {1-2},
  pages   = {225--251},
  doi     = {10.1016/S0004-3702(03)00038-3}
}

@article{KoberBagnellPeters2013IJRR,
  author  = {Kober, Jens and Bagnell, J. Andrew and Peters, Jan},
  title   = {Reinforcement learning in robotics: A survey},
  journal = {The International Journal of Robotics Research},
  year    = {2013},
  volume  = {32},
  number  = {11},
  pages   = {1238--1274},
  doi     = {10.1177/0278364913495721}
}

@inproceedings{KollerParr1999FactoredValue,
  author    = {Daphne Koller and Ronald Parr},
  title     = {Computing Factored Value Functions for Policies in Structured {MDP}s},
  booktitle = {Proceedings of the Sixteenth International Joint Conference on Artificial Intelligence (IJCAI)},
  year      = {1999},
  pages     = {1332--1338}
}

@inproceedings{KollerParr2000PolicyIterationFactored,
  author    = {Daphne Koller and Ronald Parr},
  title     = {Policy Iteration for Factored {MDP}s},
  booktitle = {Proceedings of the Sixteenth Conference on Uncertainty in Artificial Intelligence (UAI 2000)},
  year      = {2000},
  pages     = {326--334},
  publisher = {Morgan Kaufmann}
}

@article{Krishnan2016HIRL,
  author    = {Sanjay Krishnan and Animesh Garg and Richard Liaw and Lauren Miller and Florian T. Pokorny and Ken Goldberg},
  title     = {{HIRL}: Hierarchical Inverse Reinforcement Learning for Long{-}Horizon Tasks with Delayed Rewards},
  journal   = {arXiv preprint arXiv:1604.06508},
  year      = {2016},
  url       = {https://arxiv.org/abs/1604.06508}
}

@article{LakeBaroni2023Nature,
  author  = {Brenden M. Lake and Marco Baroni},
  title   = {Human-like systematic generalization through a meta-learning neural network},
  journal = {Nature},
  year    = {2023},
  volume  = {623},
  number  = {7985},
  pages   = {115--121},
  doi     = {10.1038/s41586-023-06668-3},
  url     = {https://www.nature.com/articles/s41586-023-06668-3}
}

@article{LakeSalakhutdinovTenenbaum2015PPI,
  author  = {Lake, Brenden M. and Salakhutdinov, Ruslan and Tenenbaum, Joshua B.},
  title   = {Human-level concept learning through probabilistic program induction},
  journal = {Science},
  year    = {2015},
  volume  = {350},
  number  = {6266},
  pages   = {1332--1338},
  doi     = {10.1126/science.aab3050}
}

@article{LasryLions2007MFG,
  author  = {Lasry, Jean-Michel and Lions, Pierre-Louis},
  title   = {Mean field games},
  journal = {Japanese Journal of Mathematics},
  year    = {2007},
  volume  = {2},
  number  = {1},
  pages   = {229--260}
}

@inproceedings{Levy2019HAC,
  author    = {Andrew Levy and George Konidaris and Robert Platt and Kate Saenko},
  title     = {Learning Multi{-}Level Hierarchies with Hindsight},
  booktitle = {International Conference on Learning Representations},
  year      = {2019},
  url       = {https://openreview.net/forum?id=ryzECoAcY7}
}

@inproceedings{LiWalshLittman2006,
  author    = {Lihong Li and Thomas J. Walsh and Michael L. Littman},
  title     = {Towards a Unified Theory of State Abstraction for {MDP}s},
  booktitle = {International Symposium on Artificial Intelligence and Mathematics},
  year      = {2006},
  url       = {https://rbr.cs.umass.edu/aimath06/proceedings/P21.pdf}
}

@inproceedings{Machado2017Laplacian,
  author    = {Marlos C. Machado and Marc G. Bellemare and Michael Bowling},
  title     = {A Laplacian Framework for Option Discovery in Reinforcement Learning},
  booktitle = {Proceedings of the 34th International Conference on Machine Learning},
  year      = {2017},
  volume    = {70},
  series    = {Proceedings of Machine Learning Research},
  pages     = {2295--2304},
  publisher = {PMLR},
  url       = {https://proceedings.mlr.press/v70/machado17a/machado17a.pdf}
}

@inproceedings{Nachum2018HIRO,
  author    = {Ofir Nachum and Shixiang Shane Gu and Honglak Lee and Sergey Levine},
  title     = {Data{-}Efficient Hierarchical Reinforcement Learning},
  booktitle = {Advances in Neural Information Processing Systems 31},
  year      = {2018},
  url       = {https://papers.neurips.cc/paper/7591-data-efficient-hierarchical-reinforcement-learning.pdf}
}

@inproceedings{NgRussell2000IRL,
  author    = {Andrew Y. Ng and Stuart J. Russell},
  title     = {Algorithms for Inverse Reinforcement Learning},
  booktitle = {Proceedings of the 17th International Conference on Machine Learning},
  year      = {2000},
  pages     = {663--670},
  publisher = {Morgan Kaufmann},
  url       = {https://ai.stanford.edu/~ang/papers/icml00-irl.pdf}
}

@inproceedings{OsbandVanRoy2014NearOptimal,
  author    = {Ian Osband and Benjamin {Van Roy}},
  title     = {Near-Optimal Reinforcement Learning in Factored {MDP}s},
  booktitle = {Advances in Neural Information Processing Systems (NeurIPS)},
  year      = {2014}
}

@inproceedings{ParrRussell1997,
  author    = {Ronald Parr and Stuart J. Russell},
  title     = {Reinforcement Learning with Hierarchies of Machines},
  booktitle = {Advances in Neural Information Processing Systems 10},
  year      = {1997},
  pages     = {1043--1049},
  publisher = {MIT Press},
  url       = {https://papers.neurips.cc/paper/1384-reinforcement-learning-with-hierarchies-of-machines.pdf}
}

@article{PoluSutskever2020GPTf,
  author  = {Polu, Stanislas and Sutskever, Ilya},
  title   = {Generative Language Modeling for Automated Theorem Proving},
  journal = {CoRR},
  year    = {2020},
  volume  = {abs/2009.03393}
}

@inproceedings{PoupartBoutilierPatrascuSchuurmans2002PiecewiseLinear,
  author    = {Pascal Poupart and Craig Boutilier and Relu Patrascu and Dale Schuurmans},
  title     = {Piecewise Linear Value Function Approximation for Factored {MDP}s},
  booktitle = {Proceedings of the Eighteenth National Conference on Artificial Intelligence (AAAI)},
  year      = {2002},
  pages     = {292--299}
}

@inproceedings{Raghavan2013Embedded,
  author    = {Hema Raghavan and Roni Khardon and Saeed Salman and Yaron Schapire},
  title     = {Embedded Feature Selection for Factored Reinforcement Learning},
  booktitle = {Advances in Neural Information Processing Systems (NeurIPS)},
  year      = {2013}
}

@inproceedings{Rakelly2019PEARL,
  author    = {Kate Rakelly and Aurick Zhou and Deirdre Quillen and Chelsea Finn and Sergey Levine},
  title     = {{PEARL}: Probabilistic Embeddings for Actor{-}Critic Reinforcement Learning},
  booktitle = {Proceedings of the 36th International Conference on Machine Learning},
  year      = {2019},
  volume    = {97},
  series    = {Proceedings of Machine Learning Research},
  pages     = {4171--4180},
  publisher = {PMLR},
  url       = {https://proceedings.mlr.press/v97/rakelly19a/rakelly19a.pdf}
}

@inproceedings{ReedDeFreitas2016NPI,
  author    = {Reed, Scott and de Freitas, Nando},
  title     = {Neural Programmer-Interpreters},
  booktitle = {International Conference on Learning Representations (ICLR)},
  year      = {2016}
}

@article{Rusu2016ProgressiveNets,
  author    = {Andrei A. Rusu and Neil C. Rabinowitz and Guillaume Desjardins and Hubert Soyer and James Kirkpatrick and Koray Kavukcuoglu and Razvan Pascanu and Raia Hadsell},
  title     = {Progressive Neural Networks},
  journal   = {arXiv preprint arXiv:1606.04671},
  year      = {2016},
  url       = {https://arxiv.org/abs/1606.04671}
}

@inproceedings{SchuurmansPatrascu2001DirectValueApprox,
  author    = {Dale Schuurmans and Relu Patrascu},
  title     = {Direct Value-Approximation for Factored {MDP}s},
  booktitle = {Advances in Neural Information Processing Systems 15 (NeurIPS)},
  year      = {2002},
  pages     = {1579--1586}
}

@inproceedings{StaubinHoeyBoutilier2000APRICODD,
  author    = {Robert St-Aubin and Jesse Hoey and Craig Boutilier},
  title     = {APRICODD: Approximate Policy Construction using Decision Diagrams},
  booktitle = {Advances in Neural Information Processing Systems 13 (NeurIPS)},
  year      = {2000},
  pages     = {1089--1095}
}

@inproceedings{StrehlDiukLittman2007StructureLearning,
  author    = {Alexander L. Strehl and Carlos Diuk and Michael L. Littman},
  title     = {Efficient Structure Learning in Factored-State {MDP}s},
  booktitle = {Proceedings of the Twenty-Second National Conference on Artificial Intelligence (AAAI)},
  year      = {2007}
}

@article{Sukhbaatar2017ASP,
  author    = {Sainbayar Sukhbaatar and Zita Marinho and Ilya Kostrikov and Gabriel Synnaeve and Arthur Szlam and Rob Fergus},
  title     = {Intrinsic Motivation and Automatic Curricula via Asymmetric Self{-}Play},
  journal   = {arXiv preprint arXiv:1703.05407},
  year      = {2017},
  url       = {https://arxiv.org/abs/1703.05407}
}

@article{Sukhbaatar2018HSP,
  author    = {Sainbayar Sukhbaatar and Zeming Lin and Ilya Kostrikov and Gabriel Synnaeve and Arthur Szlam and Rob Fergus},
  title     = {Learning Goal Embeddings via Self{-}Play for Hierarchical Reinforcement Learning},
  journal   = {arXiv preprint arXiv:1811.09083},
  year      = {2018},
  url       = {https://arxiv.org/abs/1811.09083}
}

@inproceedings{Sutton1990Dyna,
  author    = {Sutton, Richard S.},
  title     = {Integrated Architectures for Learning, Planning, and Reacting Based on Approximating Dynamic Programming},
  booktitle = {Proceedings of the Seventh International Conference on Machine Learning (ICML)},
  year      = {1990},
  pages     = {216--224}
}

@inproceedings{TianEllisTenenbaum2020Drawing,
  author    = {Tian, Jiayuan and Ellis, Kevin and Tenenbaum, Joshua B.},
  title     = {Learning Abstract Structure for Drawing by Efficient Motor Program Induction},
  booktitle = {Advances in Neural Information Processing Systems (NeurIPS)},
  year      = {2020}
}

@inproceedings{TianQianSra2020MinimaxFactoredRL,
  author    = {Tian Tian and Nianyu Qian and Suvrit Sra},
  title     = {Towards Minimax Optimal Reinforcement Learning in Factored {MDP}s},
  booktitle = {Advances in Neural Information Processing Systems (NeurIPS)},
  year      = {2020}
}

@inproceedings{Vezhnevets2017FeUdal,
  author    = {Alexander Sasha Vezhnevets and Simon Osindero and Tom Schaul and Nicolas Heess and Max Jaderberg and David Silver and Koray Kavukcuoglu},
  title     = {Fe{U}dal Networks for Hierarchical Reinforcement Learning},
  booktitle = {Proceedings of the 34th International Conference on Machine Learning},
  year      = {2017},
  volume    = {70},
  series    = {Proceedings of Machine Learning Research},
  pages     = {3540--3549},
  publisher = {PMLR},
  url       = {https://proceedings.mlr.press/v70/vezhnevets17a/vezhnevets17a.pdf}
}

@article{WatkinsDayan1992Qlearning,
  author  = {Watkins, Christopher J. C. H. and Dayan, Peter},
  title   = {Q-learning},
  journal = {Machine Learning},
  year    = {1992},
  volume  = {8},
  number  = {3--4},
  pages   = {279--292}
}

@inproceedings{WhiteMartinezRudolph2010RP,
  author    = {White, Spencer K. and Martinez, Tony and Rudolph, George},
  title     = {Generating a Novel Sort Algorithm Using Reinforcement Programming},
  booktitle = {IEEE Congress on Evolutionary Computation (CEC)},
  year      = {2010},
  pages     = {1--8},
  doi       = {10.1109/CEC.2010.5586457}
}

@article{WolpertMacready1997NFL,
  author  = {Wolpert, David H. and Macready, William G.},
  title   = {No Free Lunch Theorems for Optimization},
  journal = {IEEE Transactions on Evolutionary Computation},
  year    = {1997},
  volume  = {1},
  number  = {1},
  pages   = {67--82},
  month   = apr,
  doi     = {10.1109/4235.585893}
}

@inproceedings{Wu2021INT,
  author    = {Wu, Yuhuai and Jiang, Albert Qiaochu and Ba, Jimmy and Grosse, Roger B.},
  title     = {INT: An Inequality Benchmark for Evaluating Generalization in Theorem Proving},
  booktitle = {International Conference on Learning Representations (ICLR)},
  year      = {2021}
}

@inproceedings{XuTewari2020OracleEfficient,
  author    = {Ziping Xu and Ambuj Tewari},
  title     = {Near-Optimal Reinforcement Learning in Factored {MDP}s: Oracle-Efficient and Adaptive Algorithms for the Non-Episodic Setting},
  booktitle = {Advances in Neural Information Processing Systems (NeurIPS)},
  year      = {2020}
}

@inproceedings{XuTewari2020RLFactoredUnknown,
  author    = {Yihuai Xu and Ambuj Tewari},
  title     = {Reinforcement Learning in Factored Markov Decision Processes with Unknown Structure},
  booktitle = {Advances in Neural Information Processing Systems (NeurIPS)},
  year      = {2020}
}

@inproceedings{Yang2018MeanFieldMARL,
  author    = {Yang, Yaodong and Luo, Rui and Li, Ming and Zhou, Ming and Zhang, Weinan and Wang, Jun},
  title     = {Mean Field Multi-Agent Reinforcement Learning},
  booktitle = {Proceedings of the 35th International Conference on Machine Learning (ICML)},
  year      = {2018},
  series    = {Proceedings of Machine Learning Research},
  volume    = {80},
  pages     = {5571--5580}
}

@inproceedings{Yang2023LeanDojo,
  author    = {Yang, Kaiyu and Swope, Aidan and Gu, Alex and Chalamala, Rahul and Song, Peiyang and Yu, Shixing and Godil, Saad and Prenger, Ryan and Anandkumar, Anima},
  title     = {LeanDojo: Theorem Proving with Retrieval-Augmented Language Models},
  booktitle = {Advances in Neural Information Processing Systems (NeurIPS)},
  year      = {2023}
}

@inproceedings{Ziebart2008MaxEntIRL,
  author    = {Brian D. Ziebart and Andrew Maas and J. Andrew Bagnell and Anind K. Dey},
  title     = {Maximum Entropy Inverse Reinforcement Learning},
  booktitle = {Proceedings of the AAAI Conference on Artificial Intelligence},
  year      = {2008},
  pages     = {1433--1438},
  url       = {https://aaai.org/Papers/AAAI/2008/AAAI08-227.pdf}
}

@article{deFariasVanRoy2004ConstraintSampling,
  author  = {D. P. {de Farias} and Benjamin {Van Roy}},
  title   = {On Constraint Sampling in the Linear Programming Approach to Approximate Dynamic Programming},
  journal = {Mathematics of Operations Research},
  year    = {2004},
  volume  = {29},
  number  = {3},
  pages   = {462--478},
  doi     = {10.1287/moor.1040.0093}
}

@inproceedings{Bacon2017OptionCritic,
  author    = {Pierre{-}Luc Bacon and Jean Harb and Doina Precup},
  title     = {The Option{-}Critic Architecture},
  booktitle = {Proceedings of the Thirty-First AAAI Conference on Artificial Intelligence (AAAI-17)},
  year      = {2017},
  pages     = {1726--1734},
  url       = {https://ojs.aaai.org/index.php/AAAI/article/view/10916}
}

@inproceedings{RavindranBarto2003,
  author    = {Balaraman Ravindran and Andrew G. Barto},
  title     = {SMDP Homomorphisms: An Algebraic Approach to Abstraction in Semi-Markov Decision Processes},
  booktitle = {Proceedings of the Eighteenth International Joint Conference on Artificial Intelligence (IJCAI-03)},
  year      = {2003},
  pages     = {1011--1016},
  url       = {https://www.ijcai.org/Proceedings/03/Papers/145.pdf}
}

@article{Matiisen2017TSCL,
  author  = {Tambet Matiisen and Avital Oliver and Taco Cohen and John Schulman},
  title   = {Teacher--Student Curriculum Learning},
  journal = {IEEE Transactions on Neural Networks and Learning Systems},
  year    = {2020},
  volume  = {31},
  number  = {9},
  pages   = {3732--3740},
  doi     = {10.1109/TNNLS.2019.2934906},
  note    = {First posted as arXiv:1707.0183}
}

@inproceedings{Barreto2018DeepSF,
  author    = {Andr{'e} Barreto and Diana Borsa and John Quan and Tom Schaul and David Silver and Matteo Hessel and Daniel Mankowitz and Augustin {
{Z}}{'i}dek and R{'e}mi Munos},
  title     = {Transfer in Deep Reinforcement Learning Using Successor Features and Generalised Policy Improvement},
  booktitle = {Proceedings of the 35th International Conference on Machine Learning},
  series    = {Proceedings of Machine Learning Research},
  volume    = {80},
  pages     = {501--510},
  year      = {2018},
  editor    = {Jennifer Dy and Andreas Krause},
  publisher = {PMLR},
  url       = {https://proceedings.mlr.press/v80/barreto18a.html}
}

@article{deFariasVanRoy2003LPADP,
  author  = {D. P. de Farias and Benjamin Van Roy},
  title   = {The Linear Programming Approach to Approximate Dynamic Programming},
  journal = {Operations Research},
  year    = {2003},
  volume  = {51},
  number  = {6},
  pages   = {850--865},
  doi     = {10.1287/opre.51.6.850.24925}
}

@article {BartoHRLReview,
   author = {Barto, A. G. and Mahadevan, S.},
   title = {Recent Advances in Hierarchical Reinforcement Learning},
   journal = {Discrete Event Dynamic Systems: Theory and Applications},
   publisher = {Kluwer Netherlands},
   pages = {341-379},
   volume = {13},
   issue = {4},
   year = {2003}
}

@article {DietterichHRL,
   author = {Dietterich, T. G.},
   title = {Hierarchical reinforcement learning with the {MAXQ} value function decomposition},
   journal = {Journal of Artificial Intelligence Research},
   pages = {227-303},
   volume = {13},
   year = {2000}
}

@ARTICLE{MahadevanMaggioni:JMLR:07,
  author = {S. Mahadevan and M. Maggioni},
  title = {Proto-Value Functions: A Laplacian Framework for Learning Representation and Control in {M}arkov Decision Processes},
  journal = {Journal of Machine Learning Research},
  volume = {8},
  pages = {2169--2231},
  year = {2007},
}

@book{Puterman,
    author = {M. L. Puterman},
    publisher = {Wiley},
    title = {Markov Decision Processes},
    year = {1994}
}

@article {SuttonOptions,
   author = {Sutton, R. S. and Precup, D. and Singh, S. },
   title = {Between {MDP}s and semi-{MDP}s: A Framework for Temporal Abstraction in Reinforcement Learning},
   journal = {Artificial Intelligence},
   pages = {181-211},
   volume = {112},
   year = {1999}
}

@book{10.5555/1396348,
author = {Bertsekas, Dimitri P.},
title = {Dynamic Programming and Optimal Control, Vol. II},
year = {2007},
isbn = {1886529302},
publisher = {Athena Scientific},
edition = {3rd}
}

@book{10.5555/528623,
author = {Puterman, Martin L.},
title = {Markov Decision Processes: Discrete Stochastic Dynamic Programming},
year = {1994},
isbn = {0471619779},
publisher = {John Wiley \& Sons, Inc.},
address = {USA},
edition = {1st}
}

@article{JMLR:v21:20-212,
  author  = {Sanmit Narvekar and Bei Peng and Matteo Leonetti and Jivko Sinapov and Matthew E. Taylor and Peter Stone},
  title   = {Curriculum Learning for Reinforcement Learning Domains: A Framework and Survey},
  journal = {Journal of Machine Learning Research},
  year    = {2020},
  volume  = {21},
  number  = {181},
  pages   = {1--50},
  url     = {http://jmlr.org/papers/v21/20-212.html}
}

@article{Pozharskiy2020ManifoldLF,
  title={Manifold learning for accelerating coarse-grained optimization},
  author={Dmitry Pozharskiy and Noah J. Wichrowski and Andrew B. Duncan and Grigorios A. Pavliotis and Ioannis G. Kevrekidis},
  journal={Journal of Computational Dynamics},
  year={2020},
  url={https://api.semanticscholar.org/CorpusID:210156978}
}

@inproceedings{
sukhbaatar2018intrinsic,
title={Intrinsic Motivation and Automatic Curricula via Asymmetric Self-Play},
author={Sainbayar Sukhbaatar and Zeming Lin and Ilya Kostrikov and Gabriel Synnaeve and Arthur Szlam and Rob Fergus},
booktitle={International Conference on Learning Representations},
year={2018},
url={https://openreview.net/forum?id=SkT5Yg-RZ},
}

@INPROCEEDINGS{5586457,
  author={White, Spencer K. and Martinez, Tony and Rudolph, George},
  booktitle={IEEE Congress on Evolutionary Computation}, 
  title={Generating a novel sort algorithm using Reinforcement Programming}, 
  year={2010},
  volume={},
  number={},
  pages={1-8},
  doi={10.1109/CEC.2010.5586457}}

@article{Sukhbaatar2018LearningGE,
  title={Learning Goal Embeddings via Self-Play for Hierarchical Reinforcement Learning},
  author={Sainbayar Sukhbaatar and Emily L. Denton and Arthur D. Szlam and R. Fergus},
  journal={ArXiv},
  year={2018},
  volume={abs/1811.09083}
}

\end{document}